\documentclass{article} 
\usepackage{arxiv}

\usepackage{amsmath,amsfonts,bm}

\def\1{\bm{1}}

\DeclareMathAlphabet{\mathsfit}{\encodingdefault}{\sfdefault}{m}{sl}
\SetMathAlphabet{\mathsfit}{bold}{\encodingdefault}{\sfdefault}{bx}{n}

\usepackage{hyperref}
\usepackage{url}

\usepackage{minitoc}

\usepackage[utf8]{inputenc} 
\usepackage[T1]{fontenc}    
\usepackage{hyperref}       
\usepackage{url}            
\usepackage{booktabs}       
\usepackage{amsfonts}       
\usepackage{nicefrac}       
\usepackage{microtype}      
\usepackage{xcolor}         
\usepackage{caption}

\usepackage{placeins}
\usepackage{microtype}
\usepackage{graphicx}
\usepackage{subcaption}
\usepackage{booktabs}
\usepackage{hyperref}
\usepackage{multirow}

\usepackage{arydshln}
\usepackage{amsmath}
\usepackage{amssymb}
\usepackage{mathtools}
\usepackage{amsthm}

\usepackage{setspace}
\usepackage{amsmath, amssymb}
\usepackage{algorithm}
\usepackage{algorithmic}
\usepackage{bm}        
\usepackage{mathtools} 

\usepackage[inline]{enumitem}
\usepackage{wrapfig}

\theoremstyle{plain}
\newtheorem{theorem}{Theorem}[section]
\newtheorem{proposition}[theorem]{Proposition}
\newtheorem{lemma}[theorem]{Lemma}

\theoremstyle{definition}
\newtheorem{definition}[theorem]{Definition}

\theoremstyle{remark}
\newtheorem{remark}[theorem]{Remark}
\newtheorem{example}[theorem]{Example}

\newcommand{\ie}{\textit{i}.\textit{e}., }
\newcommand{\eg}{\textit{e}.\textit{g}., }

\title{Validity-Preserving Hierarchical RL for Joint Routing and Switch Placement in EDA}

\author{
\textbf{
Dorian Gailhard$^{1}$,
Ugo Lecerf$^{2}$,
Enzo Tartaglione$^{1}$,
Donatello Conte$^{2}$,
Lirida Naviner$^{1}$,
Jhony H. Giraldo$^{1}$
}\\~\\
$^{1}$LTCI, Télécom Paris, Institut Polytechnique de Paris, France
\qquad
$^{2}$Arteris IP\\
\texttt{$^{1}$\{name.surname\}@telecom-paris.fr}
\qquad
\texttt{$^{2}$\{name.surname\}@arteris.com}
}

\begin{document}

\maketitle

\begin{abstract}

Routing and switch placement are fundamental combinatorial optimization problems in chip design, requiring the joint optimization of routing topology and physical placement under strict structural, geometric and logical constraints. Existing approaches typically rely on carefully engineered heuristics that incorporate strong problem-specific biases to navigate the enormous space of possible designs.
In this work, we introduce a hierarchical reinforcement learning framework for joint routing and switch placement at the level of logical communication routes. Starting from a minimal routing graph, our method progressively constructs increasingly expressive solutions through three coupled operations: \textit{switch expansion}, \textit{switch placement}, and \textit{route refinement}. These operations preserve routing validity by construction, restricting exploration to feasible configurations where every communicating initiator-target pair has one assigned loop-free route. We explore the induced solution space using Gumbel Monte Carlo Tree Search, showing that neural-guided search substantially improves solution quality over non-learning optimization methods. Furthermore, pretraining across floorplans provides a strong initialization
for fine-tuning on unseen instances.
\end{abstract}

\section{Introduction}
\label{sec:intro}

Routing and placement are central optimization problems in modern chip design. Given a set of communicating components and their physical environment, routing must establish connections while optimizing objectives such as wirelength and congestion under physical and topological constraints \cite{marculescu2008outstanding, HU20011}. These problems are combinatorial: the routing structure, the switch locations, and the individual communication routes interact, and even restricted Steiner-type formulations arising in Very-Large-Scale Integration (VLSI) design are computationally hard \cite{9057662, ihler1999class}.

Classical approaches address the complexity stemming from Electronic Design Automation (EDA) through carefully designed combinatorial optimization procedures and domain-specific heuristics \cite{chu2005fast, HU20011, 5219222}. Learning-based methods, and Reinforcement Learning (RL) in particular, provide an alternative in which the strategy used to explore a combinatorial design space can itself be learned. Recent successes in placement and routing demonstrate the potential of this approach for chip design \cite{mirhoseini2021graph, liu2021rest, chen2022reinforcement, du2023hubrouter, liu2024neuralsteiner, du2025oarest}.

Routing and switch placement are tightly coupled: \textit{switch locations
determine which routing structures are effective, while the routing structure
determines which switch locations are useful}.
Optimizing these decisions sequentially can therefore sacrifice solution quality by committing to decisions before accounting for their effect on the
other problem.
We instead optimize routing and switch placement jointly. We deliberately study a simplified physical model that captures the interaction between shared routing topology, switch placement, and communication routes while abstracting away constraints such as routing layers, capacities, vias, and detailed design rules. We view this formulation as a first step toward an optimization core that can subsequently be extended with richer physical constraints and cost models.

In this work, we introduce a hierarchical construction process that progressively builds solutions through three operations:
\textit{i) switch expansion}, which introduces a new switch;
\textit{ii) switch placement}, which assigns its physical location; and
\textit{iii) route refinement}, which updates the routes affected by the expansion.
The resulting hierarchy restricts exploration to a subset of feasible
solutions while ensuring that a globally optimal solution remains reachable,
thereby reducing the search space without prescribing how it should be explored.
We formulate this construction process as a sequential decision problem and learn graph policies to guide exploration using Gumbel Monte Carlo Tree Search (MCTS) \cite{danihelka2022policy}. Finally, we train a policy jointly across floorplans and study whether the resulting initialization can accelerate optimization on previously unseen instances.
Our contributions are as follows:
\begin{itemize}[leftmargin=0.5cm]
\itemsep0em

\item We introduce a hierarchical graph formulation for joint routing and
switch placement that restricts search to feasible configurations while
ensuring that a globally optimal solution remains reachable
(Section~\ref{sec:method}).

\item We combine this formulation with a learned graph policy and Gumbel MCTS, and show that the resulting search can exploit additional compute to progressively improve solution quality and outperform the evaluated non-learning baselines (Section \ref{sec:implementation} and Table \ref{tab:routing_results}).

\item We demonstrate that learned priors can be transferred across floorplans: a policy trained jointly on multiple instances provides a transferable search prior that accelerates optimization when fine-tuned on previously unseen floorplans (Figure \ref{fig:transfer}).

\end{itemize}

\section{Related Work}
\label{sec:related_work}

\noindent\textbf{Classical methods.}
Classical physical-design methods rely on optimized combinatorial and continuous optimization procedures.
Analytical placement methods such as RePlAce~\cite{cheng2018replace} optimize differentiable placement objectives under density constraints, while physical routing commonly relies on Steiner-tree construction~\cite{9057662,HU20011,chu2005fast}, shortest-path and maze-routing procedures~\cite{5219222}, and iterative rip-up-and-reroute~\cite{1585385,5390304,9120211}.
Application-specific Network-on-Chip (NoC) synthesis similarly considers the joint design of communication architectures for a given application.
\cite{ahonen2004topology} optimize application-specific network topologies using communication requirements and physical wirelength information, while \cite{morgan2013unified} jointly optimize topology selection, core mapping, and traffic routing through a multi-objective genetic algorithm.
More broadly, classical NoC synthesis methods rely on combinatorial optimization and domain-specific heuristics to navigate large design spaces~\cite{marculescu2008outstanding}.

\noindent\textbf{Learning-based methods.}
Machine learning has been applied to several physical-design problems. AlphaChip~\cite{mirhoseini2021graph} formulates macro placement as a sequential decision problem and uses RL to optimize placement quality. Subsequent works have incorporated routing information into learning-based physical-design pipelines.
\cite{cheng2021joint} use routing results to evaluate placement quality, while \cite{cheng2022policy} combine RL-based placement with a conditional generative routing model. Other approaches have explored black-box optimization for macro placement~\cite{shi2023macro} and combinations of RL and tree search~\cite{geng2024reinforcement}.

Learning has also been applied directly to routing and Steiner-tree construction.
\cite{liao2020deep} formulate physical routing as a sequential RL problem, while REST~\cite{liu2021rest} uses RL to construct rectilinear Steiner minimum trees. For obstacle-aware routing, \cite{chen2022reinforcement} combine RL with MCTS for Steiner-point selection. Other approaches learn intermediate geometric structures: HubRouter~\cite{du2023hubrouter} generates hubs that guide pin-hub connections, while NeuralSteiner~\cite{liu2024neuralsteiner} predicts candidate Steiner points.
More recently, OAREST~\cite{du2025oarest} uses RL for obstacle-avoiding
rectilinear Steiner minimum tree construction and introduces a restricted
representation shown to retain an optimal solution.

Our setting differs from these Steiner-tree formulations by optimizing a shared communication infrastructure for multiple communication pairs rather than a tree connecting the pins of a single net. Our formulation jointly optimizes the topology and placement of intermediate switches and the route assigned to each communication pair, allowing communications to share physical infrastructure. It is also related to application-specific NoC synthesis, but focuses on a simplified geometric setting with fixed communicating components and directly constructs the intermediate infrastructure. Within this setting, our hierarchical construction provides a structured search space that preserves feasible assigned routes and contains a globally optimal solution under our model.

\section{Problem Formulation}
\label{sec:problem}

\noindent \textbf{Notations.}
We use calligraphic letters (\eg $\mathcal{V}$) for sets, with cardinality
$|\mathcal{V}|$.
Bold lowercase letters denote vectors (\eg $\mathbf{p}$).
For a point $\mathbf{p}\in\mathbb{R}^2$, we write
$\mathbf{p}=[x,y]$ for its horizontal and vertical coordinates.

\noindent \textbf{Basic definitions.}
A directed graph $G=(\mathcal{V},\mathcal{E})$ consists of a set of
vertices $\mathcal{V}$ and a set of directed edges
$\mathcal{E}\subseteq\mathcal{V}\times\mathcal{V}$.
An edge $(u,v)\in\mathcal{E}$ is directed from $u$ to $v$.
A directed path from $v_0$ to $v_k$ is a sequence of vertices
$
    \pi=(v_0,v_1,\ldots,v_k)
$
such that $(v_j,v_{j+1})\in\mathcal{E}$ for all
$j\in\{0,\ldots,k-1\}$.
We denote by
$
    E(\pi)=\{(v_j,v_{j+1}) : 0\leq j<k\}
$
the set of edges traversed by the directed path $\pi$.
A path is \emph{simple} if it contains no repeated vertex, \ie
$v_j\neq v_\ell$ for all $j\neq\ell$.

\subsection{Problem Setting}

We consider the joint optimization of routing and switch placement for a
fixed set of communicating components.
Let $\mathcal{I}$ and $\mathcal{T}$ denote the sets of initiators and
targets, representing the source and destination endpoints of communication
requests, respectively, such as processing, memory, or other IP blocks.
Their positions are fixed within a rectangular floorplan
$\Omega=[0,W]\times[0,H]\subset\mathbb{R}^2$.
The communication requirements are specified by a set
$\mathcal{R}\subseteq\mathcal{I}\times\mathcal{T}$ of communication pairs,
\ie initiator-target pairs for which a communication route must exist.
We additionally consider a set of rectangular blockages
$\mathcal{B}\subset\Omega$, whose interiors cannot contain switches or be
traversed by wires.

Given these inputs, we jointly determine the number and positions of
intermediate switches, the routing structure connecting them, and the route
assigned to each communication pair.
Let $\mathcal{S}$ denote the set of switches, with
$|\mathcal{S}|\leq S_{\max}$, and let
$\mathbf{p}_v=[x_v,y_v]\in\Omega\setminus\mathcal{B}$
denote the position of each switch $v\in\mathcal{S}$.
Together with the fixed initiators and targets, these switches define a
directed routing graph
$\mathcal{G}=(\mathcal{V},\mathcal{E})$, where
$\mathcal{V}=\mathcal{I}\cup\mathcal{T}\cup\mathcal{S}$.
For each communication pair $(i,t)\in\mathcal{R}$, the routing solution
assigns a simple directed path
$\pi_{i,t}=(v_0,\ldots,v_k)$, $v_0=i$, $v_k=t$,
that traverses at least one switch, with all intermediate nodes belonging to
$\mathcal{S}$.
The routing graph is induced by these paths, such that
$
    \mathcal{E}
    =
    \bigcup_{(i,t)\in\mathcal{R}} E(\pi_{i,t}).
$
We make several simplifying assumptions: initiator, target, and switch dimensions and pin-level constraints are ignored, and routing density, congestion, and bandwidth constraints are not modeled.

\subsection{Objective}

We consider rectilinear routing, where wires are composed exclusively of
horizontal and vertical segments.
For two nodes $u,v$ with positions $\mathbf{p}_u$ and $\mathbf{p}_v$, let
$\Gamma_{\mathcal{B}}(u,v)$ denote the set of rectilinear paths between
$\mathbf{p}_u$ and $\mathbf{p}_v$ that do not intersect the interior of any
blockage.
We define the obstacle-avoiding rectilinear distance as
$
    d_{\mathcal{B}}(u,v)
    =
    \min_{\gamma\in\Gamma_{\mathcal{B}}(u,v)}
    \operatorname{len}(\gamma).
$
When an unobstructed Manhattan-shortest path exists, this reduces to
$
    d_{\mathcal{B}}(u,v)
    =
    |x_u-x_v|+|y_u-y_v|.
$

Both the physical extent of the interconnect and the lengths of individual
communication paths are important considerations in on-chip network
design~\cite{ahonen2004topology,morgan2013unified}.
We therefore distinguish between the total physical wirelength of the
routing graph and the route lengths of individual communications.
The total wirelength is
$
    L_{\mathrm{wire}}
    =
    \sum_{\{u,v\}:\,(u,v)\in\mathcal E\,\lor\,(v,u)\in\mathcal E}
    d_{\mathcal B}(u,v),
$
where each physical connection is counted once, regardless of its traversal
direction or the number of communication pairs using it.
For a route $\pi_{i,t}=(v_0=i,\ldots,v_k=t)$, its length is
$
    \operatorname{len}(\pi_{i,t})
    =
    \sum_{j=0}^{k-1} d_{\mathcal{B}}(v_j,v_{j+1}),
$
and the total route length is
$
    L_{\mathrm{route}}
    =
    \sum_{(i,t)\in\mathcal{R}} \operatorname{len}(\pi_{i,t}).
$
The wirelength encourages compact and shared wires, while the route length
discourages long communication paths between individual initiator--target pairs. Let $\Pi_{i,t}(\mathcal S)$ denote the set of simple directed paths from
$i$ to $t$ that traverse at least one switch and whose intermediate vertices
belong to $\mathcal S$, \ie
$
\Pi_{i,t}(\mathcal S)
=
\left\{
\pi_{i,t}=(v_0,\ldots,v_k)
\;\middle|\;
v_0=i,\;
v_k=t,\;
k\geq 2,\;
v_j\in\mathcal S\ \forall j\in\{1,\ldots,k-1\},\;
v_j\neq v_\ell\ \forall j\neq \ell
\right\}$: the joint routing and switch-placement problem is defined as
\begin{equation}
\label{eq:problem}
\min_{(\mathcal S,\mathbf p,\pi)\in\mathcal F}
\;
L_{\mathrm{wire}}+\lambda L_{\mathrm{route}},
\qquad
\mathcal F
=
\left\{
\begin{aligned}
&\left(
\mathcal S,
\{\mathbf p_v\}_{v\in\mathcal S},
\{\pi_{i,t}\}_{(i,t)\in\mathcal R}
\right)
\ \bigm|\
\mathbf p_v\in\Omega\setminus\mathcal B
\ \forall v\in\mathcal S
\\
&|\mathcal S|\leq S_{\max};
\quad
\pi_{i,t}\in\Pi_{i,t}(\mathcal S)
\ \forall(i,t)\in\mathcal R
\end{aligned}
\right\}.
\end{equation}
In our experiments, we set $\lambda=\tfrac{1}{2}$.
The problem is closely related to the rectilinear Steiner tree problem,
which is NP-hard~\cite{garey1977rectilinear}.

\begin{figure}[t]
    \centering
    \includegraphics[width=\linewidth]{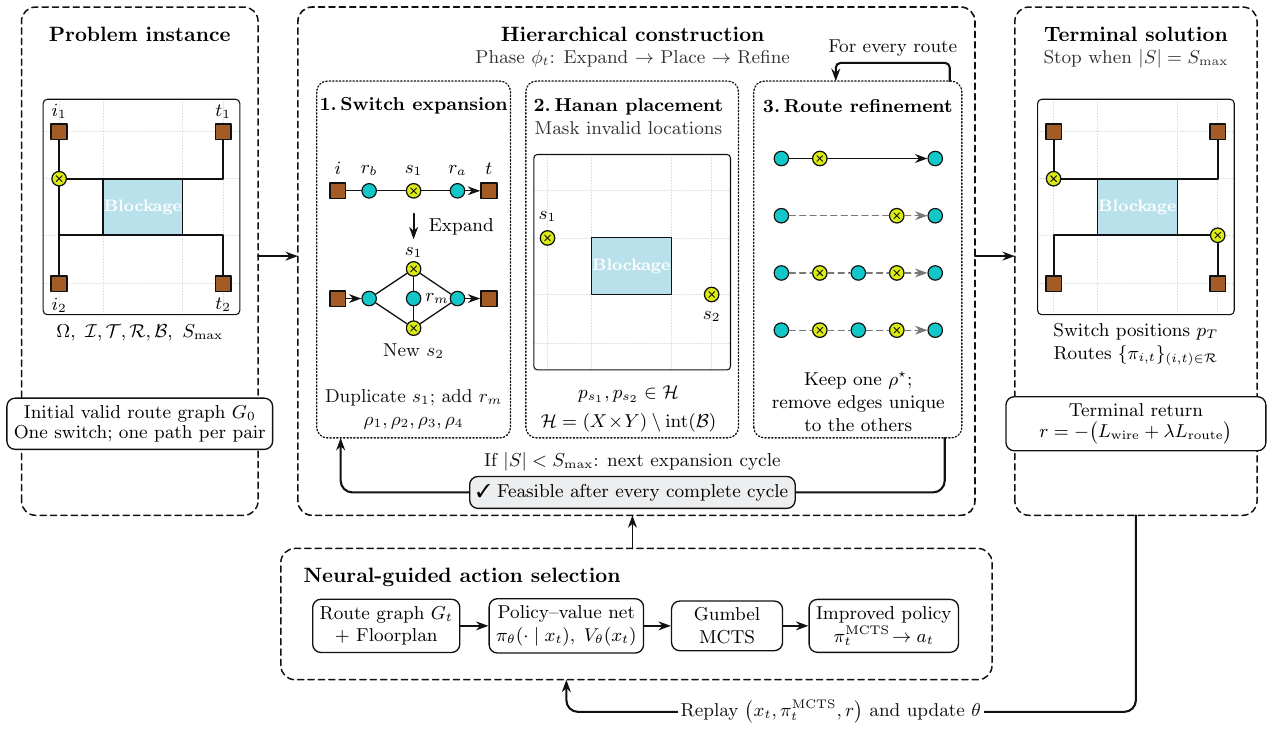}
    \caption{Overview of our routing and switch-placement method. Gumbel MCTS guides repeated switch expansion, Hanan-grid placement, and route refinement while preserving routing feasibility.}
    \label{fig:pipeline}
\end{figure}

\section{Method} \label{sec:method}

\textbf{Overview.}
We formulate joint routing and switch placement as an iterative construction process over a graph representation.
Starting from a minimal feasible routing solution containing a single switch, the routing graph is progressively expanded by introducing new switches, placing them on an extended Hanan grid~\cite{hanan1966steiner}, and updating the routes affected by each introduction.
Restricting switch placement to this finite set of candidate locations does
not exclude a globally optimal solution, but replaces the continuous
placement space with a finite one.
Each complete refinement step returns a feasible routing configuration, so that search is restricted to valid solutions rather than arbitrary routing graphs.
This process is repeated up to a predefined switch budget.
We formulate these decisions as a Markov Decision Process and learn a graph policy and value function using Gumbel MCTS~\cite{danihelka2022policy}, with solution quality evaluated according to the objective in Equation~\eqref{eq:problem}.
Figure~\ref{fig:pipeline} summarizes our proposed method.
Complete proofs of all propositions are provided in Appendix~\ref{app:proofs}.

\subsection{Route-Node Representation}

The routing graph defined in Section~\ref{sec:problem} specifies the physical connectivity induced by the communication routes.
However, an edge may be shared by several communication pairs.
Representing these assignments as variable-size edge attributes is inconvenient for our construction process.
We therefore make individual route assignments explicit in the graph by introducing \emph{route nodes}.

\begin{definition}[Route-node conversion]
Let $\mathcal{G}=(\mathcal{V},\mathcal{E})$ be a routing graph induced by a collection of communication routes.
For each edge $(u,v)\in\mathcal{E}$, let
$
    \mathcal{R}(u,v)
    =
    \{(i,t)\in\mathcal{R} : (u,v)\in E(\pi_{i,t})\}
$
denote the communication pairs whose routes traverse that edge.
We replace $(u,v)$ by one two-edge path
$
    u \rightarrow r_{i,t}^{u,v} \rightarrow v
$
for each $(i,t)\in\mathcal{R}(u,v)$, where
$r_{i,t}^{u,v}$ is a route node associated with communication pair $(i,t)$.
\end{definition}

The resulting representation makes each use of a physical connection by a communication pair explicit.
Route assignments can therefore be modified through local graph operations rather than variable-size edge attributes.
Figure~\ref{fig:couples_as_nodes} illustrates the conversion.
The transformed graph is bipartite between physical nodes
(initiators, targets, and switches) and route nodes.
Each route node has exactly one incoming and one outgoing edge and is associated with a single communication pair.
Moreover, for every communication pair traversing a switch, an incoming route node is paired with an outgoing route node of the same communication pair.
These structural properties are invariants preserved by the construction operations introduced below.

\begin{figure}[t]
\centering
\begin{subfigure}{0.45\textwidth}
\centering
\includegraphics[width=\linewidth]{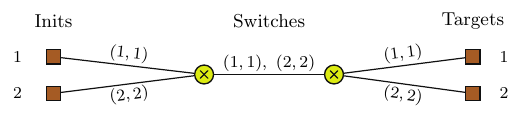}
\caption{Routing graph with communication pairs stored as edge attributes.}
\end{subfigure}
\hfill
\begin{subfigure}{0.45\textwidth}
\centering
\includegraphics[width=\linewidth]{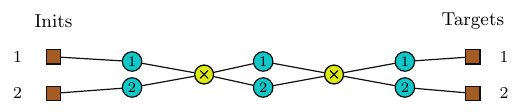}
\caption{Graph after conversion to route nodes. Route nodes labeled \textit{(1)} correspond to pair $(1,1)$ and those labeled \textit{(2)} to pair $(2,2)$.}
\end{subfigure}
\caption{Route-node conversion. Each communication pair traversing an edge is represented explicitly by a route node.}
\label{fig:couples_as_nodes}
\end{figure}

\subsection{Candidate Switch Locations}

\begin{wrapfigure}{R}{0.33\textwidth}
    \centering
    \vspace{-13pt}
    \includegraphics[width=\linewidth]{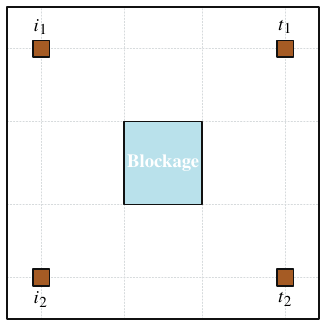}
    \caption{Extended Hanan grid induced by initiators, targets, and
    blockage corners.}
    \label{fig:hanan_grid}
    \vspace{-7pt}
\end{wrapfigure}

Switch positions are initially continuous variables over the floorplan.
To obtain a finite placement space, we restrict candidate locations to an
extended \emph{Hanan grid}, a classical construction for rectilinear routing
\cite{hanan1966steiner}.
Let $\mathcal{P}$ denote the positions of all initiators and targets, and let
$\mathcal{C}_{\mathcal{B}}$ denote the corners of the rectangular blockages.
We define the sets of horizontal and vertical coordinates
$
    \mathcal{X} =
    \{x(p) : p\in\mathcal{P}\}
    \cup
    \{x(c) : c\in\mathcal{C}_{\mathcal{B}}\},
$
$
    \mathcal{Y} =
    \{y(p) : p\in\mathcal{P}\}
    \cup
    \{y(c) : c\in\mathcal{C}_{\mathcal{B}}\}.
$
The set of candidate switch locations is then
\begin{equation}
    \mathcal{H}
    =
    \{(x,y)\in \mathcal{X}\times \mathcal{Y} :
    (x,y)\notin\operatorname{int}(\mathcal{B})\}.
\end{equation}
Figure~\ref{fig:hanan_grid} illustrates the resulting grid.

We have the following result :

\begin{proposition}[Optimal Switch Placement]
\label{prop:hanan_placement}
Under rectilinear routing, and for fixed positions of initiators, targets, and blockages, there exists an optimal solution in which every switch is placed at an intersection of the extended Hanan grid $\mathcal{H}$.
\end{proposition}

Thus, restricting switch placement to $\mathcal{H}$ preserves an
optimal solution while reducing the continuous placement problem to a
finite set of candidate locations.
Therefore, in Equation \eqref{eq:problem}, we can replace the search space $\mathbf{p}_v\in\Omega\setminus\mathcal{B},
\ \forall v\in\mathcal{S}$ by $\mathbf{p}_v\in\mathcal{H},
\ \forall v\in\mathcal{S}$.

\subsection{Switch Expansion and Route Refinement}

In this section, we describe how the routing graph is locally expanded and refined.
Given a switch $s_1$, switch expansion introduces a new switch $s_2$ and
temporarily introduces routing alternatives associated with $s_1$ and $s_2$.
Route refinement subsequently resolves these alternatives independently for
each affected communication pair.
Those two operations are formally defined as follows.

\begin{definition}[Switch expansion]
\label{def:switch_expansion}
Let $s_1$ be the switch selected for expansion, and let
$\mathcal{N}_b(s_1)$ and $\mathcal{N}_a(s_1)$ denote the sets of route nodes
immediately preceding and following $s_1$, respectively.
Switch expansion introduces a new switch $s_2$.
For every $r_b\in\mathcal{N}_b(s_1)$, we add the edge
$r_b\rightarrow s_2$, and for every $r_a\in\mathcal{N}_a(s_1)$, we add the
edge $s_2\rightarrow r_a$.

For each communication pair $(i,t)$ traversing $s_1$, we additionally
introduce a route node $r_m^{i,t}$ associated with $(i,t)$ and add the edges
$
    s_1\rightarrow r_m^{i,t}\rightarrow s_2,
    \,\,
    s_2\rightarrow r_m^{i,t}\rightarrow s_1.
$
\end{definition}

\begin{remark}
Informally, the expansion duplicates the selected switch and introduces, for
each affected communication pair, an additional route node between the two
switches. If $r_b$ and $r_a$ denote the route nodes immediately before and
after the expanded switch, this exposes four local routing alternatives:
\[
\begin{array}{ll}
\rho_1: r_b\rightarrow s_1\rightarrow r_a,
&
\rho_2: r_b\rightarrow s_2\rightarrow r_a,
\\[1mm]
\rho_3: r_b\rightarrow s_1\rightarrow r_m\rightarrow s_2\rightarrow r_a,
&
\rho_4: r_b\rightarrow s_2\rightarrow r_m\rightarrow s_1\rightarrow r_a.
\end{array}
\]
Route refinement subsequently selects one of these four alternatives.
\end{remark}

\begin{definition}[Route refinement]
\label{def:route_refinement}
For each communication pair affected by the expansion, let
$\{\rho_1,\rho_2,\rho_3,\rho_4\}$ denote its four local routing alternatives,
where each $\rho_j$ is identified with its set of edges.
Given a selected alternative $\rho^\star$, route refinement removes all
edges belonging exclusively to the unselected alternatives:
$
    \mathcal{E}
    \leftarrow
    \mathcal{E}
    \setminus
    \left(
        \bigcup_{\rho_j\neq\rho^\star}
        E(\rho_j)
        \setminus E(\rho^\star)
    \right).
$
Route nodes left disconnected by this operation are removed.
\end{definition}

\begin{figure}[t]
\centering
\includegraphics[width=\linewidth]{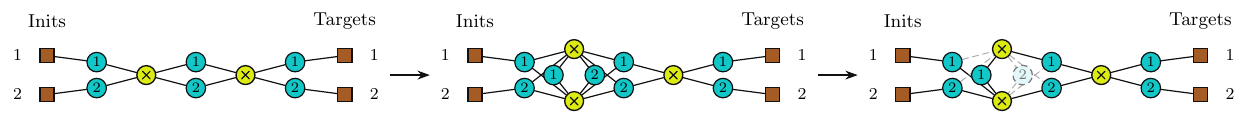}
\caption{Switch expansion and route refinement. Expansion of a switch (leftmost yellow node in the first part of the figure)
introduces a second switch and exposes four local routing alternatives for
each affected communication pair. Route refinement selects one alternative,
restoring a unique local route while leaving the remainder of the
communication path unchanged.}
\label{fig:expansion_refinement}
\end{figure}

Figure~\ref{fig:expansion_refinement} illustrates the expansion and
subsequent refinement process.
Although the expanded intermediate graph contains multiple alternatives for
the affected communication pairs, refinement restores a feasible routing
configuration.

\begin{proposition}[Validity preservation]
\label{prop:validity_preservation}
Assume the routing configuration before expansion assigns exactly one simple
directed path to every communication pair.
After applying switch expansion and route refinement
(Definitions~\ref{def:switch_expansion} and~\ref{def:route_refinement}),
the resulting routing configuration again assigns exactly one simple
directed path to every communication pair.
\end{proposition}

This proposition states that the feasibility of the constructed routing is an invariant of the expansion--refinement cycle, \ie starting from a minimal feasible routing with a single switch, any routing constructed using those two operations remains feasible.
However, the hierarchical construction cannot generate every feasible routing
configuration. Indeed, consider the following example:

\vspace{0.15cm}

\begin{example}
Consider three switches $s_1,s_2,s_3$ and three communication routes
$\pi_1=(s_1,s_2,s_3)$, $\pi_2=(s_2,s_3,s_1)$, $\pi_3=(s_3,s_1,s_2)$.
This configuration cannot be generated by the hierarchical construction following Definitions \ref{def:switch_expansion} and \ref{def:route_refinement}.
Indeed, consider the switch introduced last and assume, without loss of
generality, that it is $s_3$.
This switch must be introduced by expanding either $s_1$ or $s_2$.
For any route containing both the parent switch and $s_3$, these two
switches are consecutive immediately after the expansion. Since $s_3$ is
introduced last, no subsequent expansion can insert another switch between
them.
If $s_3$ is introduced by expanding $s_1$, this contradicts
$\pi_1=(s_1,s_2,s_3)$.
If it is introduced by expanding $s_2$, this contradicts
$\pi_3=(s_3,s_1,s_2)$.
Hence the configuration cannot be generated by the hierarchical
construction.
\end{example}

Nevertheless, the following result shows that excluding such configurations
does not sacrifice global optimality: at least one globally optimal
solution always belongs to the hierarchical search space.

\begin{proposition}[Optimality of the hierarchical search space]
\label{prop:hierarchical_optimality}
There exists a globally optimal solution to Equation~\eqref{eq:problem} that
can be obtained from the initial configuration through a finite sequence of
switch expansions, placements, and route refinements.
\end{proposition}

\subsection{Routing as a Sequential Markov Decision Process}
\label{sec:markov_decision_process}

We formulate the hierarchical construction process as a Markov Decision
Process (MDP) $(\mathcal{X},\mathcal{A},\mathcal{P},\mathcal{R})$.
An episode starts from a single switch connecting every communication pair and progressively refines an initially feasible routing configuration
through successive switch expansions, switch placements, and route
refinements.

\noindent\textbf{State space.}
At step $t$, the state is defined as
$
    x_t =
    \left(
        G_t,\mathbf{p}_t,\phi_t,
        \mathcal{Q}^{\mathrm{place}}_t,
        \mathcal{Q}^{\mathrm{route}}_t
    \right),
$
where $G_t=(\mathcal{V}_t,\mathcal{E}_t)$ is the current route-node graph,
$\mathbf{p}_t=\{\mathbf{p}_v\}_{v\in\mathcal{S}_t}$ contains the positions
of the current switches, and $\phi_t$ denotes the current decision phase.
The queues $\mathcal{Q}^{\mathrm{place}}_t$ and
$\mathcal{Q}^{\mathrm{route}}_t$ contain, respectively, the switches awaiting
placement and the communication routes awaiting refinement.
The next placement or refinement decision always operates on the first element
of the corresponding queue.

\noindent\textbf{Initial state.}
The initial graph $G_0$ contains a single switch $s^{(0)}$.
For every communication pair $(i,t)\in\mathcal{R}$, two route nodes
$r_{i,t}^{\mathrm{in}}$ and $r_{i,t}^{\mathrm{out}}$ form the path
$
    i
    \rightarrow r_{i,t}^{\mathrm{in}}
    \rightarrow s^{(0)}
    \rightarrow r_{i,t}^{\mathrm{out}}
    \rightarrow t.
$
The initial switch is placed at a fixed valid location
$\mathbf{p}_{s^{(0)}}\in\mathcal{H}$.
Both queues are initially empty,
$\mathcal{Q}^{\mathrm{place}}_0
=\mathcal{Q}^{\mathrm{route}}_0=\emptyset$,
and the initial phase corresponds to switch expansion.
The initial state therefore represents a feasible routing solution.

\noindent\textbf{Action space.}
The set of admissible actions depends on the current phase $\phi_t$.
During switch expansion, the action selects a switch
$s\in\mathcal{S}_t$ to expand according to
Definition~\ref{def:switch_expansion}.
During switch placement, the action selects a valid candidate location
$\mathbf{p}\in\mathcal{H}$ for the first switch in
$\mathcal{Q}^{\mathrm{place}}_t$.
During route refinement, the action selects one of the four local routing
alternatives
$
    \rho^\star\in\{\rho_1,\rho_2,\rho_3,\rho_4\}
$
for the first route in $\mathcal{Q}^{\mathrm{route}}_t$.

\noindent\textbf{Transition dynamics.}
The environment is deterministic and Markovian.
Each construction cycle begins with an expansion action, which introduces a
new switch that initially inherits the position of the expanded switch.
The two switches affected by the expansion are added to
$\mathcal{Q}^{\mathrm{place}}_t$, and the affected communication routes are
added to $\mathcal{Q}^{\mathrm{route}}_t$.
The placement and refinement queues are then processed sequentially.
Each action operates on and removes the first element of the corresponding
queue.
Once $\mathcal{Q}^{\mathrm{place}}_t$ is empty, the process transitions from
placement to route refinement.
Once $\mathcal{Q}^{\mathrm{route}}_t$ is also empty, the resulting graph again
represents a feasible routing configuration and the process returns to the
switch expansion phase.
The episode terminates when the switch budget $S_{\max}$ is reached and the
final refinement cycle has been completed.

\noindent\textbf{Reward.}
We use a sparse terminal reward corresponding to the negative routing
objective of Equation~\eqref{eq:problem}:
$r_t = 0$ for all $t<T$, and $r_T = -L_{\mathrm{wire}}(G_T)-\lambda L_{\mathrm{route}}(G_T)$,
where $\lambda=\tfrac{1}{2}$ in our experiments.

\section{Implementation}
\label{sec:implementation}

We learn a shared graph policy and value network over the sequential decision
process defined in Section~\ref{sec:markov_decision_process}.
The network jointly encodes the current routing graph, blockage geometry, and
candidate switch locations with separate policy heads for switch expansion,
placement, and route refinement.
Invalid actions are masked throughout construction.
We use Gumbel MCTS~\cite{danihelka2022policy} to explore the resulting search
space, using the learned policy and value function to guide tree search and
training them toward search-improved targets.
Following the risk-seeking value estimation of
AlphaTensor~\cite{fawzi2022discovering}, we train the value function toward
the top $25\%$ of observed returns rather than their average.
Per-floorplan PopArt~\cite{hessel2019multi} is used for value normalization.

To isolate the contribution of tree search, we additionally train
PPO-EWMA~\cite{hilton2022batch} on the same hierarchical MDP.
PPO-EWMA uses the same policy-value architecture, action masking, terminal
objective, and PopArt value normalization, but selects actions directly from the
learned policy without tree search.
Architectural details, optimization procedures, and hyperparameters are
provided in Appendix~\ref{app:experiments}.

\section{Experiments and Results}
\label{sec:experiments}

Our experiments investigate whether
(\textit{i}) learned guidance improves over non-learning optimization methods,
(\textit{ii}) explicit tree search improves over direct policy optimization,
and (\textit{iii}) pretraining across floorplans improves optimization on
unseen instances.
Experimental details, numerical results, and routing visualizations are provided
in Appendices~\ref{app:experiments}, \ref{app:detailed_results}, and
\ref{app:visualization}, respectively.

\subsection{Datasets and Experimental Setup}

\noindent \textbf{Floorplans.}
We evaluate our model on 28 synthetic square floorplans with rectangular blockages and
initiator--target communication constraints.
We use 24 floorplans for pretraining and hold out the remaining four for
transfer experiments.
The instances contain 18--25 communication pairs, 3--5 initiators, and 5--8
targets, with switch budgets ranging from 2 to 5.
The unobstructed area covers 58.4\%--81.3\% of each floorplan.
We construct an extended Hanan grid from terminal coordinates and blockage
boundaries, yielding grids ranging from $24 \times 23$ to $42 \times 44$
candidate positions.

\noindent \textbf{Metric.}
We evaluate solutions using the objective defined in
Section~\ref{sec:problem}, \ie $L_{\mathrm{wire}} + \frac{1}{2} L_{\mathrm{route}}$.
Wirelength counts each physical connection once, irrespective of the
number or direction of routes using it, while route length counts every
route traversal. For readability, all lengths are normalized by the
side length of the corresponding floorplan.

\noindent \textbf{Baselines.}
We compare Gumbel MCTS~\cite{danihelka2022policy} against classical,
model-free search, and direct policy optimization baselines.
\textit{Heuristic} is a deterministic obstacle-aware, Steiner-inspired
constructive method that greedily selects switch locations from the extended
Hanan grid and sequentially routes communication pairs using their marginal
contribution to the objective.
\textit{Random Search}~\cite{karnopp1963random} uniformly samples legal actions
within our hierarchical framework, isolating the benefit of
learned guidance.
\textit{Genetic Algorithm}~\cite{holland1992genetic} evolves complete
hierarchical action sequences using selection, crossover, mutation, and random
immigration, providing a stronger non-learning search baseline.
Finally, \textit{PPO-EWMA}~\cite{hilton2022batch} uses the same policy-value
architecture and hierarchical environment as Gumbel MCTS, but acts directly
from the learned policy without tree search, isolating the contribution of
explicit search.

All methods except \textit{Heuristic} are run for 48 hours on the 24 pretraining
floorplans.
\textit{Heuristic} is instead run once as a deterministic constructive
procedure.
Fine-tuning is performed for 24 hours, with transfer results aggregated over
three independent runs.
For both PPO-EWMA and Gumbel MCTS, we compare fine-tuning from the respective
pretrained checkpoint against training from scratch on each of the four
held-out floorplans.

\subsection{Results and Discussion}

\begin{table*}[t]
\centering
\caption{Objective values on the $24$ training floorplans (lower is better).
Best results are shown in bold and second-best results are underlined.}
\label{tab:routing_results}
\setlength{\tabcolsep}{3pt}
\resizebox{\textwidth}{!}{
\begin{tabular}{@{}l *{12}{c}@{}}
\toprule
\textbf{Method}
& \textbf{1}
& \textbf{2}
& \textbf{3}
& \textbf{4}
& \textbf{5}
& \textbf{6}
& \textbf{7}
& \textbf{8}
& \textbf{9}
& \textbf{10}
& \textbf{11}
& \textbf{12}
\\
\midrule
Heuristic
& \underline{11.902} & \underline{11.255} & \underline{15.349} & \underline{11.330} & 14.847 & 11.126 & \underline{8.090} & \underline{10.301} & 11.292 & 14.615 & 14.741 & \underline{13.418} \\
Random search
& 18.706 & 20.799 & 24.042 & 17.420 & 23.629 & 19.310 & 12.672 & 18.574 & 18.663 & 24.931 & 22.584 & 20.775 \\
Genetic algorithm
& 14.382 & 13.999 & 16.854 & 12.622 & 17.970 & 13.107 & 8.224 & 11.954 & 14.152 & 15.799 & 16.478 & 16.504 \\
\midrule
PPO-EWMA
& 12.166 & 11.484 & 15.670 & 11.576 & \underline{14.717} & \underline{10.960} & 8.595 & 10.352 & \underline{11.142} & \underline{13.329} & \underline{14.151} & \textbf{13.361} \\
Gumbel MCTS
& \textbf{11.333} & \textbf{10.880} & \textbf{14.274} & \textbf{9.784} & \textbf{13.926} & \textbf{9.918} & \textbf{7.739} & \textbf{9.897} & \textbf{10.820} & \textbf{13.028} & \textbf{13.995} & \textbf{13.361} \\
\bottomrule
\end{tabular}
}
\par\vspace{4pt}\noindent
\resizebox{\textwidth}{!}{
\begin{tabular}{@{}l *{12}{c}@{}}
\toprule
\textbf{Method}
& \textbf{13}
& \textbf{14}
& \textbf{15}
& \textbf{16}
& \textbf{17}
& \textbf{18}
& \textbf{19}
& \textbf{20}
& \textbf{21}
& \textbf{22}
& \textbf{23}
& \textbf{24}
\\
\midrule
Heuristic
& \underline{13.662} & 15.003 & \underline{14.512} & 15.284 & \underline{16.571} & \underline{16.366} & \underline{16.277} & 16.277 & 14.909 & 15.415 & 14.007 & \underline{15.684} \\
Random search
& 21.425 & 21.875 & 19.338 & 20.146 & 20.943 & 30.810 & 29.030 & 20.421 & 23.683 & 21.337 & 22.457 & 20.583 \\
Genetic algorithm
& 15.238 & 15.057 & \textbf{14.134} & \textbf{15.236} & \textbf{15.239} & 19.881 & 22.062 & \textbf{15.726} & 16.580 & 14.531 & 14.621 & \textbf{15.277} \\
\midrule
PPO-EWMA
& \textbf{13.287} & \underline{13.593} & \textbf{14.134} & \underline{15.269} & \textbf{15.239} & 19.147 & 18.488 & \underline{15.817} & \underline{13.095} & \underline{13.989} & \underline{12.409} & \textbf{15.277} \\
Gumbel MCTS
& \textbf{13.287} & \textbf{13.466} & \textbf{14.134} & \textbf{15.236} & \textbf{15.239} & \textbf{13.601} & \textbf{14.373} & \textbf{15.726} & \textbf{12.779} & \textbf{13.471} & \textbf{12.386} & \textbf{15.277} \\
\bottomrule
\end{tabular}
}
\end{table*}

\noindent \textbf{Comparison with optimization baselines.}
Table~\ref{tab:routing_results} reports the objective obtained on the $24$
training floorplans.
The Heuristic achieves strong results on several instances, but its performance
is highly instance-dependent.
Random Search performs poorly throughout. This shows the need for effective
guidance within the hierarchical search space.
The Genetic Algorithm is considerably stronger and competitive on several
instances, but struggles on more challenging floorplans, such as 18 and 19.
PPO-EWMA improves upon these baselines on most instances, and shows the benefit
of learned guidance, but still exhibits substantial performance gaps on some
instances.
Gumbel MCTS achieves the strongest and most consistent performance, including
on instances where the other methods struggle.
This is particularly evident on floorplans~18 and~19, where it obtains
objectives of $13.601$ and $14.373$, compared with $19.147$ and $18.488$ for
PPO-EWMA and $19.881$ and $22.062$ for the Genetic Algorithm. Among all methods, Gumbel MCTS benefits the most from additional
compute, continuing to improve as the search budget increases.

\noindent \textbf{Transfer to unseen floorplans.}
We next investigate whether training across multiple floorplans produces
reusable policies that facilitate optimization of previously unseen instances.
We pretrain on the $24$ training floorplans and evaluate transfer on four
held-out floorplans.
For each target instance, we compare fine-tuning from the fixed pretrained
checkpoint against training the same model from random initialization.
Results are aggregated over three independent fine-tuning runs.
Figure~\ref{fig:transfer} shows that pretraining substantially accelerates
optimization on held-out floorplans.
Fine-tuning starts from stronger solutions than training from scratch and
reaches competitive solutions using less target-instance optimization.
On some instances, fine-tuning also reaches better solutions within the
available optimization budget.

\begin{figure*}[t]
\centering
\begin{subfigure}[b]{0.48\textwidth}
    \centering
    \includegraphics[width=\linewidth]{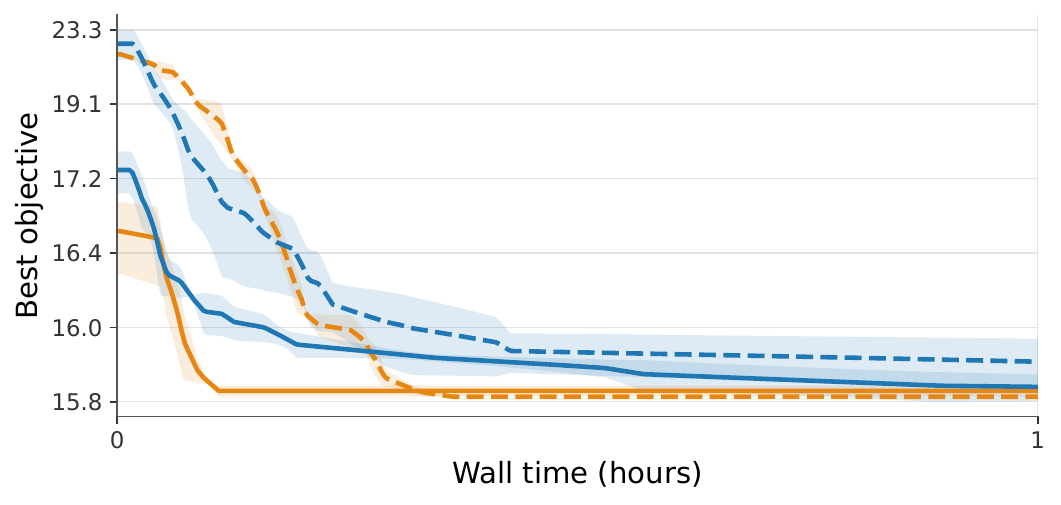}
    \caption{Fine-tuning instance 1}
\end{subfigure}
\hfill
\begin{subfigure}[b]{0.48\textwidth}
    \centering
    \includegraphics[width=\linewidth]{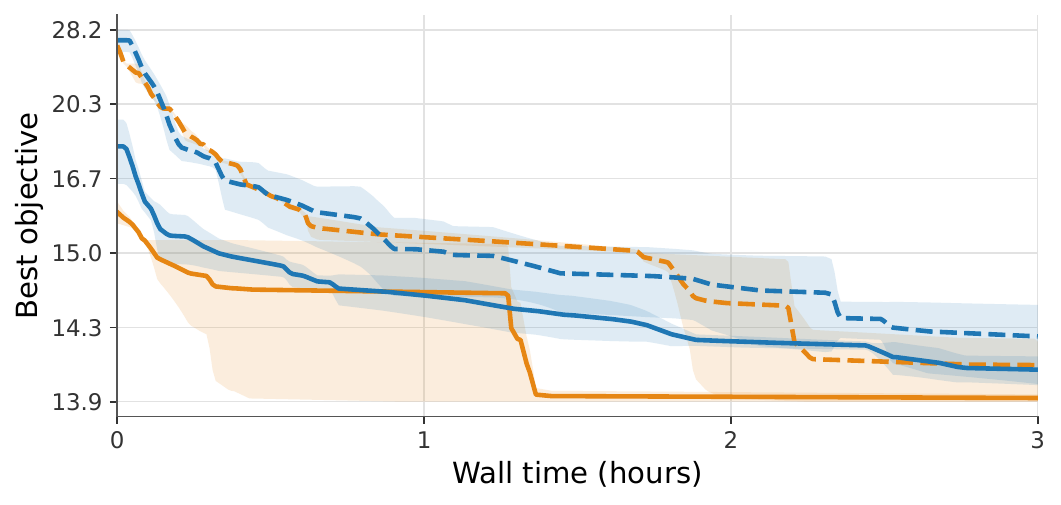}
    \caption{Fine-tuning instance 2}
\end{subfigure}
\\[0.8em]
\begin{subfigure}[b]{0.48\textwidth}
    \centering
    \includegraphics[width=\linewidth]{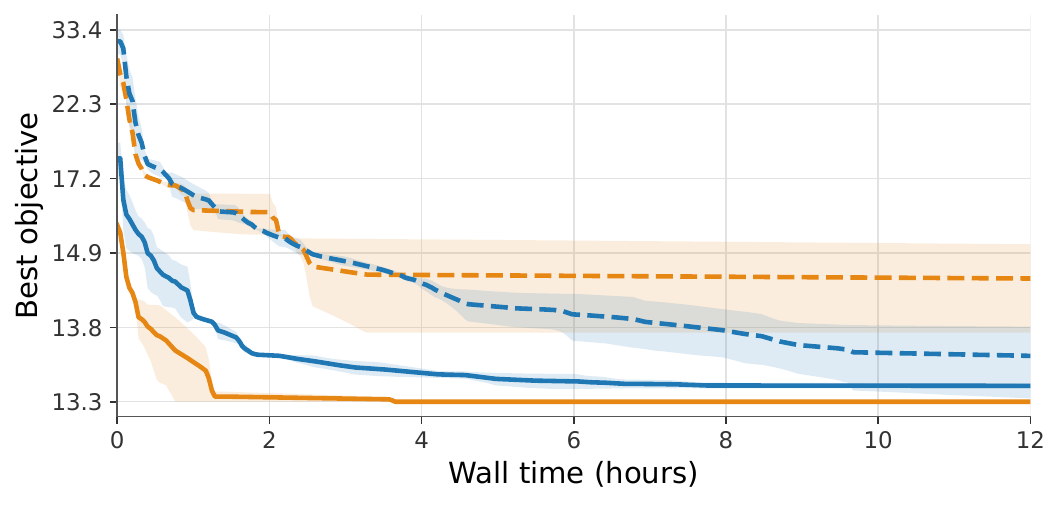}
    \caption{Fine-tuning instance 3}
\end{subfigure}
\hfill
\begin{subfigure}[b]{0.48\textwidth}
    \centering
    \includegraphics[width=\linewidth]{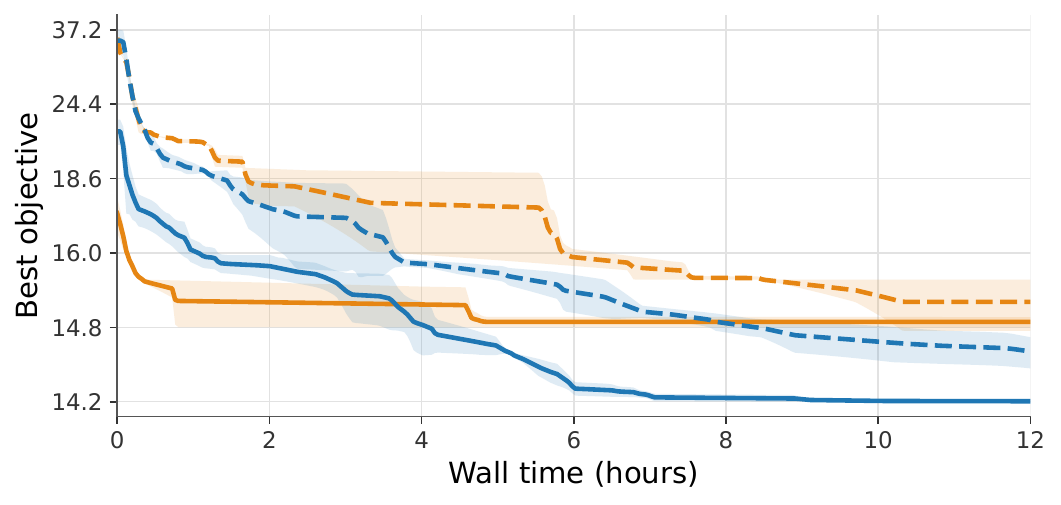}
    \caption{Fine-tuning instance 4}
\end{subfigure}
\caption{Comparison between optimization from scratch and fine-tuning from a
policy pretrained on the $24$ training floorplans, evaluated on held-out
floorplans. Each curve shows the mean over three independent runs, and the
shaded region indicates one standard deviation. The objective is shown on a
logarithmic scale. For readability, the time axis is truncated once all
methods are within $1\%$ of their respective best objective values.
PPO-EWMA is shown in orange and Gumbel MCTS in blue. Solid lines correspond
to fine-tuning from the pretrained policy, while dotted lines correspond to
optimization from scratch.}
\label{fig:transfer}
\end{figure*}

\subsection{Limitations}
\label{sec:limitations}

Our formulation deliberately omits several aspects of physical design.
The current environment does not model component dimensions, pin-level
constraints, unit density, routing congestion, or bandwidth constraints, and the
optimization objective considers only physical wirelength and
communication-route length.
Similarly, our experiments are limited to relatively small routing instances
compared with real-world NoC designs, which can contain up to tens of thousands
of communication connections.
While the formulation does not assume fixed instance sizes, larger instances increase the computational cost.
Extending the formulation to richer physical constraints and objectives, and
evaluating it at industrial scale, are left to future work.

\section{Conclusion}

We introduced a learning-based framework for joint routing and switch placement
based on a validity-preserving graph hierarchical formulation.
The formulation restricts switch locations to an extended Hanan grid and
enforces feasibility at every complete expansion-refinement cycle.
Together, these choices discretize the placement space and restrict search to
feasible solutions, substantially reducing the difficulty of the search.
Experiments across diverse floorplans show that Gumbel MCTS effectively
explores the resulting search space, consistently outperforming classical
optimization baselines and direct policy optimization, with particularly large
improvements on more challenging instances.
Pretraining across multiple floorplans also provides a useful initialization
for unseen instances, reducing the time required to find high-quality
solutions.

\section*{Reproducibility Statement}
All proofs are given in Appendix \ref{app:proofs}. Appendices \ref{app:implementation} and \ref{app:experiments} provide all the details for the implementation and experiments. The code will be made public upon acceptance of the paper.

\section*{AI Use Statement}
Generative AI was used to assist with literature review, polishing the writing,
and coding. It was also used to double-check and refine the proofs in
Appendix~\ref{app:proofs}. In particular, it provided a counterexample to an
initial version of Proposition~\ref{prop:hierarchical_optimality}, which
prompted us to refine the proposition statement.
We take responsibility for the final content of this work, including all text,
claims, code, and other materials produced with the assistance of generative AI.

\bibliographystyle{unsrt}
\bibliography{main}

\newpage
\appendix

\doparttoc
\faketableofcontents

\addcontentsline{toc}{section}{Appendix}

\begingroup
\renewcommand{\partname}{}
\renewcommand{\thepart}{}
\part{{\Large Appendix}}
\endgroup

\parttoc
\newpage


\section{Proofs} \label{app:proofs}

In this section, we provide proofs for all the propositions in the main paper.

\subsection{Optimality on the Hanan Grid}

\begin{proposition}[Optimal Switch Placement]
Under rectilinear routing, and for fixed positions of initiators, targets, and blockages, there exists an optimal solution in which every switch is placed at an intersection of the extended Hanan grid $\mathcal{H}$.
\end{proposition}

\begin{proof}
We first establish a geometric property used in the argument.
Recall that blockage interiors are forbidden, while their boundaries are
admissible.

\begin{lemma}
\label{lem:horizontal_cell_path}
Let $C$ be the set of switches sharing an off-grid $x$-coordinate $x$, and let
$x^-<x<x^+$ be consecutive coordinates obtained by augmenting the vertical
Hanan lines with all switch $x$-coordinates.
Every physical connection incident to $C$ admits a shortest realization whose
length varies affinely when $C$ is translated horizontally within
$[x^-,x^+]$.
\end{lemma}

\begin{proof}
Since $x^-$ and $x^+$ are consecutive augmented coordinates, no vertical
blockage boundary lies in the open slab $(x^-,x^+)\times\mathbb R$.
Consequently, from any feasible point $(x,y)$ in the slab, the horizontal
segment to either boundary is feasible: otherwise a rectangular blockage would
have a vertical boundary inside the slab or would contain $(x,y)$.

Consider first a shortest connection from $s=(x,y_s)\in C$ to an endpoint
outside $C$.
The other endpoint cannot lie strictly inside the slab, since terminal
coordinates are Hanan coordinates and all switch coordinates were included in
the augmentation.
Hence the connection reaches one of the slab boundaries; suppose it first
reaches $x^+$ at $b=(x^+,y_b)$.

The prefix from $s$ to $b$ can be replaced by a horizontal segment from
$(x,y_s)$ to $(x^+,y_s)$ followed by a vertical segment along $x^+$ to $b$.
This path is feasible and has length
$x^+-x+|y_b-y_s|$, the rectilinear lower bound between $s$ and $b$.
It is therefore also shortest.
After translating $s$ by $\delta$, only its initial horizontal segment changes,
so the connection length changes by $-\delta$.
A connection reaching $x^-$ analogously changes by $+\delta$.

Now consider a connection whose two endpoints belong to $C$.
If their vertical segment is feasible, it remains feasible under a common
translation within the slab and its length is constant.
Otherwise, any feasible connection between them must leave the slab.
Applying the preceding argument at both endpoints gives a shortest realization
with fixed boundary attachments.
Its length therefore changes by $-2\delta$, $2\delta$, or $0$, according to
the boundaries through which it leaves and re-enters the slab.

Thus every connection incident to $C$ has a shortest realization whose length
is affine in $\delta$.
\end{proof}

We now consider a globally optimal solution with objective $J^\star$ and fix
its logical routing topology.
For each physical connection $\{u,v\}$, let $m_{uv}$ be the number of
communication routes traversing it.
Its contribution to the objective has fixed positive weight
$w_{uv}=1+\lambda m_{uv}$.
Hence, for the fixed topology, the optimal objective can be written as
$J^\star=\sum_{\{u,v\}\in E} w_{uv}\,d_{\mathcal B}(u,v)$,
where $d_{\mathcal B}(u,v)$ denotes the shortest obstacle-avoiding rectilinear
distance between $u$ and $v$.

We first align the $x$-coordinates.
Let $\mathcal X$ be the set of vertical Hanan coordinates and define
$\widehat{\mathcal X}=\mathcal X\cup\{x_s:s\in\mathcal S\}$.
Group switches sharing the same $x\notin\mathcal X$, and let $N_x$ be the
number of such off-grid groups.

Consider one group $C$ at $x$, with adjacent coordinates
$x^-<x<x^+$ in $\widehat{\mathcal X}$.
Choose for each incident connection a shortest realization given by
Lemma~\ref{lem:horizontal_cell_path}.
If $C$ is translated by $\delta$, with
$x+\delta\in[x^-,x^+]$, the weighted length of these realizations is affine:
$\widetilde J(\delta)=J(0)+a\delta$ for some constant $a$.
Since $\delta=0$ lies between the two boundary displacements, at least one
boundary satisfies $\widetilde J(\delta)\leq J(0)$.

Move $C$ to such a boundary.
The constructed paths certify a feasible solution with no larger objective;
replacing them by shortest paths at the new switch positions can only improve
it further.
If the reached coordinate is in $\mathcal X$, $C$ is now Hanan-aligned;
otherwise it merges with another switch group.
In either case, $N_x$ strictly decreases.
Repeating this operation therefore places every switch on a vertical Hanan
line after finitely many moves, without increasing the objective.

Applying the same argument to the $y$-coordinates, while keeping the
$x$-coordinates fixed, places every switch at an intersection of
$\mathcal H$ without increasing the objective.
The paths maintained during these translations are feasible, but need not be
shortest at the final switch positions.
We therefore replace each physical connection by a shortest
obstacle-avoiding rectilinear path between its final endpoints.
Such a path exists because the maintained path provides a feasible connection,
and this replacement can only decrease the objective.
Thus, starting from the globally optimal value $J^\star$, we obtain a feasible
grid-aligned solution with objective at most $J^\star$.
By global optimality, its objective must equal $J^\star$, completing the proof.
\end{proof}

\subsection{Correct and Static Routing at Every Step of the Algorithm}

\begin{proposition}[Validity preservation]
Assume the routing configuration before expansion assigns exactly one simple
directed path to every communication pair.
After applying switch expansion and route refinement
(Definitions~\ref{def:switch_expansion} and~\ref{def:route_refinement}),
the resulting routing configuration again assigns exactly one simple
directed path to every communication pair.
\end{proposition}

\begin{proof}
Let $G=(\mathcal{V},\mathcal{E})$ be a valid routing configuration and let
$s_1$ be the switch selected for expansion.
Consider a communication pair $(i,t)\in\mathcal{R}$.

If $\pi_{i,t}$ does not traverse $s_1$, neither switch expansion nor route
refinement modifies its route, so $\pi_{i,t}$ remains a unique simple
directed path.

Now suppose that $\pi_{i,t}$ traverses $s_1$.
Because $\pi_{i,t}$ is simple, it visits $s_1$ exactly once.
Let $r_b$ and $r_a$ denote the route nodes immediately preceding and
following $s_1$ on $\pi_{i,t}$.
Switch expansion leaves the remainder of $\pi_{i,t}$ unchanged and replaces
the local segment
$
    r_b \rightarrow s_1 \rightarrow r_a
$
by the four alternatives
$
\rho_1: r_b\rightarrow s_1\rightarrow r_a
$, $
\rho_2: r_b\rightarrow s_2\rightarrow r_a
$, $
\rho_3: r_b\rightarrow s_1\rightarrow r_m\rightarrow s_2\rightarrow r_a
$ and $
\rho_4: r_b\rightarrow s_2\rightarrow r_m\rightarrow s_1\rightarrow r_a
$,
where $s_2$ is the newly introduced switch and $r_m$ is the corresponding
route node for $(i,t)$.

Each $\rho_j$ is a simple directed path from $r_b$ to $r_a$.
Indeed, $s_2$ and $r_m$ are newly introduced nodes, while $s_1$ occurs only
once in each alternative.
Route refinement selects exactly one $\rho_j$ and removes the edges belonging
exclusively to the remaining alternatives.
Replacing the original local segment by the selected $\rho_j$ therefore
yields exactly one simple directed path from $i$ to $t$.

Since this argument applies independently to every communication pair
affected by the expansion, while all unaffected routes remain unchanged, the
refined configuration assigns exactly one simple directed path to every
communication pair.
\end{proof}

\subsection{Optimality of the Hierarchical Search Space}

\begin{proposition}[Optimality of the hierarchical search space]
There exists a globally optimal solution to Equation~\eqref{eq:problem} that
can be obtained from the initial configuration through a finite sequence of
switch expansions, placements, and route refinements.
\end{proposition}

\begin{proof}
We first establish a structural property of optimal routing solutions.
For a communication route $\pi$ containing two switches $s_a$ and $s_b$,
let $\pi[s_a,s_b]$ denote the physical subpath between them, independently
of its traversal direction.

\begin{lemma}[Common-subpath property]
\label{lem:common_subpath}
There exists a globally optimal routing configuration such that, for any pair
of switches $s_a$ and $s_b$, all communication routes containing both switches
use the same switch sequence between them, possibly in reverse order.
\end{lemma}

\begin{proof}
Index the switches up to the budget $S_{\max}$, and assign every possible
undirected connection $e$ between terminals and switches a positive
tie-breaking weight
$
    \omega_e = 2^{k(e)},
$
where the exponents $k(e)$ are distinct.
For a route $\pi=(v_0,\ldots,v_k)$, define
$
    \tau(\pi)
    =
    \sum_{j=1}^{k}\omega_{\{v_{j-1},v_j\}}.
$
Since routes are simple, two distinct paths between the same endpoints,
up to reversal, have different values of $\tau$.

Among all globally optimal routing configurations, choose one minimizing
$
    \Psi
    =
    \sum_{(i,t)\in\mathcal R}\tau(\pi_{i,t}).
$
Such a configuration exists because the switch budget and the number of
communication pairs are finite, and hence only finitely many simple route
sequences are possible.

Suppose, for contradiction, that two communication routes contain the same
switches $s_a$ and $s_b$ but use different switch sequences $P_1$ and $P_2$
between them.
Orient both sequences from $s_a$ to $s_b$, and let $L(P)$ denote their
route-length contribution.
Without loss of generality, assume that $L(P_1)<L(P_2)$, or that
$L(P_1)=L(P_2)$ and $\tau(P_1)<\tau(P_2)$.

Replace $P_2$ in its route by $P_1$, using the reverse of $P_1$ if necessary.
Every connection of $P_1$ is already present in the routing configuration, so
this introduces no new physical wire.
If the resulting route contains a repeated switch, erase the resulting loops;
this preserves connectivity and can only remove physical wire and route
length.

If $L(P_1)<L(P_2)$, the route-length term strictly decreases while
wirelength does not increase, contradicting global optimality.
Otherwise, the objective does not increase and therefore, by global
optimality, remains unchanged.
The resulting configuration is thus also globally optimal, but replacing
$P_2$ by $P_1$ strictly decreases $\Psi$.
Any loop removal only decreases it further because all $\omega_e$ are
positive.
This contradicts the choice of the globally optimal configuration minimizing
$\Psi$.

Therefore no such pair of routes can exist, proving the claim.
\end{proof}

Now let $G^\star$ be such a globally optimal configuration.
By Proposition~\ref{prop:hanan_placement}, we may additionally choose
$G^\star$ such that every switch is placed on the extended Hanan grid
$\mathcal H$.

We now construct a finite sequence of contractions
$
    G^\star = G^{(m)}
    \xrightarrow{C_m}
    G^{(m-1)}
    \xrightarrow{C_{m-1}}
    \cdots
    \xrightarrow{C_1}
    G^{(0)},
$
where $G^{(0)}$ contains a single switch.

Consider a configuration $G^{(k)}$ containing at least two switches.
If some communication route contains at least two switches, choose two
switches $s_a$ and $s_b$ that are consecutive on that route.
Their physical subpath is the single connection between $s_a$ and $s_b$;
hence, by Lemma~\ref{lem:common_subpath}, every other communication route
containing both switches also contains them consecutively, possibly in the
opposite order.
If no communication route contains two switches, choose any two remaining
switches $s_a$ and $s_b$, which cannot occur together on any communication
route.

For every communication pair whose route contains at least one of
$s_a$ and $s_b$, the local subpath involving these switches is therefore
one of
$
    \rightarrow s_a \rightarrow,\qquad
    \rightarrow s_b \rightarrow,\qquad
    \rightarrow s_a \rightarrow s_b \rightarrow,\qquad
    \rightarrow s_b \rightarrow s_a \rightarrow.
$
Before contracting the switches, we record this local subpath for each
such communication pair.

The contraction $C_k$ identifies $s_a$ and $s_b$ with a single
macro-switch $s$ and replaces each of these local subpaths by
$
    \rightarrow s \rightarrow.
$
All other communication routes are left unchanged.
Each contraction removes one switch. Repeating this construction therefore
yields, after finitely many steps, a single-switch configuration $G^{(0)}$.

We now reverse the sequence constructively.
For each contraction $C_k$, we recorded the two switches $(s_a,s_b)$,
their positions in $G^\star$, and the local subpath of every communication
route affected by the contraction.
Starting from $G^{(k-1)}$, we expand the corresponding macro-switch $s$
into $s_a$ and $s_b$ and assign them their positions
$
    \mathbf p_{s_a}^\star,\mathbf p_{s_b}^\star\in\mathcal H.
$
For each communication pair traversing $s$, route refinement replaces
$
    \rightarrow s \rightarrow
$
by its recorded local subpath,
$
    \rightarrow s_a \rightarrow,\qquad
    \rightarrow s_b \rightarrow,\qquad
    \rightarrow s_a \rightarrow s_b \rightarrow,\qquad\text{or}\qquad
    \rightarrow s_b \rightarrow s_a \rightarrow.
$
All other communication routes are unaffected.
Thus, the expansion, placement, and refinement operations reconstruct
$G^{(k)}$ exactly from $G^{(k-1)}$.

Reversing all contractions consequently gives
$
    G^{(0)}
    \longrightarrow
    G^{(1)}
    \longrightarrow
    \cdots
    \longrightarrow
    G^{(m)}=G^\star,
$
where every transition consists of one switch expansion, the corresponding
switch placements, and the recorded route refinements.
Hence $G^\star$ can be constructed from the initial configuration through
a finite sequence of switch expansions, placements, and route refinements.
Since $G^\star$ is globally optimal, the hierarchical search space contains
a globally optimal solution.
\end{proof}

\section{Implementation Details}
\label{app:implementation}

\subsection{Model Architecture}
\label{app:model_architecture}

The policy and value functions share a neural architecture that jointly encodes
the current routing configuration and physical floorplan.
The routing configuration is represented as an augmented directed graph
containing initiators, targets, instantiated switches, and route-bundle nodes
representing physical segments shared by one or more communication routes.
The floorplan is represented by a raster encoding of the blockage geometry and
legal candidate switch locations.

The architecture is structured as follows:
\begin{enumerate}[leftmargin=0.5cm]
    \itemsep0em

    \item \textbf{Routing graph representation:}
    Graph nodes represent initiators, targets, switches, and route bundles.
    Node and edge attributes encode node type, normalized coordinates,
    communication identities, active switches, geometric distances, and the
    current routing configuration.

    \item \textbf{Graph feature embedding:}
    Coordinate and categorical attributes are independently embedded into
    $24$-dimensional representations and combined to initialize the graph
    features used by the policy and value networks.

    \item \textbf{Floorplan encoding:}
    The physical floorplan is represented by a two-channel $128\times128$
    image.
    The first channel encodes blockage occupancy, while the second identifies
    legal candidate switch locations.
    A convolutional encoder consisting of four blocks with GroupNorm and ReLU
    activations produces spatial and global floorplan representations.
    These representations condition both the routing-graph features and the
    switch-placement predictions.

    \item \textbf{Graph processing:}
    The routing graph is processed by four directed message-passing layers
    with hidden dimension $80$.
    This produces contextual node representations that combine the current
    routing structure with the encoded physical environment.

    \item \textbf{Policy prediction:}
    Separate policy heads parameterize the three action types of the
    hierarchical MDP: switch expansion, switch placement, and route refinement.
    For expansion, the policy predicts a scalar score for each existing switch,
    normalized across the switches in the current graph.
    For placement, it predicts a two-dimensional score map over candidate
    locations for each switch being placed.
    Invalid and absent locations are masked, and each placement map is
    normalized independently.
    For refinement, the policy predicts four scores for each relevant graph
    edge, corresponding to the four refinement alternatives defined in
    Sec.~\ref{sec:method}.
    Predictions associated with edges belonging to the same affected
    communication route are averaged to obtain its four refinement scores.

    \item \textbf{Value prediction:}
    The contextual node representations are aggregated using attention pooling
    with four learned queries and passed through a value network with two hidden
    layers of dimension $160$.
    For PPO-EWMA, the network predicts a scalar state value.
    For Gumbel MCTS, following AlphaTensor~\cite{fawzi2022discovering}, it predicts eight return quantiles, with the mean of the
    two largest quantiles used as the leaf value during search.
    The quantile outputs are trained using the quantile Huber loss at levels
    $\tau_i=(i-\tfrac{1}{2})/8$.

    We use PopArt normalization with separate statistics for each floorplan.
    Statistics are updated once per fresh trajectory collection using a decay
    of $0.9$ and a minimum standard deviation of $10^{-6}$.
    Following each update, the output-layer parameters are rescaled so that
    the corresponding unnormalized value predictions remain unchanged.
\end{enumerate}

The complete policy and value architecture contains approximately
$0.7$ million trainable parameters.
Both learning methods use AdamW with learning rate $10^{-4}$ and zero weight
decay, per-GPU optimization batches of $4096$, and gradient-norm clipping at
$1$.

\subsection{PPO-EWMA}
\label{app:ppo}

We train PPO using an exponentially weighted moving-average proximal
policy~\cite{hilton2022batch}.
Complete trajectories are collected under the current behavior policy and
optimized using a clipped importance-sampling objective relative to the
exponentially averaged proximal policy.
We use undiscounted Monte Carlo returns with $\gamma=1$.

We use one optimization epoch per rollout, a value-loss coefficient of $0.5$,
an entropy coefficient of $0.05$, a clipping coefficient of $0.01$, a
proximal-policy EWMA decay of $0.889$, and a maximum importance ratio of
$100$.
Advantages are normalized separately for each floorplan using exponentially
weighted running moments with decay $0.9$.
We collect $8192$ trajectories per GPU per collection.

\subsection{Gumbel MCTS}
\label{app:gumbel_mcts_impl}

We implement Gumbel MCTS following Gumbel
AlphaZero~\cite{danihelka2022policy}.
At each decision state, the policy network provides action priors and the
value network predicts eight return quantiles.
The mean of the largest two predicted quantiles, corresponding to the upper
quartile of the predicted return distribution, is used as the scalar leaf
value during search.
At the root, actions are selected using Gumbel-perturbed policy logits and
evaluated through sequential halving.
At non-root states, simulations are allocated according to the completed-value
policy-improvement rule.

During training, we use $800$ MCTS simulations per decision and collect $512$
trajectories per floorplan per collection.
At most $128$ root actions are retained for sequential halving.
We use unit-scale Gumbel perturbations, $c_{\mathrm{visit}}=50$, and
$c_{\mathrm{scale}}=0.01$.

Search simulations are evaluated asynchronously.
Within each root-action subtree, the number of concurrent simulations is
limited to $0.08$ of the simulation budget.
A virtual loss~\cite{chaslot2008parallel} of $0.1$, expressed in normalized completed-$Q$ units, is
applied to in-flight branches to reduce collisions between concurrent
simulations.

To stabilize learning from search-improved targets, we interpolate the policy
target with the current policy using $\eta_\pi=0.25$:
$
    \pi_{\mathrm{target}}
    =
    (1-\eta_\pi)\pi_{\mathrm{prior}}
    +
    \eta_\pi\pi_{\mathrm{search}}.
$

For the quantile critic, let
$\mathbf z_{\mathrm{prior}}=(z_1,\ldots,z_8)$ denote the predicted return
quantiles and let $G$ denote the realized return.
At collection $k$, each quantile target is
\begin{equation}
    z_{i,\mathrm{target}}
    =
    (1-\eta_V^{(k)})z_{i,\mathrm{prior}}
    +
    \eta_V^{(k)}G,
    \qquad
    \eta_V^{(k)}
    =
    \max\left(0.25,\frac{1}{k}\right).
\end{equation}
The larger interpolation coefficient during the initial collections mitigates
critic cold start; afterward, it matches the policy interpolation coefficient.

MCTS samples are retained for two collections, corresponding to an expected
$20$ optimization replays per sample.
The value-loss coefficient is $0.5$.

\section{Experimental Details}
\label{app:experiments}

This section provides additional details on the datasets, baselines, training
setup, computational resources, and evaluation protocols used in our
experiments.

Unless otherwise stated, all methods are evaluated using the same routing
objective and feasibility requirements.
Random Search and the Genetic Algorithm operate directly in the same
hierarchical routing environment as the learning-based methods.
The Heuristic uses the same extended Hanan grid.
PPO-EWMA and Gumbel MCTS share the same routing environment and neural
backbone, with scalar and quantile value heads, respectively.

\subsection{Datasets and Transfer Protocol}
\label{app:datasets}

Our pretraining dataset contains $24$ synthetically generated routing
floorplans.
Each instance contains between $6$ and $18$ rectangular blockages,
$3$--$5$ initiators, $5$--$8$ targets, and $18$--$25$ required directed
communications.
The switch budget ranges from $2$ to $5$.
All floorplans are square.

For each learning algorithm, a single shared model is trained across all
pretraining floorplans, with trajectories distributed approximately uniformly
across instances during data collection.

For transfer experiments, we use four target floorplans held out from
pretraining, covering two-, three-, and four-switch routing settings.
For each target instance, we compare initialization from the corresponding
pretrained model against training from scratch.
When transferring a pretrained model, we retain all shape-compatible shared
parameters, including the floorplan encoder, message-passing layers, policy
heads, attention pooling, and hidden value-network layers.
The initiator-, target-, and route-identifier embeddings and the
floorplan-specific PopArt output layer are reinitialized.
Optimizer state, replay data, PopArt statistics, and algorithm-specific
training state are not transferred.

\subsection{Optimization Baselines}
\label{app:baselines}

\noindent\textbf{Heuristic.}
The Heuristic is a deterministic obstacle-aware, Steiner-inspired constructive
method operating on the extended Hanan grid.
Every
communication route is required to traverse at least one switch.
Candidate locations are ranked by evaluating single-switch routing solutions,
and the best $128$ candidates are retained.
The method evaluates the retained one-switch solutions and greedily expands
the best partial solution by adding switches until the prescribed maximum
budget is reached.
The best solution encountered across all intermediate switch counts is
retained, so the returned solution may use fewer switches than the maximum
budget.

For each candidate switch set, communication pairs are inserted sequentially
using obstacle-aware shortest paths that minimize their marginal contribution
to the objective.
For a segment of length $d$, introducing a new physical connection incurs
$d+0.5d$, whereas reusing an existing physical connection incurs only the
additional route cost $0.5d$.
We evaluate eight deterministic demand orderings and retain the best resulting
network.
Finally, one local-refinement pass considers single-switch replacements and
accepts strictly improving configurations.

\noindent\textbf{Random Search.}
Random Search operates directly on the same hierarchical construction process
as the learning-based methods, but uses no learned policy or value function.
At each state, it samples uniformly from the currently admissible actions,
thereby producing feasible routing configurations by construction.
It evaluates $18{,}000$ complete trajectories per GPU and generation and
retains the best solution found.

\noindent\textbf{Genetic Algorithm.}
The Genetic Algorithm represents each candidate solution as a sequence of
hierarchical decisions and evaluates it using the same routing environment.
Starting from a random population, subsequent generations combine elitist
selection, tournament selection, one-point crossover, per-decision mutation,
and random immigration.
Because action legality depends on preceding decisions, inherited actions that
are no longer admissible are replaced by uniformly sampled legal actions.
We use an elite pool of $64$ programs per floorplan, tournament size $4$,
crossover probability $0.9$, mutation probability $0.05$, and a $10\%$ random
immigrant fraction.
Each generation evaluates $18{,}000$ complete trajectories per GPU.

\subsection{Compute and Evaluation}
\label{app:compute}

Pretraining is performed on six NVIDIA L40S GPUs.
PPO-EWMA, Gumbel MCTS, Random Search, and the Genetic Algorithm are each run
for $48$ hours on the $24$ pretraining floorplans, with computation distributed
approximately uniformly across instances.
The deterministic Heuristic is run once for each floorplan.
Due to the computational cost of multi-floorplan pretraining, we train one
pretrained model for each learning algorithm.

Fine-tuning is performed for $24$ hours on a single NVIDIA L40S GPU.
Transfer results are aggregated over three independent fine-tuning runs for each
target floorplan and initialization setting.
All pretrained fine-tuning runs for a given learning algorithm are initialized
from the same fixed pretrained checkpoint.

\clearpage

\section{Detailed Numerical Results} \label{app:detailed_results}

This section provides detailed numerical results for all experiments.
For each floorplan, we additionally report its main characteristics, including
the switch budget, number of initiators, targets, and communication routes, and
free space, defined as the percentage of the floorplan area not covered by
blockages.
We report the route length, wirelength, and best objective found by each method
within its corresponding optimization budget.
We similarly report transfer results on the four held-out floorplans.

\subsection{Pretraining}

\begin{center}
\begin{minipage}{\textwidth}
\centering
\captionof{table}{Detailed numerical results for Heuristic on the 24 pretraining floorplans.}
\label{tab:detailed_heuristic}
\setlength{\tabcolsep}{5pt}
\resizebox{\textwidth}{!}{
\begin{tabular}{@{}lcccccccc@{}}
\toprule
\multicolumn{9}{c}{\textbf{Heuristic}} \\
\midrule
\textbf{Instance} &
\textbf{Switch budget} &
\textbf{Initiators} &
\textbf{Targets} &
\textbf{Routes} &
\textbf{Free space} (\%) &
\textbf{Route length} $ \downarrow $ &
\textbf{Wirelength} $ \downarrow $ &
\textbf{Objective} $ \downarrow $ \\
\midrule
1 & 4 & 5 & 8 & 21 & 71.9 & 16.842 & 3.481 & 11.902 \\
2 & 4 & 4 & 7 & 18 & 68.9 & 15.975 & 3.267 & 11.255 \\
3 & 4 & 5 & 8 & 23 & 71.0 & 19.626 & 5.536 & 15.349 \\
4 & 4 & 5 & 8 & 22 & 77.2 & 15.623 & 3.518 & 11.330 \\
5 & 4 & 5 & 8 & 19 & 76.1 & 17.938 & 5.878 & 14.847 \\
6 & 4 & 5 & 8 & 22 & 81.0 & 15.640 & 3.306 & 11.126 \\
7 & 4 & 5 & 8 & 22 & 75.0 & 9.796 & 3.192 & 8.090 \\
8 & 4 & 5 & 8 & 23 & 77.7 & 14.270 & 3.166 & 10.301 \\
9 & 4 & 5 & 8 & 21 & 78.5 & 15.834 & 3.375 & 11.292 \\
10 & 4 & 5 & 8 & 23 & 81.3 & 20.724 & 4.253 & 14.615 \\
11 & 3 & 4 & 5 & 19 & 68.0 & 21.356 & 4.063 & 14.741 \\
12 & 3 & 4 & 5 & 19 & 74.9 & 19.355 & 3.740 & 13.418 \\
13 & 3 & 4 & 5 & 19 & 75.2 & 19.670 & 3.827 & 13.662 \\
14 & 3 & 4 & 5 & 19 & 72.2 & 20.028 & 4.989 & 15.003 \\
15 & 2 & 3 & 7 & 20 & 68.5 & 19.366 & 4.829 & 14.512 \\
16 & 2 & 3 & 7 & 20 & 68.7 & 21.211 & 4.679 & 15.284 \\
17 & 2 & 3 & 7 & 20 & 63.7 & 23.431 & 4.855 & 16.571 \\
18 & 5 & 5 & 5 & 25 & 66.5 & 25.126 & 3.803 & 16.366 \\
19 & 4 & 5 & 5 & 25 & 76.2 & 24.959 & 3.797 & 16.277 \\
20 & 2 & 4 & 5 & 19 & 60.6 & 22.797 & 4.879 & 16.277 \\
21 & 4 & 4 & 5 & 19 & 58.4 & 20.369 & 4.724 & 14.909 \\
22 & 3 & 4 & 5 & 19 & 63.2 & 20.309 & 5.260 & 15.415 \\
23 & 4 & 4 & 5 & 19 & 61.8 & 18.653 & 4.681 & 14.007 \\
24 & 2 & 3 & 7 & 20 & 63.1 & 20.961 & 5.203 & 15.684 \\
\bottomrule
\end{tabular}
}
\end{minipage}
\end{center}

\begin{center}
\begin{minipage}{\textwidth}
\centering
\captionof{table}{Detailed numerical results for Random search on the 24 pretraining floorplans.}
\label{tab:detailed_random_search}
\setlength{\tabcolsep}{5pt}
\resizebox{\textwidth}{!}{
\begin{tabular}{@{}lcccccccc@{}}
\toprule
\multicolumn{9}{c}{\textbf{Random search}} \\
\midrule
\textbf{Instance} &
\textbf{Switch budget} &
\textbf{Initiators} &
\textbf{Targets} &
\textbf{Routes} &
\textbf{Free space} (\%) &
\textbf{Route length} $ \downarrow $ &
\textbf{Wirelength} $ \downarrow $ &
\textbf{Objective} $ \downarrow $ \\
\midrule
1 & 4 & 5 & 8 & 21 & 71.9 & 18.540 & 9.436 & 18.706 \\
2 & 4 & 4 & 7 & 18 & 68.9 & 23.403 & 9.098 & 20.799 \\
3 & 4 & 5 & 8 & 23 & 71.0 & 24.834 & 11.625 & 24.042 \\
4 & 4 & 5 & 8 & 22 & 77.2 & 18.655 & 8.092 & 17.420 \\
5 & 4 & 5 & 8 & 19 & 76.1 & 23.366 & 11.946 & 23.629 \\
6 & 4 & 5 & 8 & 22 & 81.0 & 20.092 & 9.264 & 19.310 \\
7 & 4 & 5 & 8 & 22 & 75.0 & 13.404 & 5.970 & 12.672 \\
8 & 4 & 5 & 8 & 23 & 77.7 & 18.610 & 9.269 & 18.574 \\
9 & 4 & 5 & 8 & 21 & 78.5 & 19.788 & 8.769 & 18.663 \\
10 & 4 & 5 & 8 & 23 & 81.3 & 26.542 & 11.660 & 24.931 \\
11 & 3 & 4 & 5 & 19 & 68.0 & 25.169 & 9.999 & 22.584 \\
12 & 3 & 4 & 5 & 19 & 74.9 & 24.155 & 8.698 & 20.775 \\
13 & 3 & 4 & 5 & 19 & 75.2 & 24.537 & 9.157 & 21.425 \\
14 & 3 & 4 & 5 & 19 & 72.2 & 24.847 & 9.451 & 21.875 \\
15 & 2 & 3 & 7 & 20 & 68.5 & 22.518 & 8.079 & 19.338 \\
16 & 2 & 3 & 7 & 20 & 68.7 & 25.037 & 7.627 & 20.146 \\
17 & 2 & 3 & 7 & 20 & 63.7 & 25.112 & 8.387 & 20.943 \\
18 & 5 & 5 & 5 & 25 & 66.5 & 33.287 & 14.167 & 30.810 \\
19 & 4 & 5 & 5 & 25 & 76.2 & 34.856 & 11.602 & 29.030 \\
20 & 2 & 4 & 5 & 19 & 60.6 & 26.344 & 7.249 & 20.421 \\
21 & 4 & 4 & 5 & 19 & 58.4 & 27.774 & 9.796 & 23.683 \\
22 & 3 & 4 & 5 & 19 & 63.2 & 25.547 & 8.564 & 21.337 \\
23 & 4 & 4 & 5 & 19 & 61.8 & 24.468 & 10.223 & 22.457 \\
24 & 2 & 3 & 7 & 20 & 63.1 & 25.941 & 7.613 & 20.583 \\
\bottomrule
\end{tabular}
}
\end{minipage}
\end{center}

\begin{center}
\begin{minipage}{\textwidth}
\centering
\captionof{table}{Detailed numerical results for Genetic algorithm on the 24 pretraining floorplans.}
\label{tab:detailed_genetic_algorithm}
\setlength{\tabcolsep}{5pt}
\resizebox{\textwidth}{!}{
\begin{tabular}{@{}lcccccccc@{}}
\toprule
\multicolumn{9}{c}{\textbf{Genetic algorithm}} \\
\midrule
\textbf{Instance} &
\textbf{Switch budget} &
\textbf{Initiators} &
\textbf{Targets} &
\textbf{Routes} &
\textbf{Free space} (\%) &
\textbf{Route length} $ \downarrow $ &
\textbf{Wirelength} $ \downarrow $ &
\textbf{Objective} $ \downarrow $ \\
\midrule
1 & 4 & 5 & 8 & 21 & 71.9 & 17.970 & 5.397 & 14.382 \\
2 & 4 & 4 & 7 & 18 & 68.9 & 16.293 & 5.853 & 13.999 \\
3 & 4 & 5 & 8 & 23 & 71.0 & 22.228 & 5.740 & 16.854 \\
4 & 4 & 5 & 8 & 22 & 77.2 & 15.589 & 4.827 & 12.622 \\
5 & 4 & 5 & 8 & 19 & 76.1 & 20.256 & 7.842 & 17.970 \\
6 & 4 & 5 & 8 & 22 & 81.0 & 14.812 & 5.701 & 13.107 \\
7 & 4 & 5 & 8 & 22 & 75.0 & 9.982 & 3.233 & 8.224 \\
8 & 4 & 5 & 8 & 23 & 77.7 & 15.332 & 4.288 & 11.954 \\
9 & 4 & 5 & 8 & 21 & 78.5 & 17.572 & 5.366 & 14.152 \\
10 & 4 & 5 & 8 & 23 & 81.3 & 19.600 & 5.999 & 15.799 \\
11 & 3 & 4 & 5 & 19 & 68.0 & 21.500 & 5.728 & 16.478 \\
12 & 3 & 4 & 5 & 19 & 74.9 & 21.160 & 5.924 & 16.504 \\
13 & 3 & 4 & 5 & 19 & 75.2 & 18.885 & 5.796 & 15.238 \\
14 & 3 & 4 & 5 & 19 & 72.2 & 20.606 & 4.754 & 15.057 \\
15 & 2 & 3 & 7 & 20 & 68.5 & 19.638 & 4.315 & 14.134 \\
16 & 2 & 3 & 7 & 20 & 68.7 & 21.019 & 4.727 & 15.236 \\
17 & 2 & 3 & 7 & 20 & 63.7 & 20.303 & 5.087 & 15.239 \\
18 & 5 & 5 & 5 & 25 & 66.5 & 25.411 & 7.175 & 19.881 \\
19 & 4 & 5 & 5 & 25 & 76.2 & 29.455 & 7.334 & 22.062 \\
20 & 2 & 4 & 5 & 19 & 60.6 & 21.177 & 5.138 & 15.726 \\
21 & 4 & 4 & 5 & 19 & 58.4 & 19.118 & 7.021 & 16.580 \\
22 & 3 & 4 & 5 & 19 & 63.2 & 19.422 & 4.821 & 14.531 \\
23 & 4 & 4 & 5 & 19 & 61.8 & 19.822 & 4.710 & 14.621 \\
24 & 2 & 3 & 7 & 20 & 63.1 & 21.269 & 4.643 & 15.277 \\
\bottomrule
\end{tabular}
}
\end{minipage}
\end{center}

\begin{center}
\begin{minipage}{\textwidth}
\centering
\captionof{table}{Detailed numerical results for PPO-EWMA on the 24 pretraining floorplans.}
\label{tab:detailed_ppo}
\setlength{\tabcolsep}{5pt}
\resizebox{\textwidth}{!}{
\begin{tabular}{@{}lcccccccc@{}}
\toprule
\multicolumn{9}{c}{\textbf{PPO-EWMA}} \\
\midrule
\textbf{Instance} &
\textbf{Switch budget} &
\textbf{Initiators} &
\textbf{Targets} &
\textbf{Routes} &
\textbf{Free space} (\%) &
\textbf{Route length} $ \downarrow $ &
\textbf{Wirelength} $ \downarrow $ &
\textbf{Objective} $ \downarrow $ \\
\midrule
1 & 4 & 5 & 8 & 21 & 71.9 & 17.150 & 3.591 & 12.166 \\
2 & 4 & 4 & 7 & 18 & 68.9 & 15.969 & 3.499 & 11.484 \\
3 & 4 & 5 & 8 & 23 & 71.0 & 20.328 & 5.506 & 15.670 \\
4 & 4 & 5 & 8 & 22 & 77.2 & 15.109 & 4.021 & 11.576 \\
5 & 4 & 5 & 8 & 19 & 76.1 & 19.440 & 4.997 & 14.717 \\
6 & 4 & 5 & 8 & 22 & 81.0 & 15.130 & 3.395 & 10.960 \\
7 & 4 & 5 & 8 & 22 & 75.0 & 10.702 & 3.244 & 8.595 \\
8 & 4 & 5 & 8 & 23 & 77.7 & 14.482 & 3.111 & 10.352 \\
9 & 4 & 5 & 8 & 21 & 78.5 & 15.860 & 3.212 & 11.142 \\
10 & 4 & 5 & 8 & 23 & 81.3 & 19.434 & 3.612 & 13.329 \\
11 & 3 & 4 & 5 & 19 & 68.0 & 20.563 & 3.869 & 14.151 \\
12 & 3 & 4 & 5 & 19 & 74.9 & 19.274 & 3.724 & 13.361 \\
13 & 3 & 4 & 5 & 19 & 75.2 & 18.415 & 4.080 & 13.287 \\
14 & 3 & 4 & 5 & 19 & 72.2 & 19.741 & 3.722 & 13.593 \\
15 & 2 & 3 & 7 & 20 & 68.5 & 19.638 & 4.315 & 14.134 \\
16 & 2 & 3 & 7 & 20 & 68.7 & 20.543 & 4.998 & 15.269 \\
17 & 2 & 3 & 7 & 20 & 63.7 & 20.303 & 5.087 & 15.239 \\
18 & 5 & 5 & 5 & 25 & 66.5 & 27.654 & 5.320 & 19.147 \\
19 & 4 & 5 & 5 & 25 & 76.2 & 27.178 & 4.899 & 18.488 \\
20 & 2 & 4 & 5 & 19 & 60.6 & 22.841 & 4.396 & 15.817 \\
21 & 4 & 4 & 5 & 19 & 58.4 & 18.409 & 3.890 & 13.095 \\
22 & 3 & 4 & 5 & 19 & 63.2 & 20.314 & 3.832 & 13.989 \\
23 & 4 & 4 & 5 & 19 & 61.8 & 18.419 & 3.199 & 12.409 \\
24 & 2 & 3 & 7 & 20 & 63.1 & 21.269 & 4.643 & 15.277 \\
\bottomrule
\end{tabular}
}
\end{minipage}
\end{center}

\begin{center}
\begin{minipage}{\textwidth}
\centering
\captionof{table}{Detailed numerical results for Gumbel MCTS on the 24 pretraining floorplans.}
\label{tab:detailed_mcts_ours}
\setlength{\tabcolsep}{5pt}
\resizebox{\textwidth}{!}{
\begin{tabular}{@{}lcccccccc@{}}
\toprule
\multicolumn{9}{c}{\textbf{Gumbel MCTS}} \\
\midrule
\textbf{Instance} &
\textbf{Switch budget} &
\textbf{Initiators} &
\textbf{Targets} &
\textbf{Routes} &
\textbf{Free space} (\%) &
\textbf{Route length} $ \downarrow $ &
\textbf{Wirelength} $ \downarrow $ &
\textbf{Objective} $ \downarrow $ \\
\midrule
1 & 4 & 5 & 8 & 21 & 71.9 & 16.782 & 2.942 & 11.333 \\
2 & 4 & 4 & 7 & 18 & 68.9 & 15.463 & 3.148 & 10.880 \\
3 & 4 & 5 & 8 & 23 & 71.0 & 19.616 & 4.466 & 14.274 \\
4 & 4 & 5 & 8 & 22 & 77.2 & 13.849 & 2.859 & 9.784 \\
5 & 4 & 5 & 8 & 19 & 76.1 & 19.014 & 4.419 & 13.926 \\
6 & 4 & 5 & 8 & 22 & 81.0 & 14.496 & 2.670 & 9.918 \\
7 & 4 & 5 & 8 & 22 & 75.0 & 9.788 & 2.845 & 7.739 \\
8 & 4 & 5 & 8 & 23 & 77.7 & 14.084 & 2.855 & 9.897 \\
9 & 4 & 5 & 8 & 21 & 78.5 & 15.784 & 2.928 & 10.820 \\
10 & 4 & 5 & 8 & 23 & 81.3 & 19.168 & 3.444 & 13.028 \\
11 & 3 & 4 & 5 & 19 & 68.0 & 20.850 & 3.570 & 13.995 \\
12 & 3 & 4 & 5 & 19 & 74.9 & 19.274 & 3.724 & 13.361 \\
13 & 3 & 4 & 5 & 19 & 75.2 & 18.415 & 4.080 & 13.287 \\
14 & 3 & 4 & 5 & 19 & 72.2 & 19.870 & 3.531 & 13.466 \\
15 & 2 & 3 & 7 & 20 & 68.5 & 19.638 & 4.315 & 14.134 \\
16 & 2 & 3 & 7 & 20 & 68.7 & 21.019 & 4.727 & 15.236 \\
17 & 2 & 3 & 7 & 20 & 63.7 & 20.303 & 5.087 & 15.239 \\
18 & 5 & 5 & 5 & 25 & 66.5 & 22.272 & 2.465 & 13.601 \\
19 & 4 & 5 & 5 & 25 & 76.2 & 23.025 & 2.861 & 14.373 \\
20 & 2 & 4 & 5 & 19 & 60.6 & 21.177 & 5.138 & 15.726 \\
21 & 4 & 4 & 5 & 19 & 58.4 & 18.801 & 3.379 & 12.779 \\
22 & 3 & 4 & 5 & 19 & 63.2 & 19.669 & 3.637 & 13.471 \\
23 & 4 & 4 & 5 & 19 & 61.8 & 18.393 & 3.190 & 12.386 \\
24 & 2 & 3 & 7 & 20 & 63.1 & 21.269 & 4.643 & 15.277 \\
\bottomrule
\end{tabular}
}
\end{minipage}
\end{center}

\vspace{2em}

\subsection{Fine-Tuning}

\begin{center}
\begin{minipage}{\textwidth}
\centering
\captionof{table}{Detailed PPO-EWMA fine-tuning results with pretrained and scratch
initialization. Values are mean $\pm$ standard deviation over three runs.}
\label{tab:finetuning_detailed_ppo}
\setlength{\tabcolsep}{4pt}
\resizebox{\textwidth}{!}{
\begin{tabular}{@{}llcccccccc@{}}
\toprule
\textbf{Instance} &
\textbf{Initialization} &
\textbf{Switch budget} &
\textbf{Initiators} &
\textbf{Targets} &
\textbf{Routes} &
\textbf{Free space} (\%) &
\textbf{Route length} $\downarrow$ &
\textbf{Wirelength} $\downarrow$ &
\textbf{Objective} $\downarrow$ \\
\midrule

\multirow{2}{*}{1}
& Pretrained & \multirow{2}{*}{2} & \multirow{2}{*}{5} & \multirow{2}{*}{6}
& \multirow{2}{*}{20} & \multirow{2}{*}{76.1}
& $21.310 \vcenter{\hbox{\scriptsize $\pm 0.020$}}$
& $5.165 \vcenter{\hbox{\scriptsize $\pm 0.000$}}$
& $15.820 \vcenter{\hbox{\scriptsize $\pm 0.010$}}$ \\
& Scratch & & & & & 
& $21.290 \vcenter{\hbox{\scriptsize $\pm 0.000$}}$
& $5.165 \vcenter{\hbox{\scriptsize $\pm 0.000$}}$
& $15.810 \vcenter{\hbox{\scriptsize $\pm 0.000$}}$ \\
\midrule

\multirow{2}{*}{2}
& Pretrained & \multirow{2}{*}{3} & \multirow{2}{*}{5} & \multirow{2}{*}{6}
& \multirow{2}{*}{20} & \multirow{2}{*}{72.9}
& $18.765 \vcenter{\hbox{\scriptsize $\pm 0.010$}}$
& $4.550 \vcenter{\hbox{\scriptsize $\pm 0.005$}}$
& $13.933 \vcenter{\hbox{\scriptsize $\pm 0.000$}}$ \\
& Scratch & & & & &
& $18.765 \vcenter{\hbox{\scriptsize $\pm 0.010$}}$
& $4.550 \vcenter{\hbox{\scriptsize $\pm 0.005$}}$
& $13.933 \vcenter{\hbox{\scriptsize $\pm 0.000$}}$ \\
\midrule

\multirow{2}{*}{3}
& Pretrained & \multirow{2}{*}{4} & \multirow{2}{*}{5} & \multirow{2}{*}{6}
& \multirow{2}{*}{20} & \multirow{2}{*}{71.1}
& $18.045 \vcenter{\hbox{\scriptsize $\pm 0.000$}}$
& $4.325 \vcenter{\hbox{\scriptsize $\pm 0.000$}}$
& $13.347 \vcenter{\hbox{\scriptsize $\pm 0.000$}}$ \\
& Scratch & & & & &
& $20.115 \vcenter{\hbox{\scriptsize $\pm 0.050$}}$
& $4.222 \vcenter{\hbox{\scriptsize $\pm 0.477$}}$
& $14.280 \vcenter{\hbox{\scriptsize $\pm 0.503$}}$ \\
\midrule

\multirow{2}{*}{4}
& Pretrained & \multirow{2}{*}{4} & \multirow{2}{*}{5} & \multirow{2}{*}{5}
& \multirow{2}{*}{25} & \multirow{2}{*}{69.1}
& $23.554 \vcenter{\hbox{\scriptsize $\pm 0.139$}}$
& $3.048 \vcenter{\hbox{\scriptsize $\pm 0.126$}}$
& $14.825 \vcenter{\hbox{\scriptsize $\pm 0.057$}}$ \\
& Scratch & & & & &
& $24.000 \vcenter{\hbox{\scriptsize $\pm 0.304$}}$
& $3.079 \vcenter{\hbox{\scriptsize $\pm 0.198$}}$
& $15.078 \vcenter{\hbox{\scriptsize $\pm 0.349$}}$ \\

\bottomrule
\end{tabular}
}
\end{minipage}
\end{center}

\begin{center}
\begin{minipage}{\textwidth}
\centering
\captionof{table}{Detailed Gumbel MCTS fine-tuning results with pretrained and scratch
initialization. Values are mean $\pm$ standard deviation over three runs.}
\label{tab:finetuning_detailed_mcts}
\setlength{\tabcolsep}{4pt}
\resizebox{\textwidth}{!}{
\begin{tabular}{@{}llcccccccc@{}}
\toprule
\textbf{Instance} &
\textbf{Initialization} &
\textbf{Switch budget} &
\textbf{Initiators} &
\textbf{Targets} &
\textbf{Routes} &
\textbf{Free space} (\%) &
\textbf{Route length} $\downarrow$ &
\textbf{Wirelength} $\downarrow$ &
\textbf{Objective} $\downarrow$ \\
\midrule

\multirow{2}{*}{1}
& Pretrained & \multirow{2}{*}{2} & \multirow{2}{*}{5} & \multirow{2}{*}{6}
& \multirow{2}{*}{20} & \multirow{2}{*}{76.1}
& $21.290 \vcenter{\hbox{\scriptsize $\pm 0.000$}}$
& $5.165 \vcenter{\hbox{\scriptsize $\pm 0.000$}}$
& $15.810 \vcenter{\hbox{\scriptsize $\pm 0.000$}}$ \\
& Scratch & & & & &
& $21.290 \vcenter{\hbox{\scriptsize $\pm 0.000$}}$
& $5.165 \vcenter{\hbox{\scriptsize $\pm 0.000$}}$
& $15.810 \vcenter{\hbox{\scriptsize $\pm 0.000$}}$ \\
\midrule

\multirow{2}{*}{2}
& Pretrained & \multirow{2}{*}{3} & \multirow{2}{*}{5} & \multirow{2}{*}{6}
& \multirow{2}{*}{20} & \multirow{2}{*}{72.9}
& $20.135 \vcenter{\hbox{\scriptsize $\pm 0.008$}}$
& $3.903 \vcenter{\hbox{\scriptsize $\pm 0.002$}}$
& $13.971 \vcenter{\hbox{\scriptsize $\pm 0.005$}}$ \\
& Scratch & & & & &
& $20.365 \vcenter{\hbox{\scriptsize $\pm 0.339$}}$
& $3.970 \vcenter{\hbox{\scriptsize $\pm 0.092$}}$
& $14.153 \vcenter{\hbox{\scriptsize $\pm 0.262$}}$ \\
\midrule

\multirow{2}{*}{3}
& Pretrained & \multirow{2}{*}{4} & \multirow{2}{*}{5} & \multirow{2}{*}{6}
& \multirow{2}{*}{20} & \multirow{2}{*}{71.1}
& $18.895 \vcenter{\hbox{\scriptsize $\pm 0.523$}}$
& $3.970 \vcenter{\hbox{\scriptsize $\pm 0.264$}}$
& $13.417 \vcenter{\hbox{\scriptsize $\pm 0.012$}}$ \\
& Scratch & & & & &
& $19.532 \vcenter{\hbox{\scriptsize $\pm 0.207$}}$
& $3.815 \vcenter{\hbox{\scriptsize $\pm 0.142$}}$
& $13.581 \vcenter{\hbox{\scriptsize $\pm 0.245$}}$ \\
\midrule

\multirow{2}{*}{4}
& Pretrained & \multirow{2}{*}{4} & \multirow{2}{*}{5} & \multirow{2}{*}{5}
& \multirow{2}{*}{25} & \multirow{2}{*}{69.1}
& $20.620 \vcenter{\hbox{\scriptsize $\pm 0.000$}}$
& $3.907 \vcenter{\hbox{\scriptsize $\pm 0.000$}}$
& $14.217 \vcenter{\hbox{\scriptsize $\pm 0.000$}}$ \\
& Scratch & & & & &
& $23.134 \vcenter{\hbox{\scriptsize $\pm 0.215$}}$
& $2.852 \vcenter{\hbox{\scriptsize $\pm 0.009$}}$
& $14.419 \vcenter{\hbox{\scriptsize $\pm 0.116$}}$ \\

\bottomrule
\end{tabular}
}
\end{minipage}
\end{center}

\clearpage

\section{Floorplans Generated by the Different Methods}\label{app:visualization}

This section provides the solution corresponding to the best objective value found for each experiment and method.

\subsection{Pretraining}

\begin{figure*}[h]
\centering
    \begin{subfigure}[t]{0.31\linewidth}
        \vspace{0pt}
        \centering
        \includegraphics[width=\linewidth]{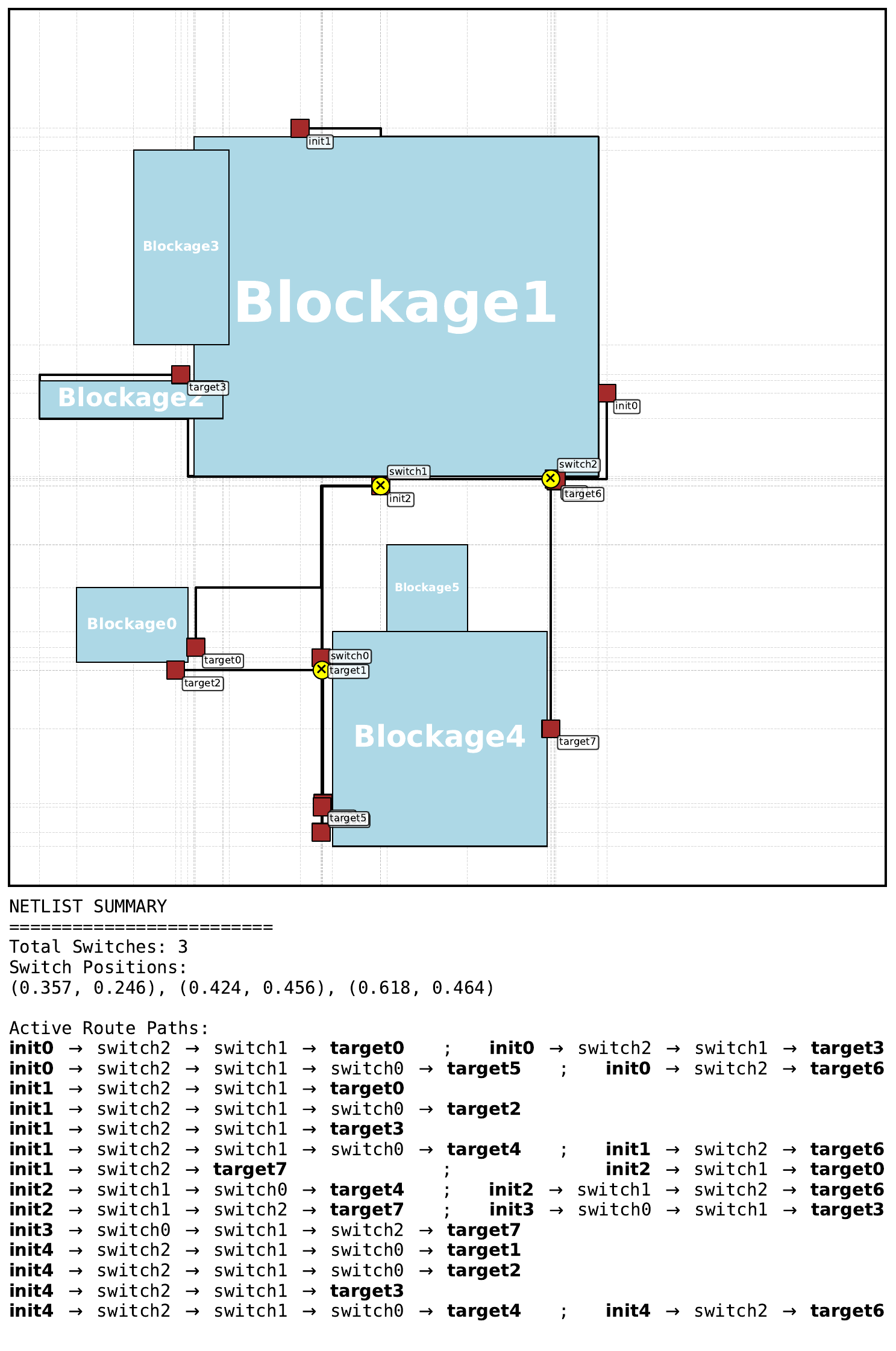}
        \caption*{Heuristic}
    \end{subfigure}
    \hfill
    \begin{subfigure}[t]{0.31\linewidth}
        \vspace{0pt}
        \centering
        \includegraphics[width=\linewidth]{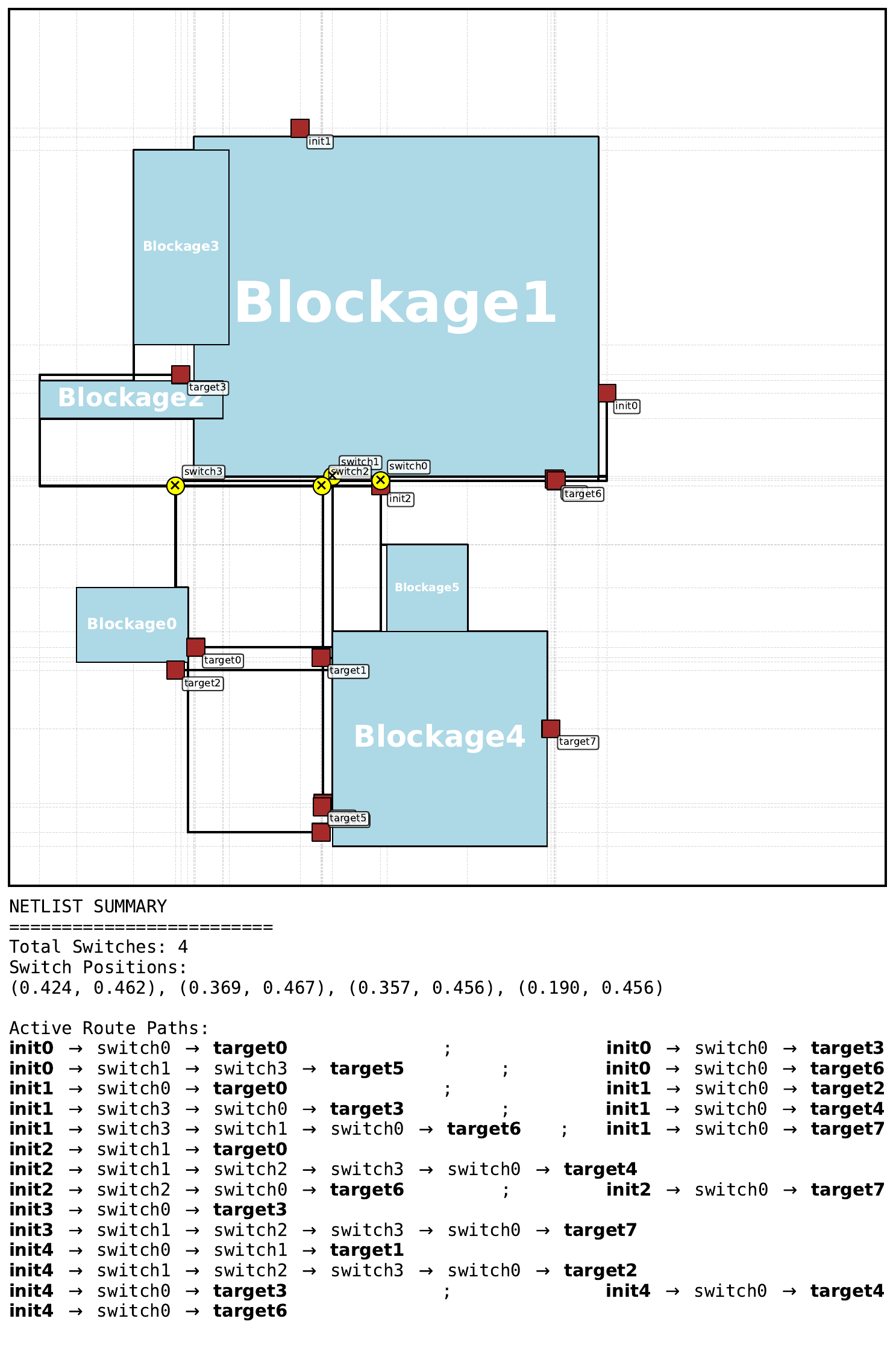}
        \caption*{Random search}
    \end{subfigure}
    \hfill
    \begin{subfigure}[t]{0.31\linewidth}
        \vspace{0pt}
        \centering
        \includegraphics[width=\linewidth]{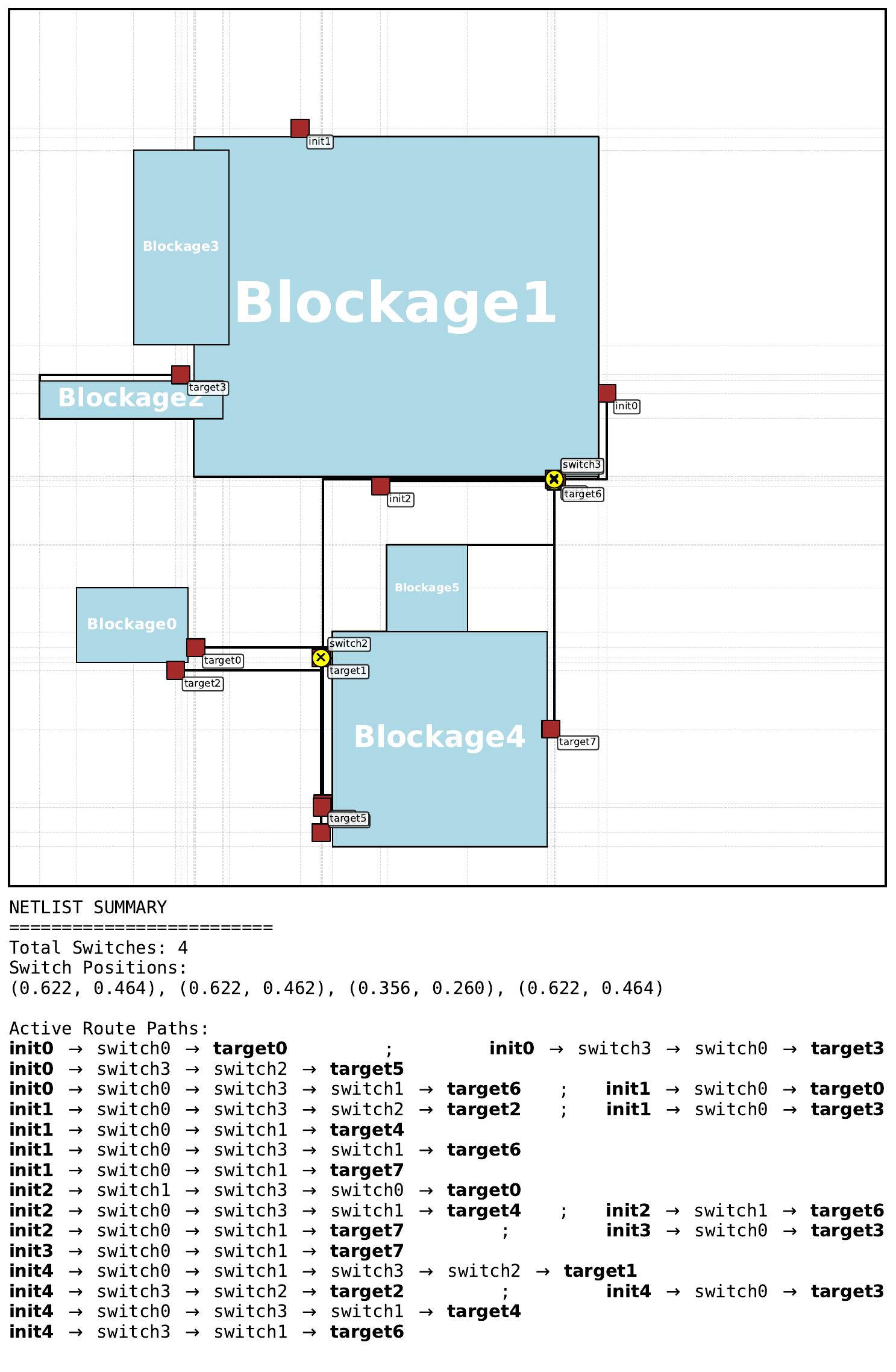}
        \caption*{Genetic algorithm}
    \end{subfigure}
    \\[0.6em]
    \begin{subfigure}[t]{0.31\linewidth}
        \vspace{0pt}
        \centering
        \includegraphics[width=\linewidth]{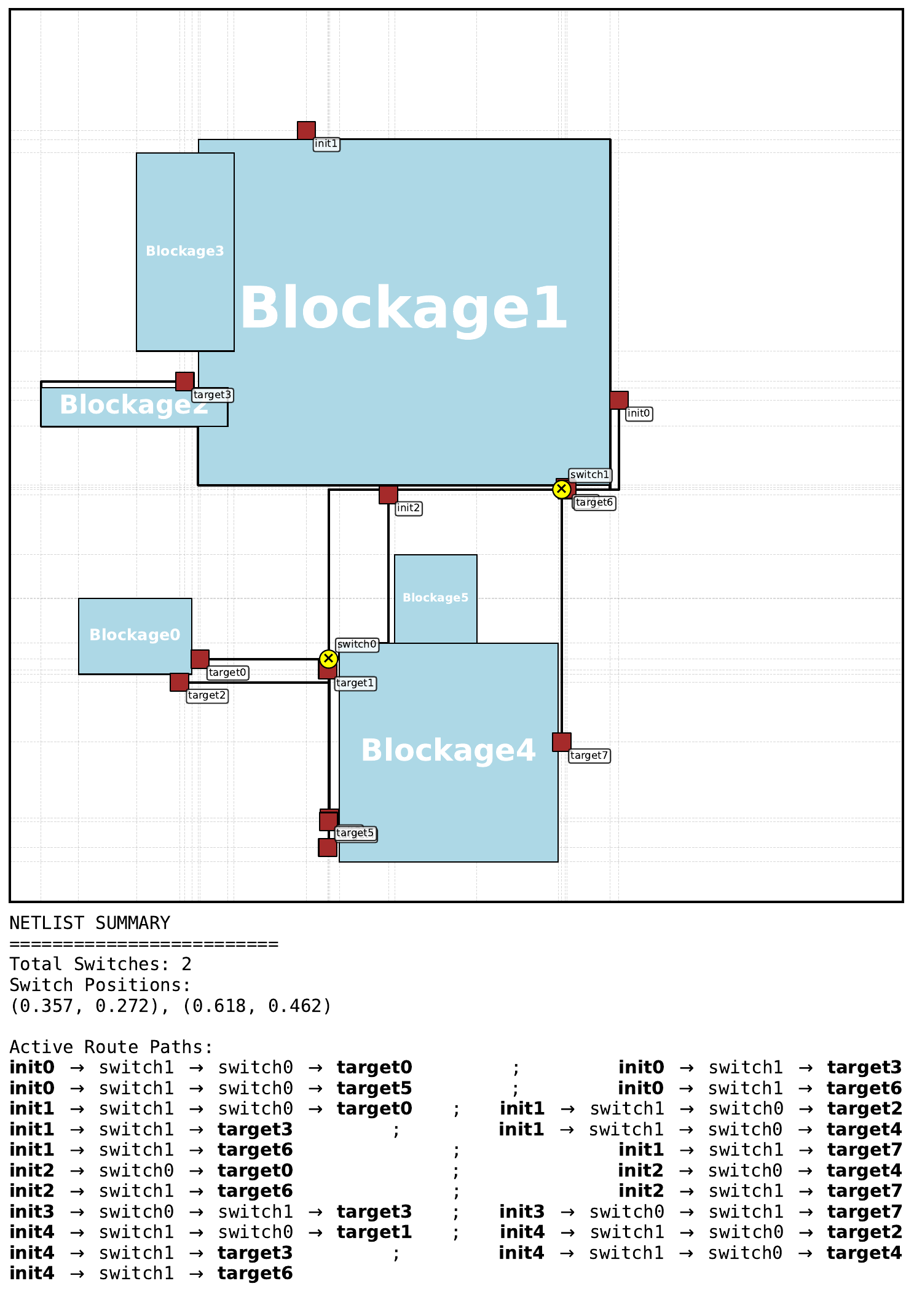}
        \caption*{PPO}
    \end{subfigure}
    \hspace{0.04\linewidth}
    \begin{subfigure}[t]{0.31\linewidth}
        \vspace{0pt}
        \centering
        \includegraphics[width=\linewidth]{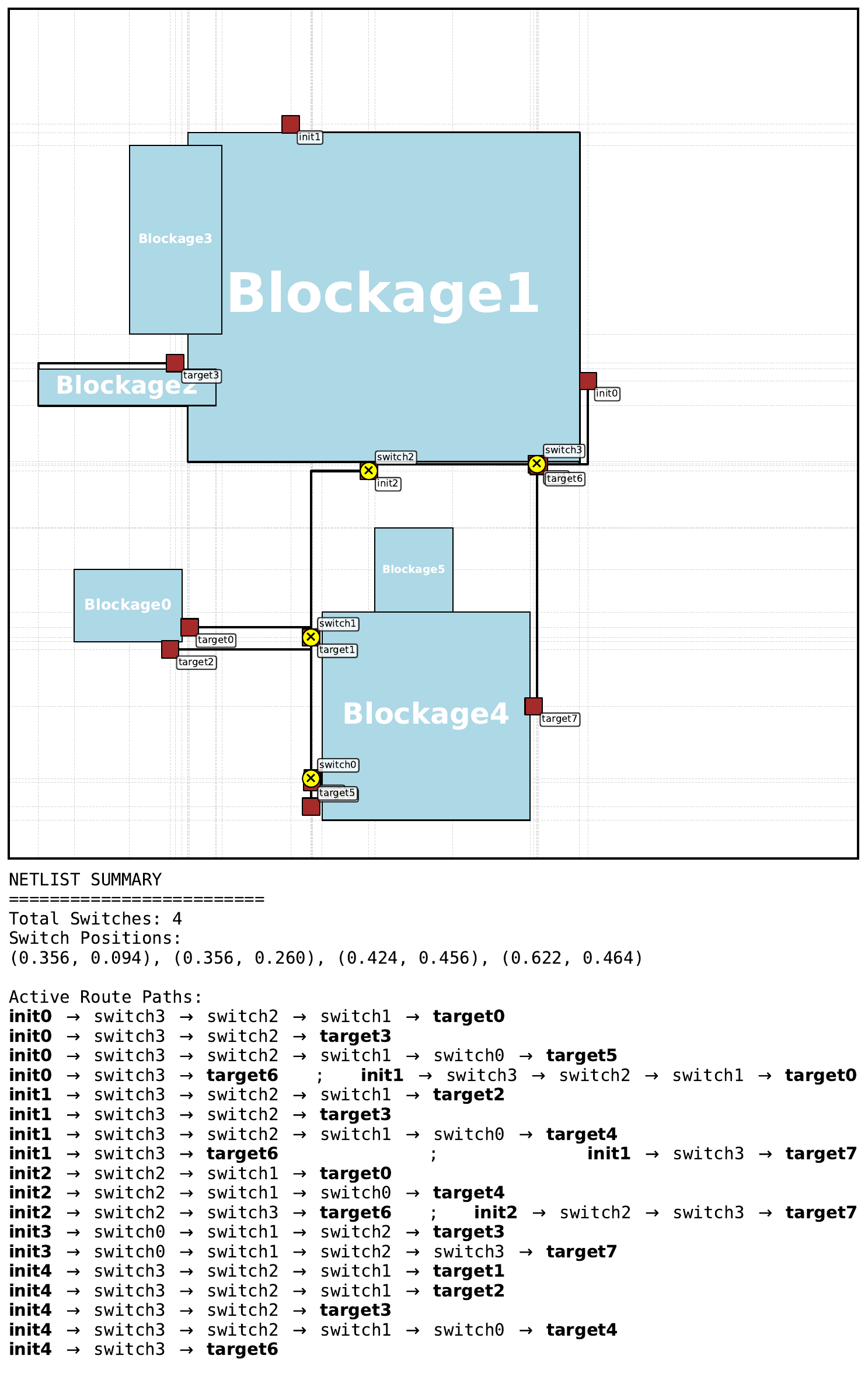}
        \caption*{MCTS}
    \end{subfigure}
\caption{Instance 1.}
\label{fig:best_pretrain_instance_1}
\end{figure*}

\begin{figure*}[h]
\centering
    \begin{subfigure}[t]{0.31\linewidth}
        \vspace{0pt}
        \centering
        \includegraphics[width=\linewidth]{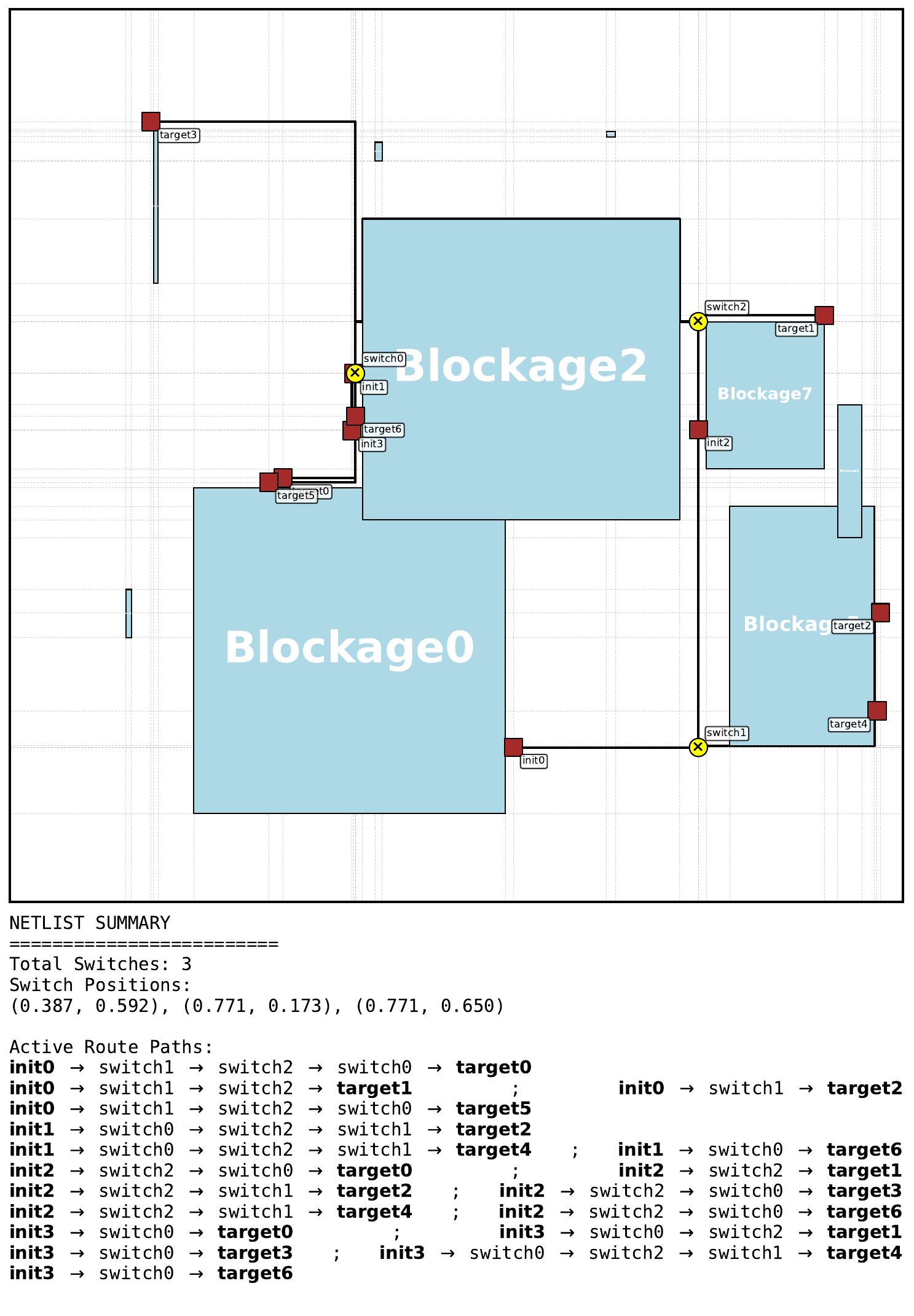}
        \caption*{Heuristic}
    \end{subfigure}
    \hfill
    \begin{subfigure}[t]{0.31\linewidth}
        \vspace{0pt}
        \centering
        \includegraphics[width=\linewidth]{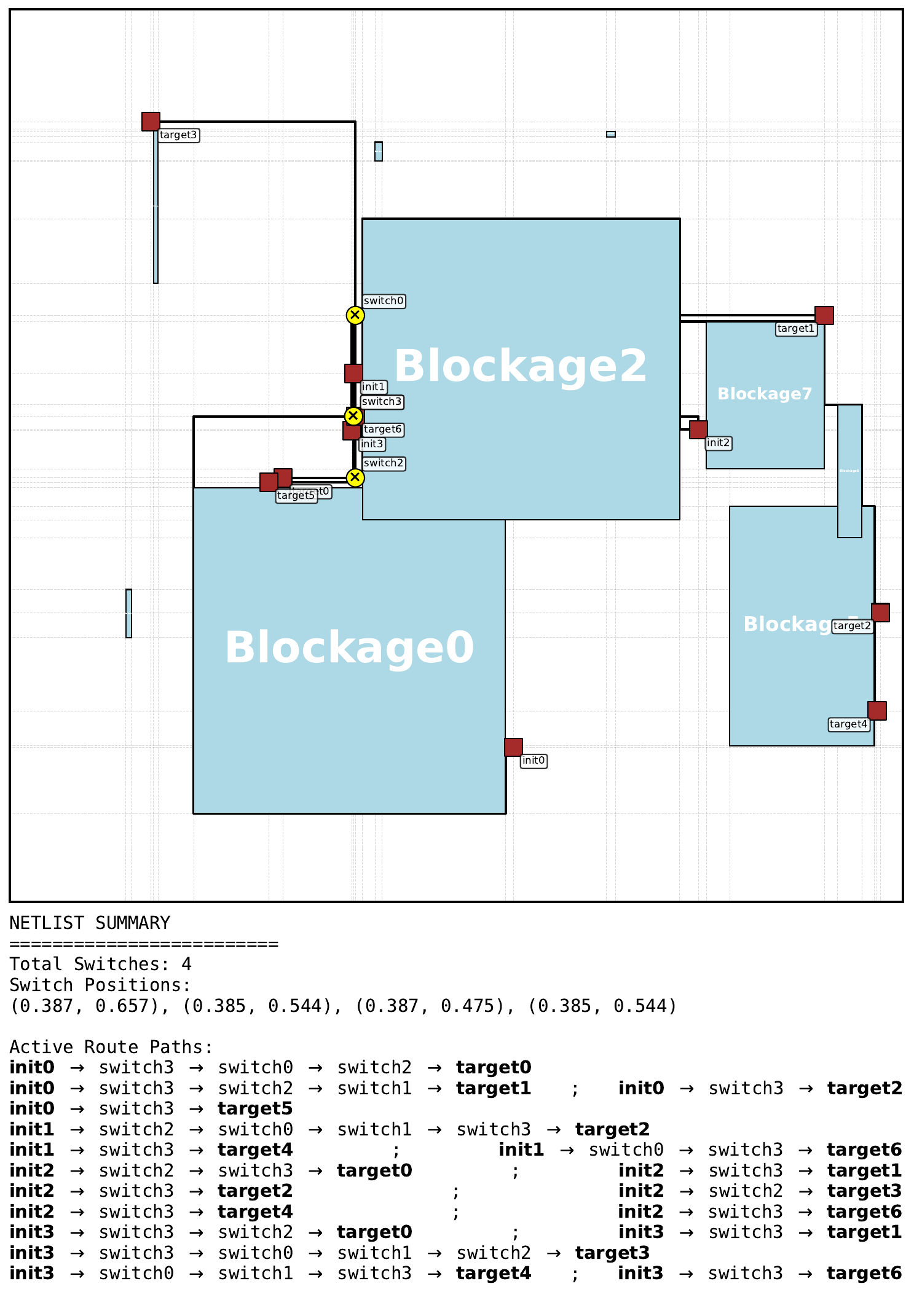}
        \caption*{Random search}
    \end{subfigure}
    \hfill
    \begin{subfigure}[t]{0.31\linewidth}
        \vspace{0pt}
        \centering
        \includegraphics[width=\linewidth]{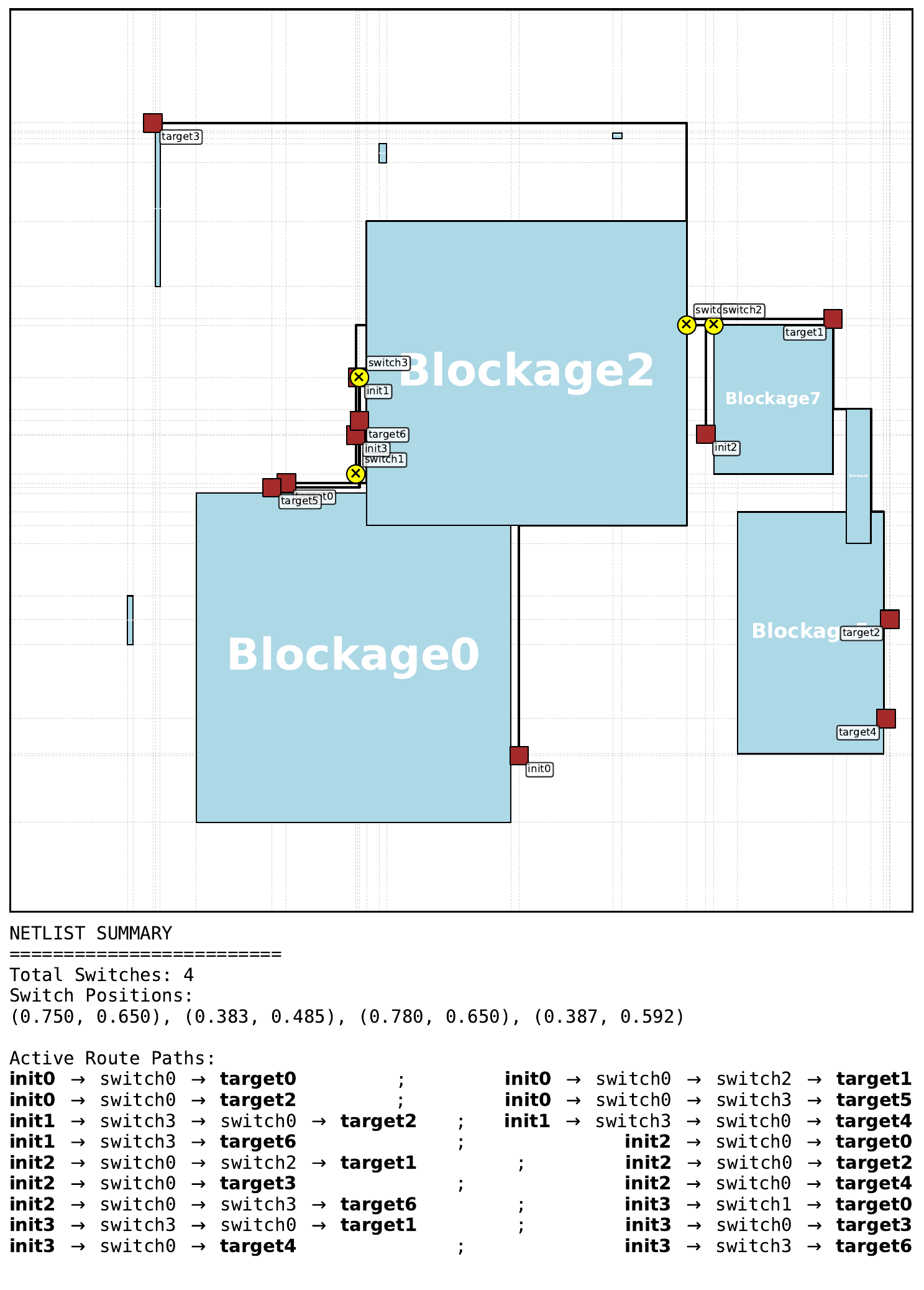}
        \caption*{Genetic algorithm}
    \end{subfigure}
    \\[0.6em]
    \begin{subfigure}[t]{0.31\linewidth}
        \vspace{0pt}
        \centering
        \includegraphics[width=\linewidth]{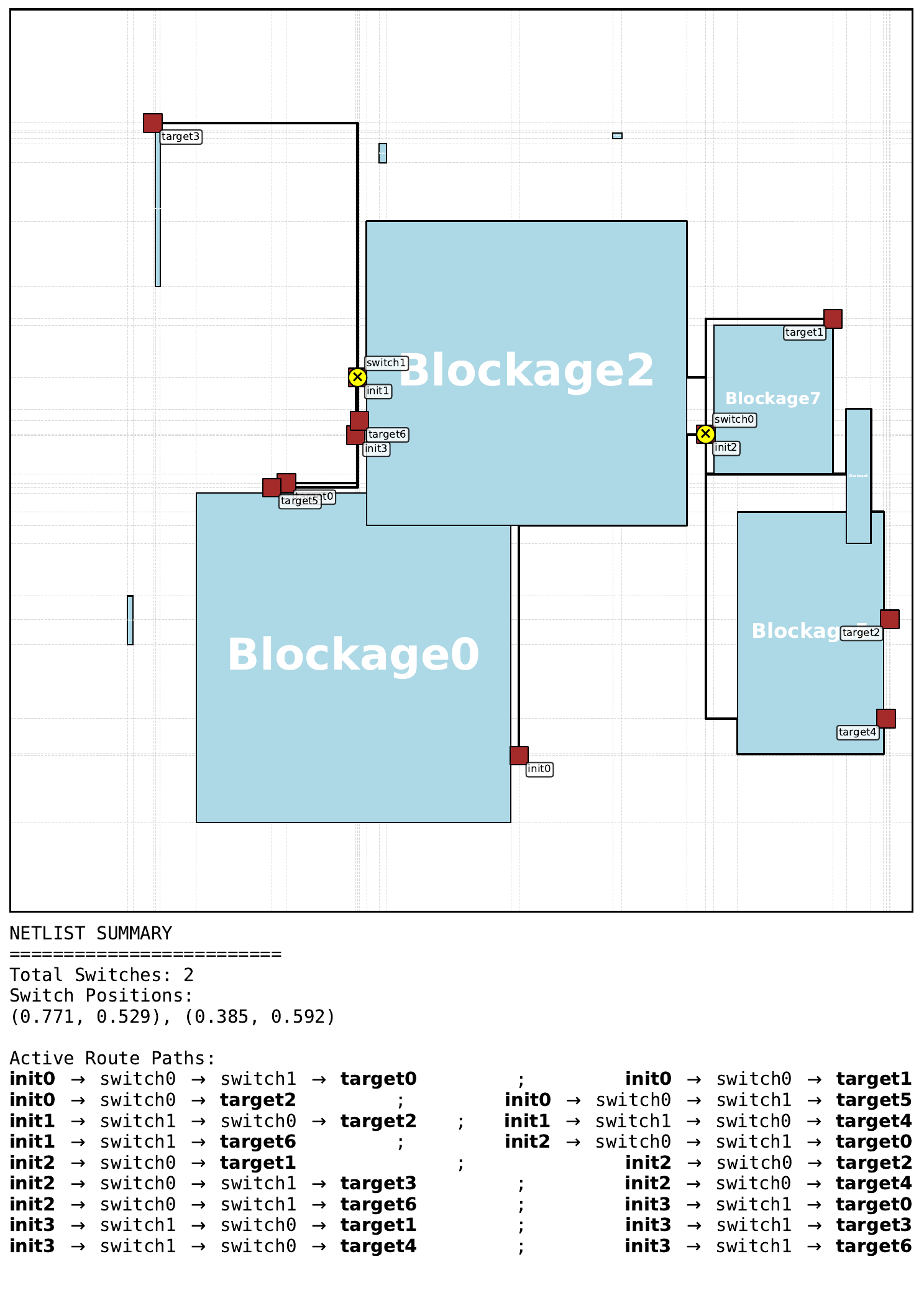}
        \caption*{PPO}
    \end{subfigure}
    \hspace{0.04\linewidth}
    \begin{subfigure}[t]{0.31\linewidth}
        \vspace{0pt}
        \centering
        \includegraphics[width=\linewidth]{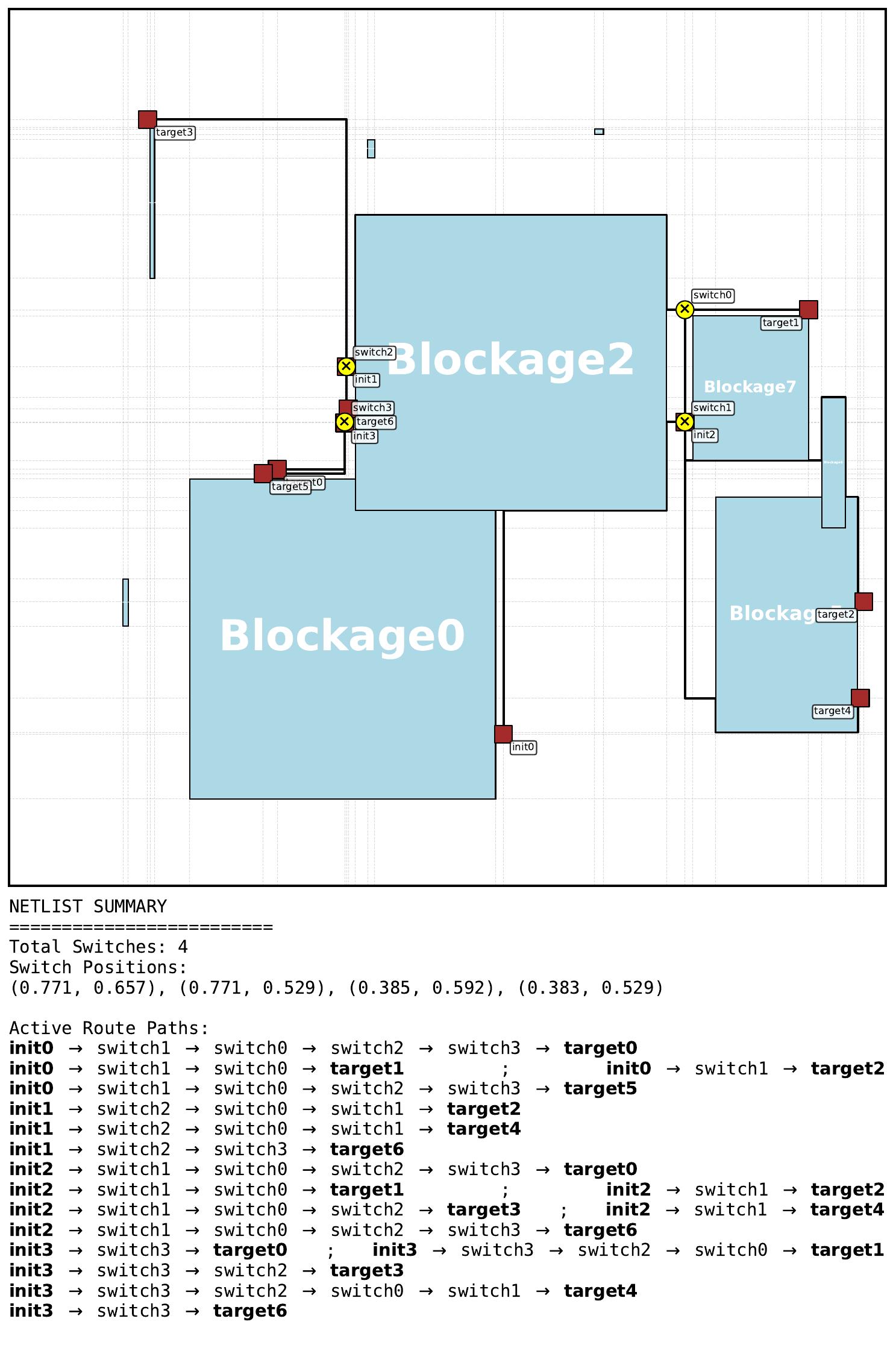}
        \caption*{MCTS}
    \end{subfigure}
\caption{Instance 2.}
\label{fig:best_pretrain_instance_2}
\end{figure*}

\begin{figure*}[h]
\centering
    \begin{subfigure}[t]{0.31\linewidth}
        \vspace{0pt}
        \centering
        \includegraphics[width=\linewidth]{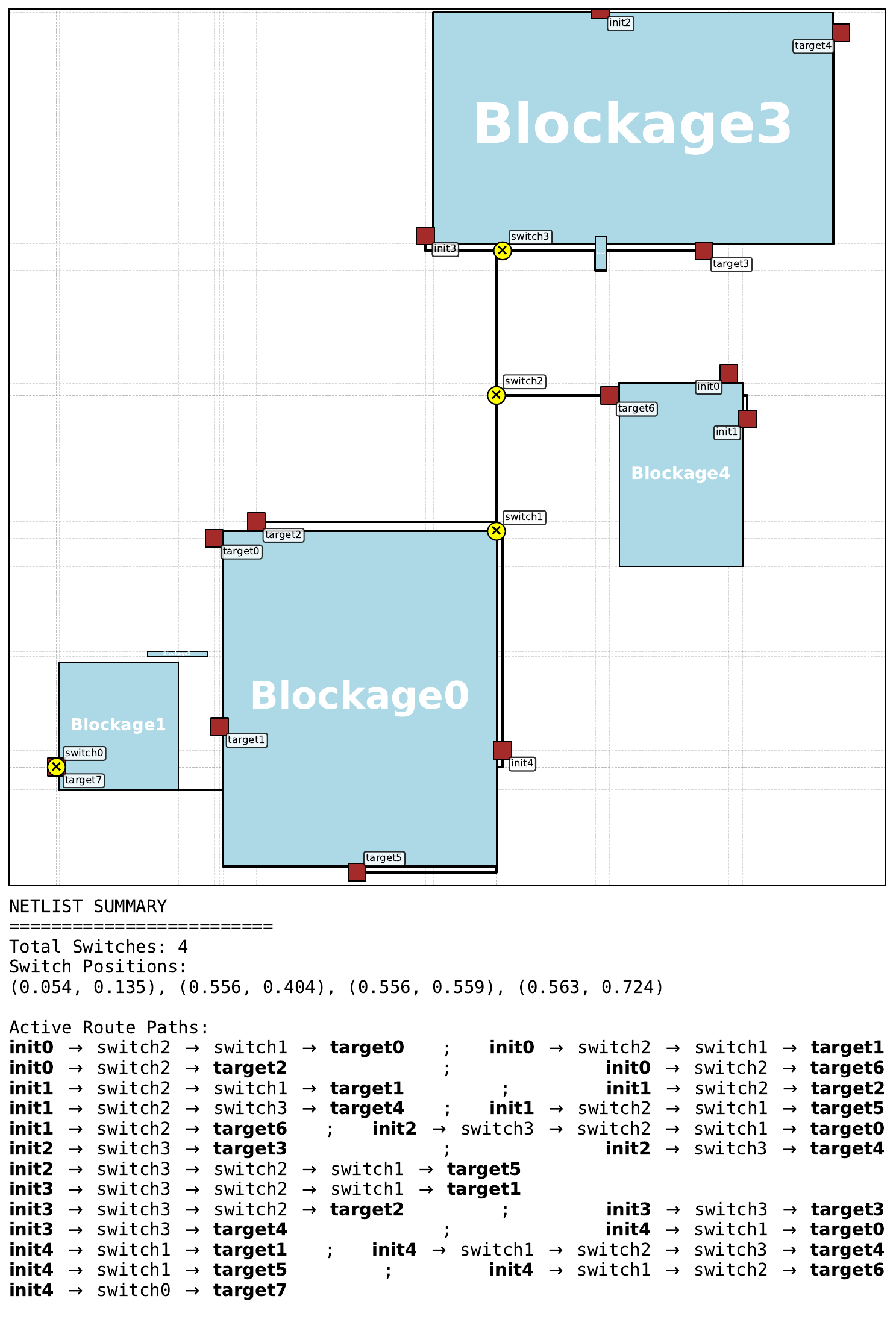}
        \caption*{Heuristic}
    \end{subfigure}
    \hfill
    \begin{subfigure}[t]{0.31\linewidth}
        \vspace{0pt}
        \centering
        \includegraphics[width=\linewidth]{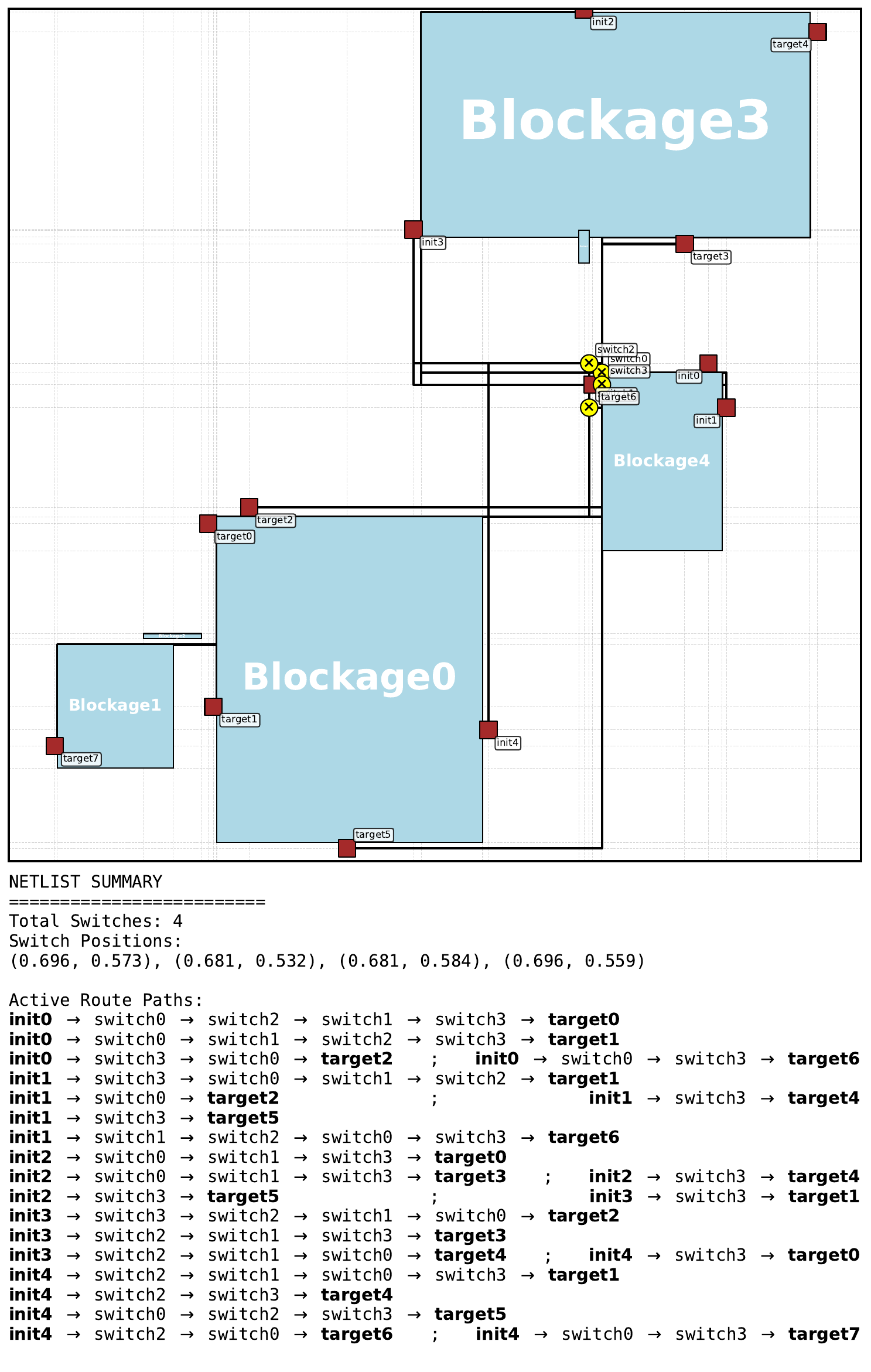}
        \caption*{Random search}
    \end{subfigure}
    \hfill
    \begin{subfigure}[t]{0.31\linewidth}
        \vspace{0pt}
        \centering
        \includegraphics[width=\linewidth]{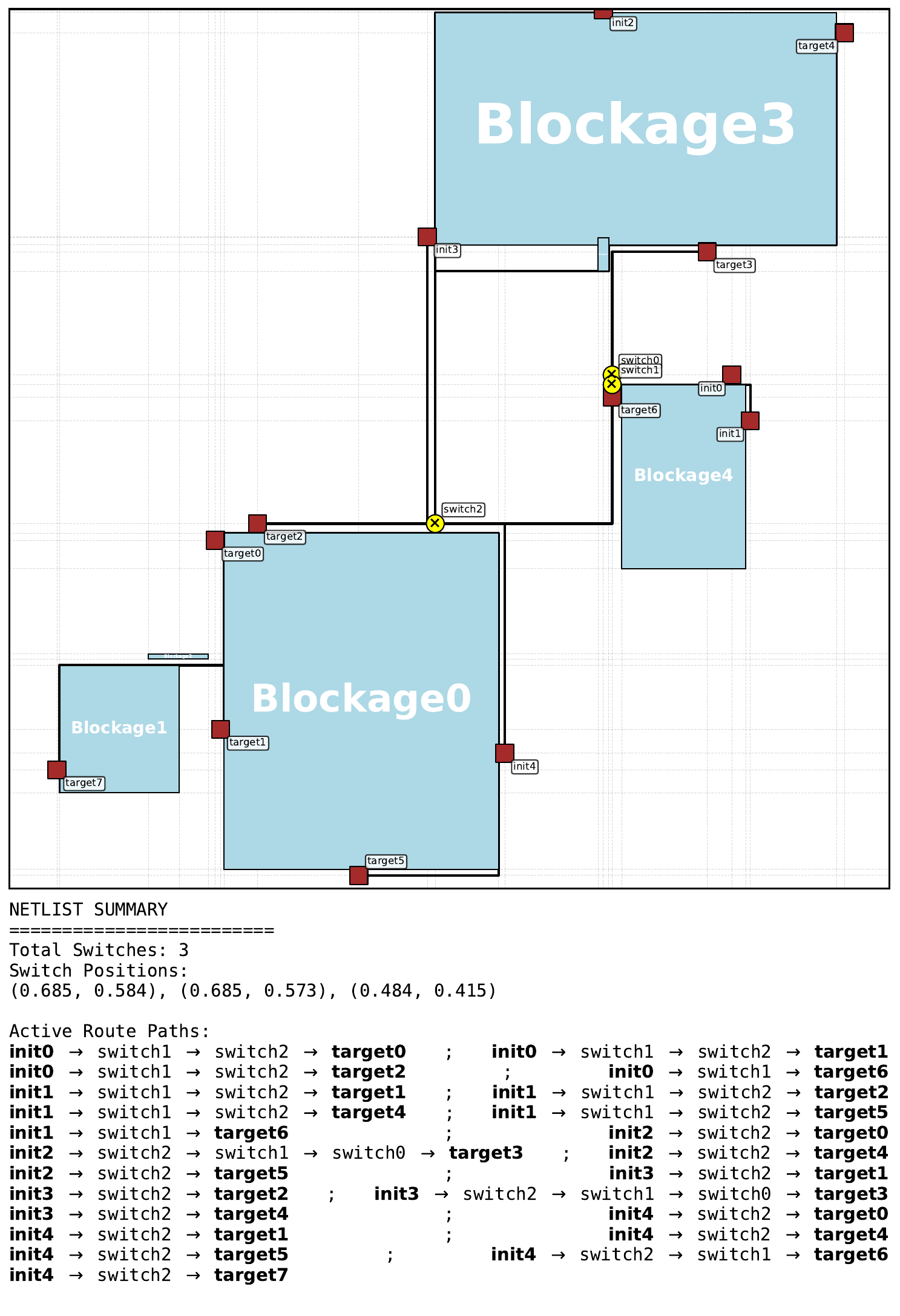}
        \caption*{Genetic algorithm}
    \end{subfigure}
    \\[0.6em]
    \begin{subfigure}[t]{0.31\linewidth}
        \vspace{0pt}
        \centering
        \includegraphics[width=\linewidth]{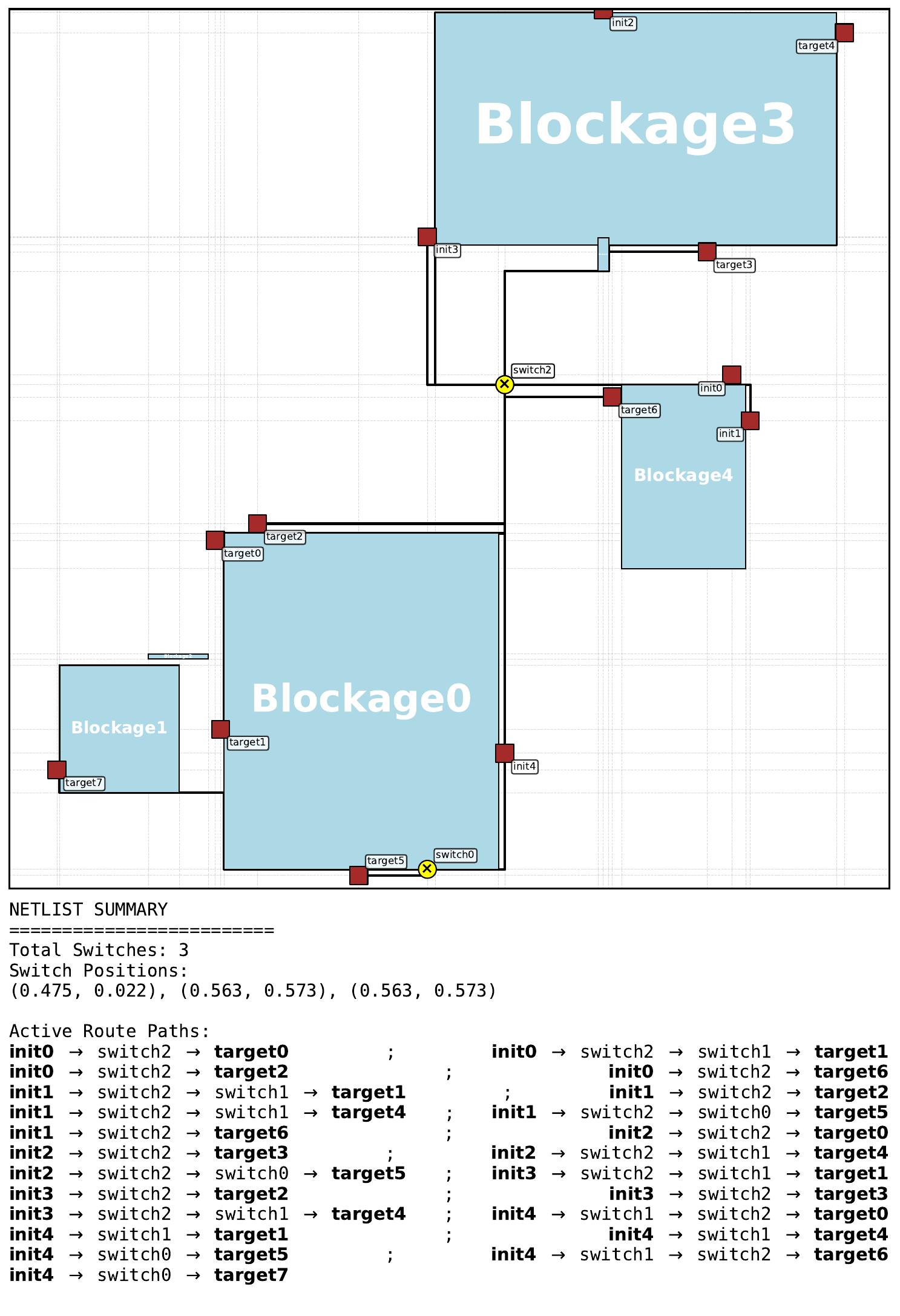}
        \caption*{PPO}
    \end{subfigure}
    \hspace{0.04\linewidth}
    \begin{subfigure}[t]{0.31\linewidth}
        \vspace{0pt}
        \centering
        \includegraphics[width=\linewidth]{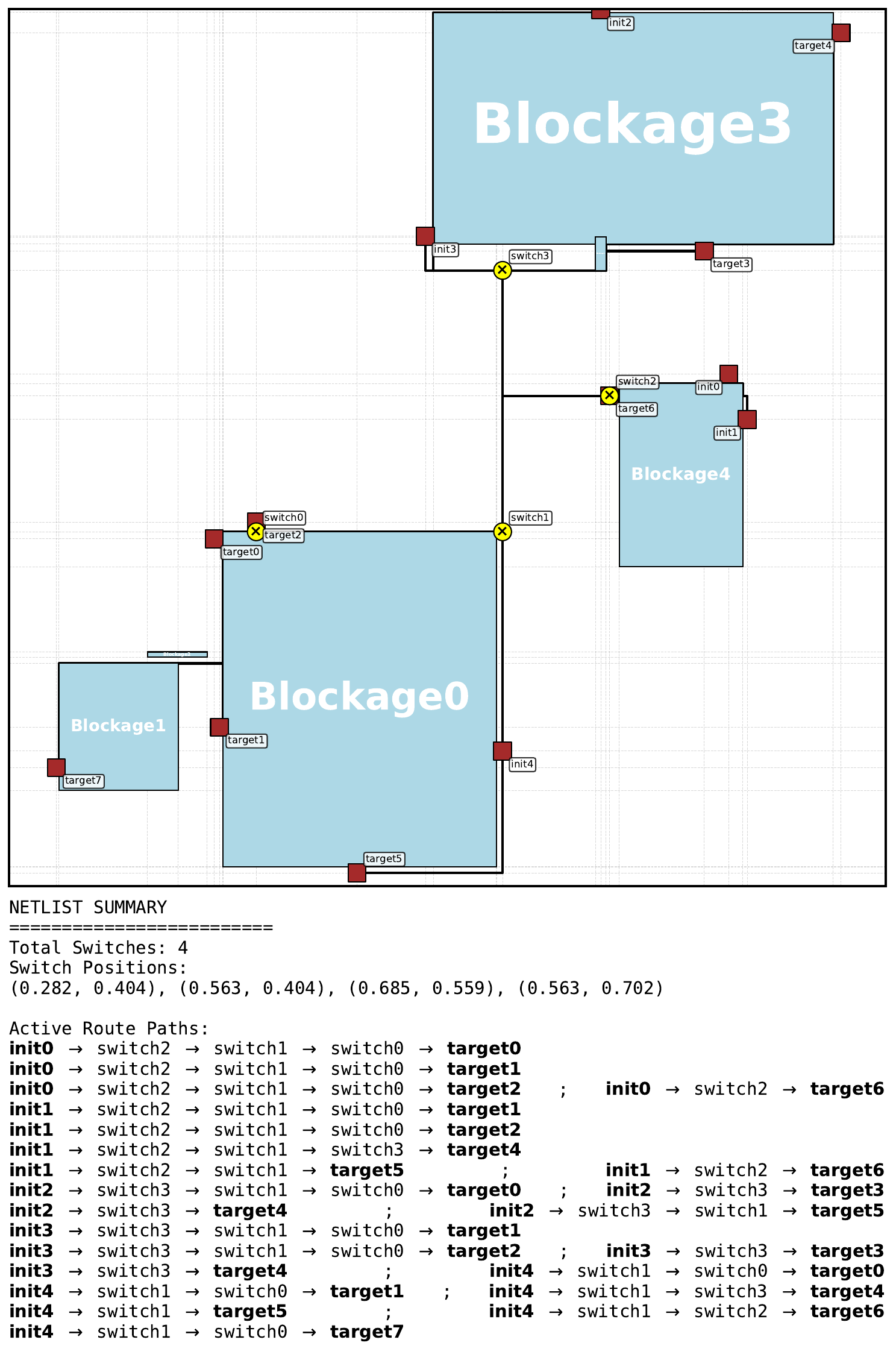}
        \caption*{MCTS}
    \end{subfigure}
\caption{Instance 3.}
\label{fig:best_pretrain_instance_3}
\end{figure*}

\begin{figure*}[h]
\centering
    \begin{subfigure}[t]{0.31\linewidth}
        \vspace{0pt}
        \centering
        \includegraphics[width=\linewidth]{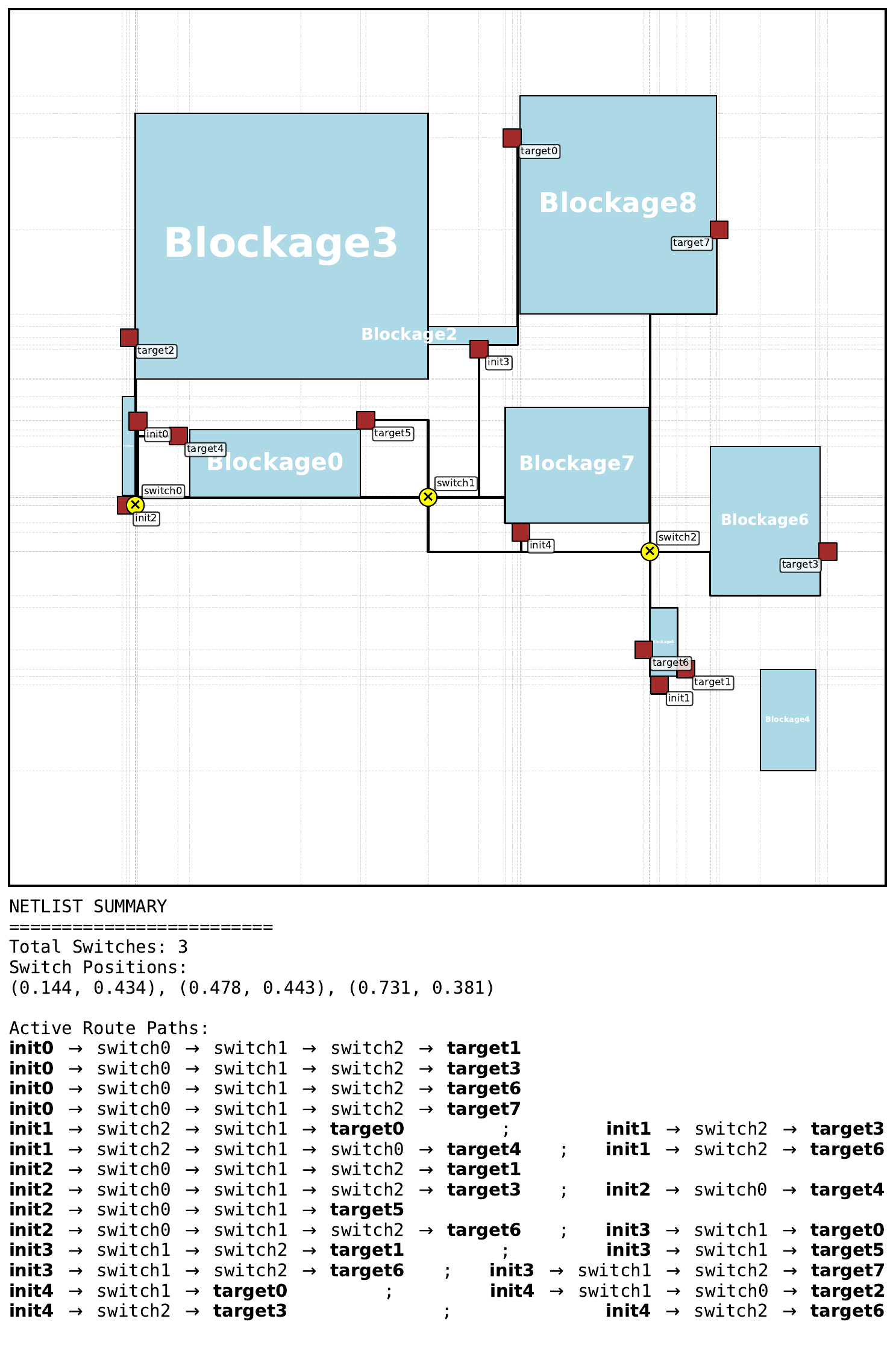}
        \caption*{Heuristic}
    \end{subfigure}
    \hfill
    \begin{subfigure}[t]{0.31\linewidth}
        \vspace{0pt}
        \centering
        \includegraphics[width=\linewidth]{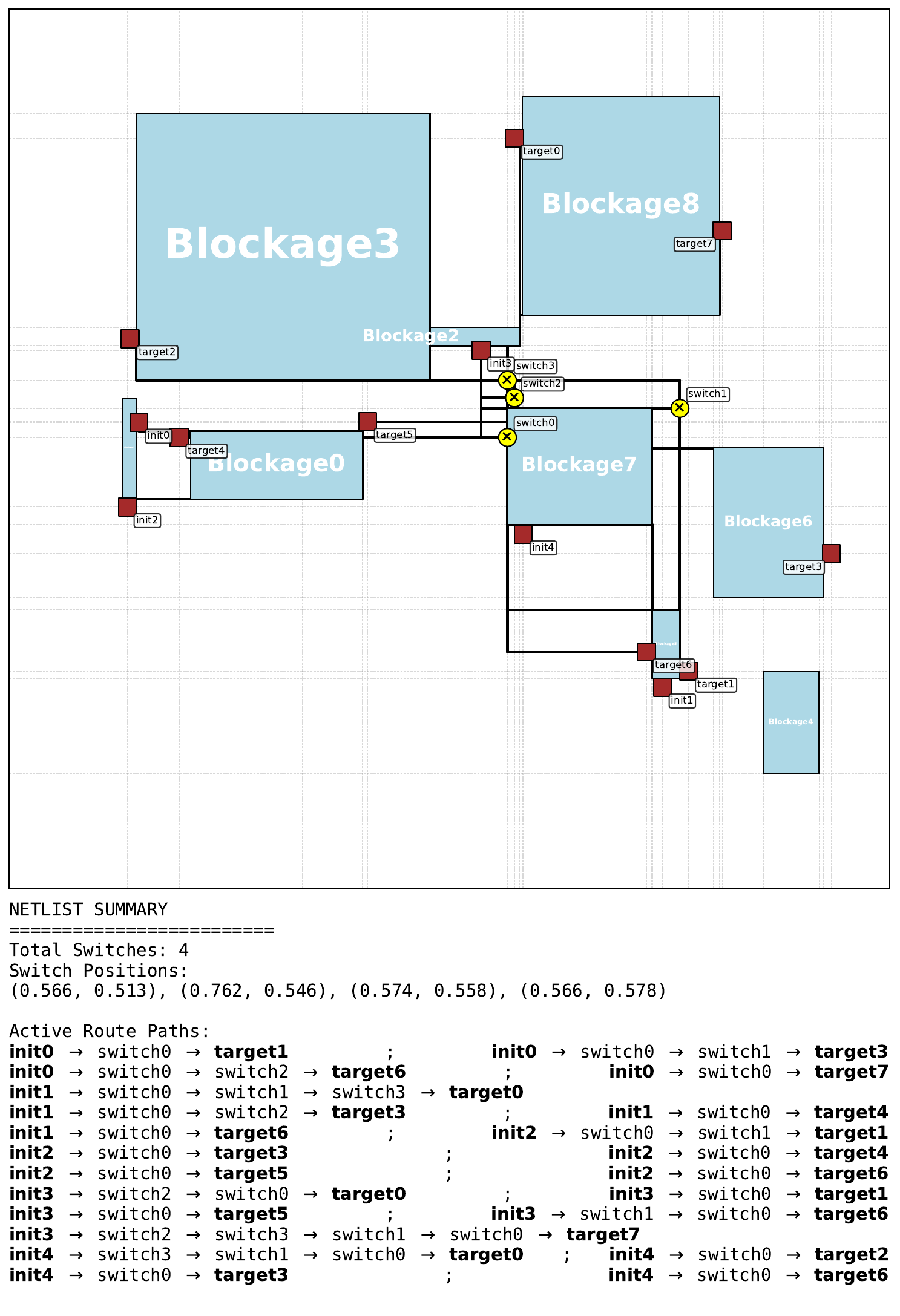}
        \caption*{Random search}
    \end{subfigure}
    \hfill
    \begin{subfigure}[t]{0.31\linewidth}
        \vspace{0pt}
        \centering
        \includegraphics[width=\linewidth]{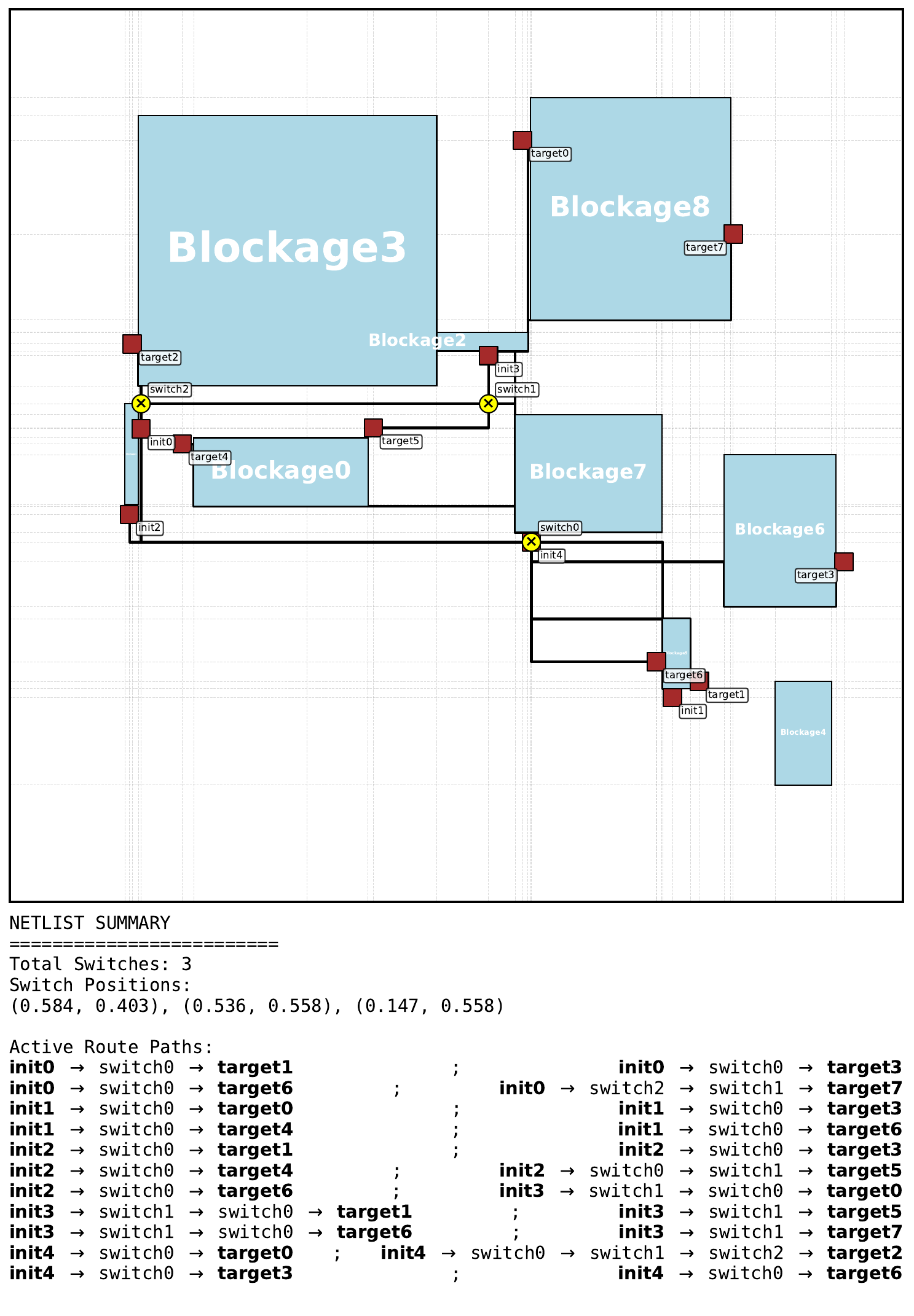}
        \caption*{Genetic algorithm}
    \end{subfigure}
    \\[0.6em]
    \begin{subfigure}[t]{0.31\linewidth}
        \vspace{0pt}
        \centering
        \includegraphics[width=\linewidth]{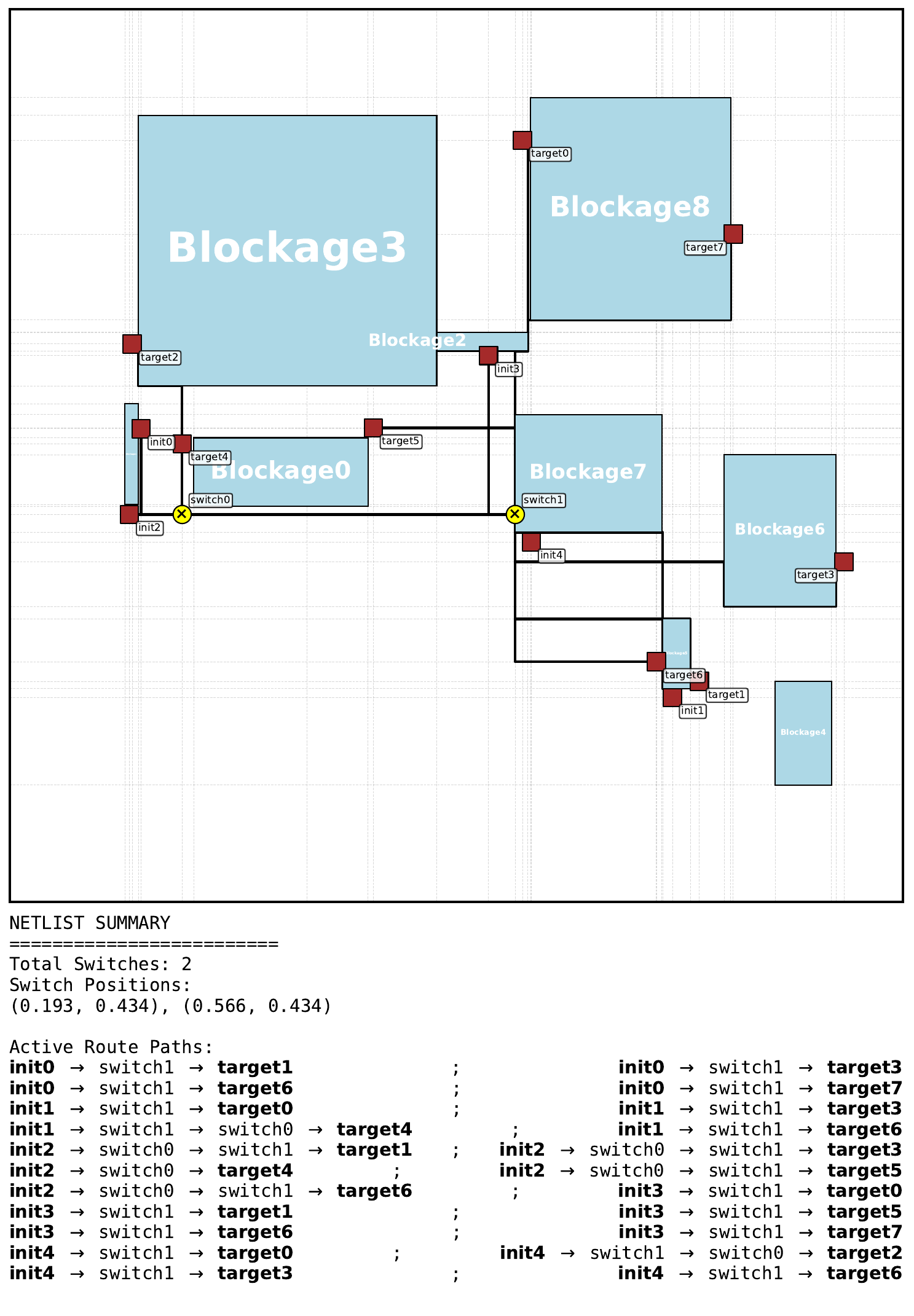}
        \caption*{PPO}
    \end{subfigure}
    \hspace{0.04\linewidth}
    \begin{subfigure}[t]{0.31\linewidth}
        \vspace{0pt}
        \centering
        \includegraphics[width=\linewidth]{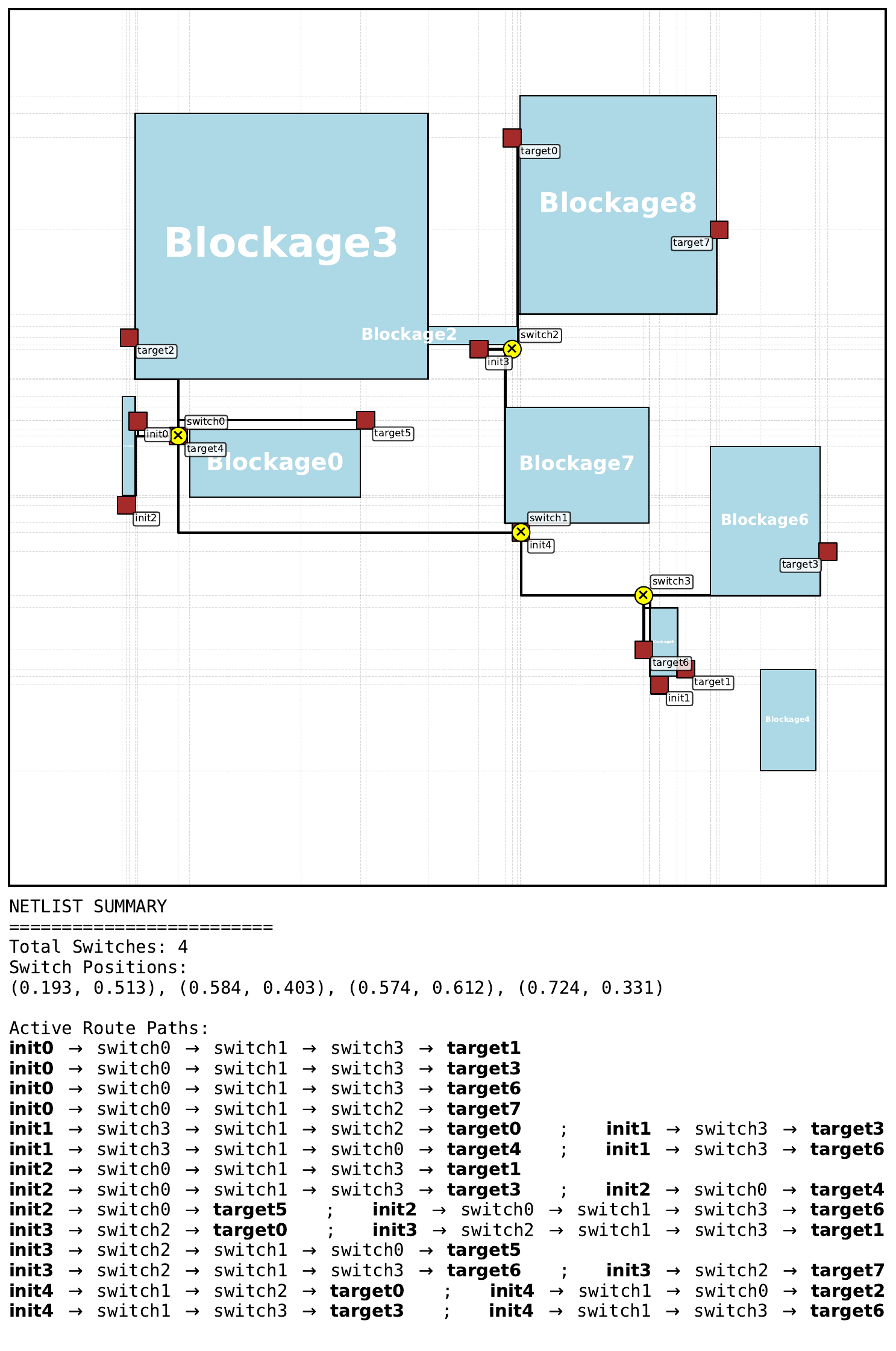}
        \caption*{MCTS}
    \end{subfigure}
\caption{Instance 4.}
\label{fig:best_pretrain_instance_4}
\end{figure*}

\begin{figure*}[h]
\centering
    \begin{subfigure}[t]{0.31\linewidth}
        \vspace{0pt}
        \centering
        \includegraphics[width=\linewidth]{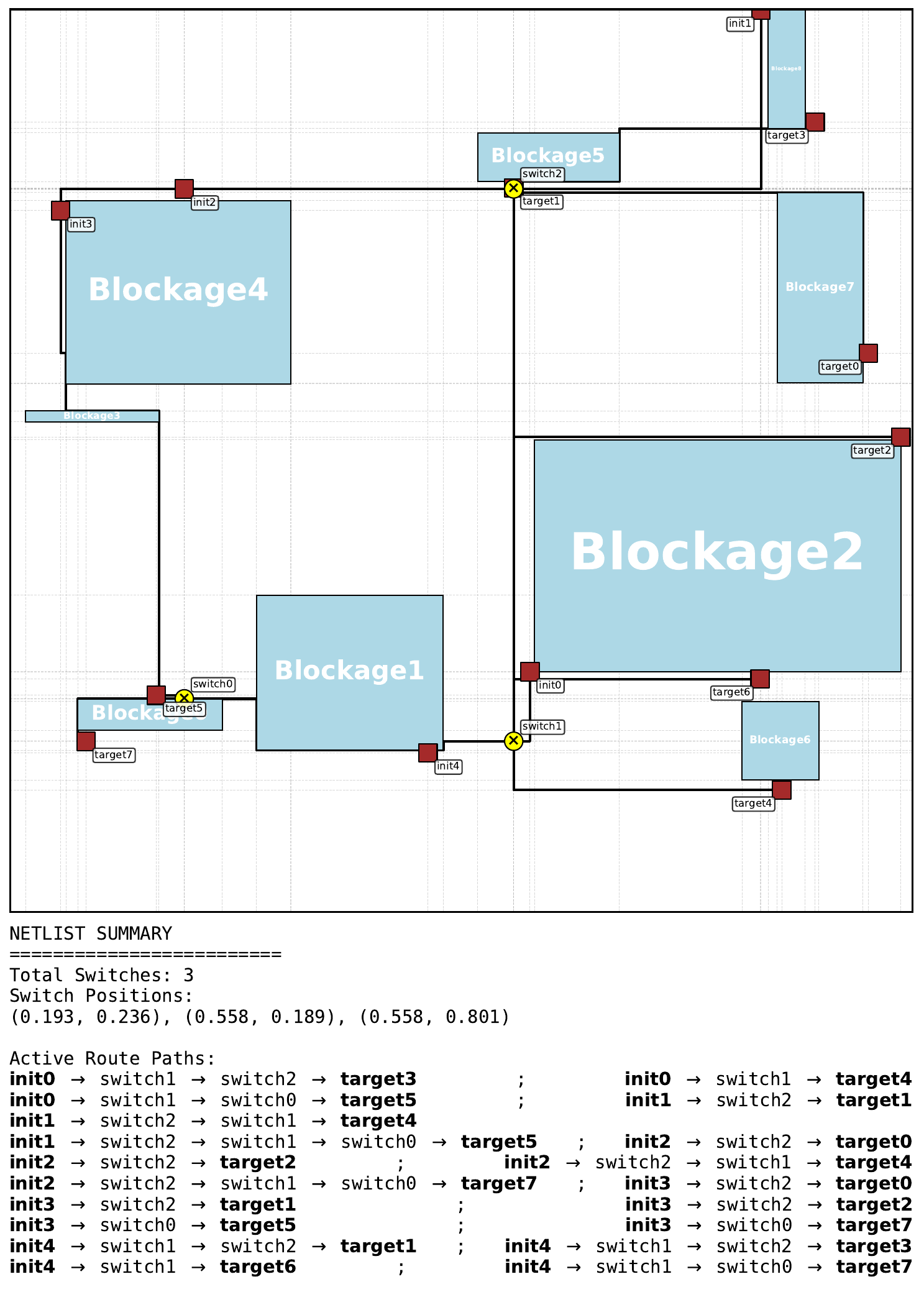}
        \caption*{Heuristic}
    \end{subfigure}
    \hfill
    \begin{subfigure}[t]{0.31\linewidth}
        \vspace{0pt}
        \centering
        \includegraphics[width=\linewidth]{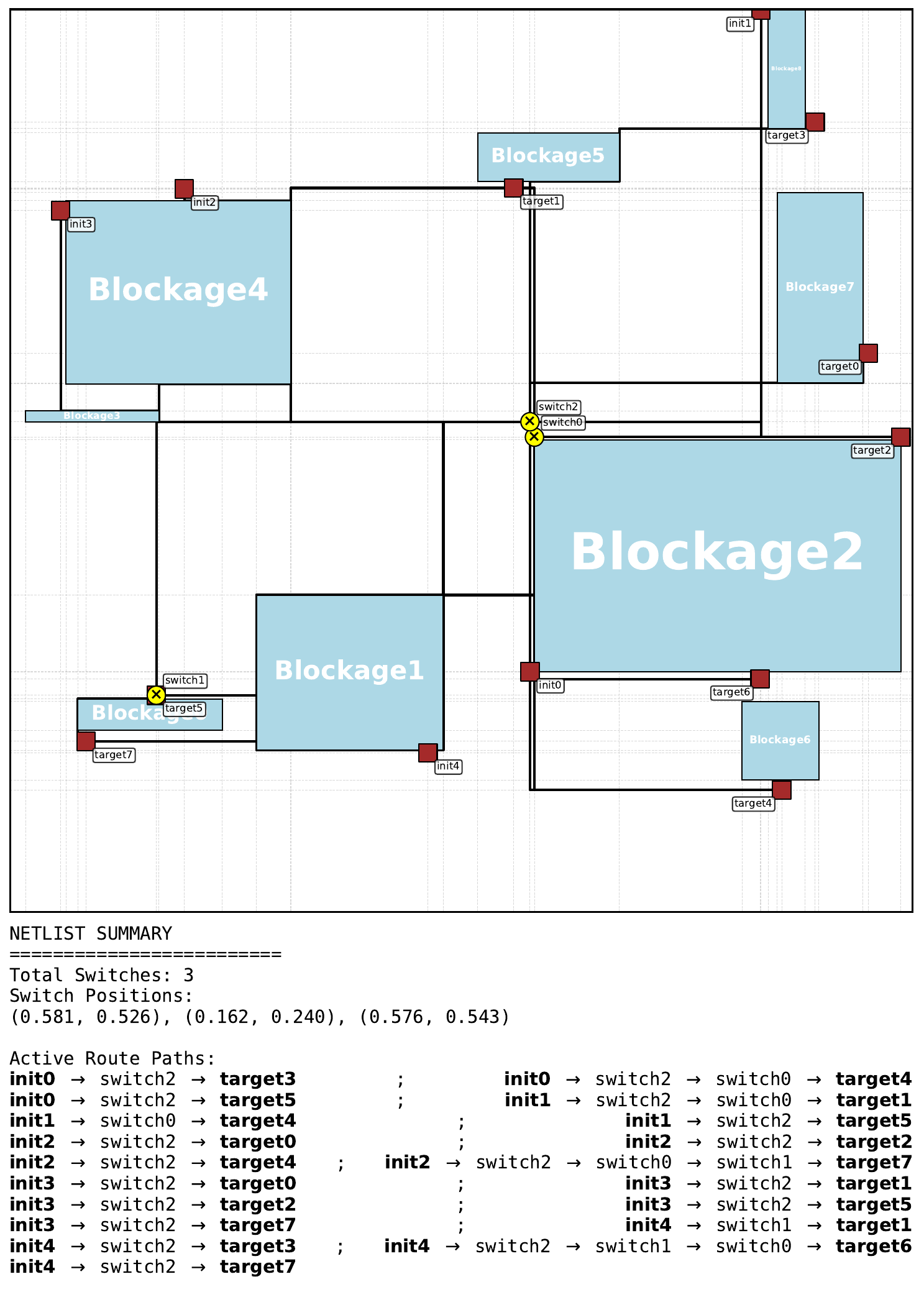}
        \caption*{Random search}
    \end{subfigure}
    \hfill
    \begin{subfigure}[t]{0.31\linewidth}
        \vspace{0pt}
        \centering
        \includegraphics[width=\linewidth]{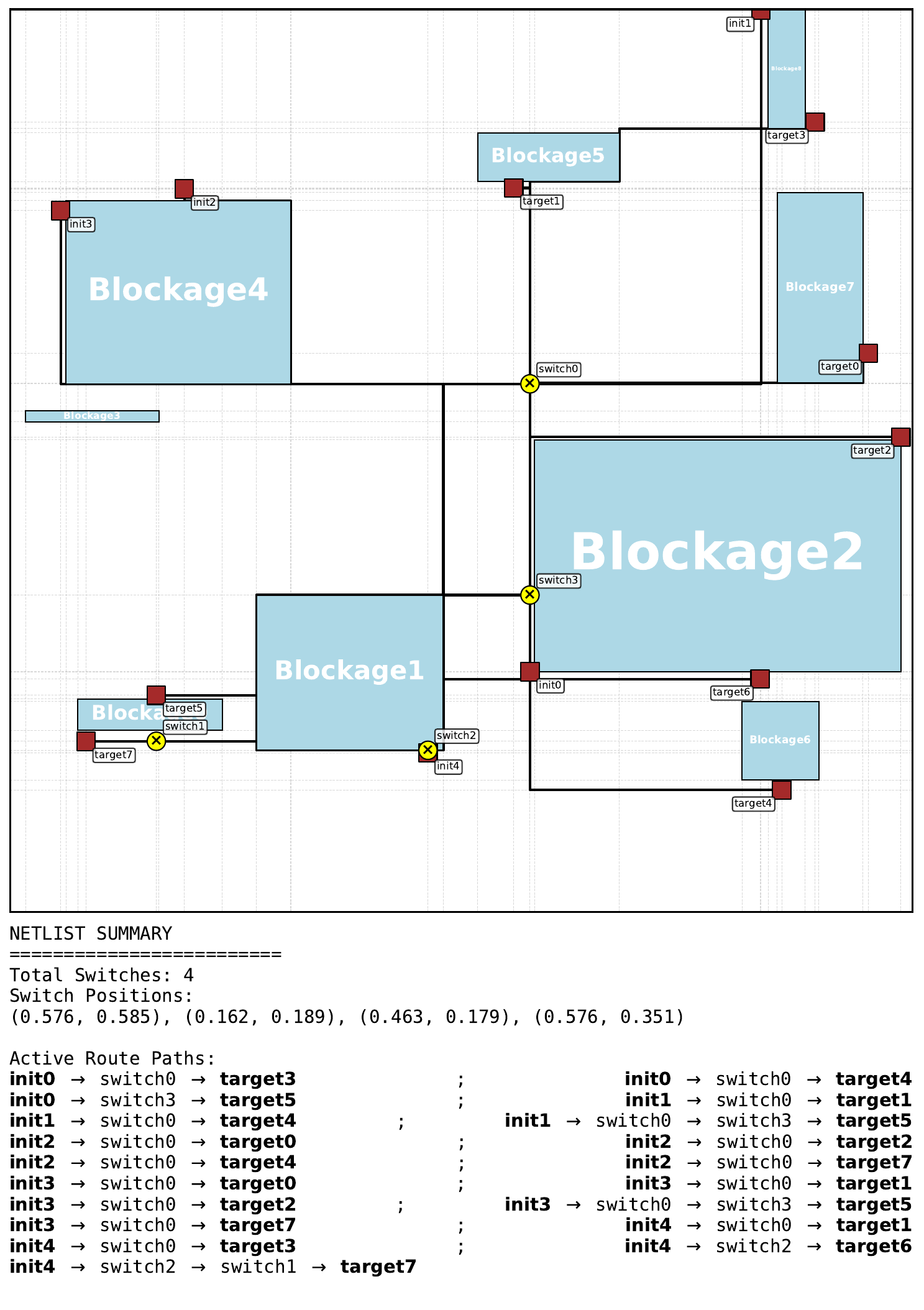}
        \caption*{Genetic algorithm}
    \end{subfigure}
    \\[0.6em]
    \begin{subfigure}[t]{0.31\linewidth}
        \vspace{0pt}
        \centering
        \includegraphics[width=\linewidth]{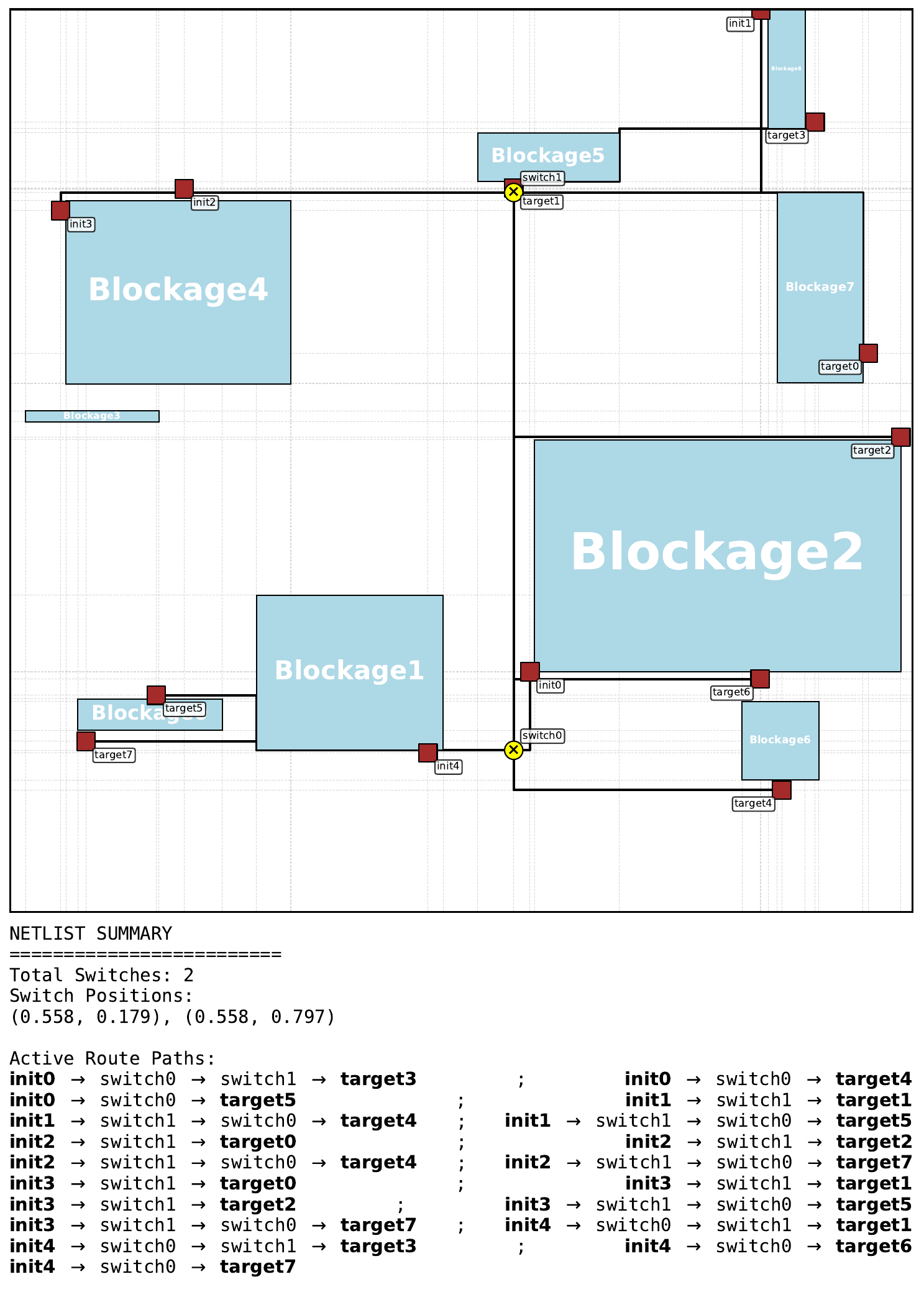}
        \caption*{PPO}
    \end{subfigure}
    \hspace{0.04\linewidth}
    \begin{subfigure}[t]{0.31\linewidth}
        \vspace{0pt}
        \centering
        \includegraphics[width=\linewidth]{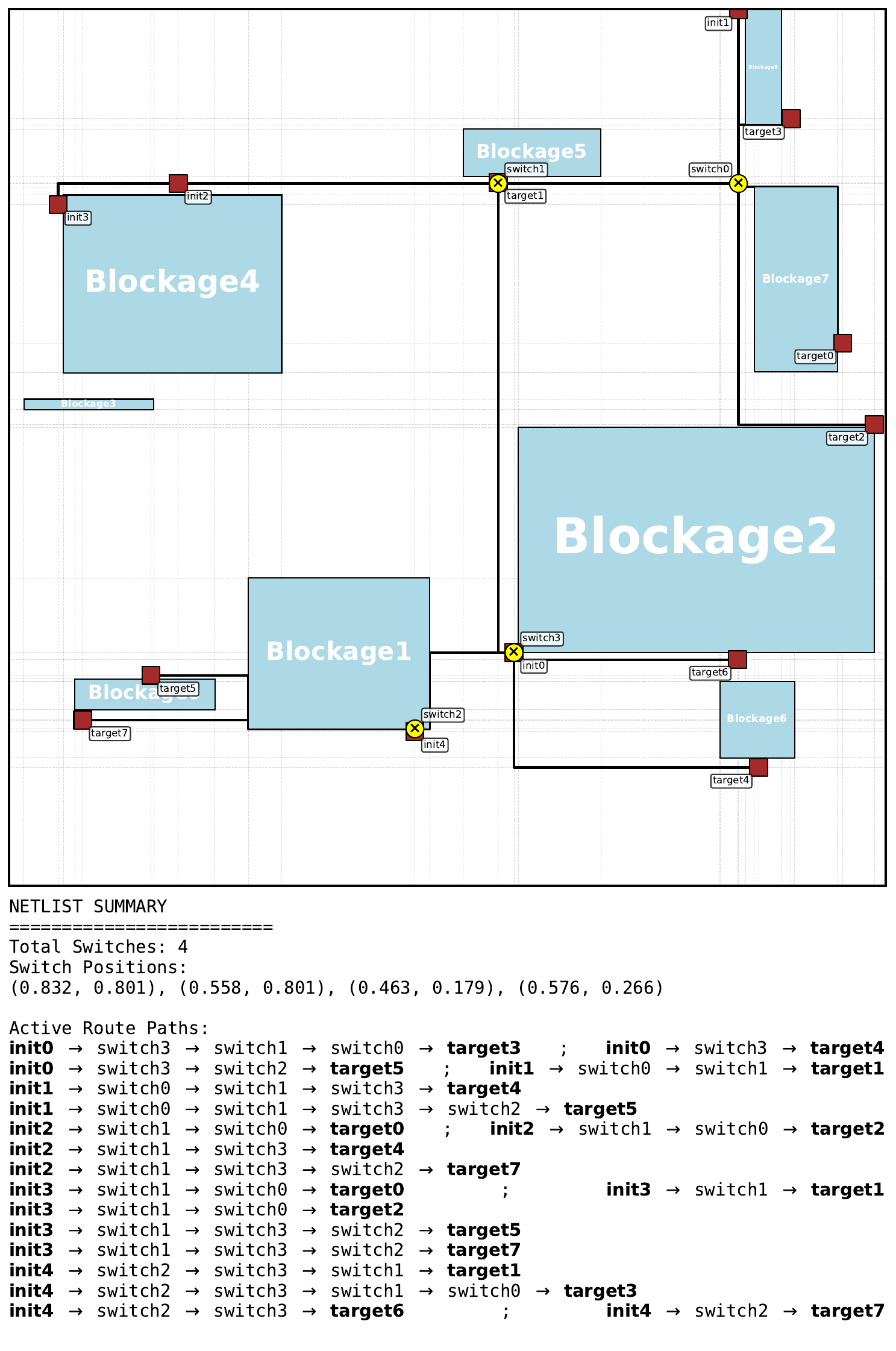}
        \caption*{MCTS}
    \end{subfigure}
\caption{Instance 5.}
\label{fig:best_pretrain_instance_5}
\end{figure*}

\begin{figure*}[h]
\centering
    \begin{subfigure}[t]{0.31\linewidth}
        \vspace{0pt}
        \centering
        \includegraphics[width=\linewidth]{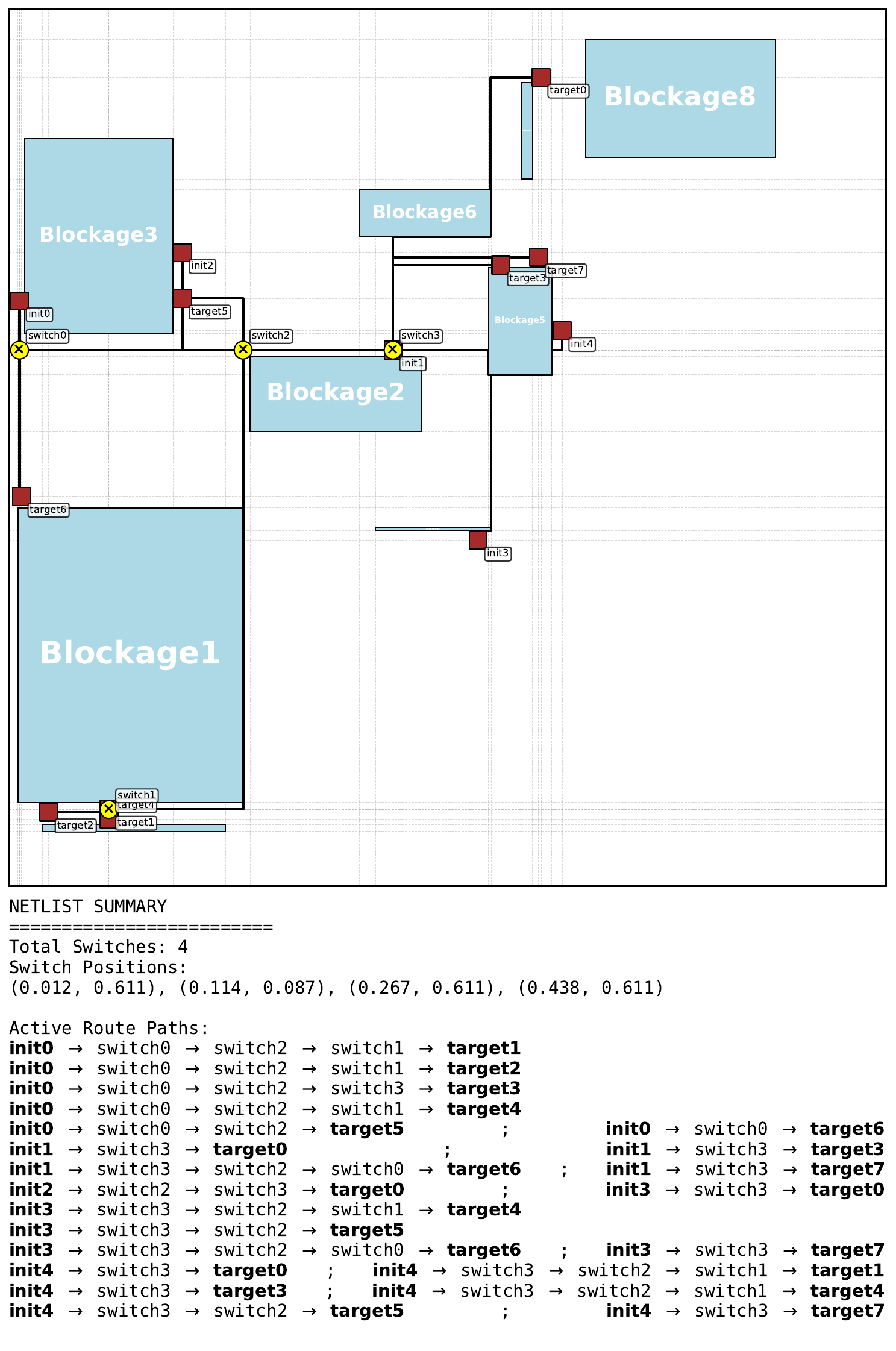}
        \caption*{Heuristic}
    \end{subfigure}
    \hfill
    \begin{subfigure}[t]{0.31\linewidth}
        \vspace{0pt}
        \centering
        \includegraphics[width=\linewidth]{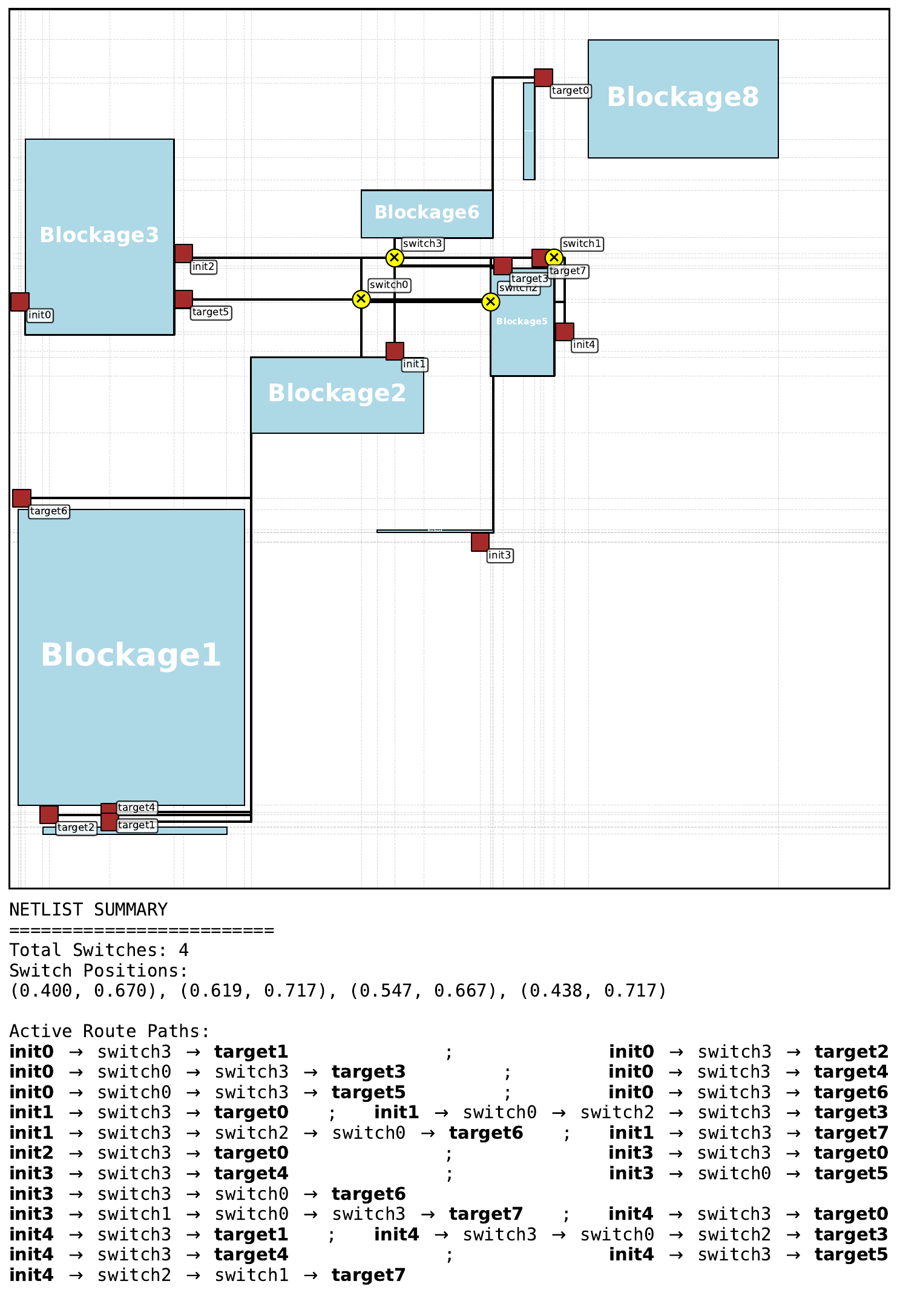}
        \caption*{Random search}
    \end{subfigure}
    \hfill
    \begin{subfigure}[t]{0.31\linewidth}
        \vspace{0pt}
        \centering
        \includegraphics[width=\linewidth]{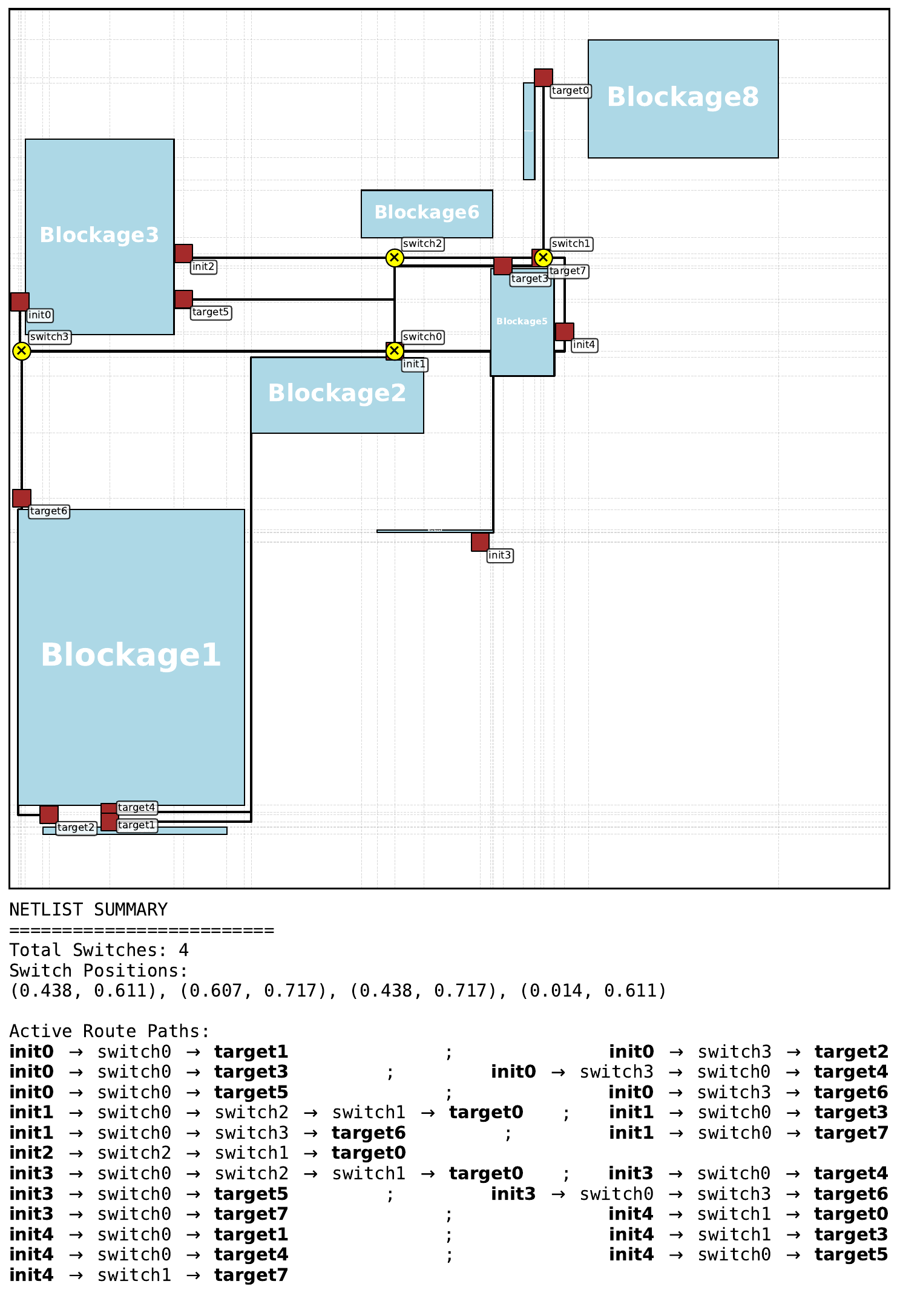}
        \caption*{Genetic algorithm}
    \end{subfigure}
    \\[0.6em]
    \begin{subfigure}[t]{0.31\linewidth}
        \vspace{0pt}
        \centering
        \includegraphics[width=\linewidth]{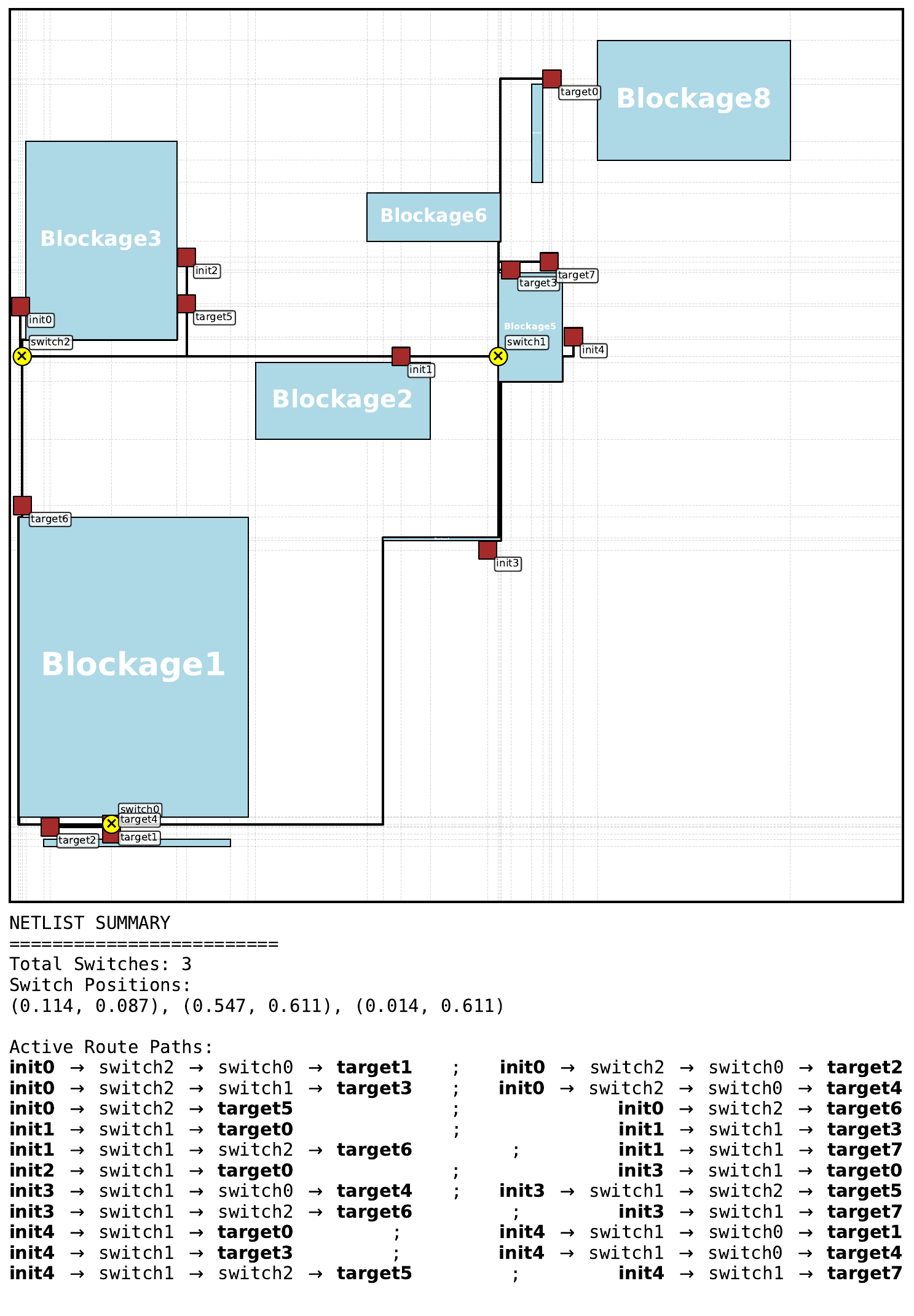}
        \caption*{PPO}
    \end{subfigure}
    \hspace{0.04\linewidth}
    \begin{subfigure}[t]{0.31\linewidth}
        \vspace{0pt}
        \centering
        \includegraphics[width=\linewidth]{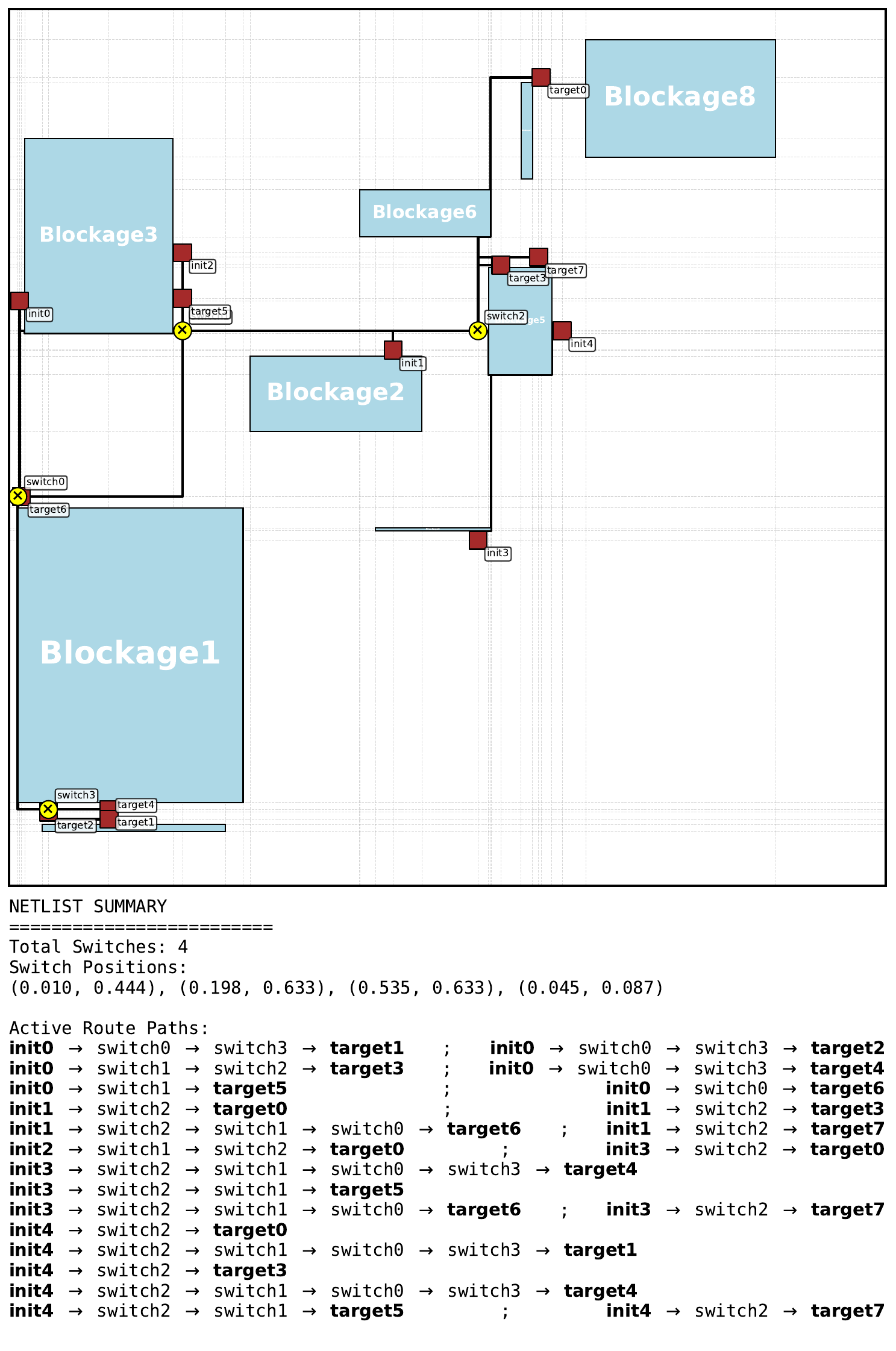}
        \caption*{MCTS}
    \end{subfigure}
\caption{Instance 6.}
\label{fig:best_pretrain_instance_6}
\end{figure*}

\begin{figure*}[h]
\centering
    \begin{subfigure}[t]{0.31\linewidth}
        \vspace{0pt}
        \centering
        \includegraphics[width=\linewidth]{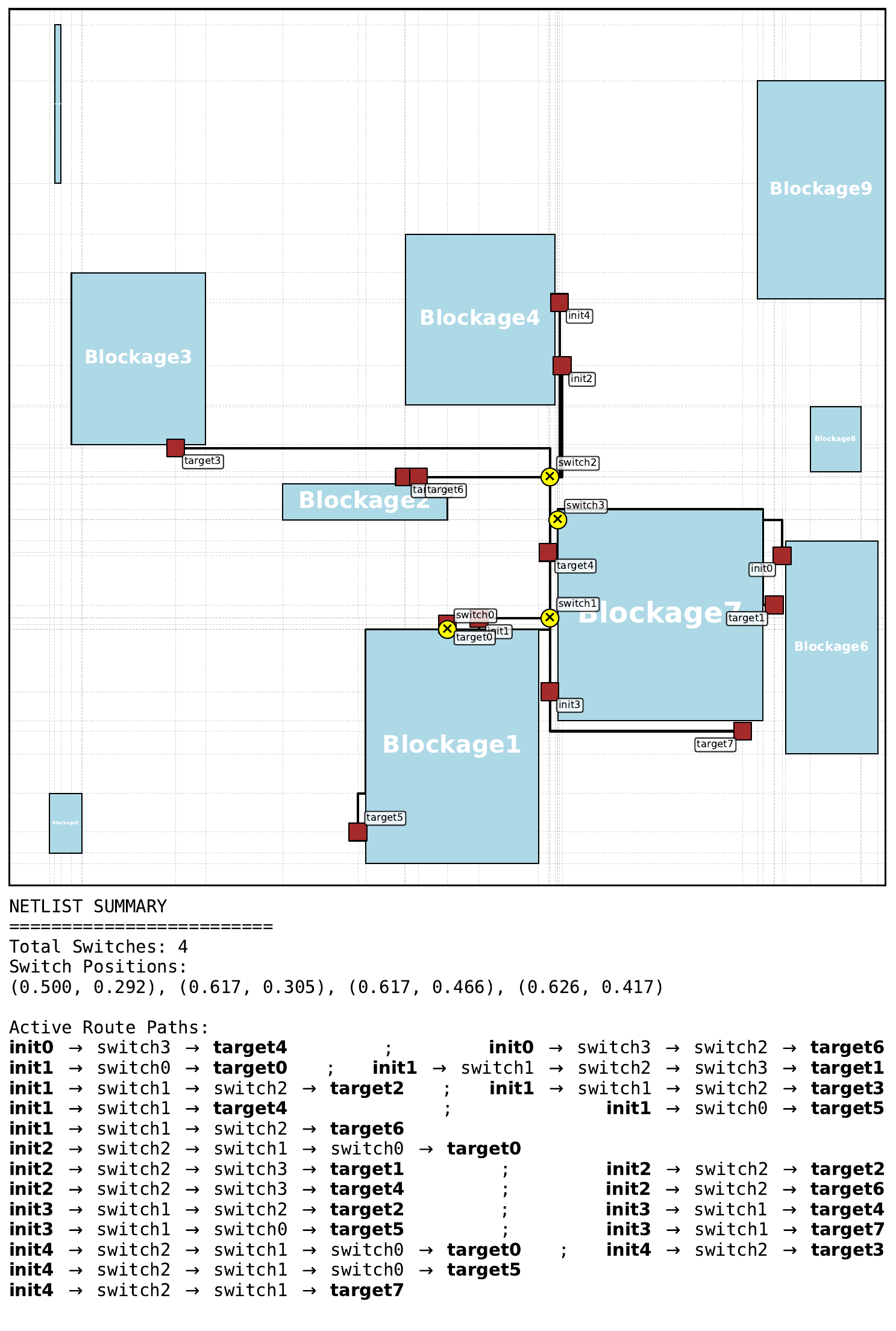}
        \caption*{Heuristic}
    \end{subfigure}
    \hfill
    \begin{subfigure}[t]{0.31\linewidth}
        \vspace{0pt}
        \centering
        \includegraphics[width=\linewidth]{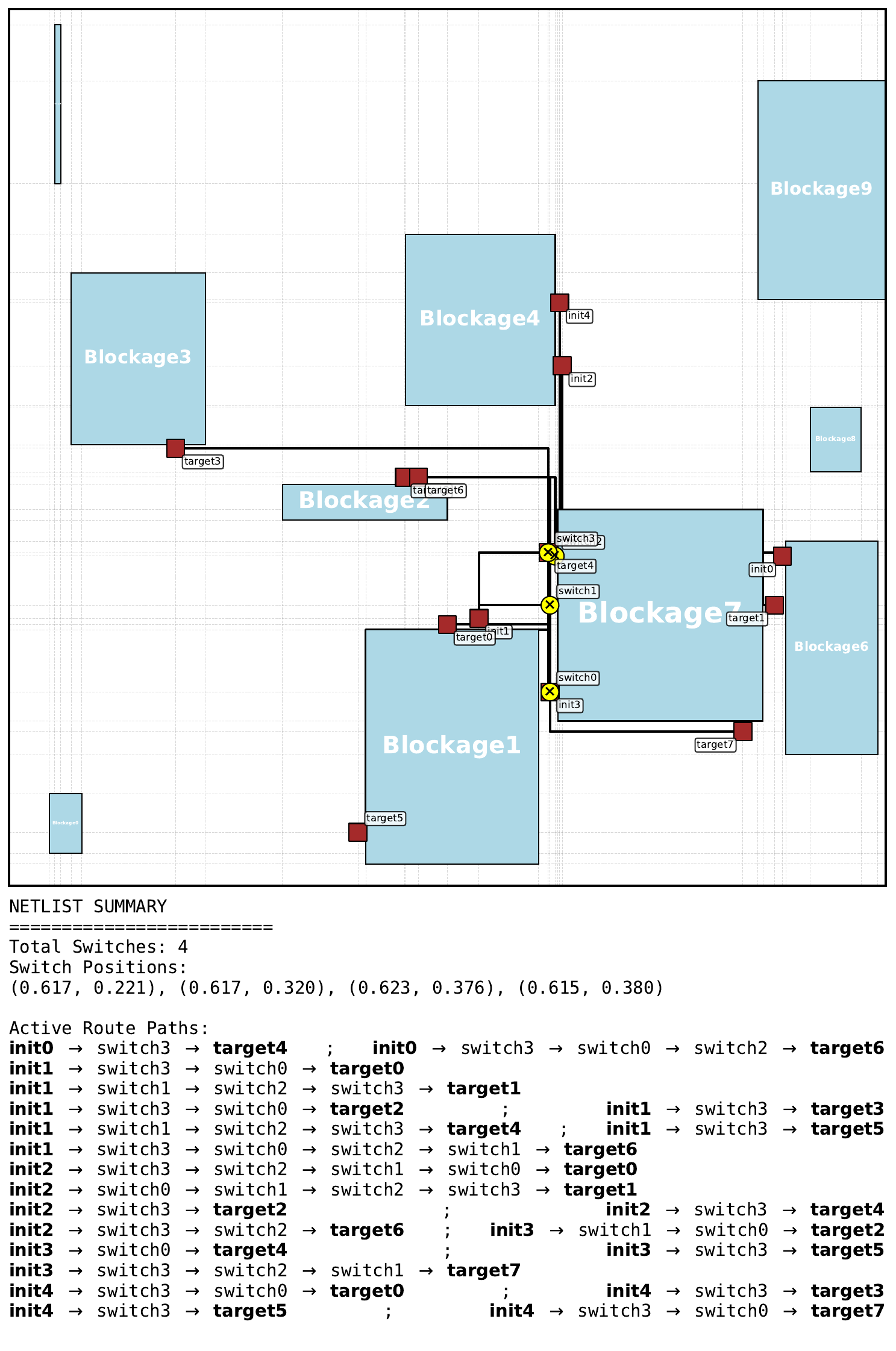}
        \caption*{Random search}
    \end{subfigure}
    \hfill
    \begin{subfigure}[t]{0.31\linewidth}
        \vspace{0pt}
        \centering
        \includegraphics[width=\linewidth]{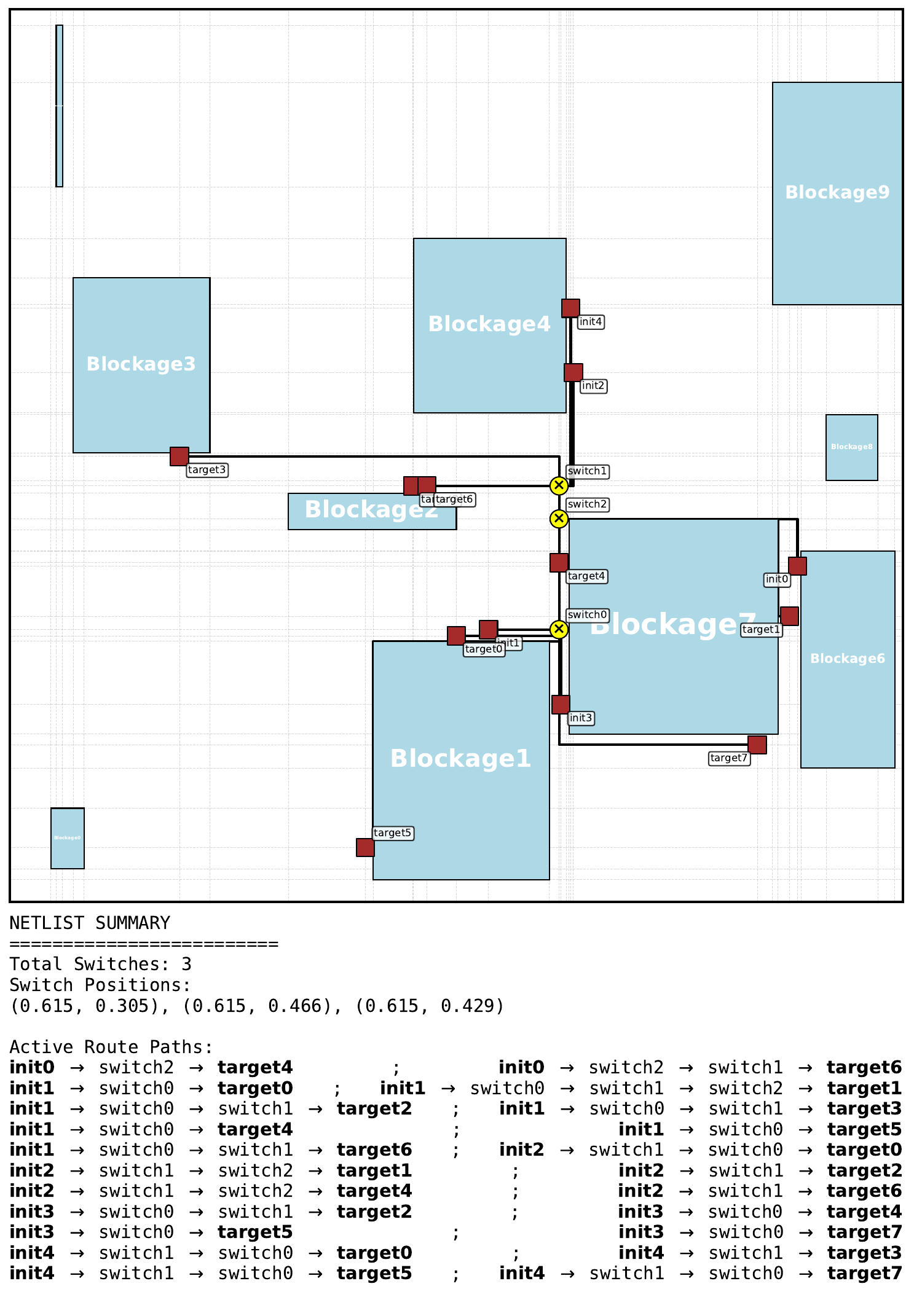}
        \caption*{Genetic algorithm}
    \end{subfigure}
    \\[0.6em]
    \begin{subfigure}[t]{0.31\linewidth}
        \vspace{0pt}
        \centering
        \includegraphics[width=\linewidth]{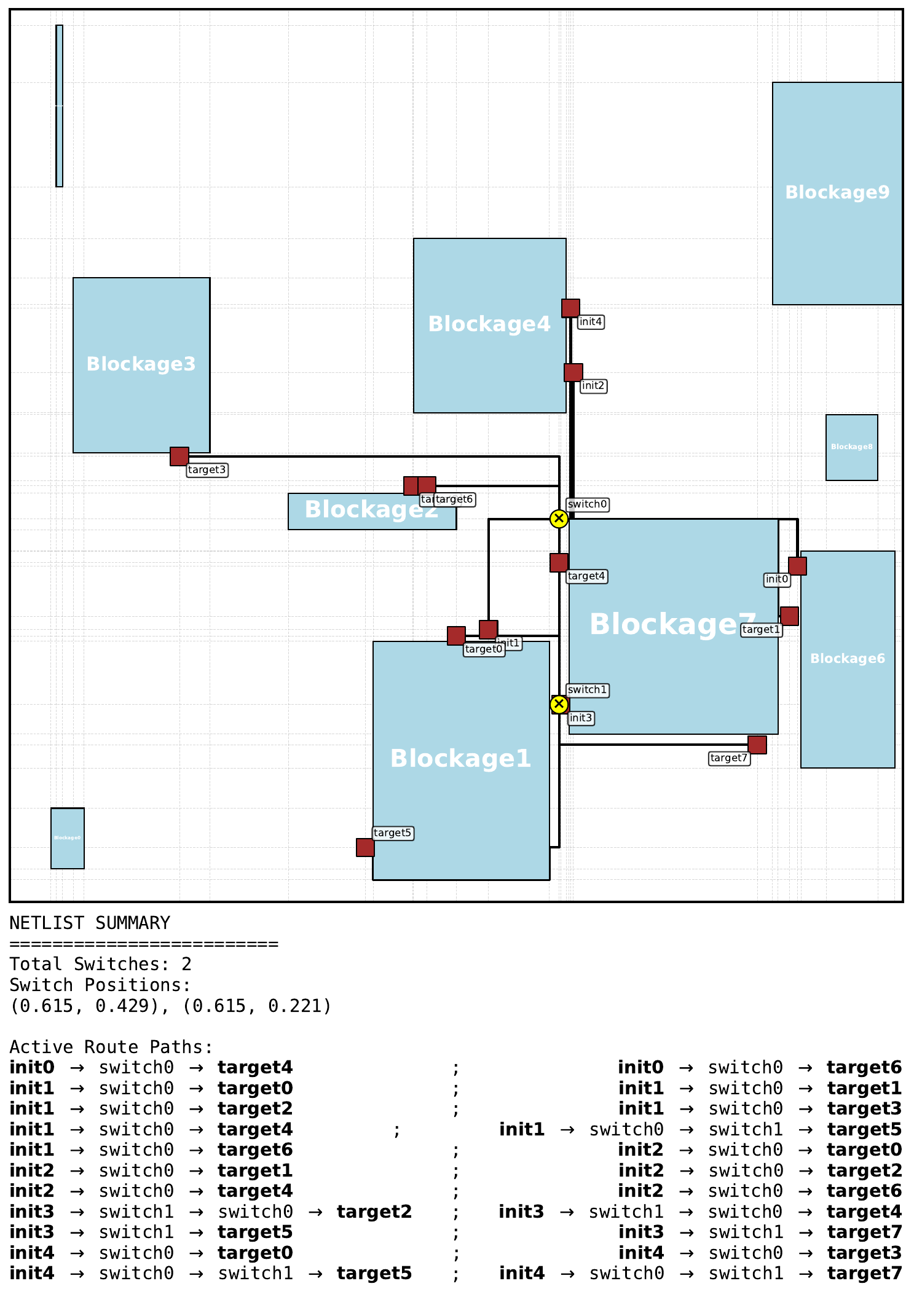}
        \caption*{PPO}
    \end{subfigure}
    \hspace{0.04\linewidth}
    \begin{subfigure}[t]{0.31\linewidth}
        \vspace{0pt}
        \centering
        \includegraphics[width=\linewidth]{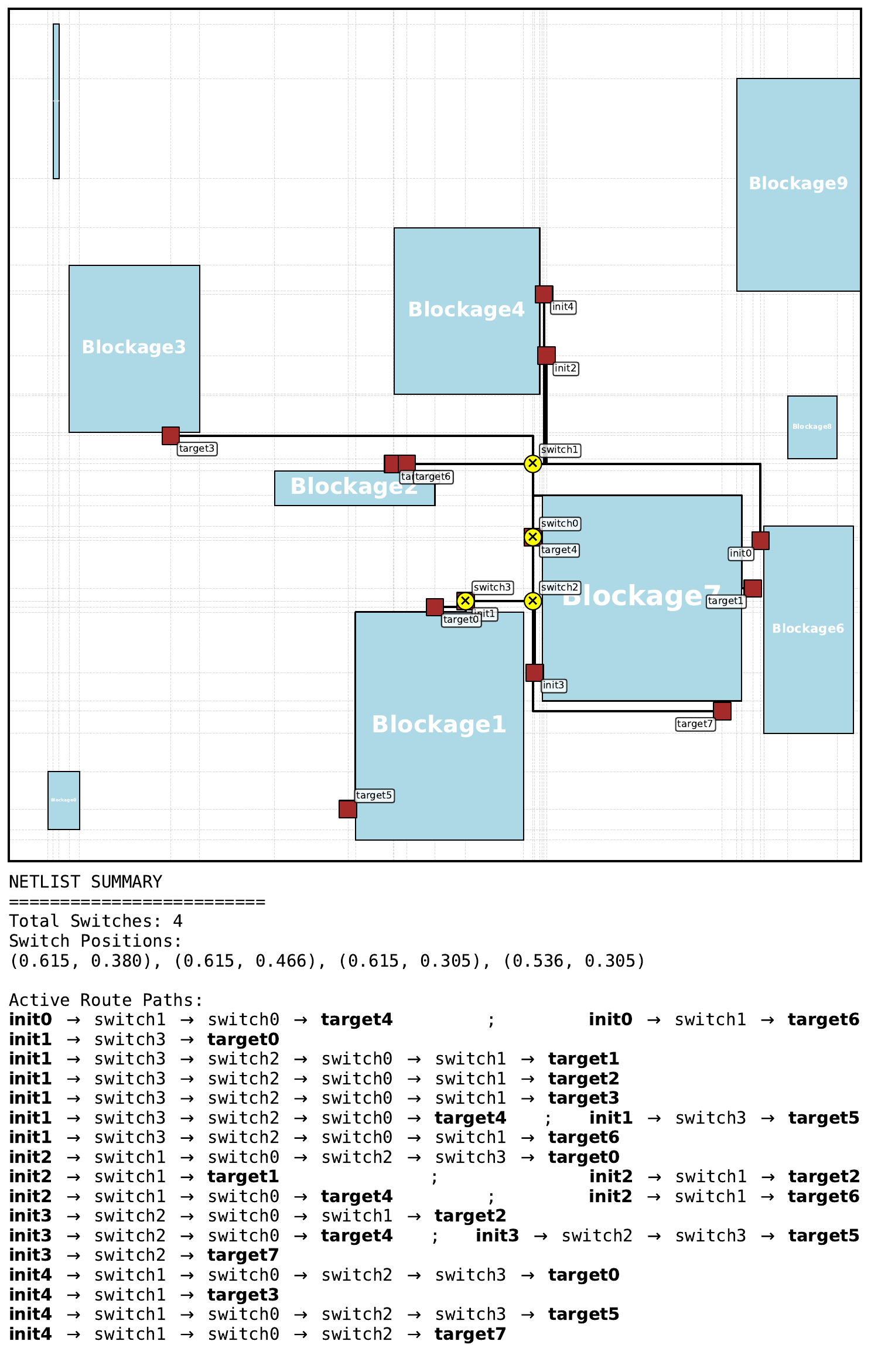}
        \caption*{MCTS}
    \end{subfigure}
\caption{Instance 7.}
\label{fig:best_pretrain_instance_7}
\end{figure*}

\begin{figure*}[h]
\centering
    \begin{subfigure}[t]{0.31\linewidth}
        \vspace{0pt}
        \centering
        \includegraphics[width=\linewidth]{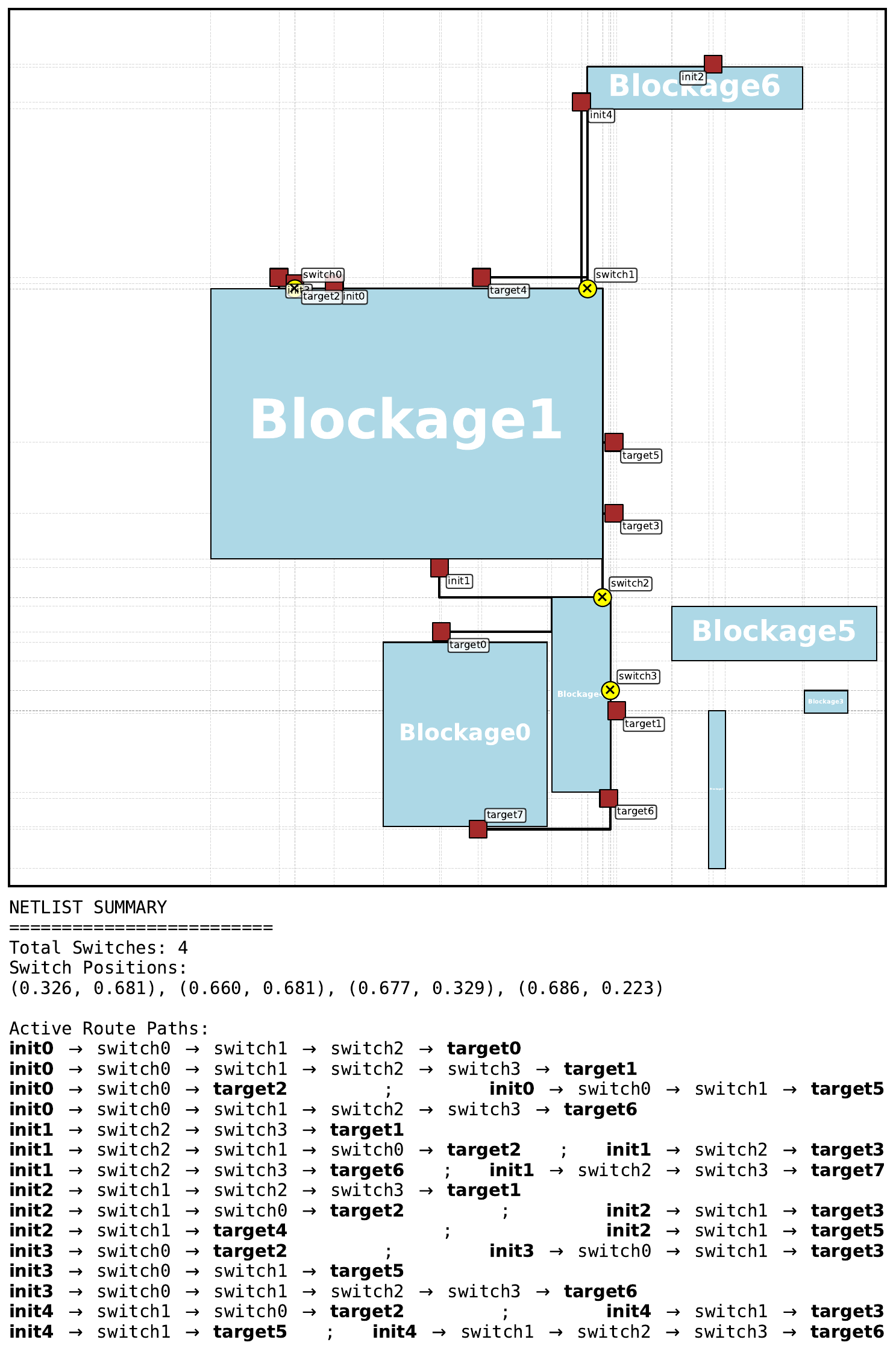}
        \caption*{Heuristic}
    \end{subfigure}
    \hfill
    \begin{subfigure}[t]{0.31\linewidth}
        \vspace{0pt}
        \centering
        \includegraphics[width=\linewidth]{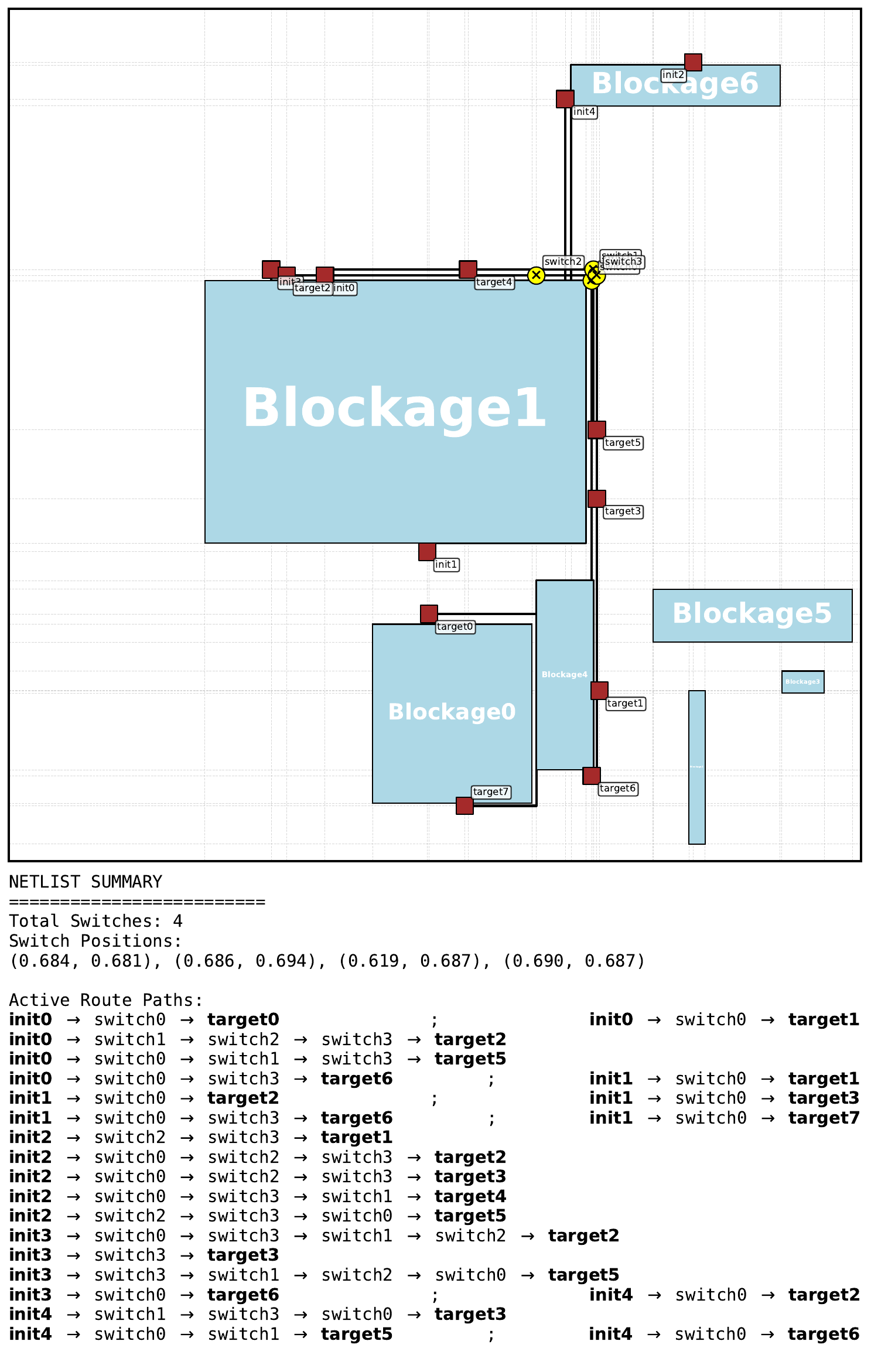}
        \caption*{Random search}
    \end{subfigure}
    \hfill
    \begin{subfigure}[t]{0.31\linewidth}
        \vspace{0pt}
        \centering
        \includegraphics[width=\linewidth]{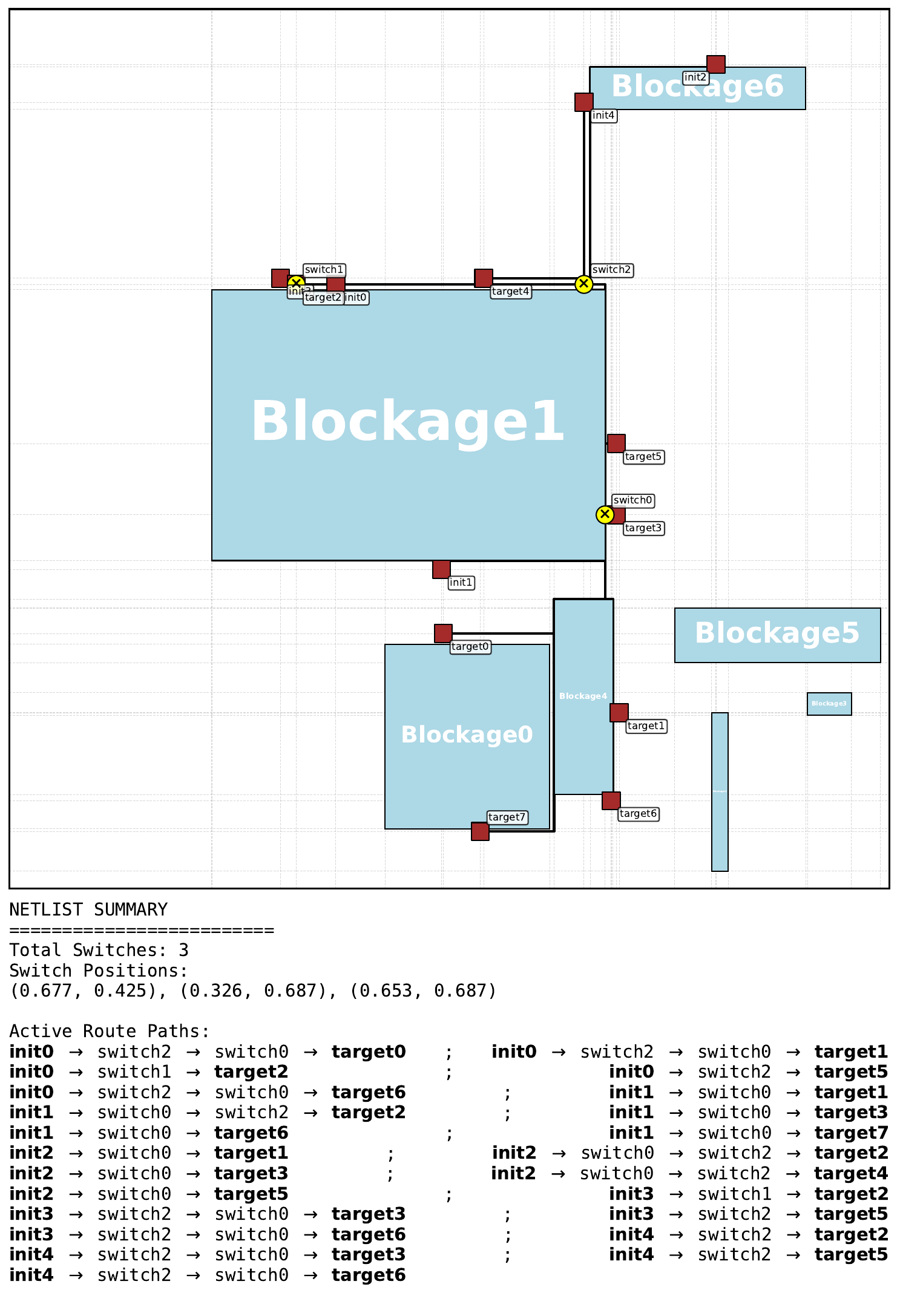}
        \caption*{Genetic algorithm}
    \end{subfigure}
    \\[0.6em]
    \begin{subfigure}[t]{0.31\linewidth}
        \vspace{0pt}
        \centering
        \includegraphics[width=\linewidth]{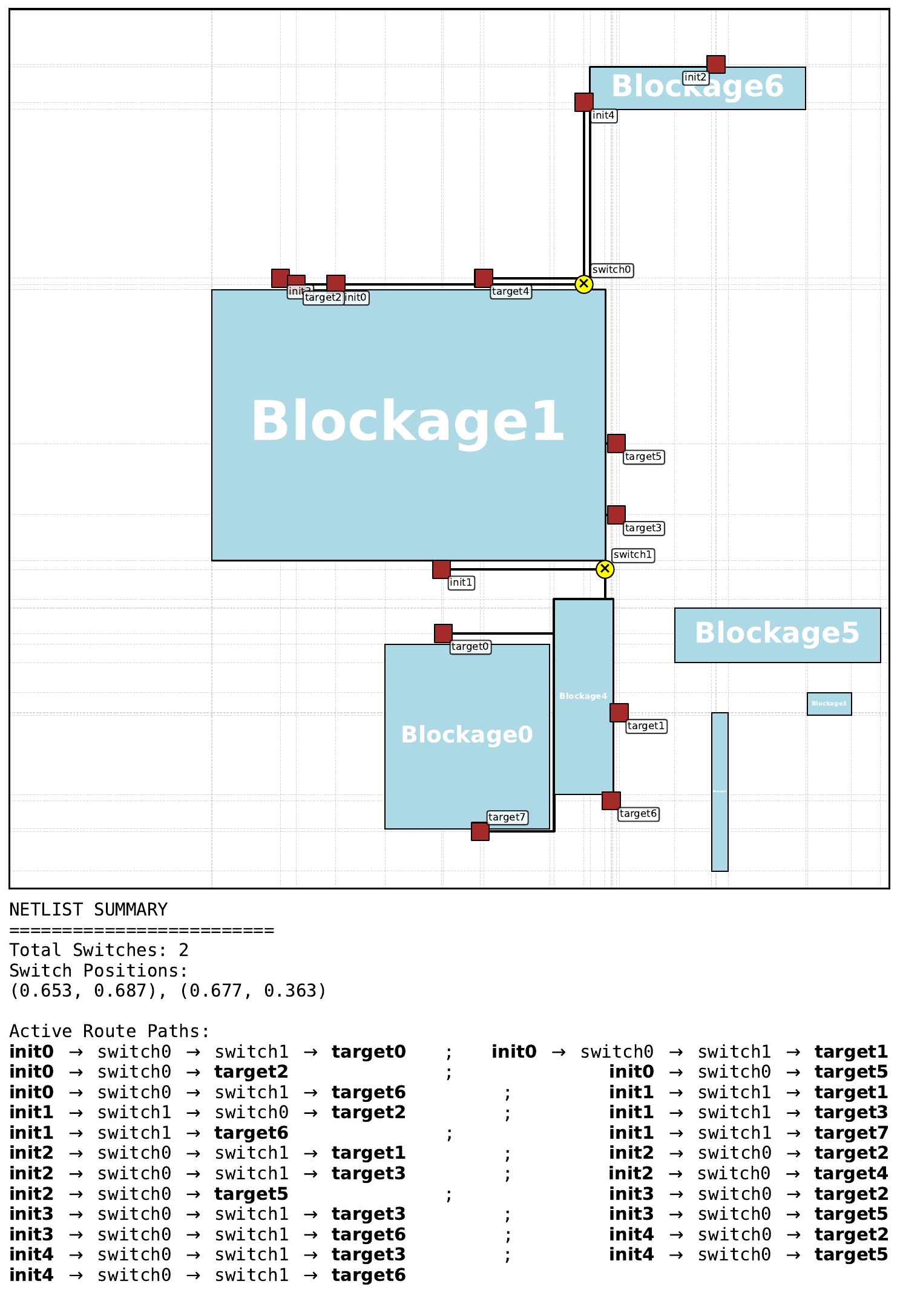}
        \caption*{PPO}
    \end{subfigure}
    \hspace{0.04\linewidth}
    \begin{subfigure}[t]{0.31\linewidth}
        \vspace{0pt}
        \centering
        \includegraphics[width=\linewidth]{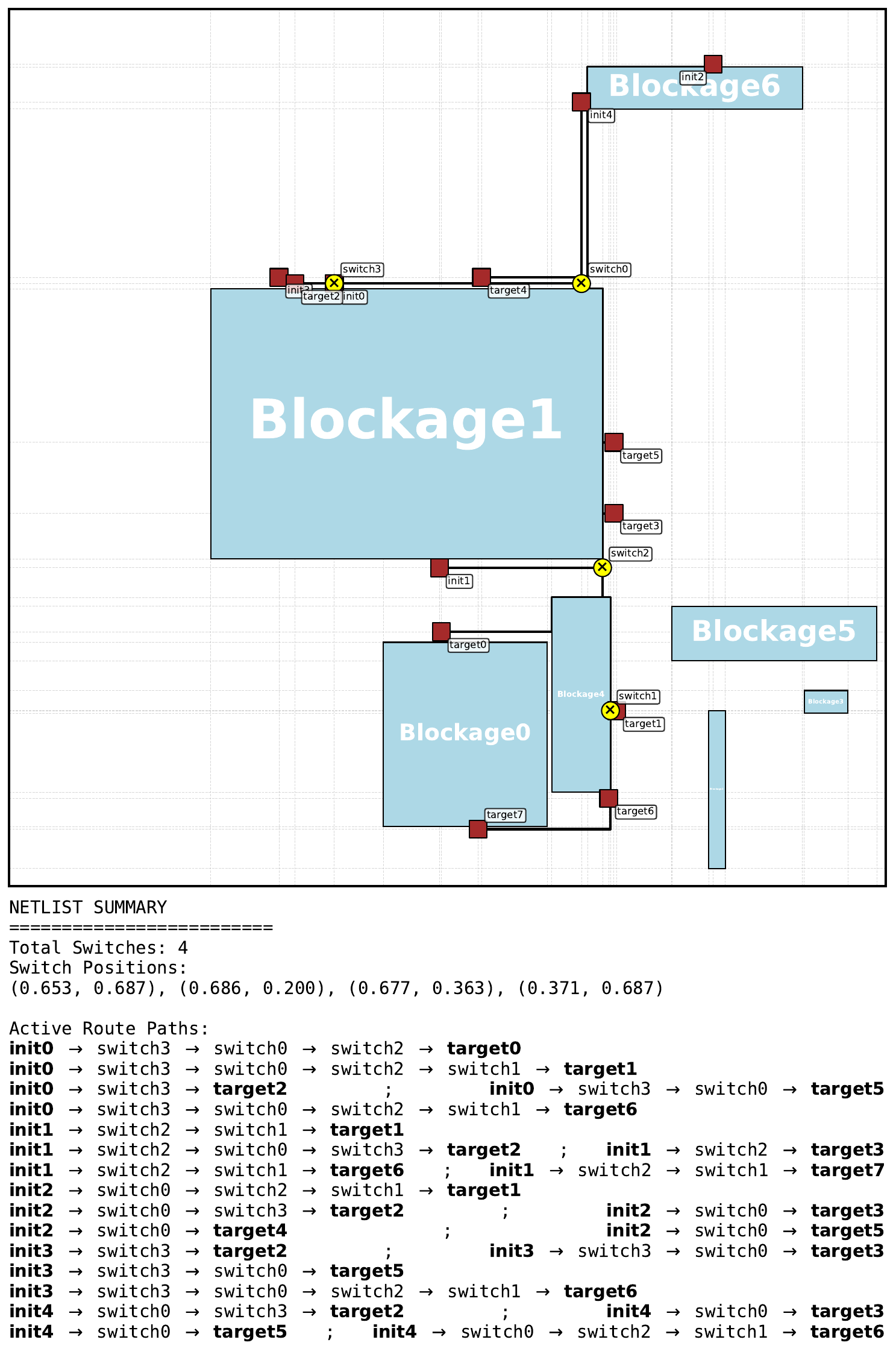}
        \caption*{MCTS}
    \end{subfigure}
\caption{Instance 8.}
\label{fig:best_pretrain_instance_8}
\end{figure*}

\begin{figure*}[h]
\centering
    \begin{subfigure}[t]{0.31\linewidth}
        \vspace{0pt}
        \centering
        \includegraphics[width=\linewidth]{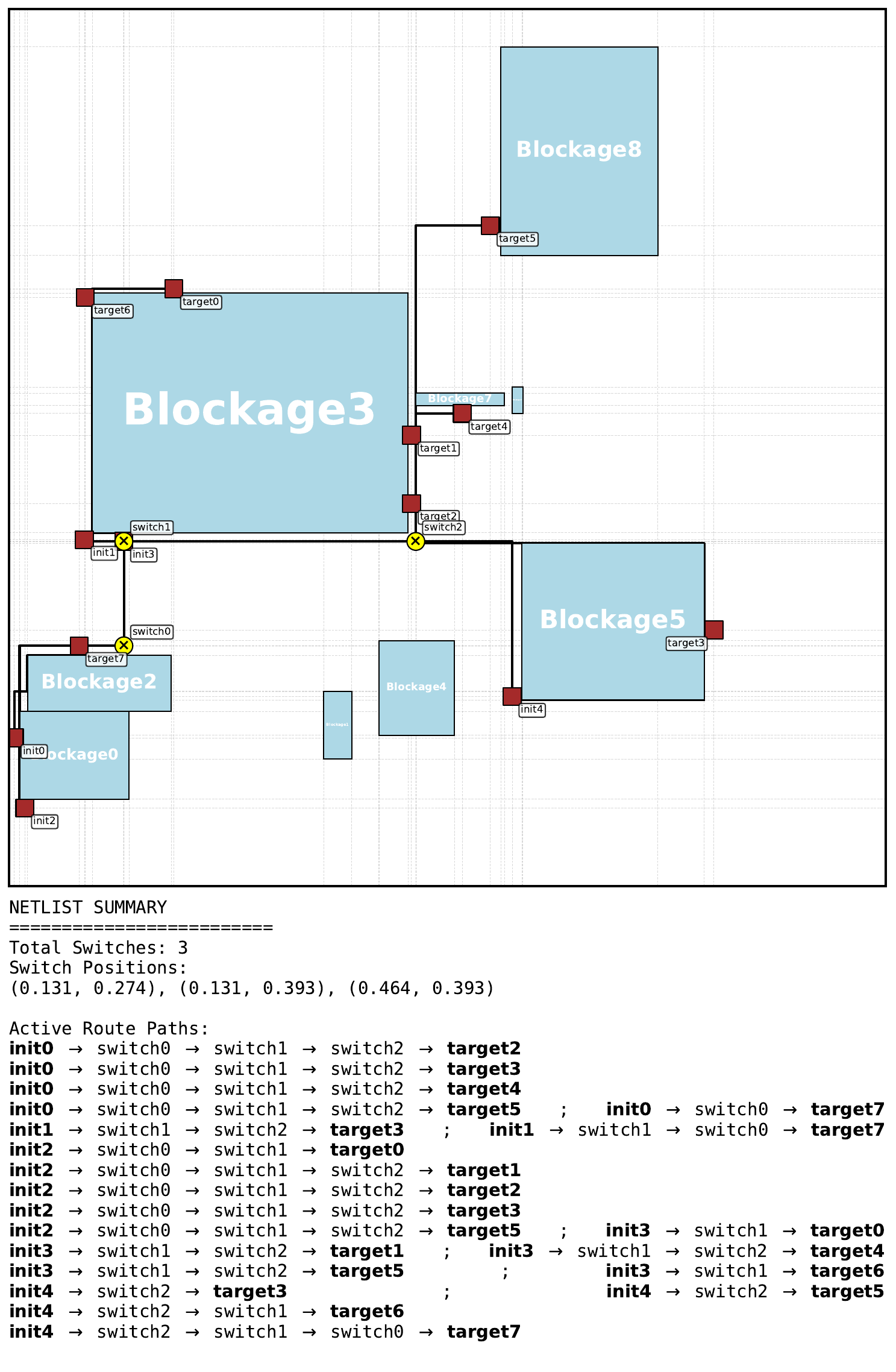}
        \caption*{Heuristic}
    \end{subfigure}
    \hfill
    \begin{subfigure}[t]{0.31\linewidth}
        \vspace{0pt}
        \centering
        \includegraphics[width=\linewidth]{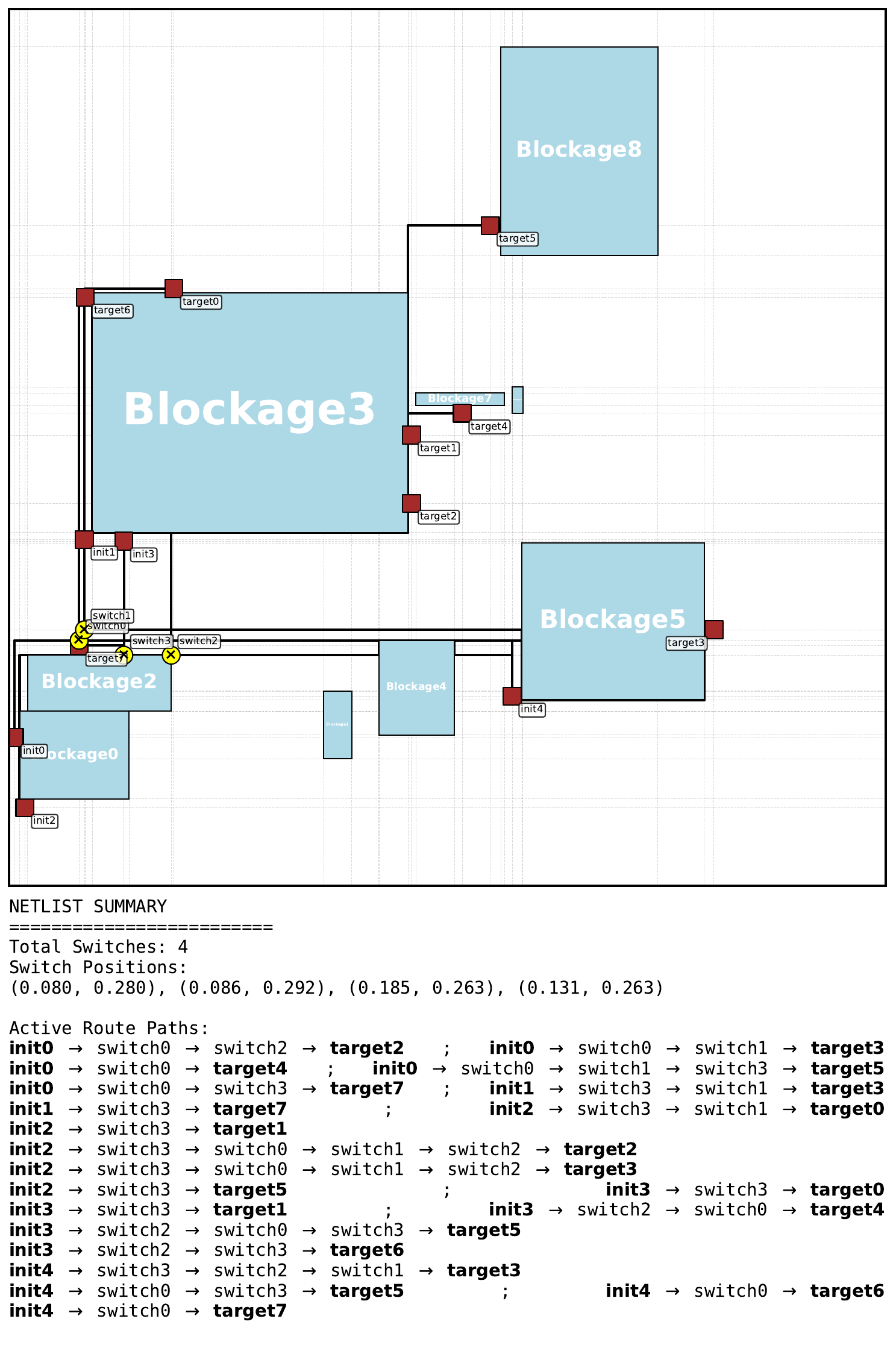}
        \caption*{Random search}
    \end{subfigure}
    \hfill
    \begin{subfigure}[t]{0.31\linewidth}
        \vspace{0pt}
        \centering
        \includegraphics[width=\linewidth]{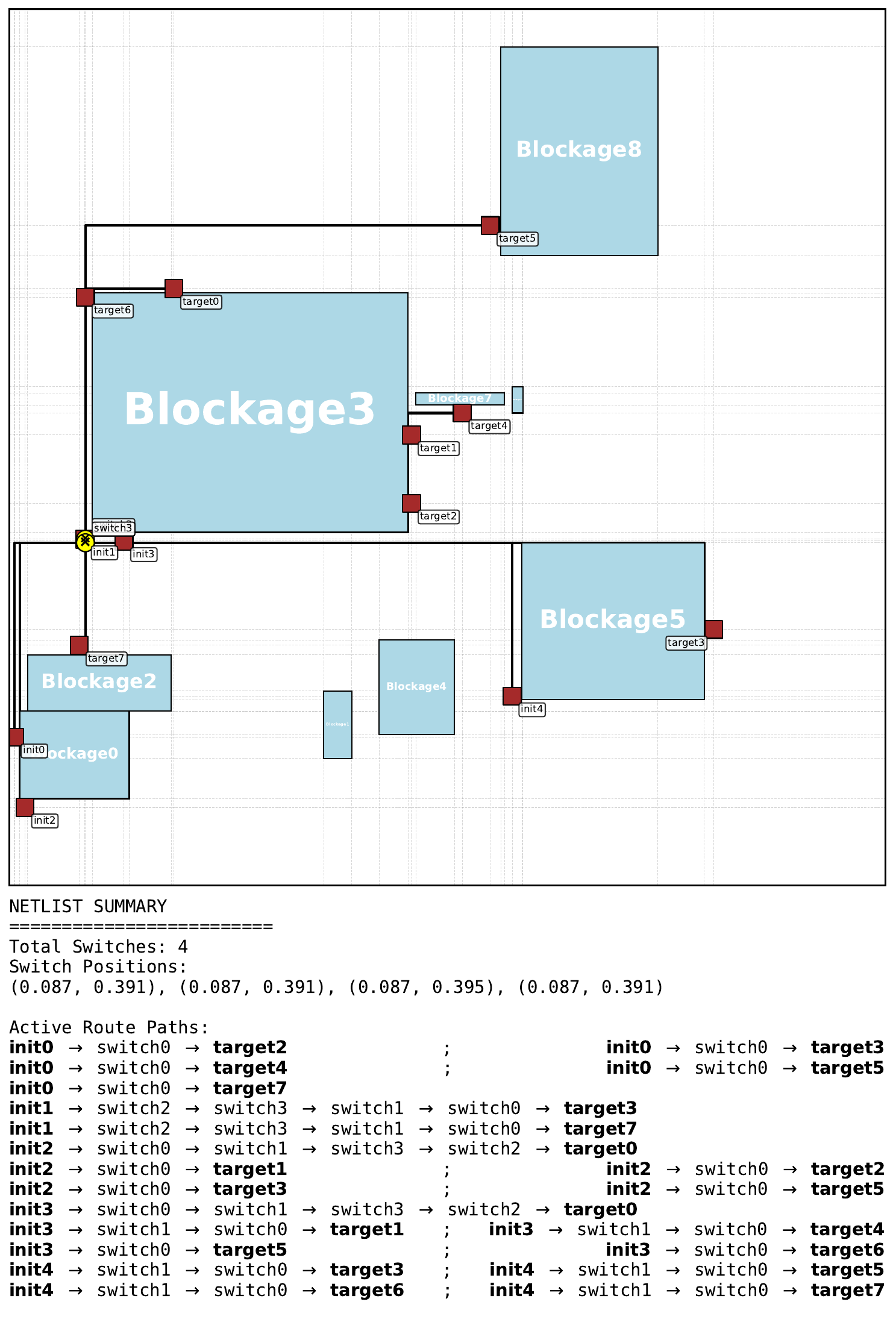}
        \caption*{Genetic algorithm}
    \end{subfigure}
    \\[0.6em]
    \begin{subfigure}[t]{0.31\linewidth}
        \vspace{0pt}
        \centering
        \includegraphics[width=\linewidth]{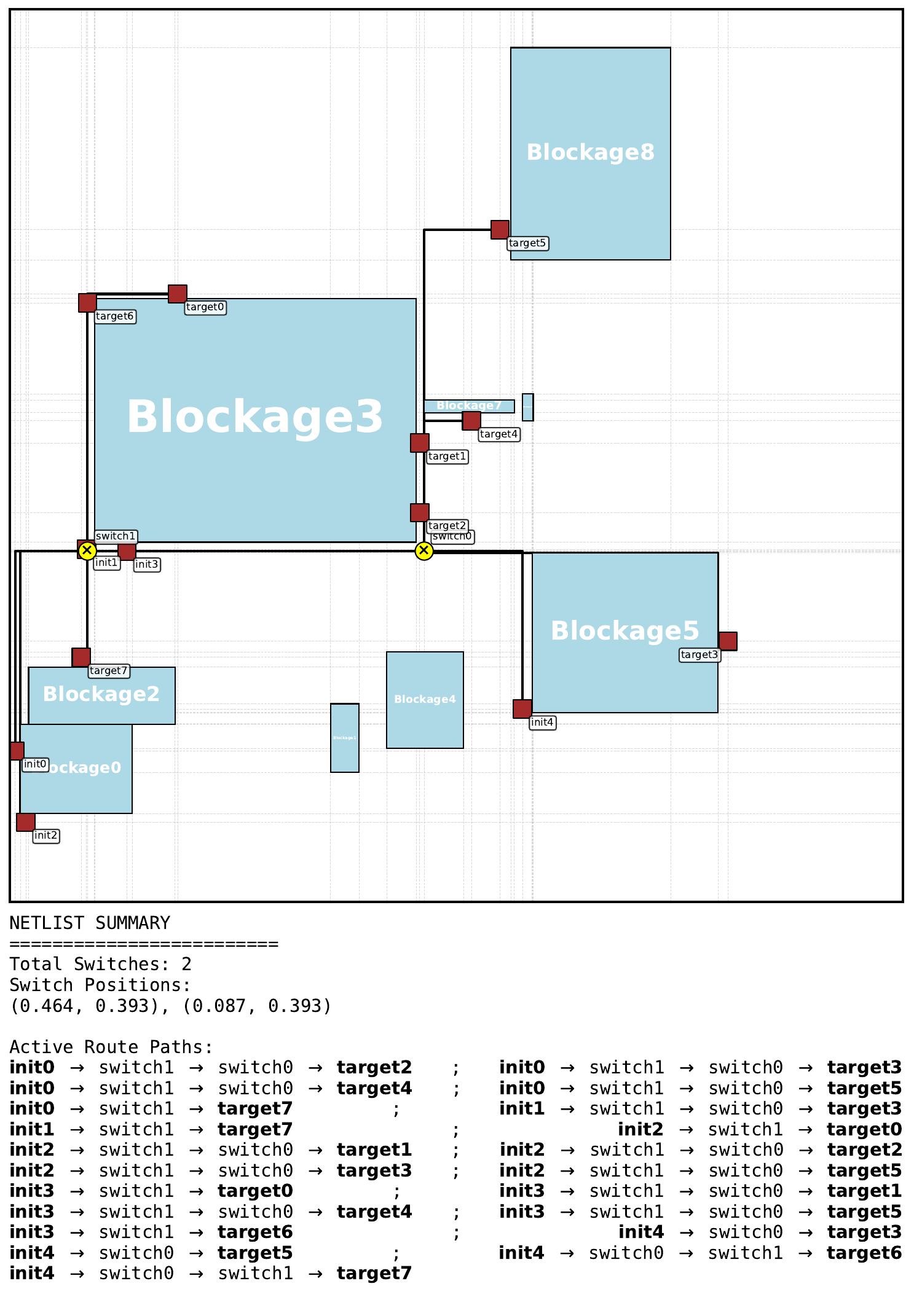}
        \caption*{PPO}
    \end{subfigure}
    \hspace{0.04\linewidth}
    \begin{subfigure}[t]{0.31\linewidth}
        \vspace{0pt}
        \centering
        \includegraphics[width=\linewidth]{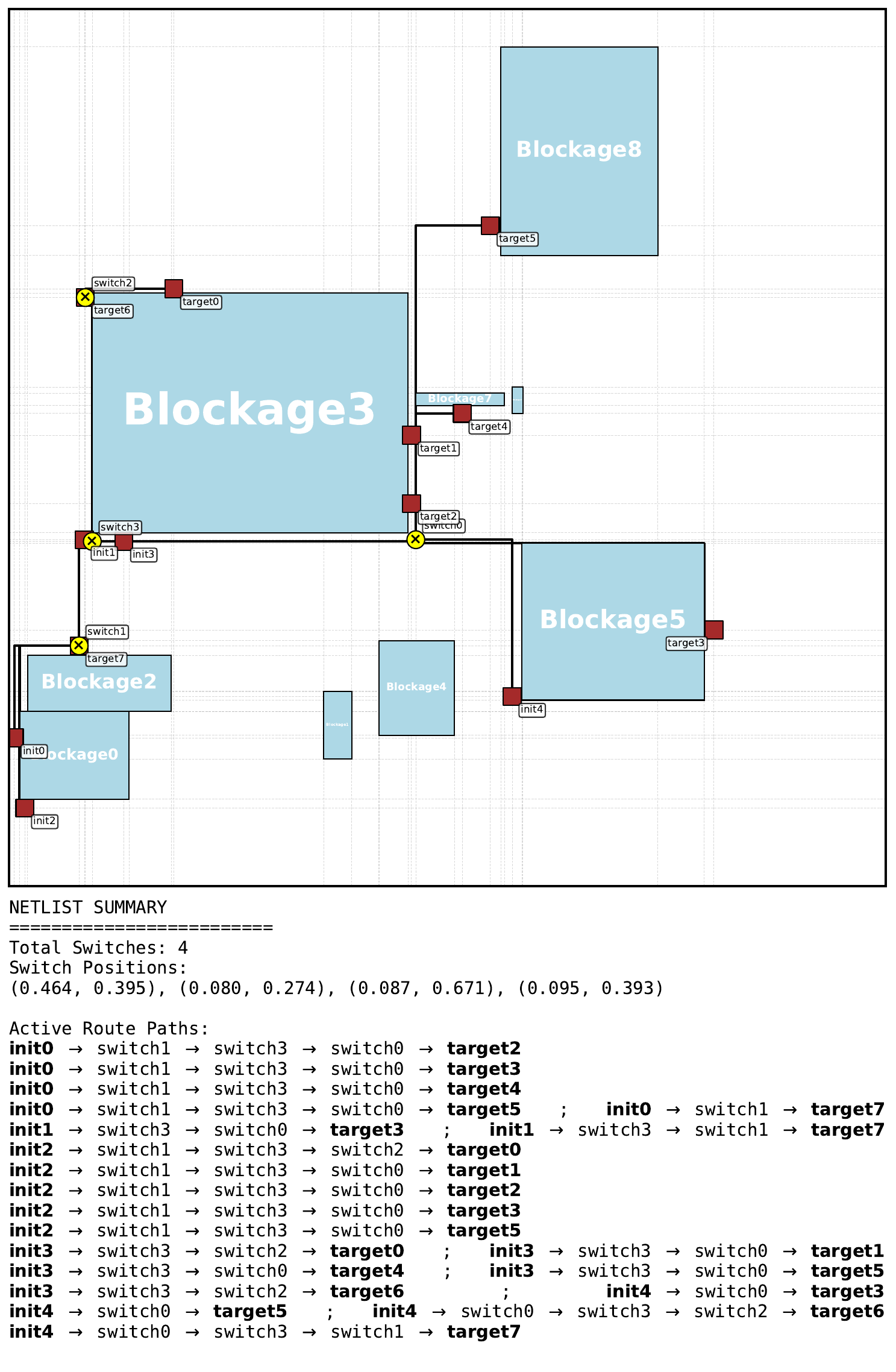}
        \caption*{MCTS}
    \end{subfigure}
\caption{Instance 9.}
\label{fig:best_pretrain_instance_9}
\end{figure*}

\begin{figure*}[h]
\centering
    \begin{subfigure}[t]{0.31\linewidth}
        \vspace{0pt}
        \centering
        \includegraphics[width=\linewidth]{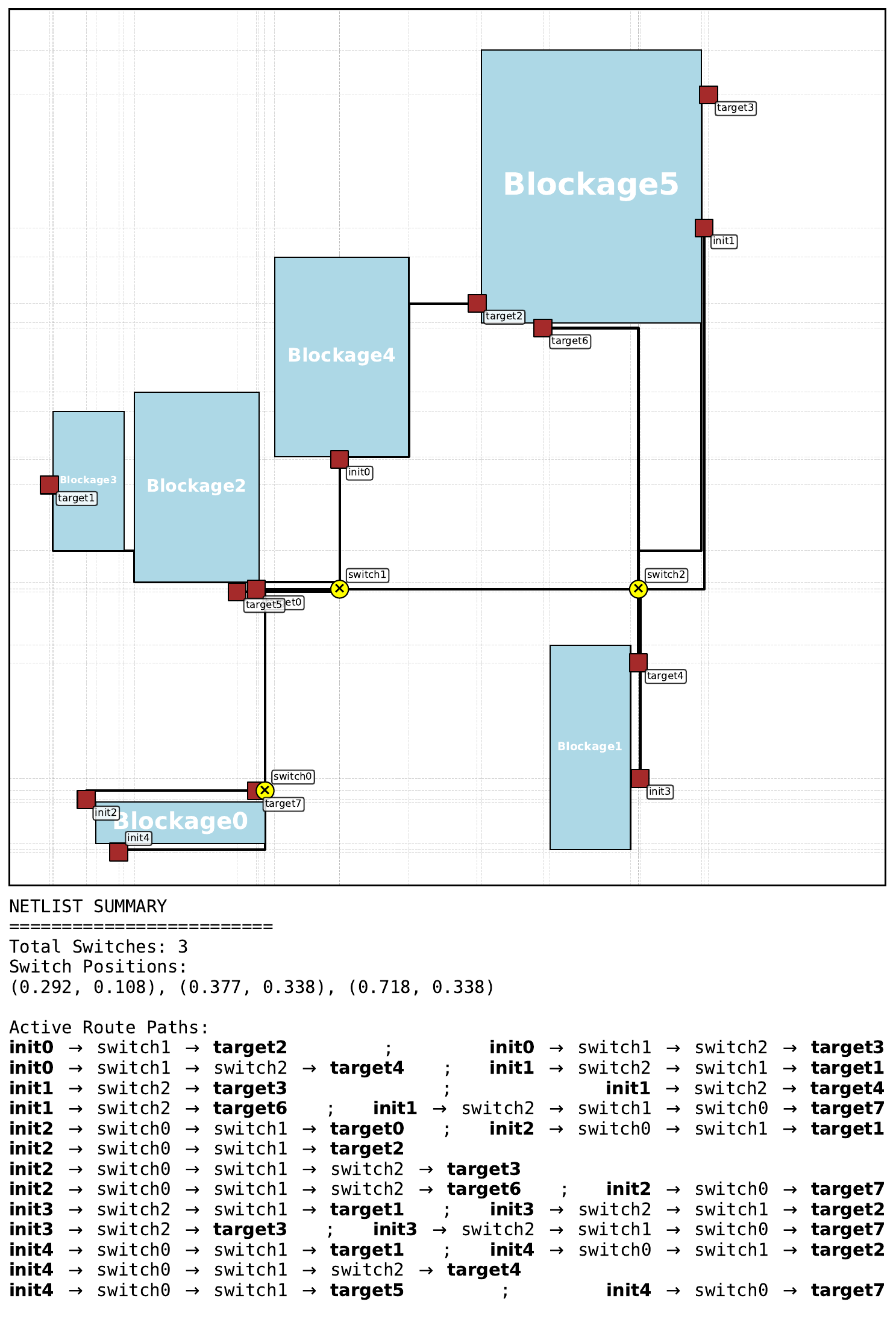}
        \caption*{Heuristic}
    \end{subfigure}
    \hfill
    \begin{subfigure}[t]{0.31\linewidth}
        \vspace{0pt}
        \centering
        \includegraphics[width=\linewidth]{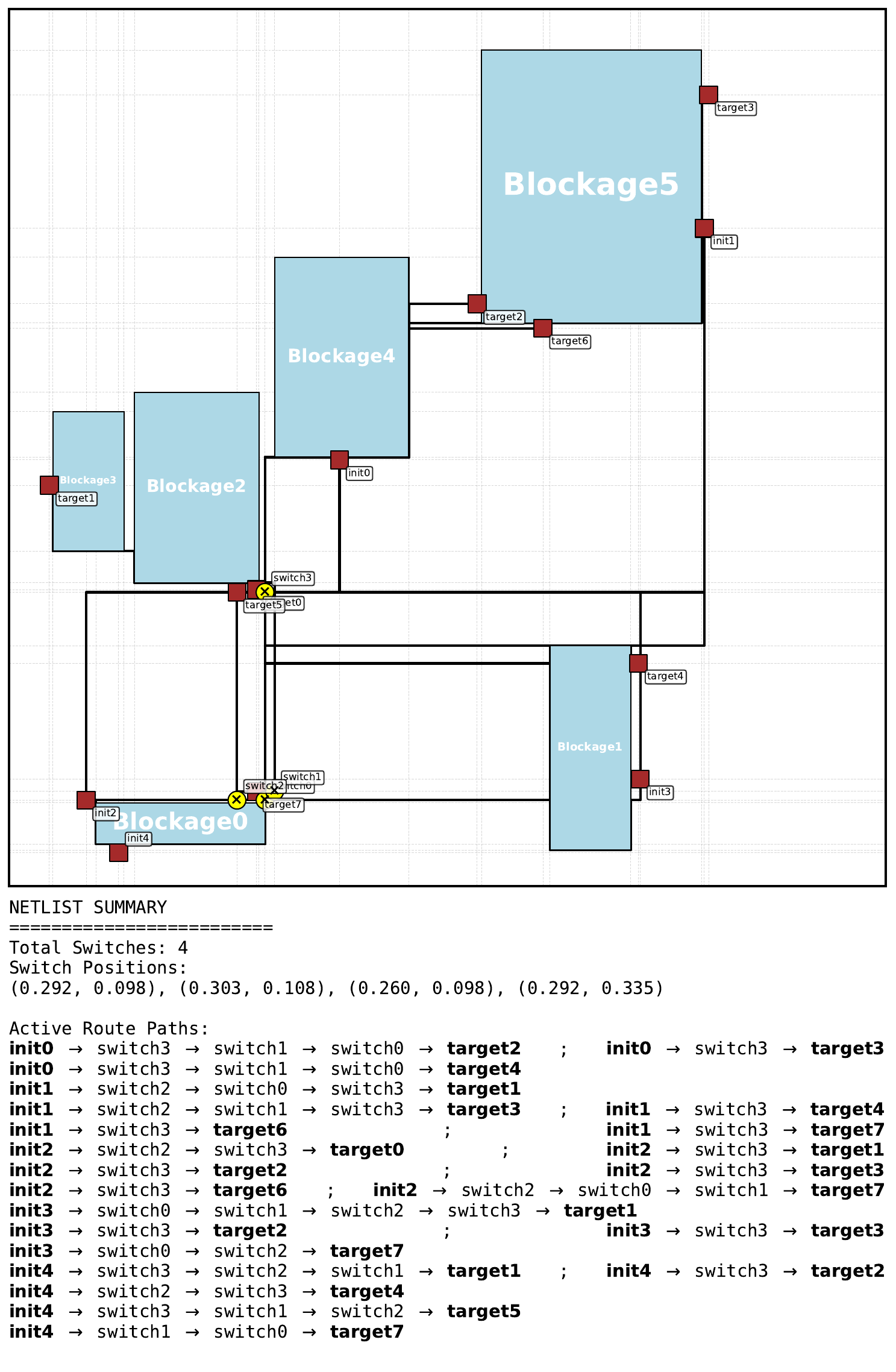}
        \caption*{Random search}
    \end{subfigure}
    \hfill
    \begin{subfigure}[t]{0.31\linewidth}
        \vspace{0pt}
        \centering
        \includegraphics[width=\linewidth]{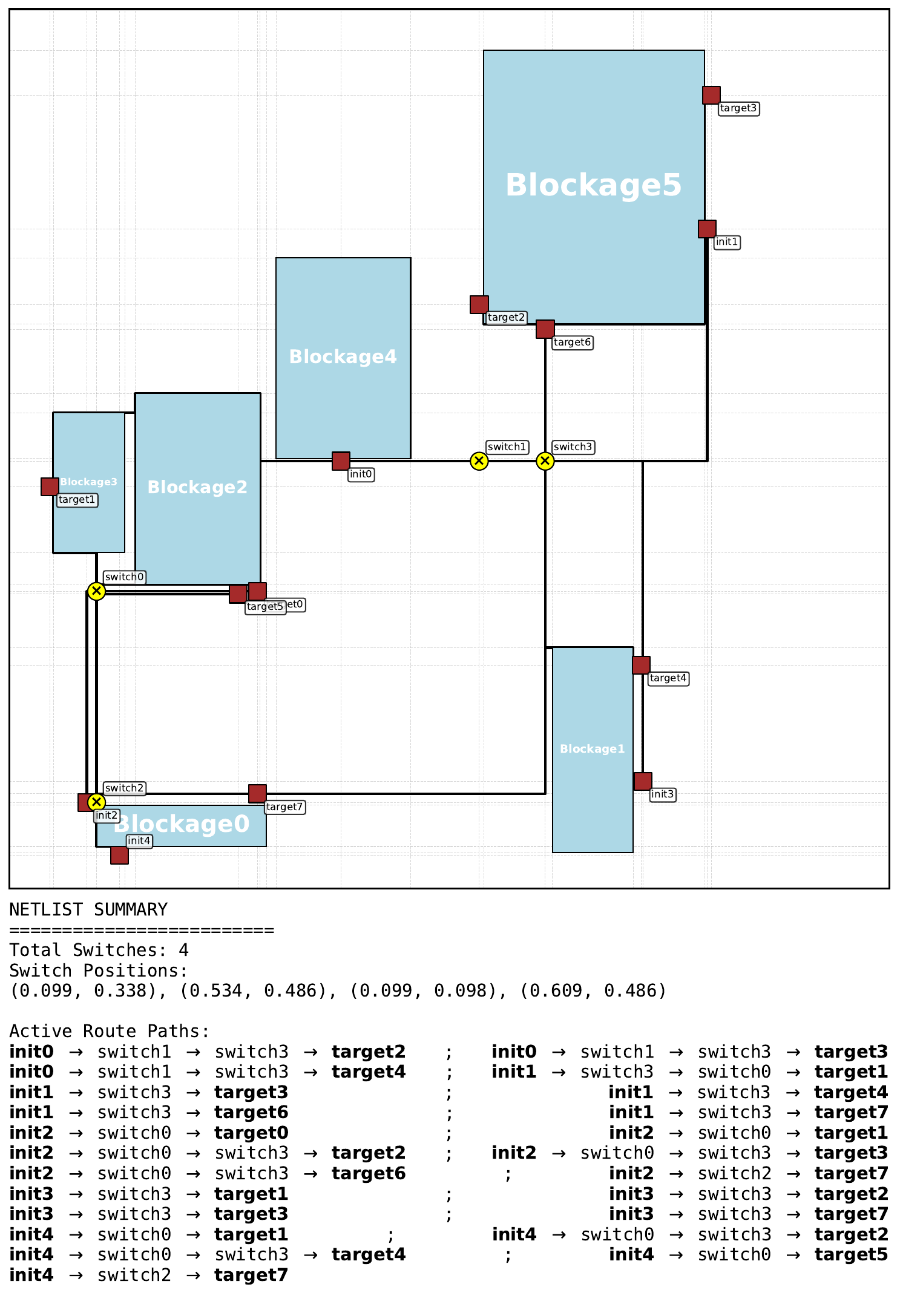}
        \caption*{Genetic algorithm}
    \end{subfigure}
    \\[0.6em]
    \begin{subfigure}[t]{0.31\linewidth}
        \vspace{0pt}
        \centering
        \includegraphics[width=\linewidth]{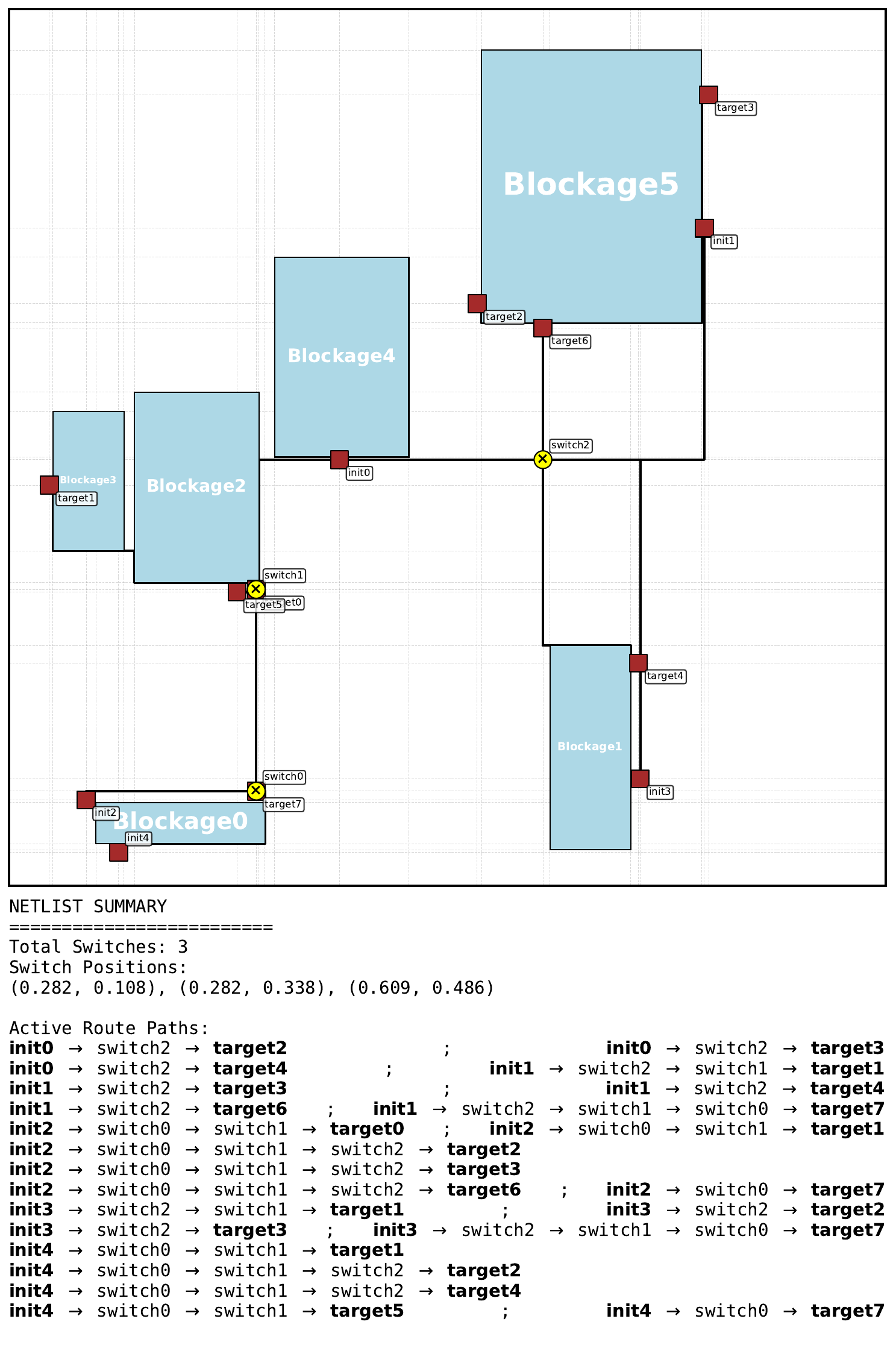}
        \caption*{PPO}
    \end{subfigure}
    \hspace{0.04\linewidth}
    \begin{subfigure}[t]{0.31\linewidth}
        \vspace{0pt}
        \centering
        \includegraphics[width=\linewidth]{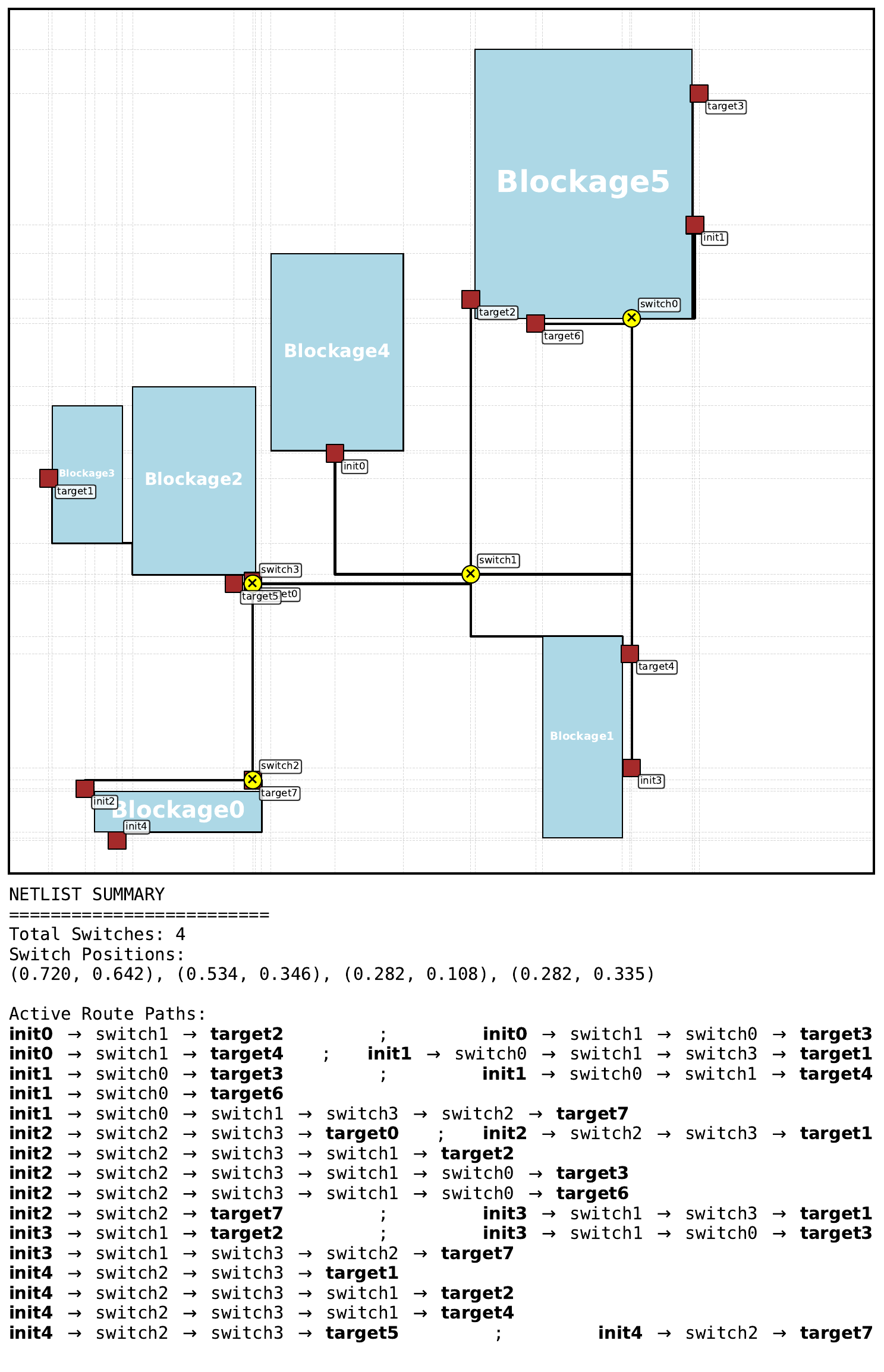}
        \caption*{MCTS}
    \end{subfigure}
\caption{Instance 10.}
\label{fig:best_pretrain_instance_10}
\end{figure*}

\begin{figure*}[h]
\centering
    \begin{subfigure}[t]{0.31\linewidth}
        \vspace{0pt}
        \centering
        \includegraphics[width=\linewidth]{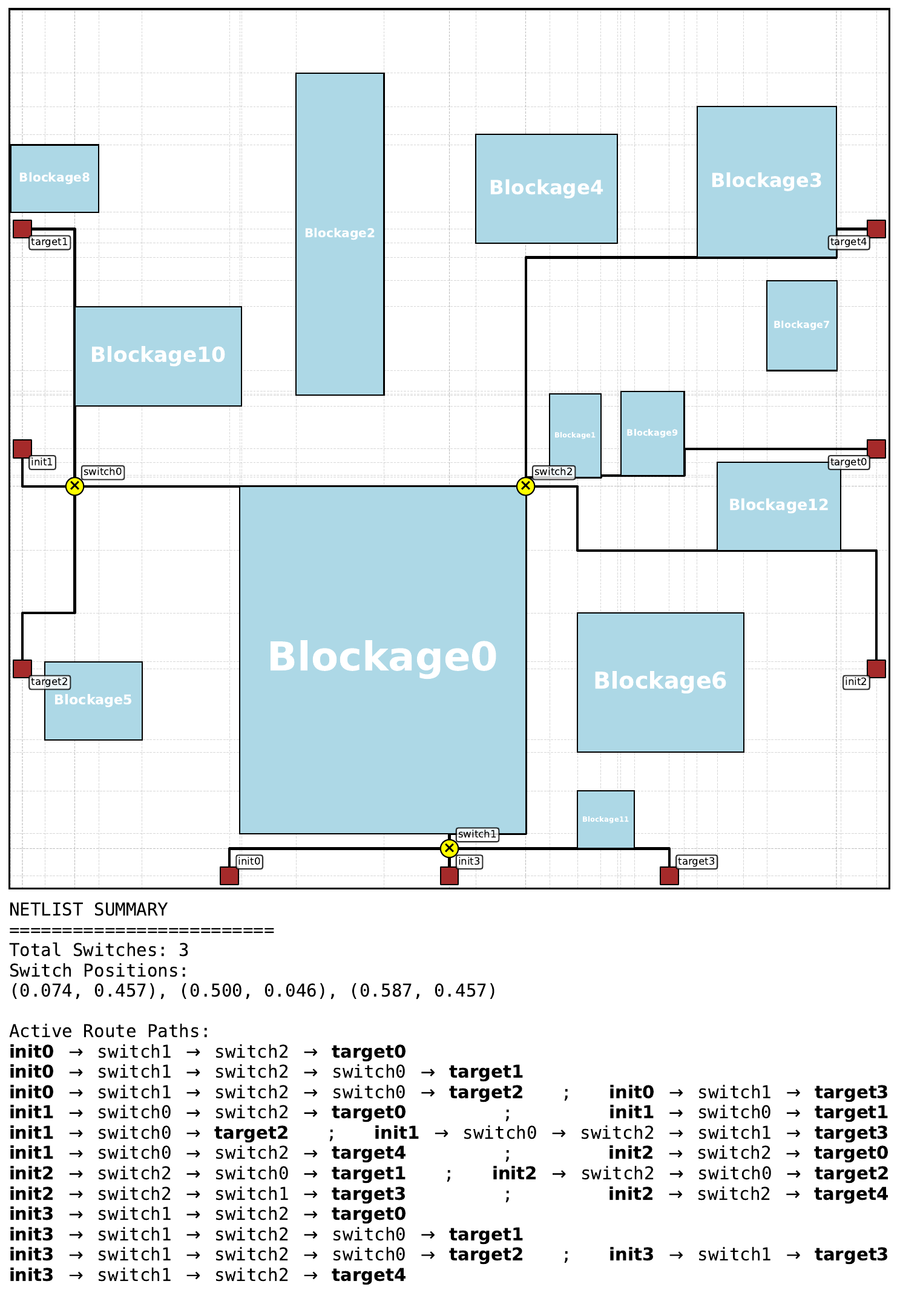}
        \caption*{Heuristic}
    \end{subfigure}
    \hfill
    \begin{subfigure}[t]{0.31\linewidth}
        \vspace{0pt}
        \centering
        \includegraphics[width=\linewidth]{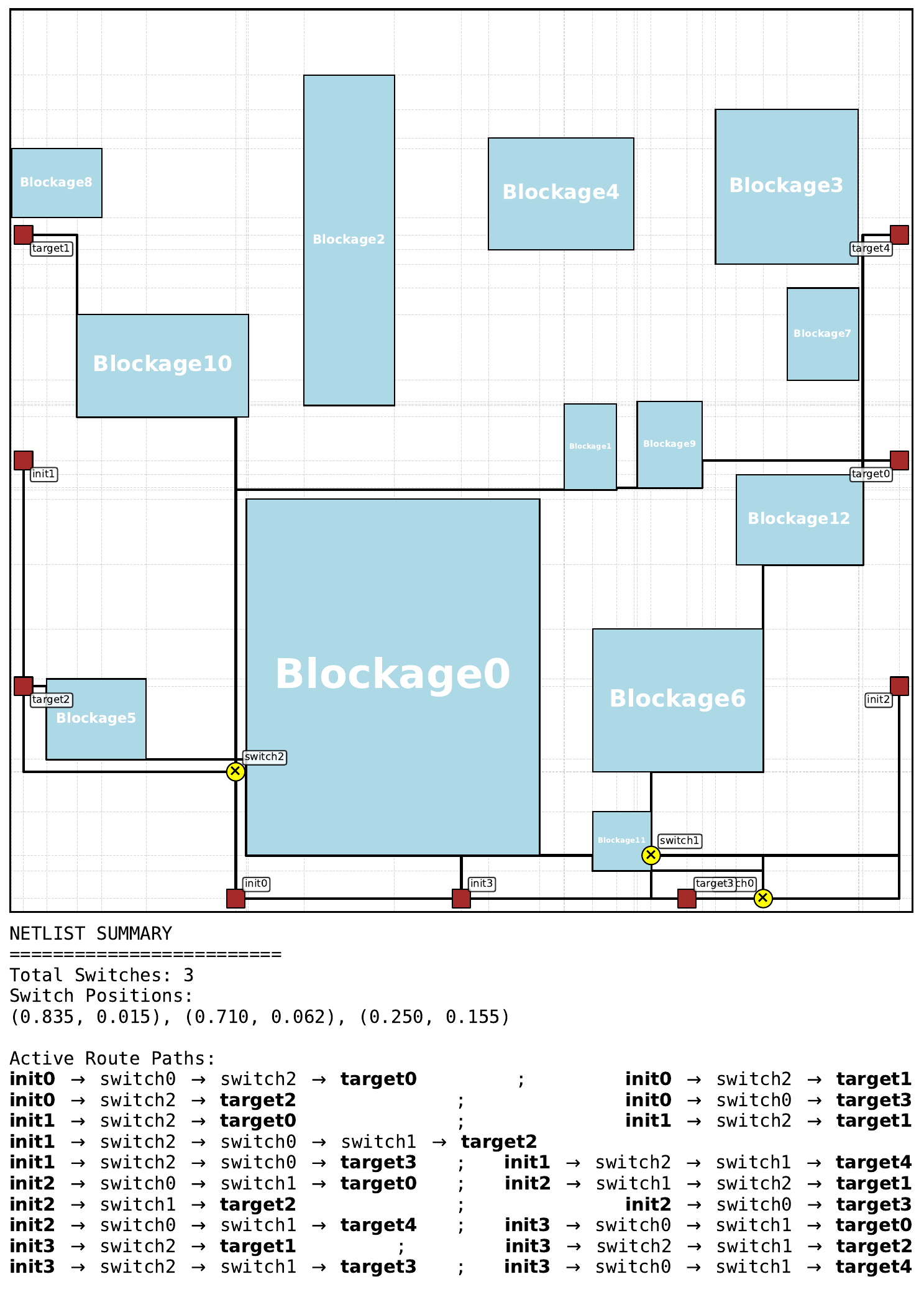}
        \caption*{Random search}
    \end{subfigure}
    \hfill
    \begin{subfigure}[t]{0.31\linewidth}
        \vspace{0pt}
        \centering
        \includegraphics[width=\linewidth]{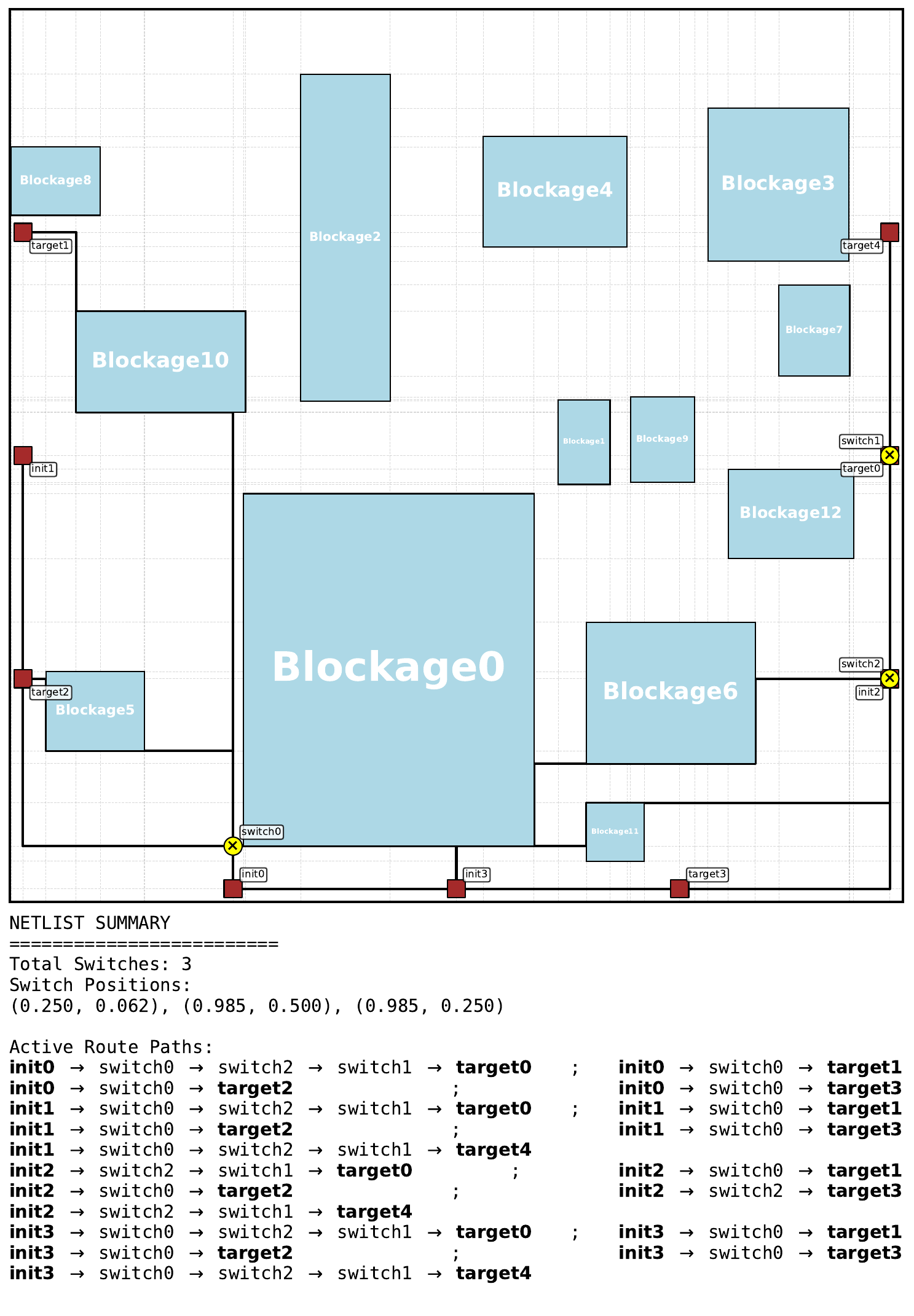}
        \caption*{Genetic algorithm}
    \end{subfigure}
    \\[0.6em]
    \begin{subfigure}[t]{0.31\linewidth}
        \vspace{0pt}
        \centering
        \includegraphics[width=\linewidth]{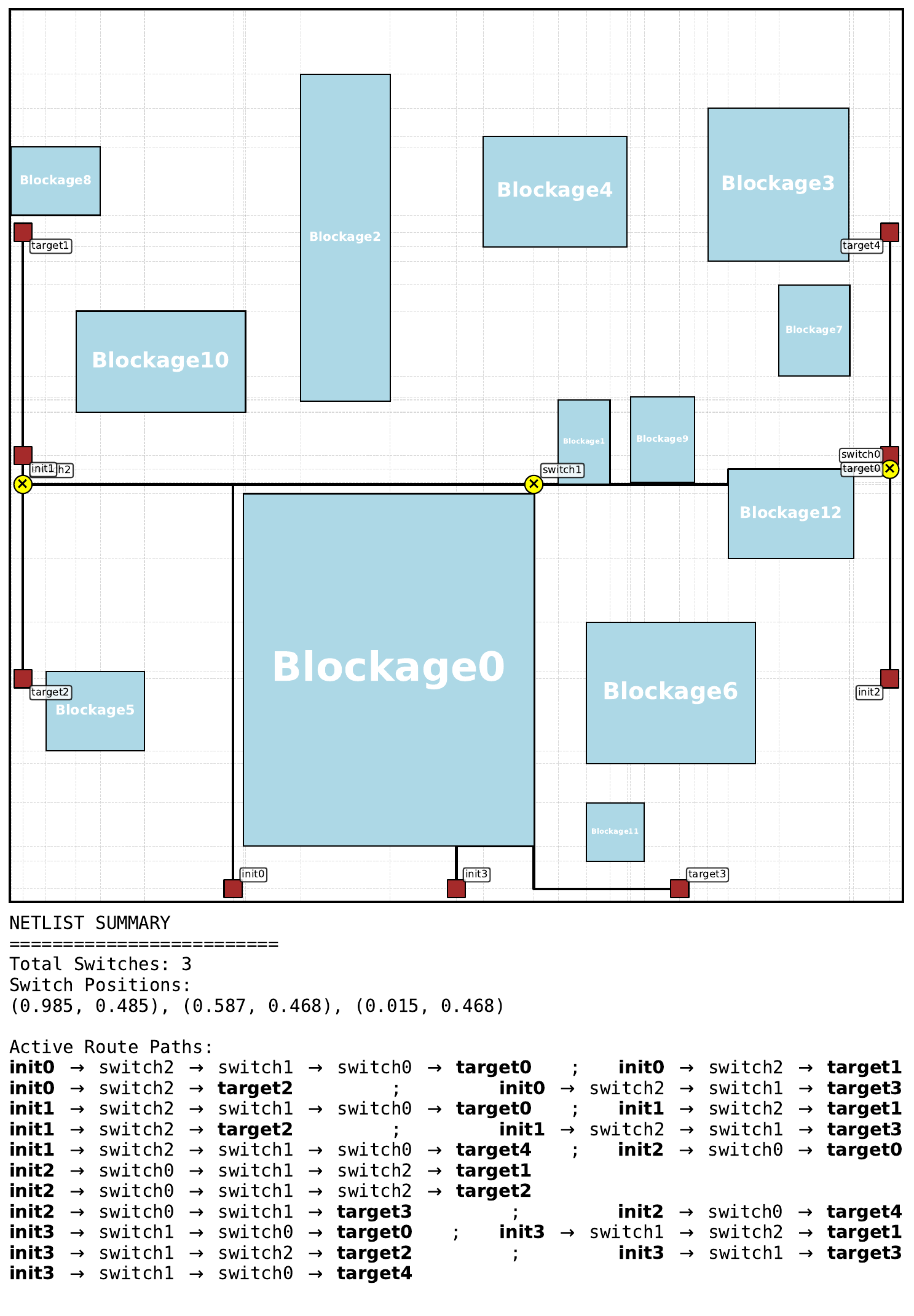}
        \caption*{PPO}
    \end{subfigure}
    \hspace{0.04\linewidth}
    \begin{subfigure}[t]{0.31\linewidth}
        \vspace{0pt}
        \centering
        \includegraphics[width=\linewidth]{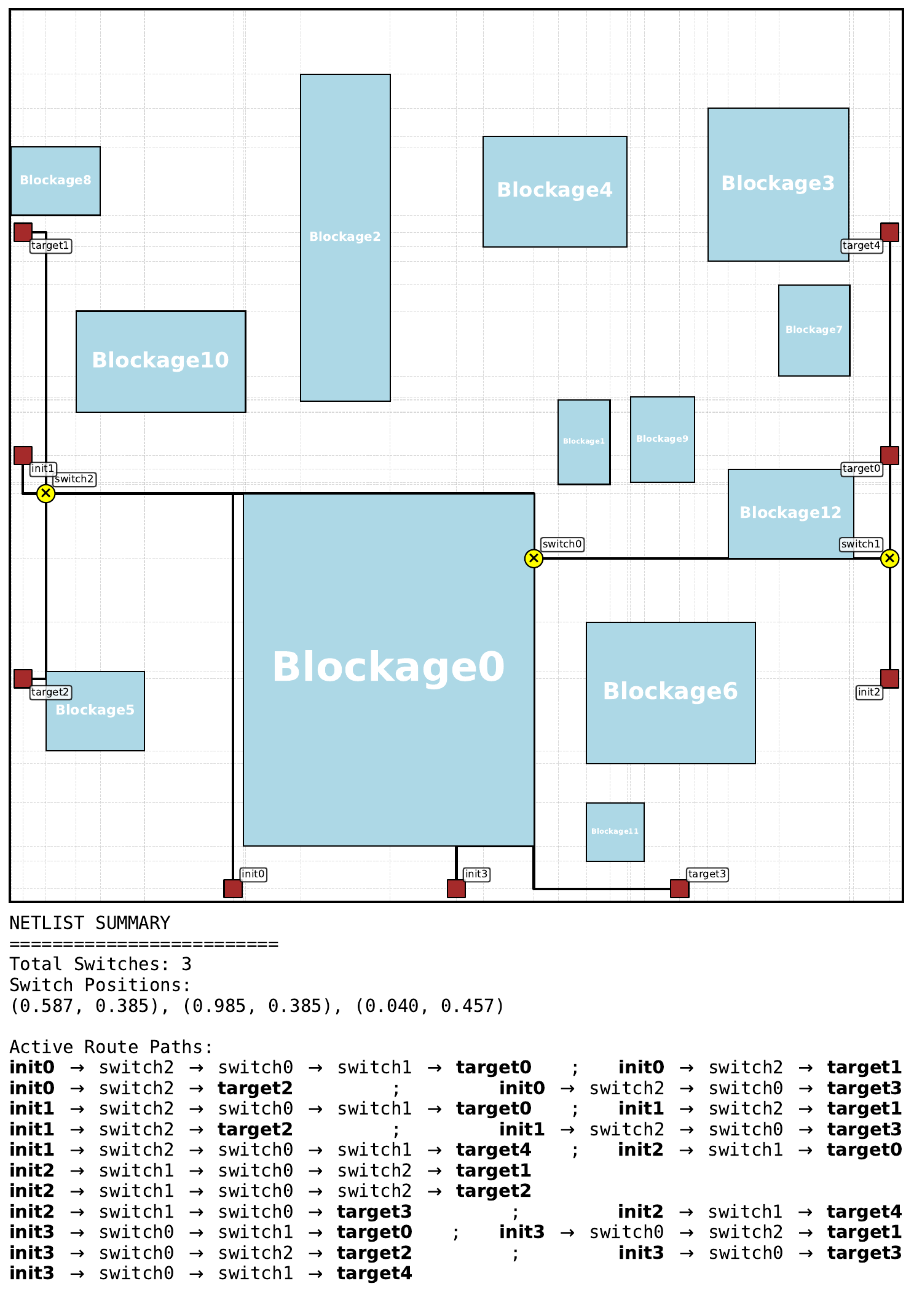}
        \caption*{MCTS}
    \end{subfigure}
\caption{Instance 11.}
\label{fig:best_pretrain_instance_11}
\end{figure*}

\begin{figure*}[h]
\centering
    \begin{subfigure}[t]{0.31\linewidth}
        \vspace{0pt}
        \centering
        \includegraphics[width=\linewidth]{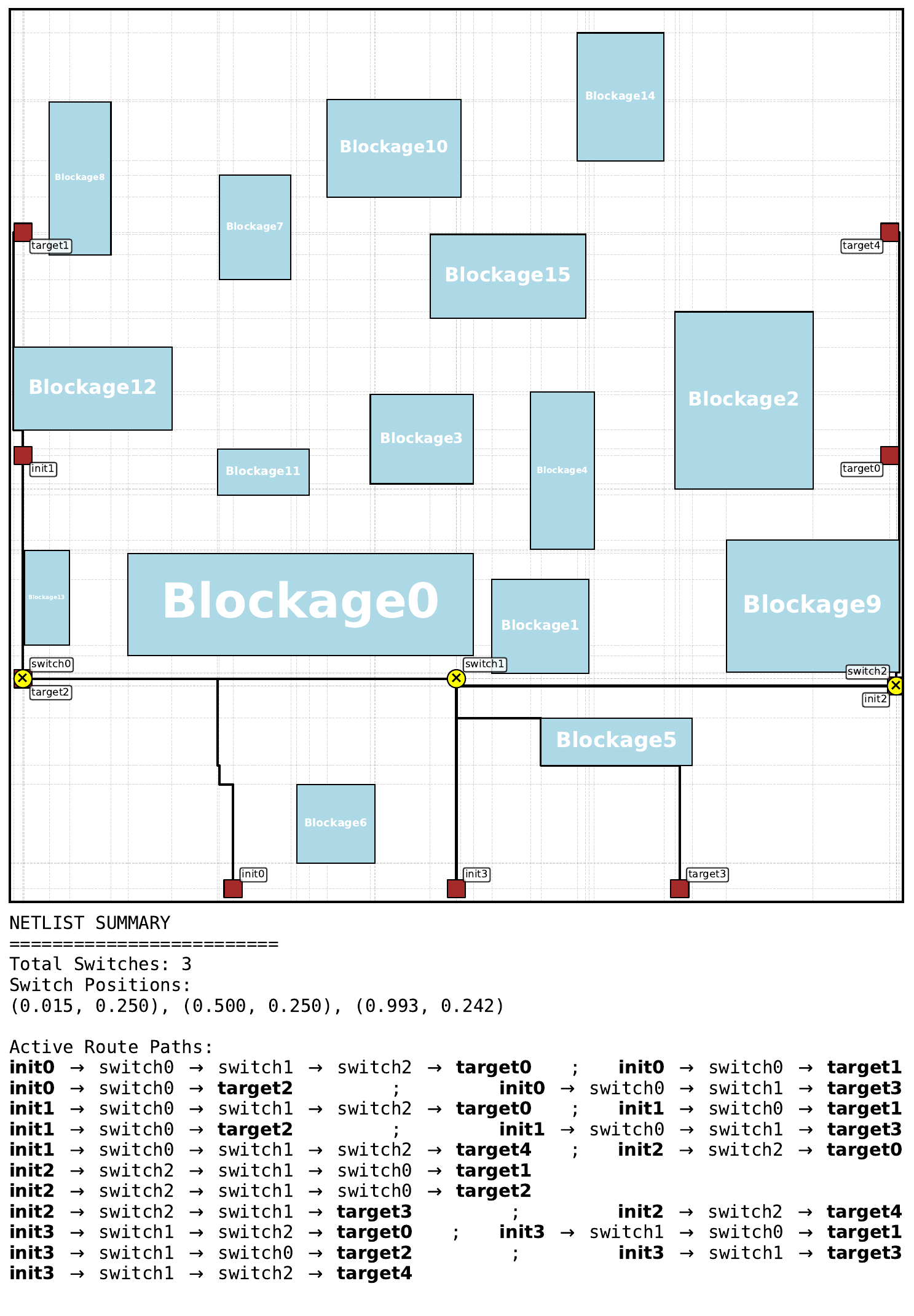}
        \caption*{Heuristic}
    \end{subfigure}
    \hfill
    \begin{subfigure}[t]{0.31\linewidth}
        \vspace{0pt}
        \centering
        \includegraphics[width=\linewidth]{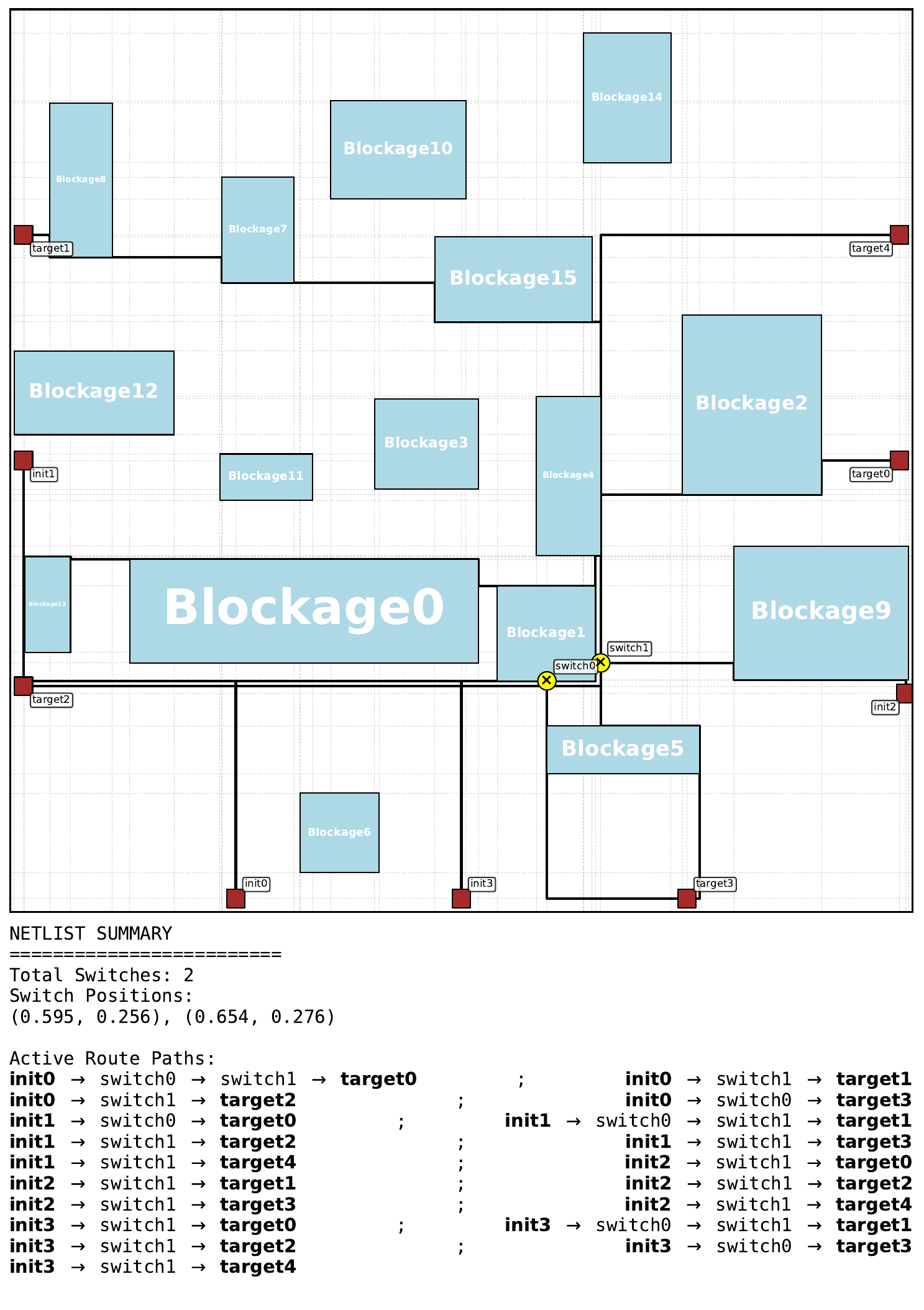}
        \caption*{Random search}
    \end{subfigure}
    \hfill
    \begin{subfigure}[t]{0.31\linewidth}
        \vspace{0pt}
        \centering
        \includegraphics[width=\linewidth]{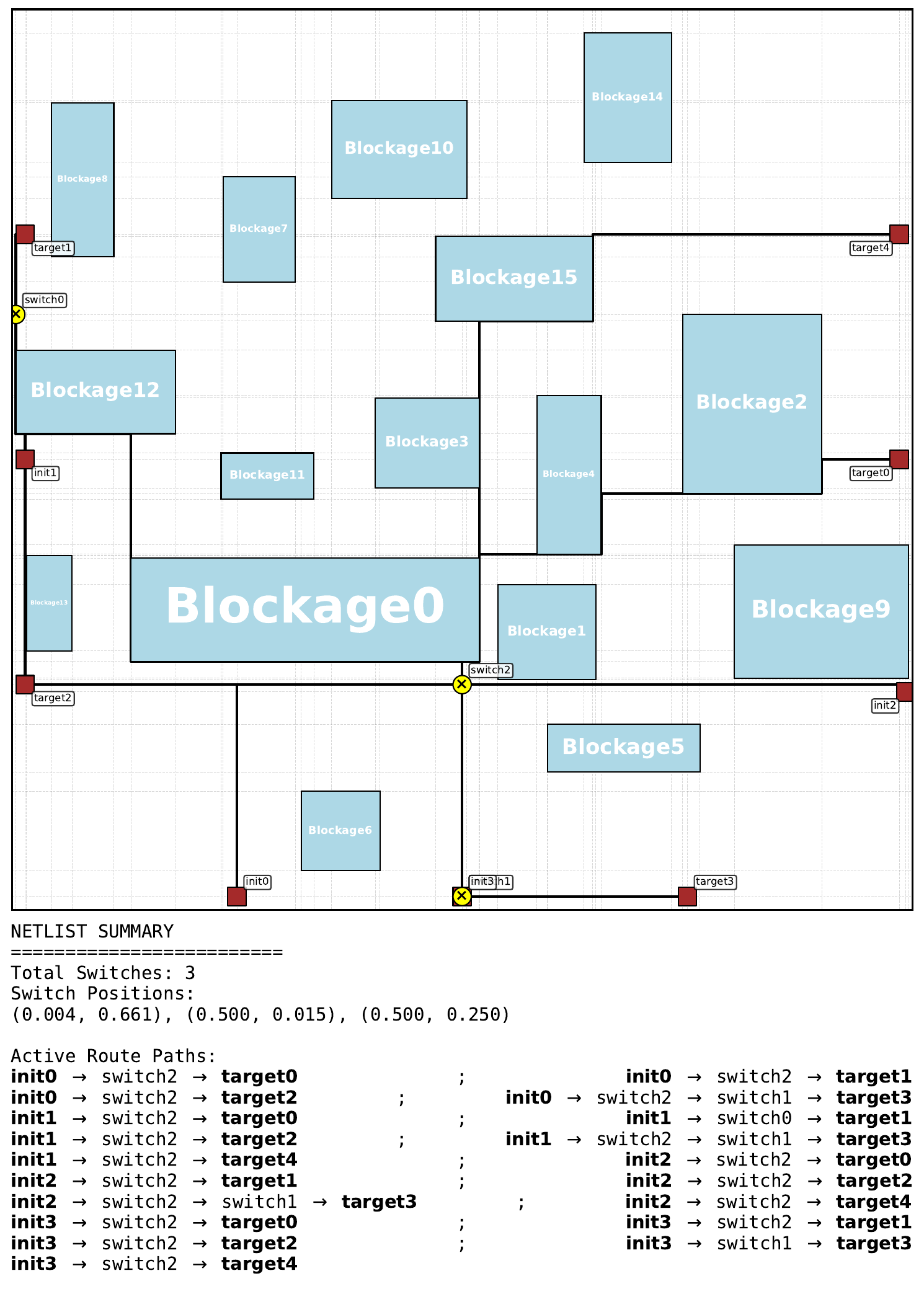}
        \caption*{Genetic algorithm}
    \end{subfigure}
    \\[0.6em]
    \begin{subfigure}[t]{0.31\linewidth}
        \vspace{0pt}
        \centering
        \includegraphics[width=\linewidth]{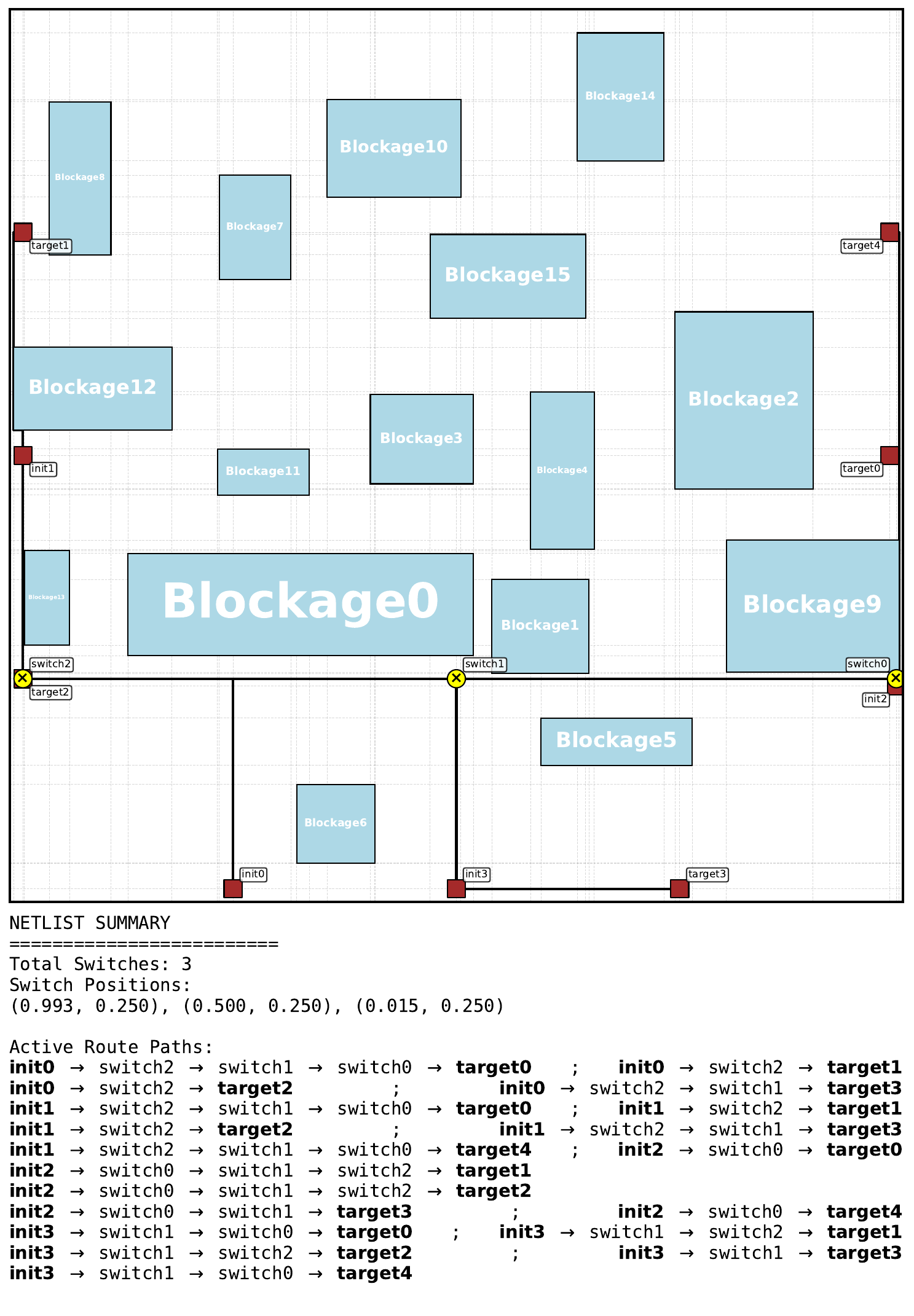}
        \caption*{PPO}
    \end{subfigure}
    \hspace{0.04\linewidth}
    \begin{subfigure}[t]{0.31\linewidth}
        \vspace{0pt}
        \centering
        \includegraphics[width=\linewidth]{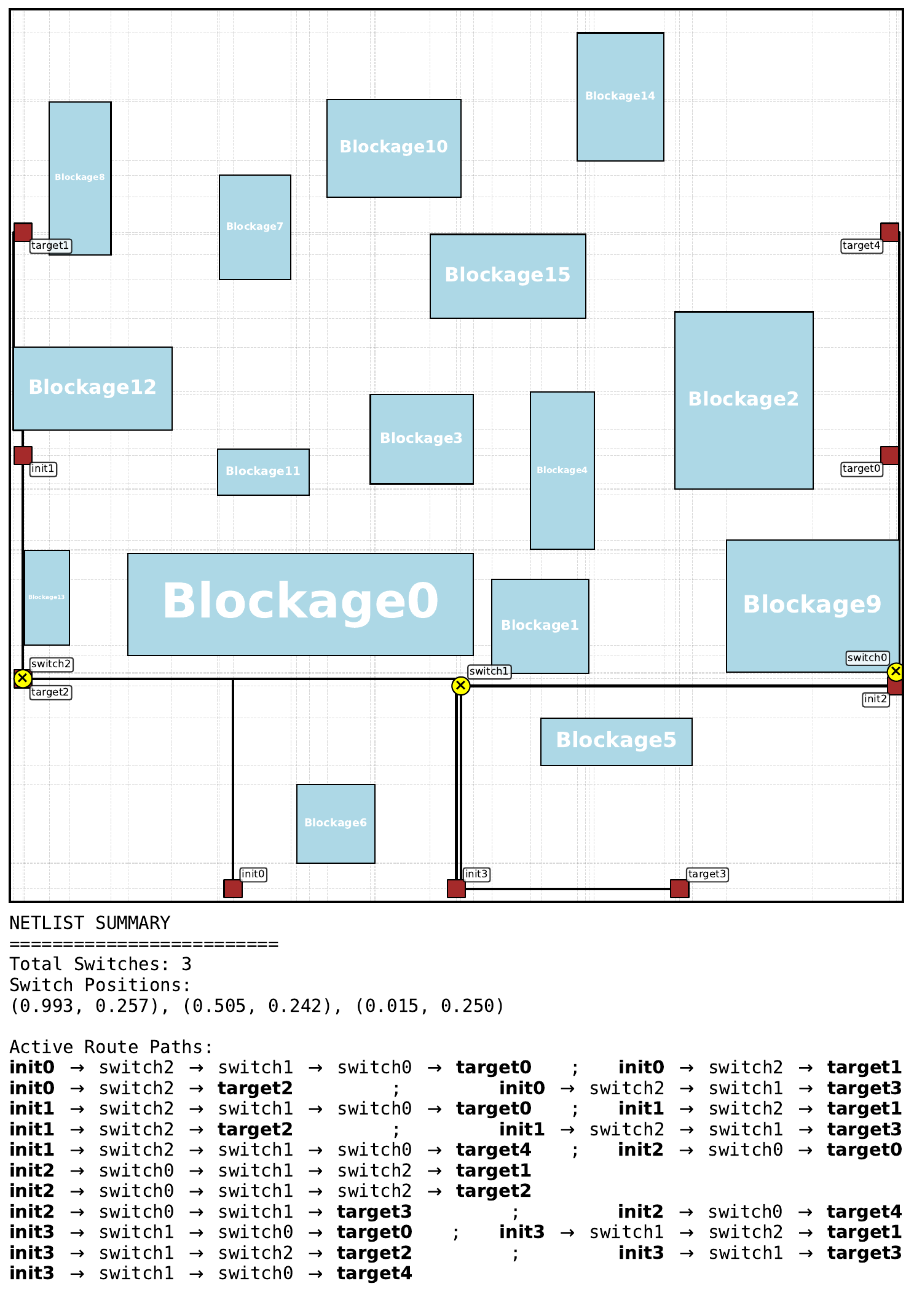}
        \caption*{MCTS}
    \end{subfigure}
\caption{Instance 12.}
\label{fig:best_pretrain_instance_12}
\end{figure*}

\begin{figure*}[h]
\centering
    \begin{subfigure}[t]{0.31\linewidth}
        \vspace{0pt}
        \centering
        \includegraphics[width=\linewidth]{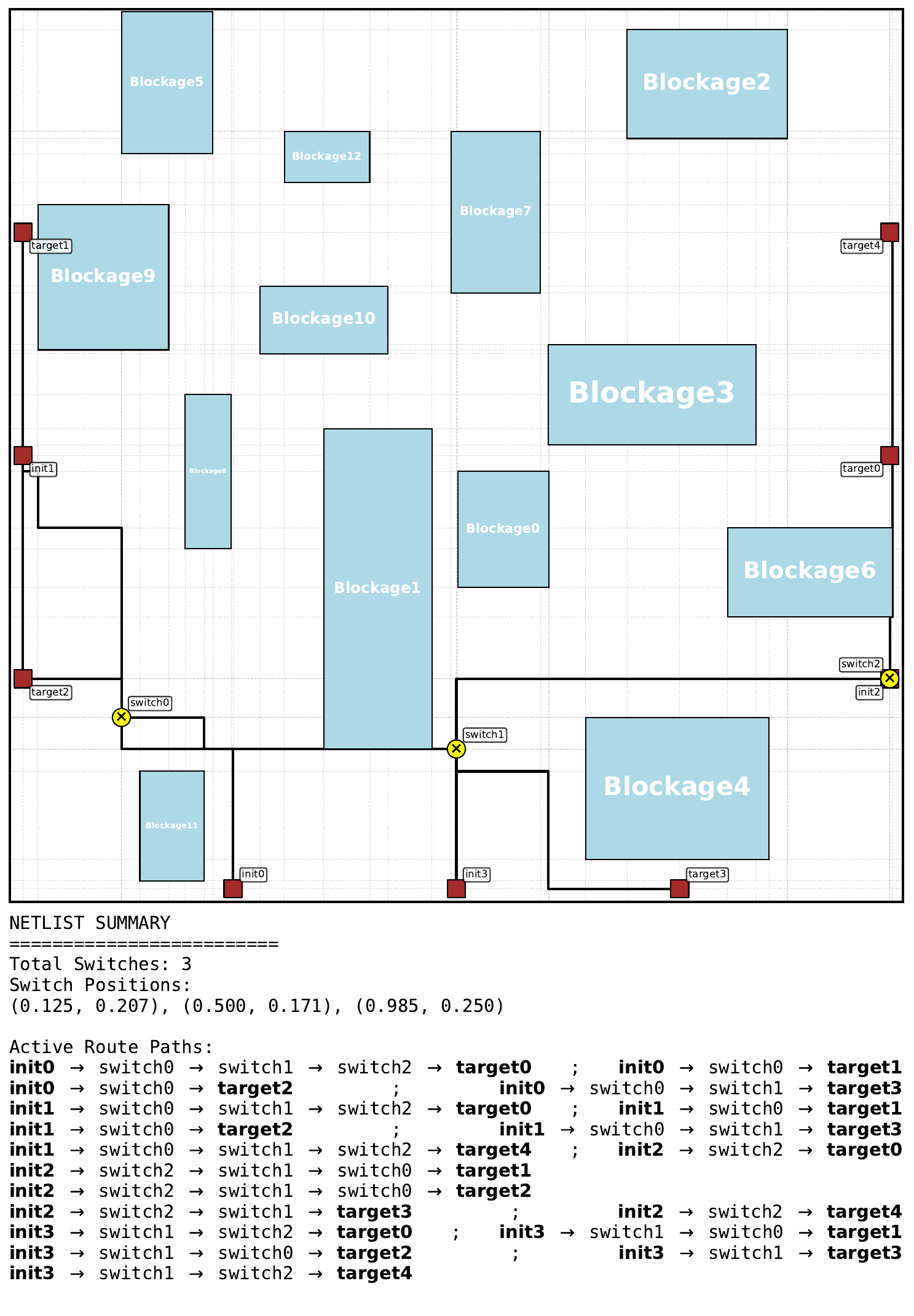}
        \caption*{Heuristic}
    \end{subfigure}
    \hfill
    \begin{subfigure}[t]{0.31\linewidth}
        \vspace{0pt}
        \centering
        \includegraphics[width=\linewidth]{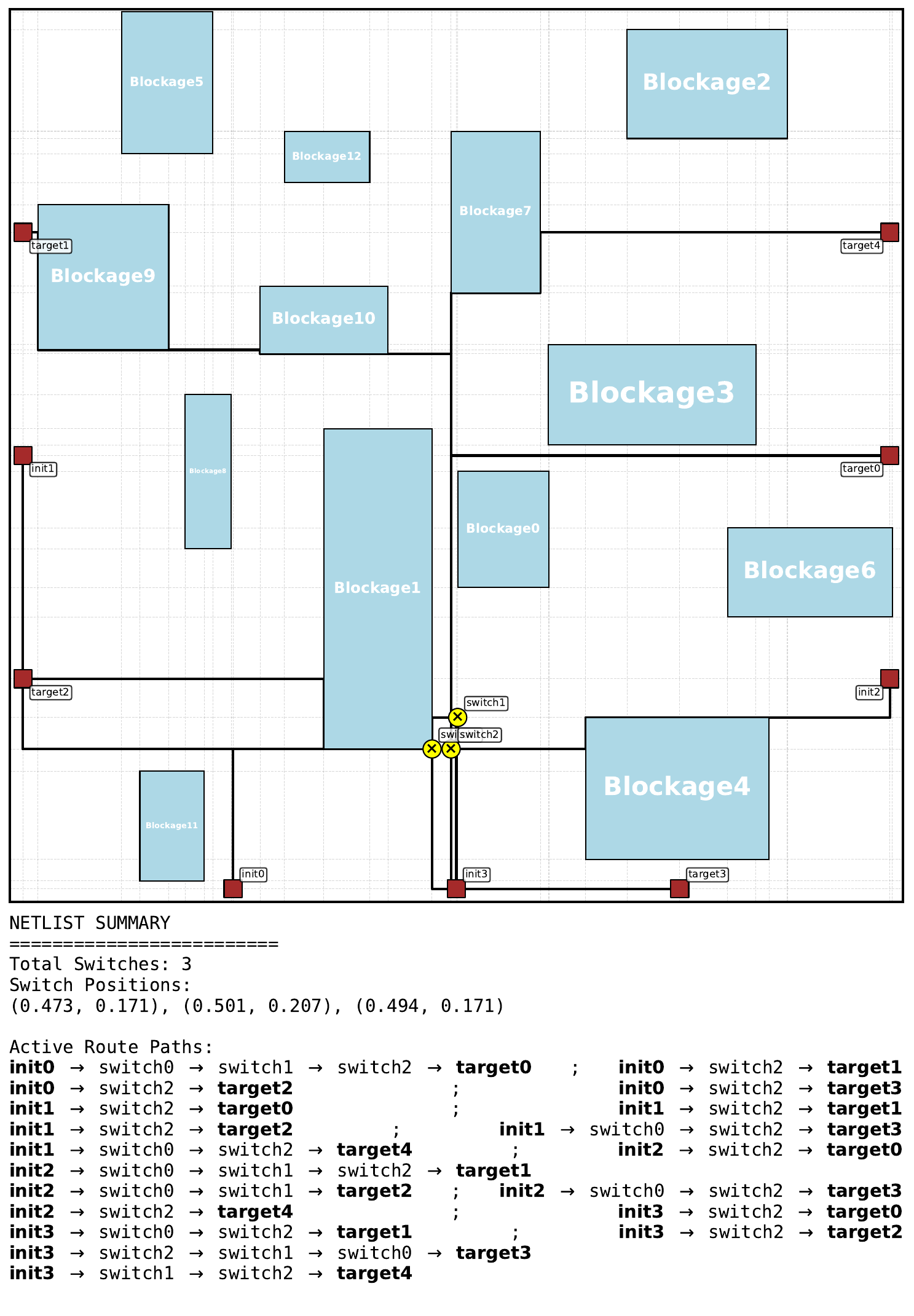}
        \caption*{Random search}
    \end{subfigure}
    \hfill
    \begin{subfigure}[t]{0.31\linewidth}
        \vspace{0pt}
        \centering
        \includegraphics[width=\linewidth]{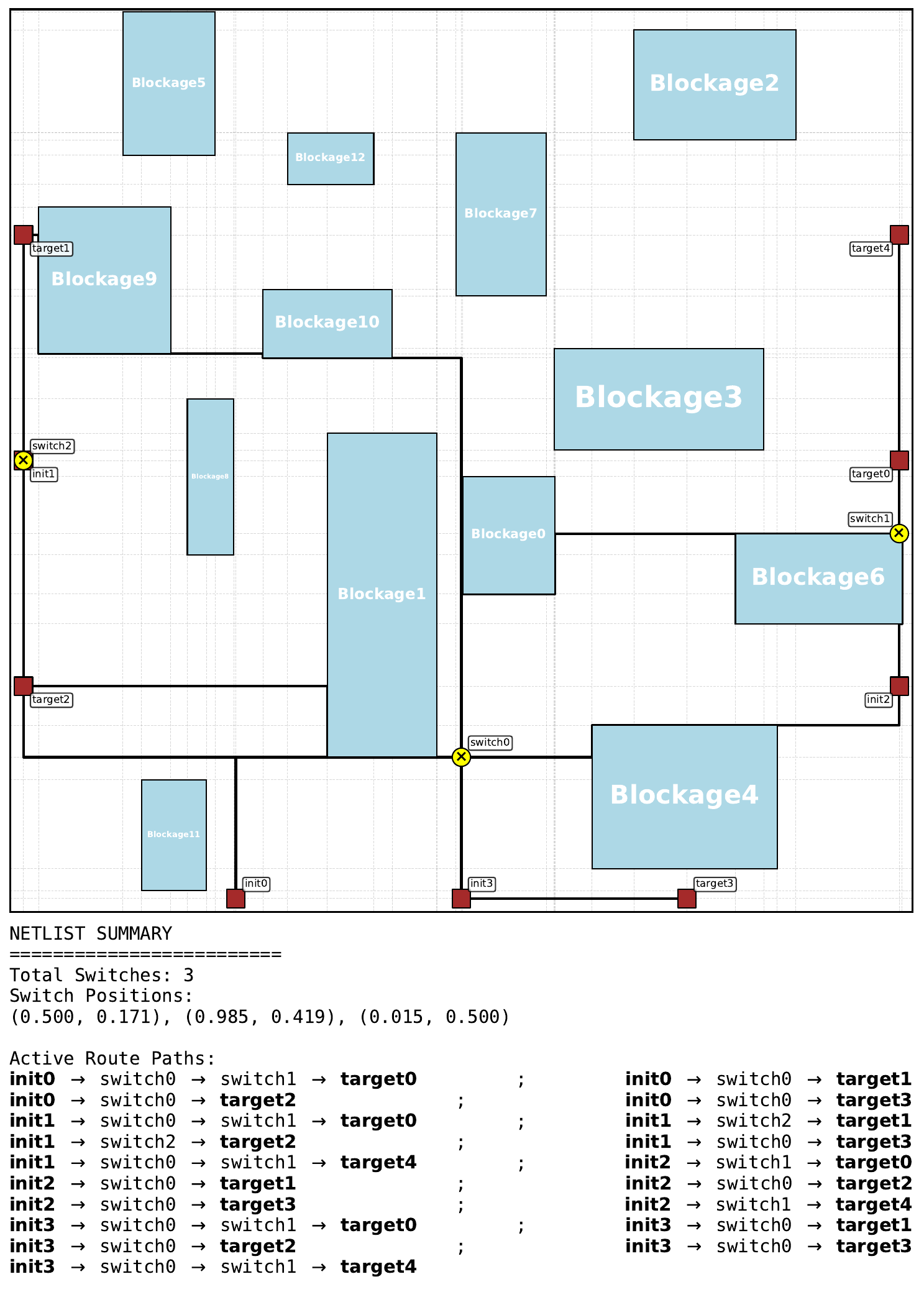}
        \caption*{Genetic algorithm}
    \end{subfigure}
    \\[0.6em]
    \begin{subfigure}[t]{0.31\linewidth}
        \vspace{0pt}
        \centering
        \includegraphics[width=\linewidth]{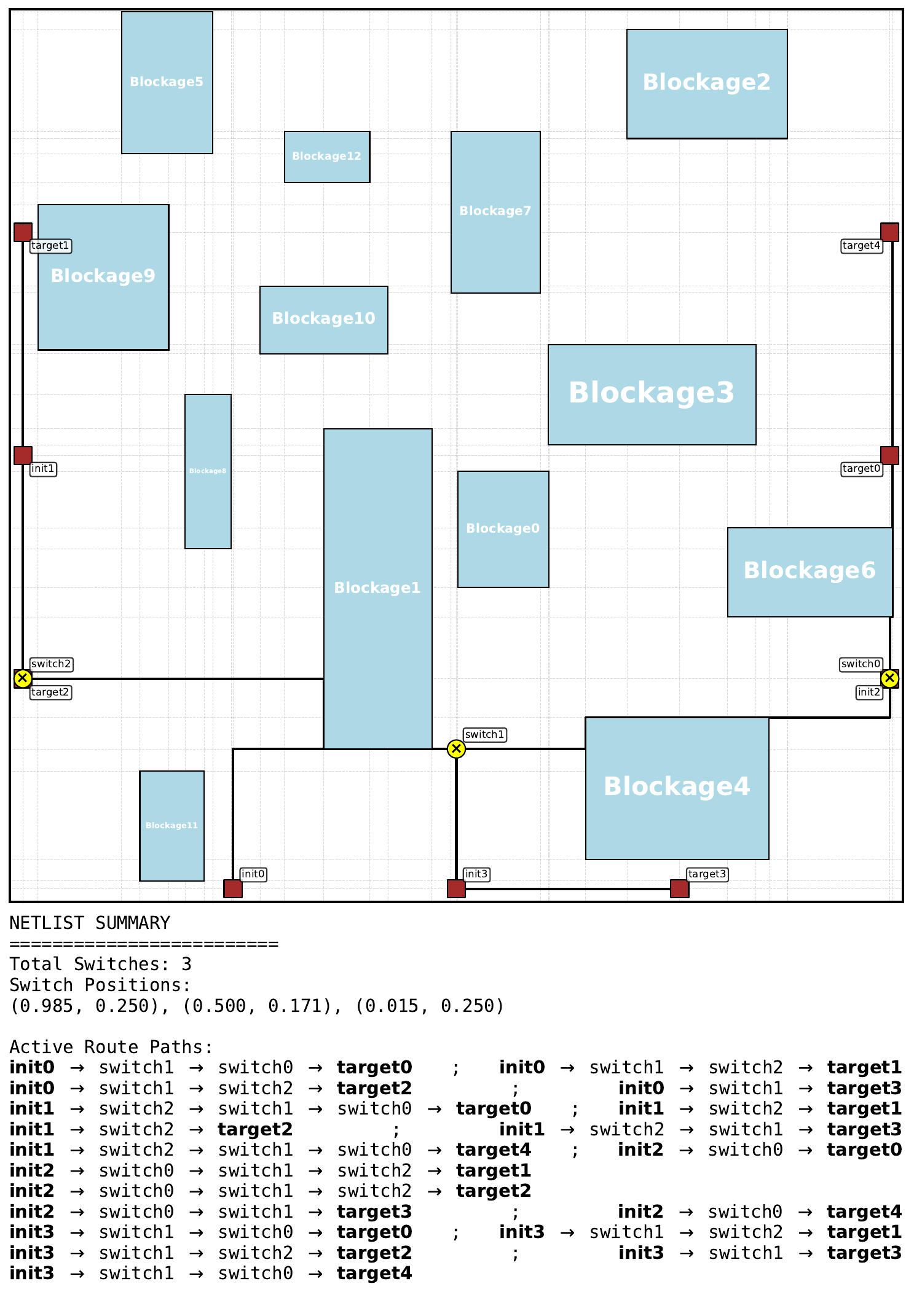}
        \caption*{PPO}
    \end{subfigure}
    \hspace{0.04\linewidth}
    \begin{subfigure}[t]{0.31\linewidth}
        \vspace{0pt}
        \centering
        \includegraphics[width=\linewidth]{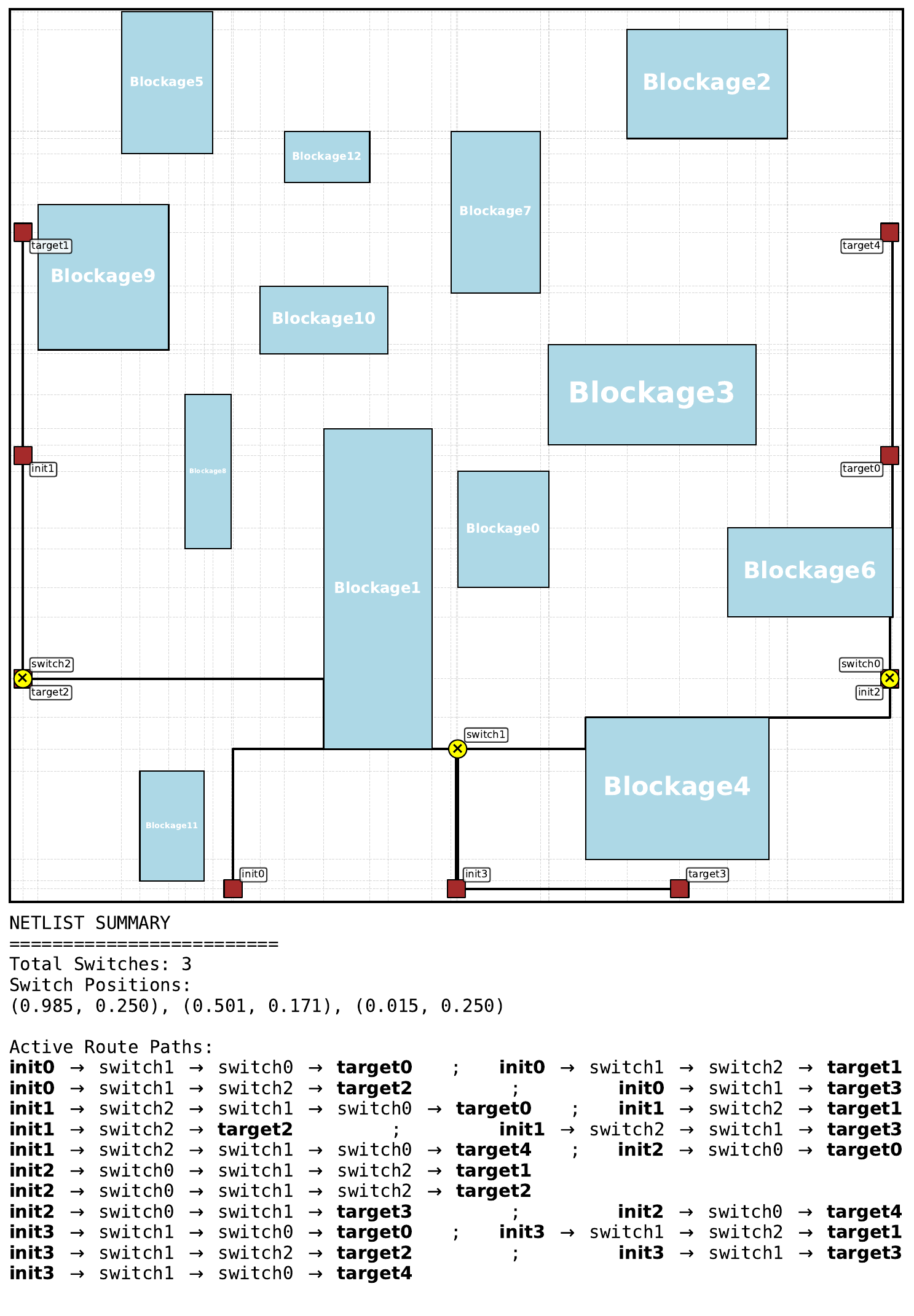}
        \caption*{MCTS}
    \end{subfigure}
\caption{Instance 13.}
\label{fig:best_pretrain_instance_13}
\end{figure*}

\begin{figure*}[h]
\centering
    \begin{subfigure}[t]{0.31\linewidth}
        \vspace{0pt}
        \centering
        \includegraphics[width=\linewidth]{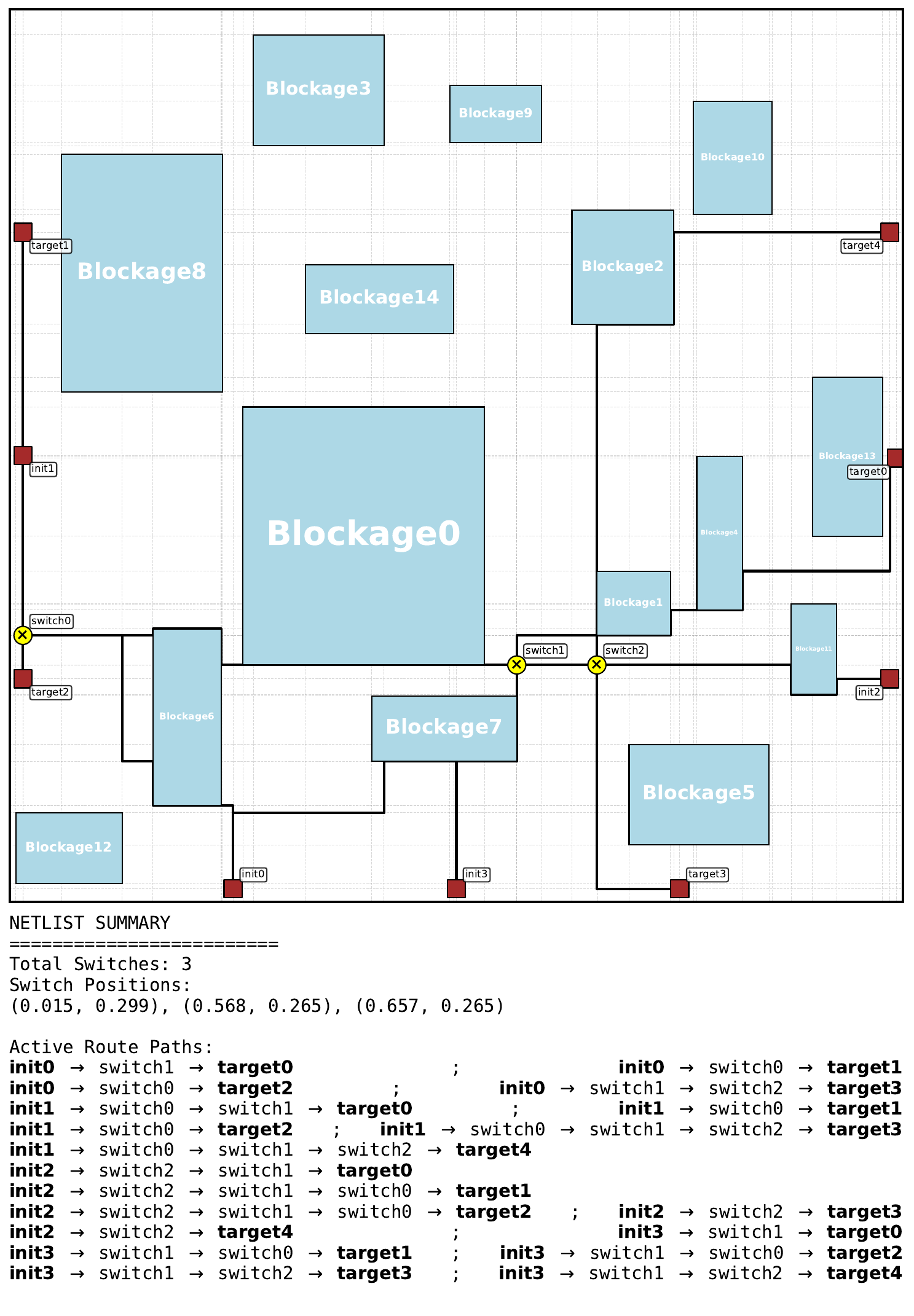}
        \caption*{Heuristic}
    \end{subfigure}
    \hfill
    \begin{subfigure}[t]{0.31\linewidth}
        \vspace{0pt}
        \centering
        \includegraphics[width=\linewidth]{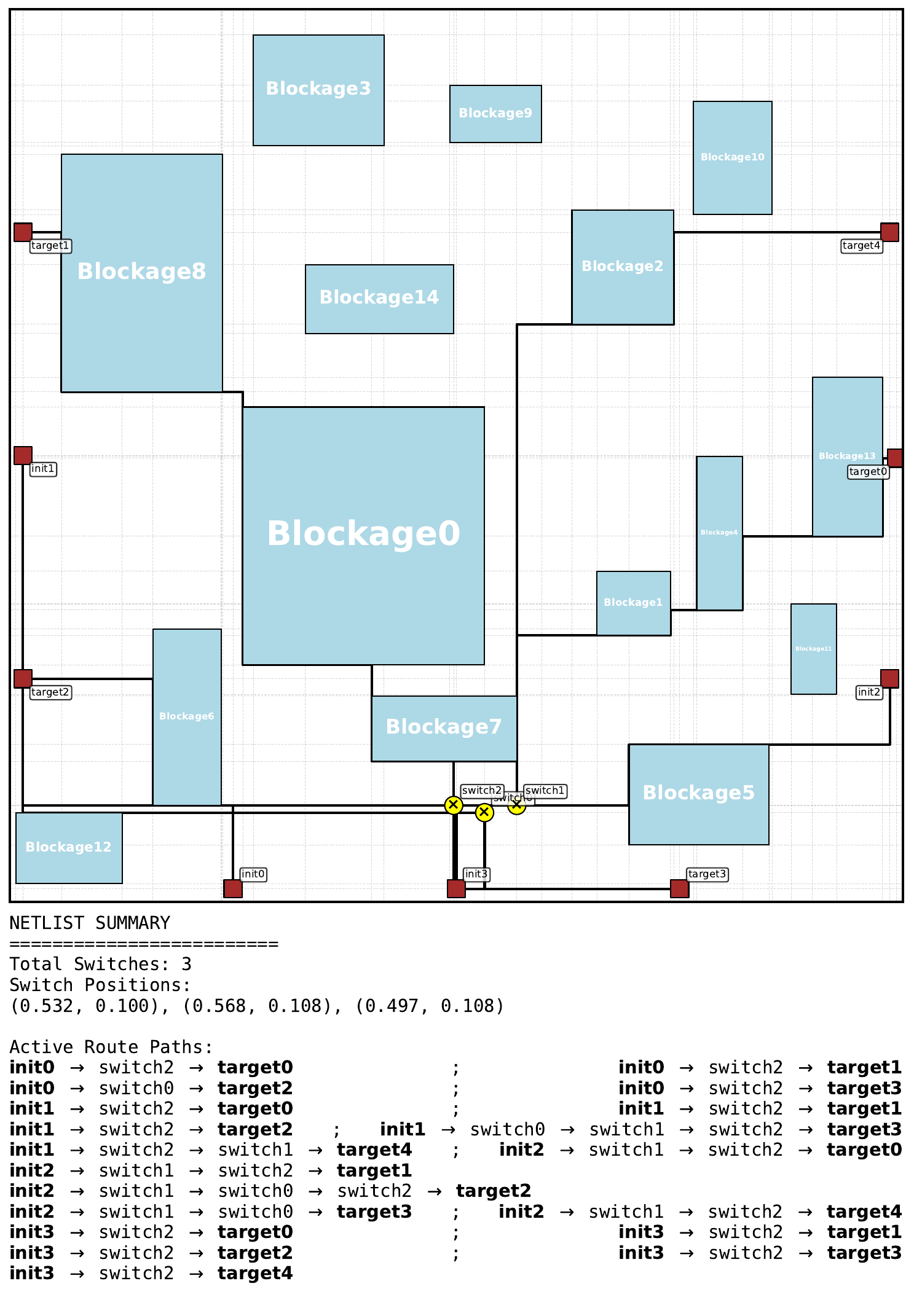}
        \caption*{Random search}
    \end{subfigure}
    \hfill
    \begin{subfigure}[t]{0.31\linewidth}
        \vspace{0pt}
        \centering
        \includegraphics[width=\linewidth]{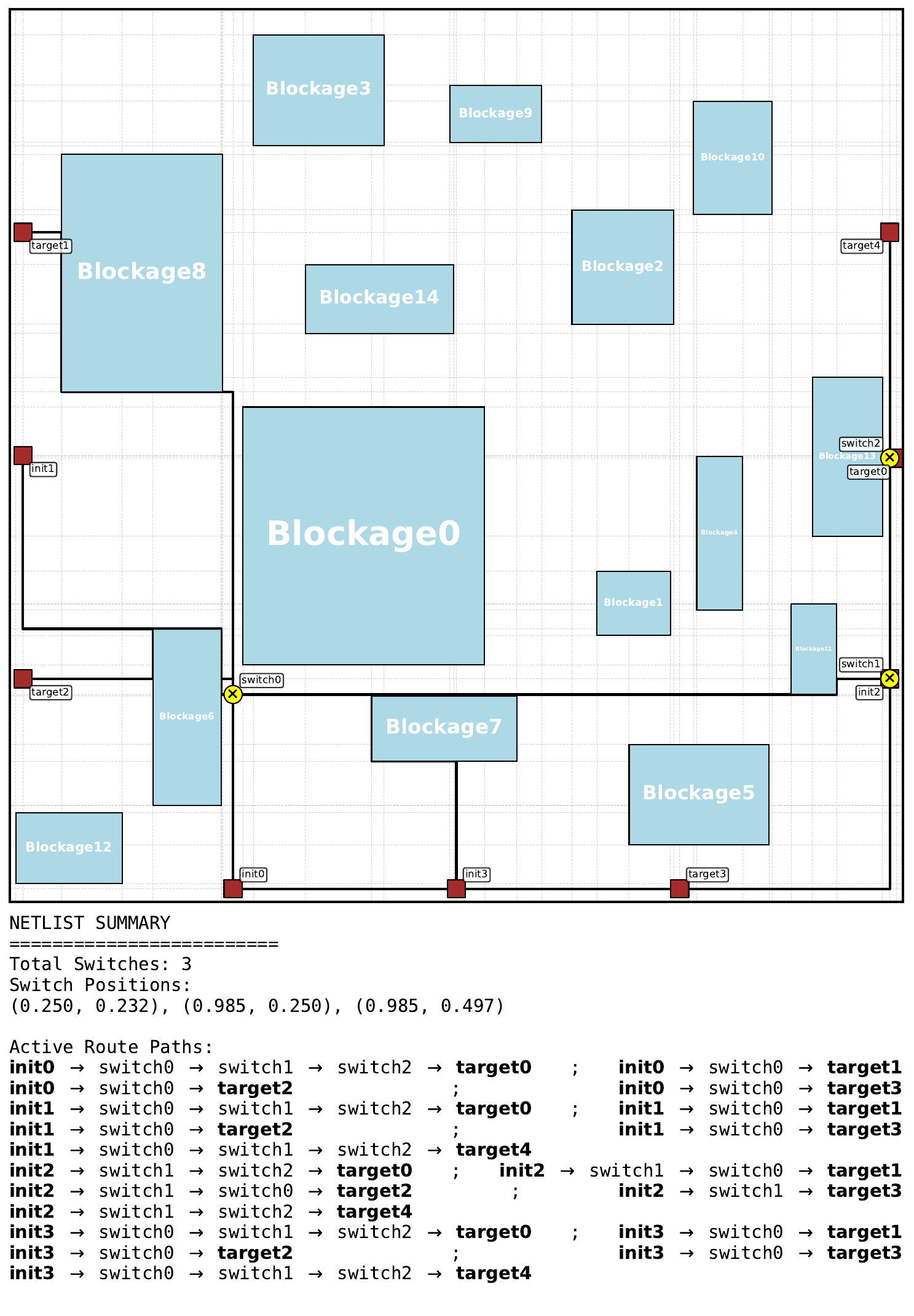}
        \caption*{Genetic algorithm}
    \end{subfigure}
    \\[0.6em]
    \begin{subfigure}[t]{0.31\linewidth}
        \vspace{0pt}
        \centering
        \includegraphics[width=\linewidth]{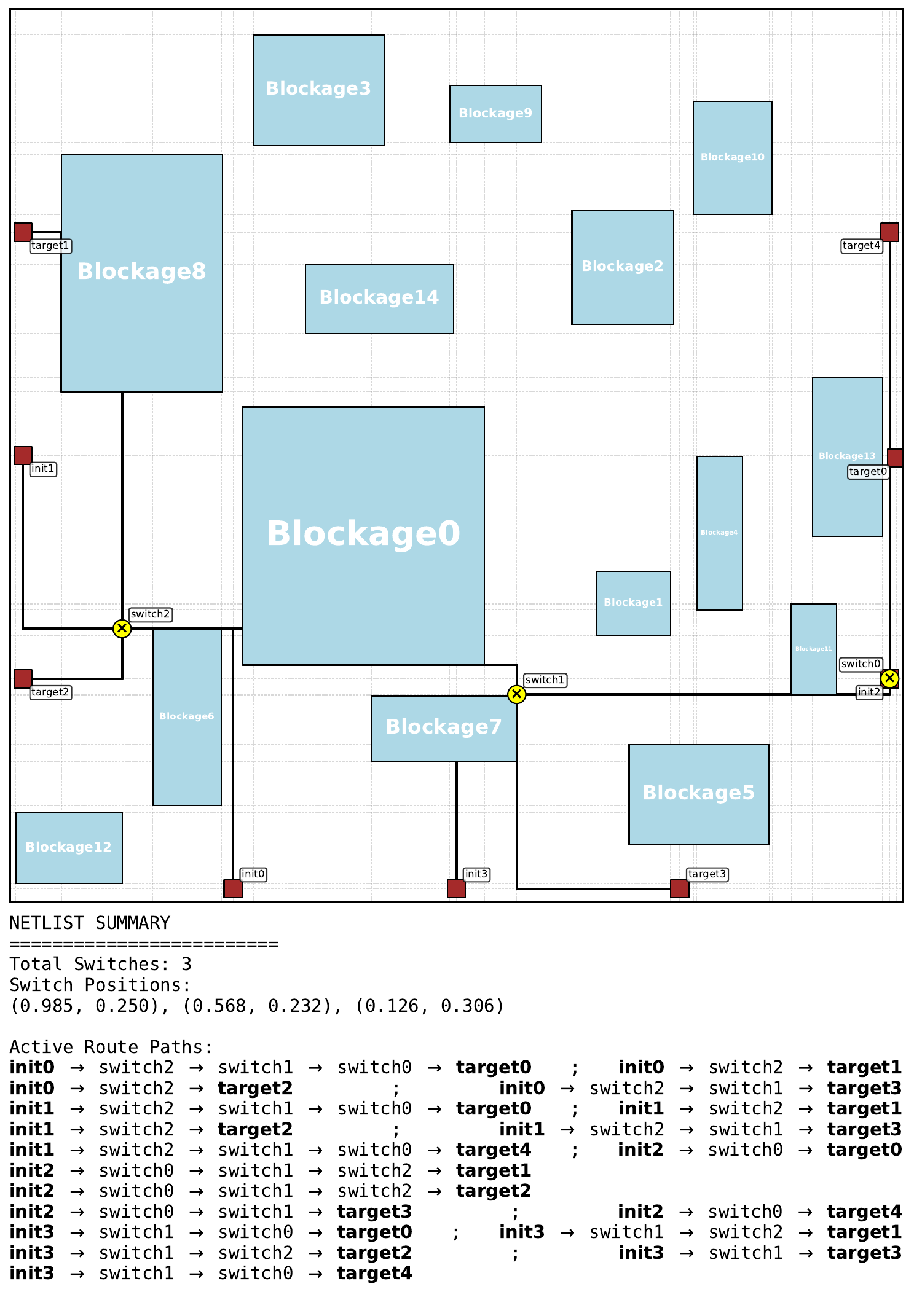}
        \caption*{PPO}
    \end{subfigure}
    \hspace{0.04\linewidth}
    \begin{subfigure}[t]{0.31\linewidth}
        \vspace{0pt}
        \centering
        \includegraphics[width=\linewidth]{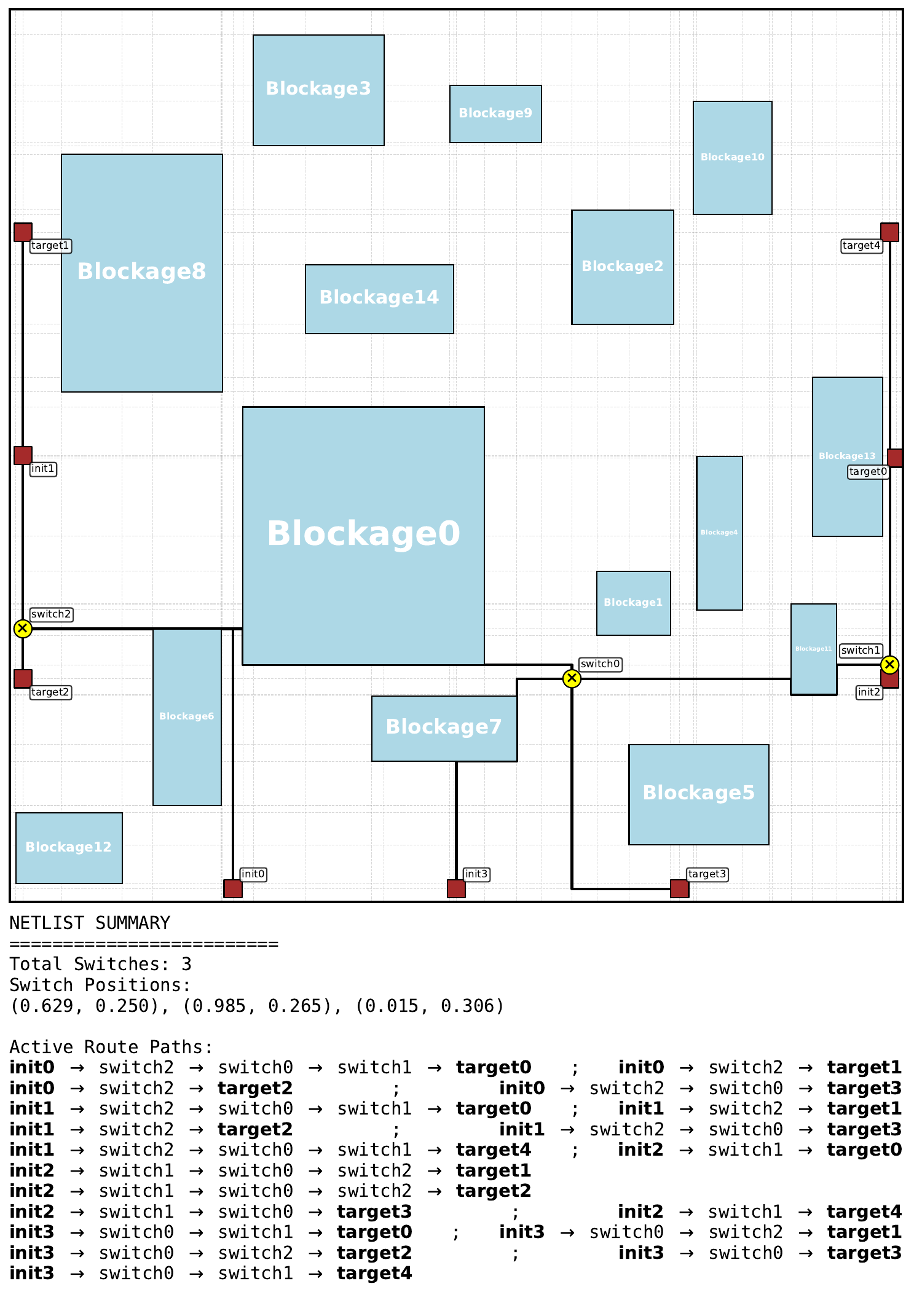}
        \caption*{MCTS}
    \end{subfigure}
\caption{Instance 14.}
\label{fig:best_pretrain_instance_14}
\end{figure*}

\begin{figure*}[h]
\centering
    \begin{subfigure}[t]{0.31\linewidth}
        \vspace{0pt}
        \centering
        \includegraphics[width=\linewidth]{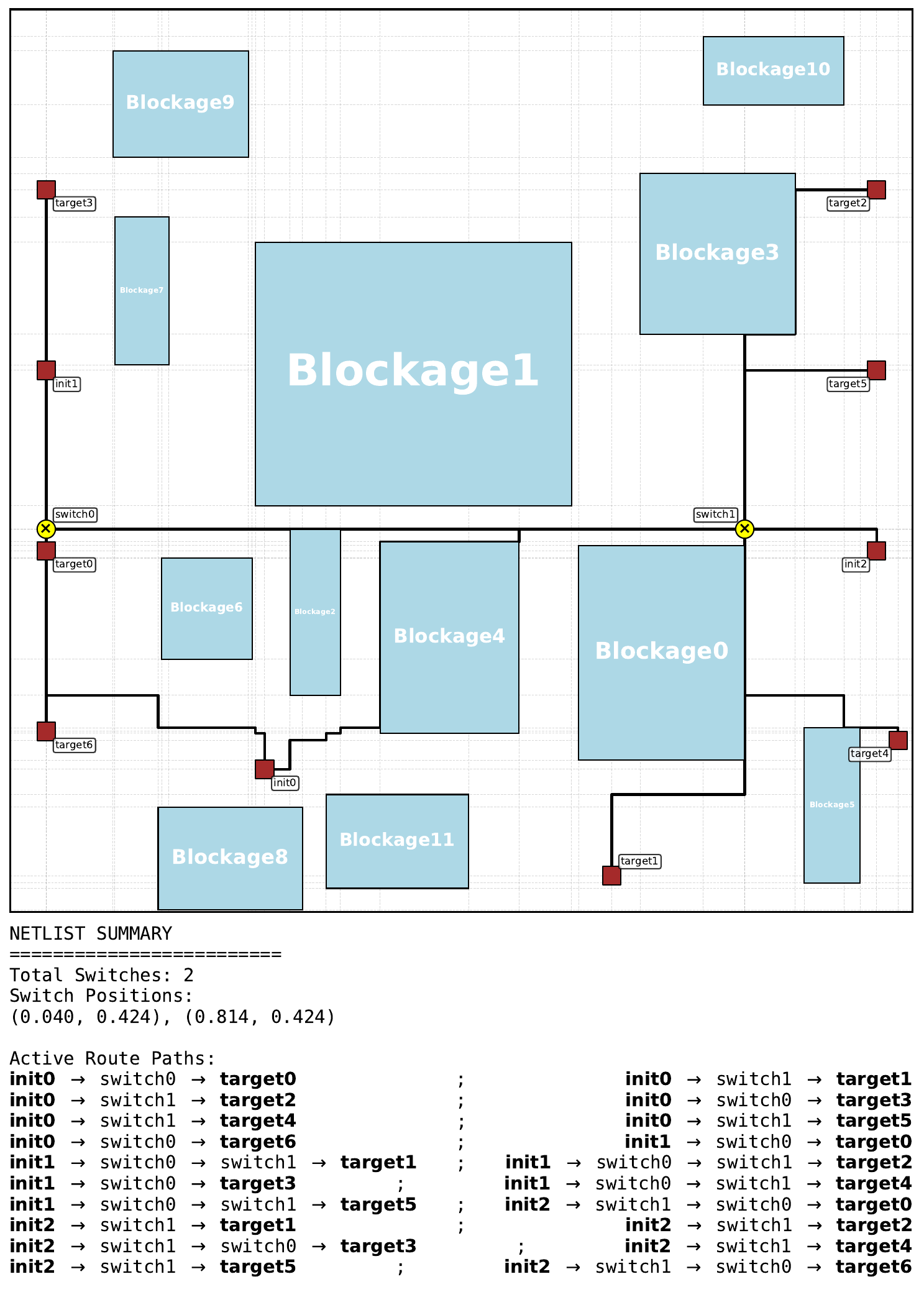}
        \caption*{Heuristic}
    \end{subfigure}
    \hfill
    \begin{subfigure}[t]{0.31\linewidth}
        \vspace{0pt}
        \centering
        \includegraphics[width=\linewidth]{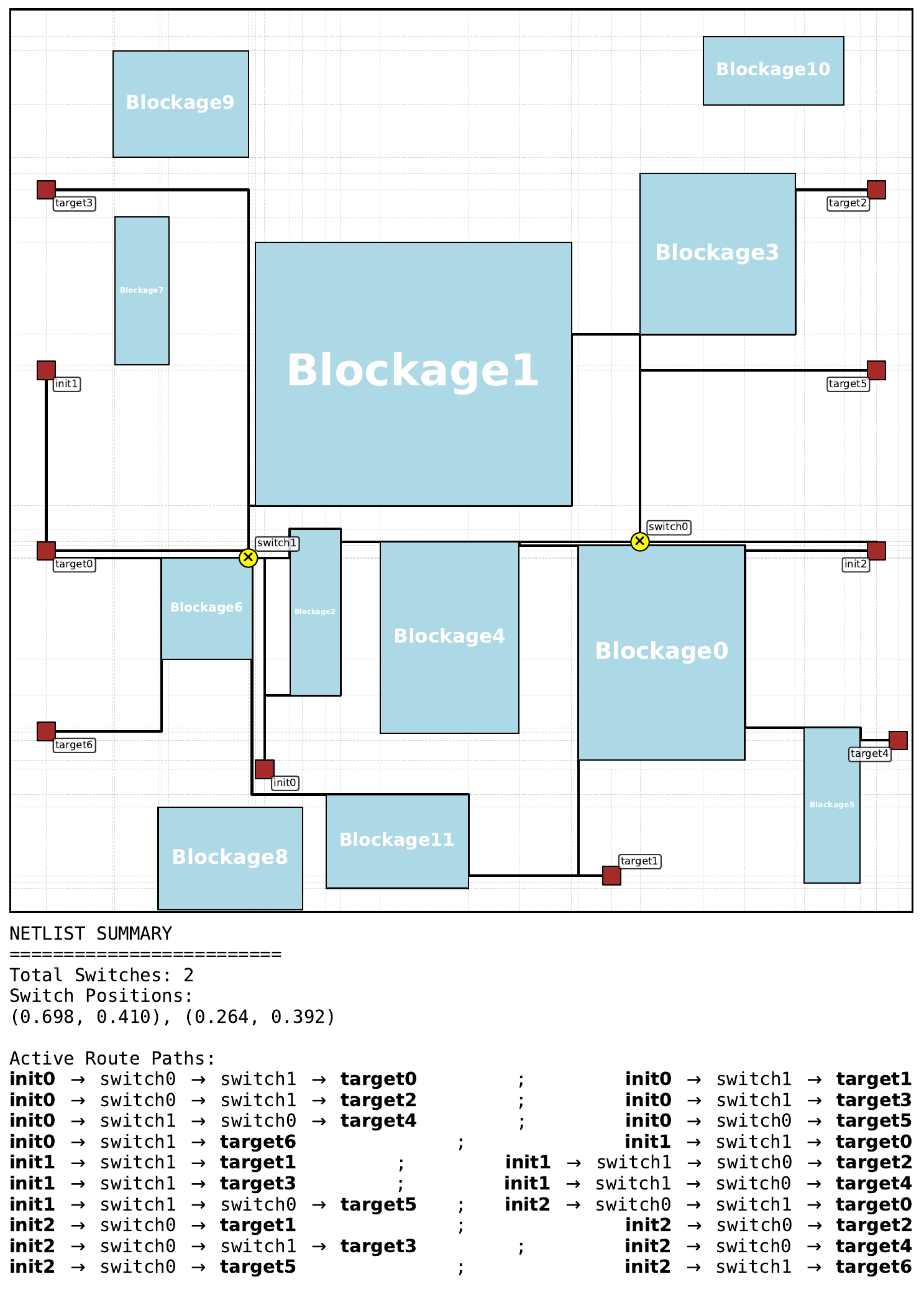}
        \caption*{Random search}
    \end{subfigure}
    \hfill
    \begin{subfigure}[t]{0.31\linewidth}
        \vspace{0pt}
        \centering
        \includegraphics[width=\linewidth]{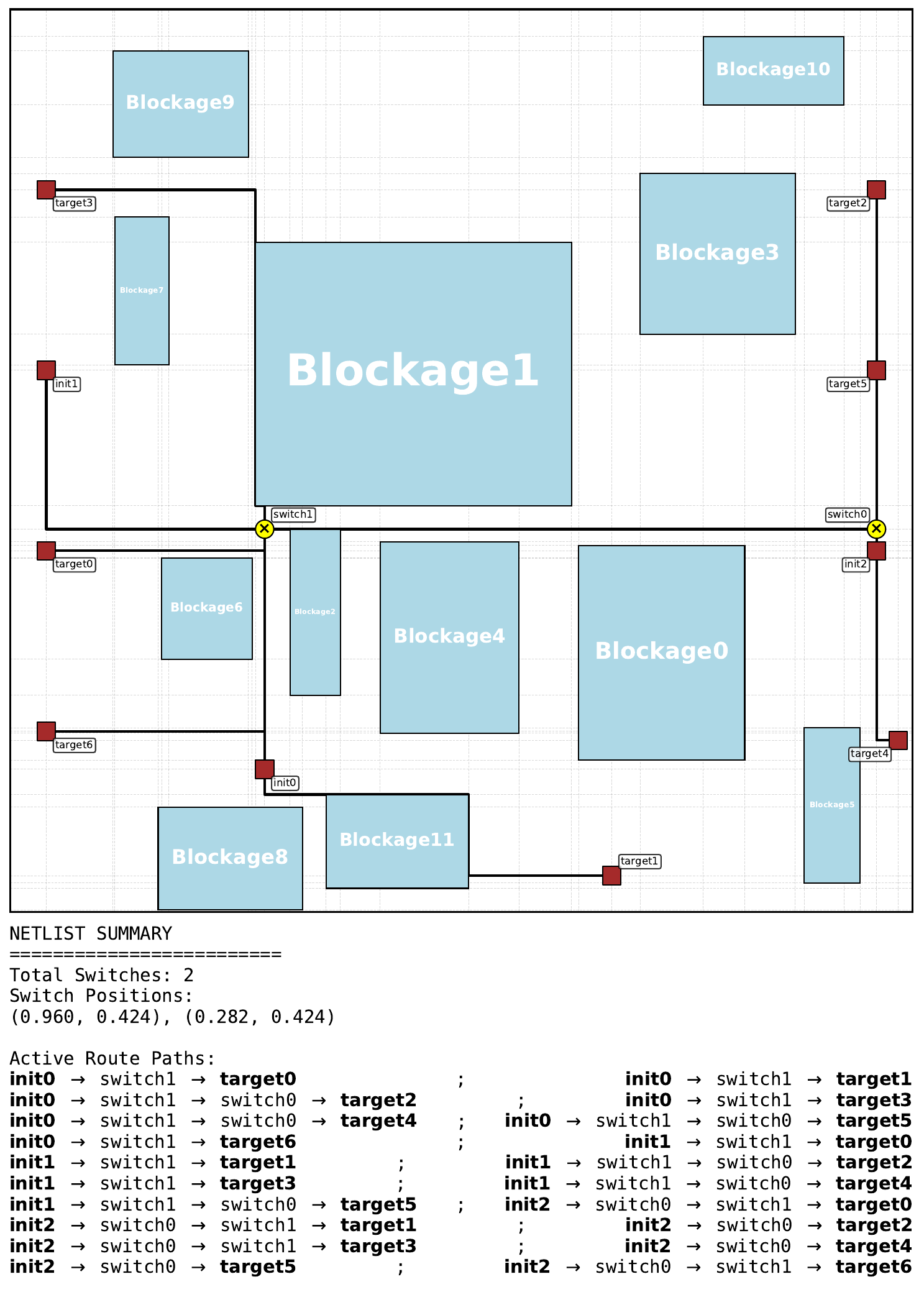}
        \caption*{Genetic algorithm}
    \end{subfigure}
    \\[0.6em]
    \begin{subfigure}[t]{0.31\linewidth}
        \vspace{0pt}
        \centering
        \includegraphics[width=\linewidth]{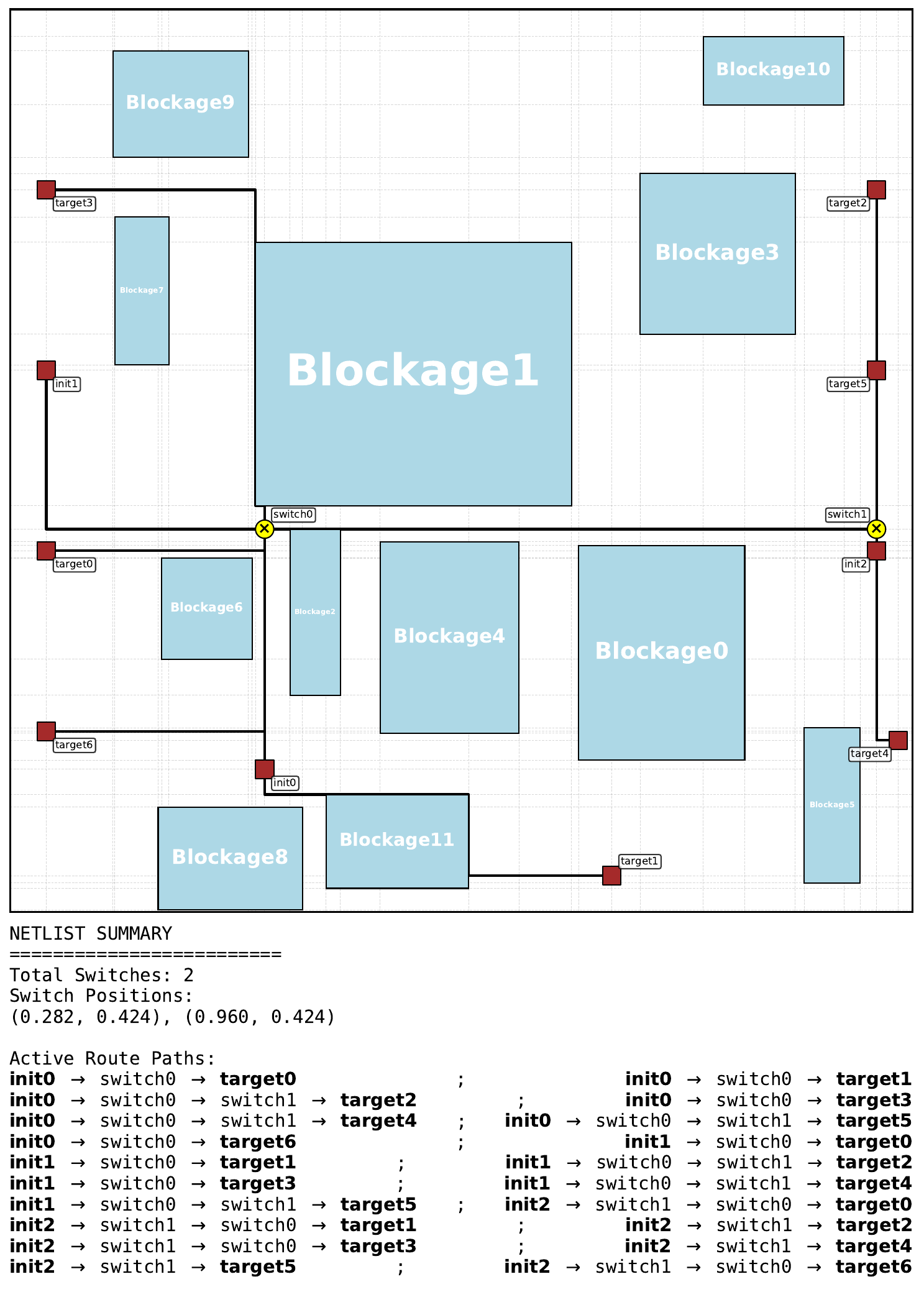}
        \caption*{PPO}
    \end{subfigure}
    \hspace{0.04\linewidth}
    \begin{subfigure}[t]{0.31\linewidth}
        \vspace{0pt}
        \centering
        \includegraphics[width=\linewidth]{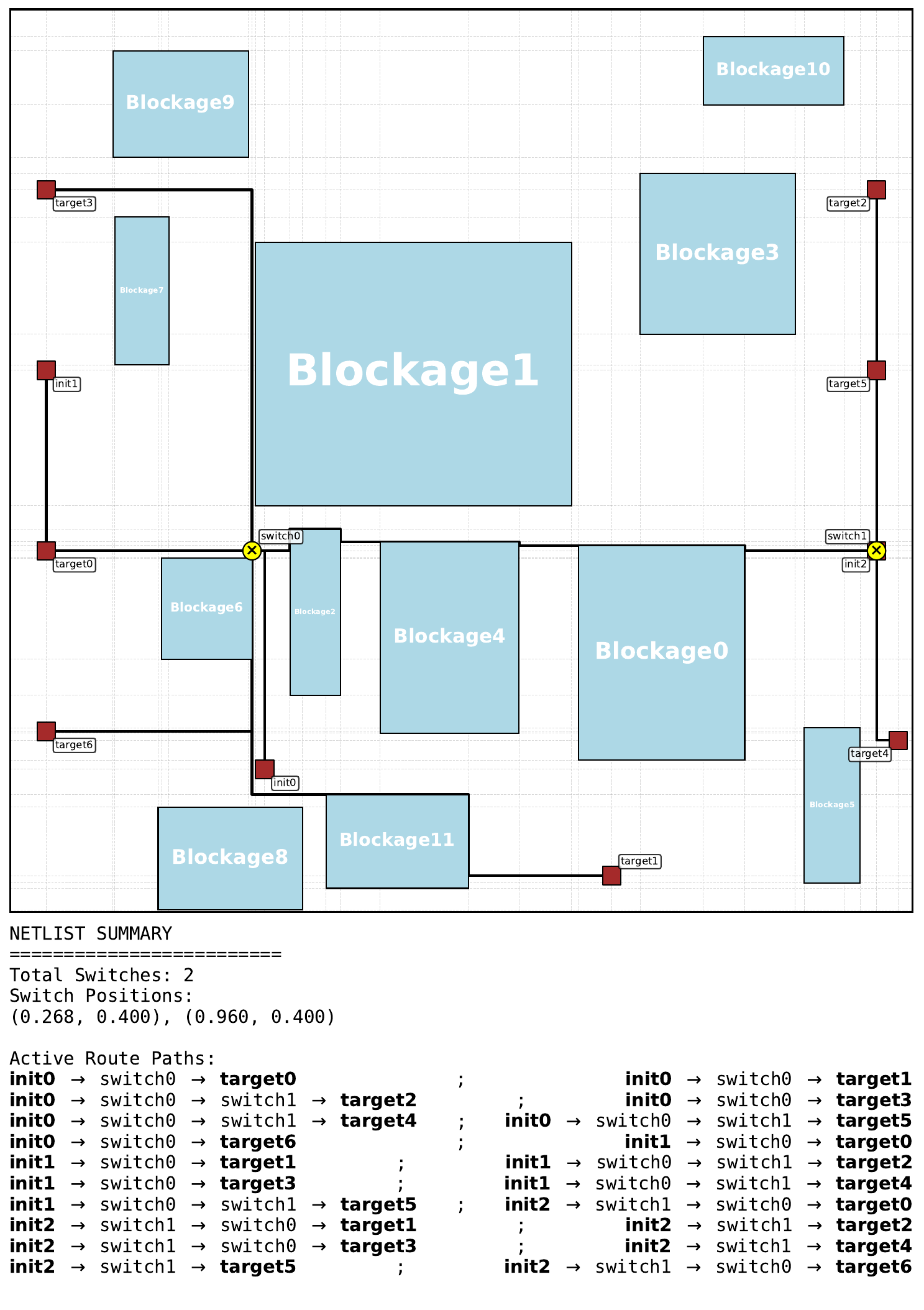}
        \caption*{MCTS}
    \end{subfigure}
\caption{Instance 15.}
\label{fig:best_pretrain_instance_15}
\end{figure*}

\begin{figure*}[h]
\centering
    \begin{subfigure}[t]{0.31\linewidth}
        \vspace{0pt}
        \centering
        \includegraphics[width=\linewidth]{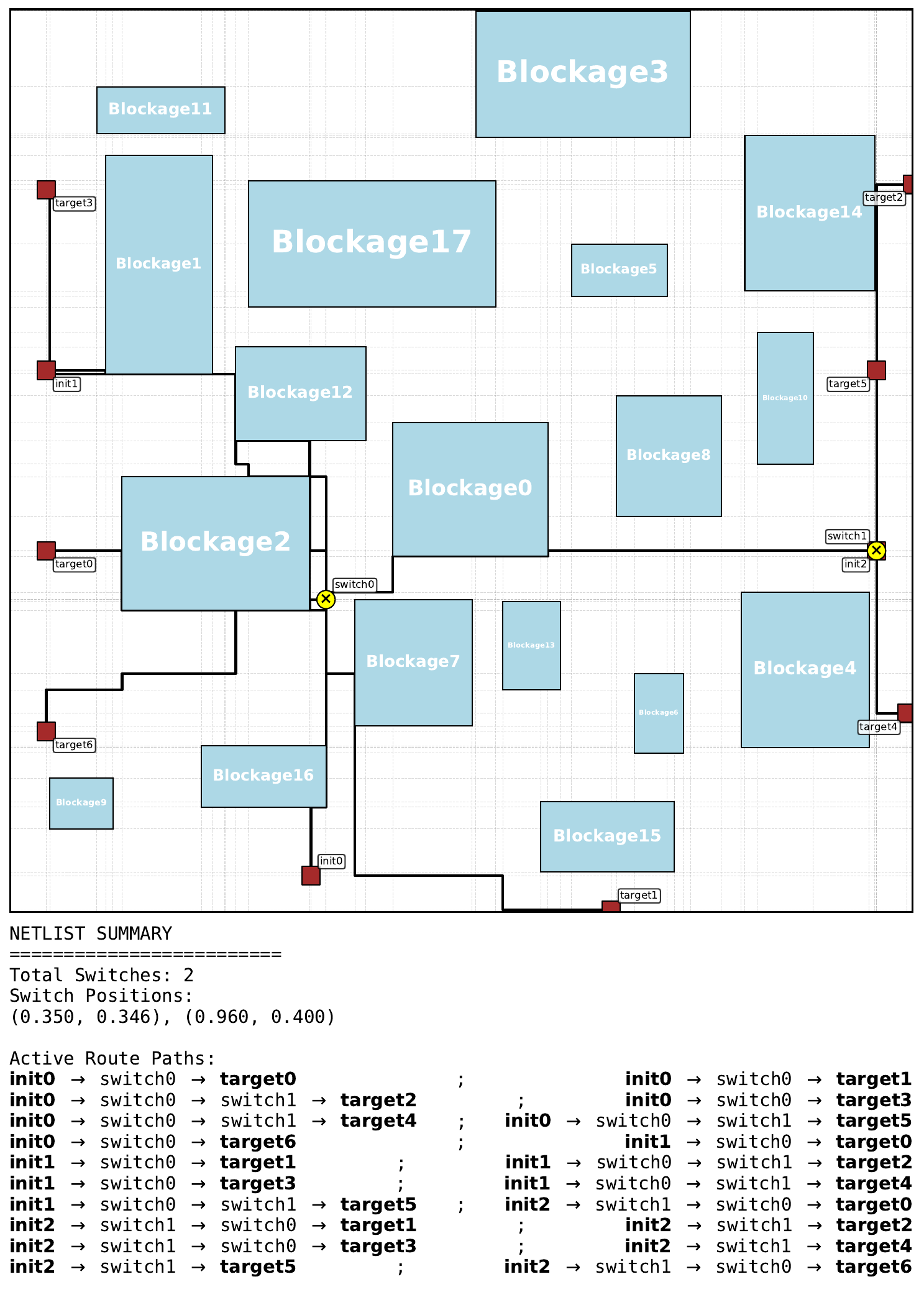}
        \caption*{Heuristic}
    \end{subfigure}
    \hfill
    \begin{subfigure}[t]{0.31\linewidth}
        \vspace{0pt}
        \centering
        \includegraphics[width=\linewidth]{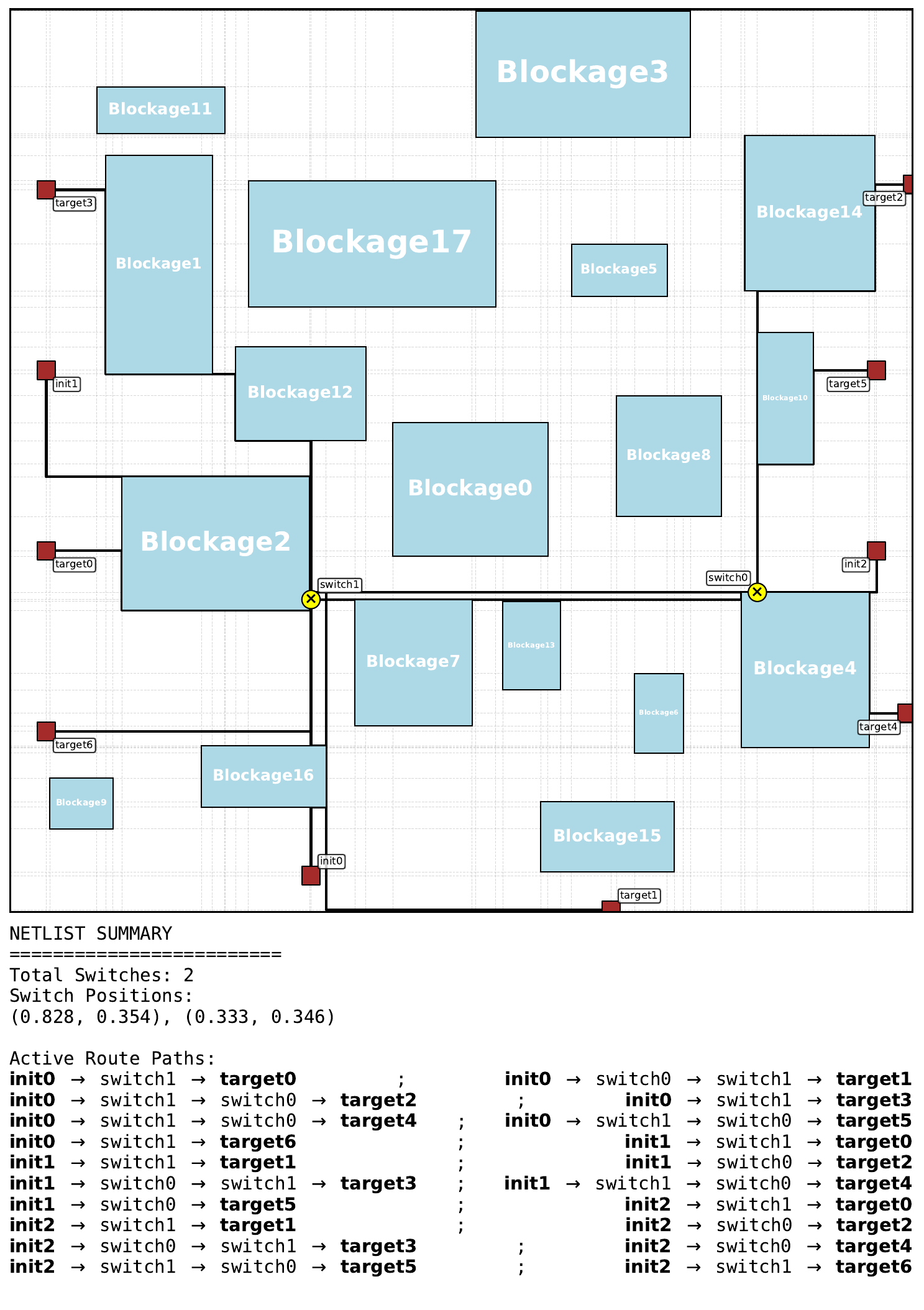}
        \caption*{Random search}
    \end{subfigure}
    \hfill
    \begin{subfigure}[t]{0.31\linewidth}
        \vspace{0pt}
        \centering
        \includegraphics[width=\linewidth]{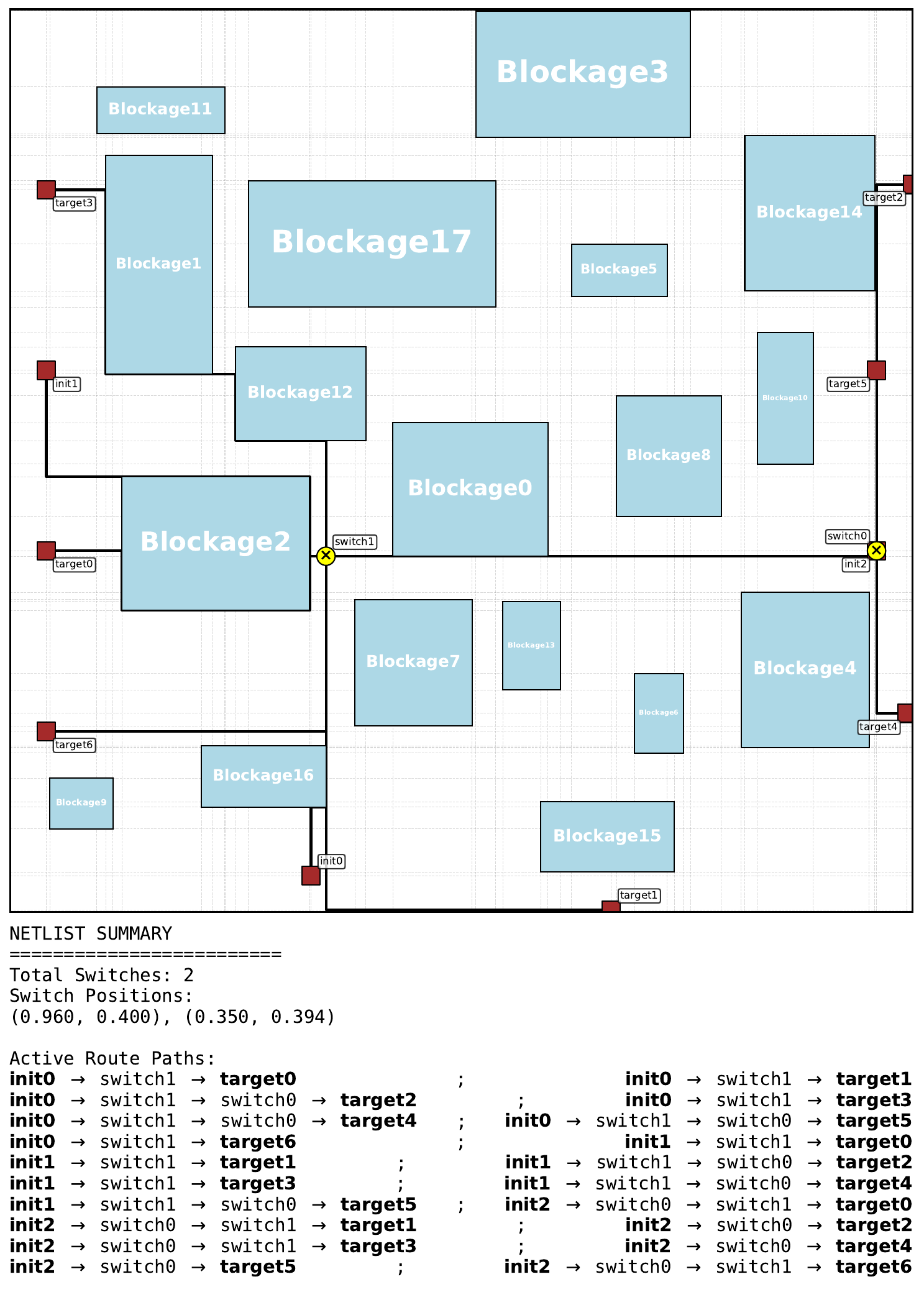}
        \caption*{Genetic algorithm}
    \end{subfigure}
    \\[0.6em]
    \begin{subfigure}[t]{0.31\linewidth}
        \vspace{0pt}
        \centering
        \includegraphics[width=\linewidth]{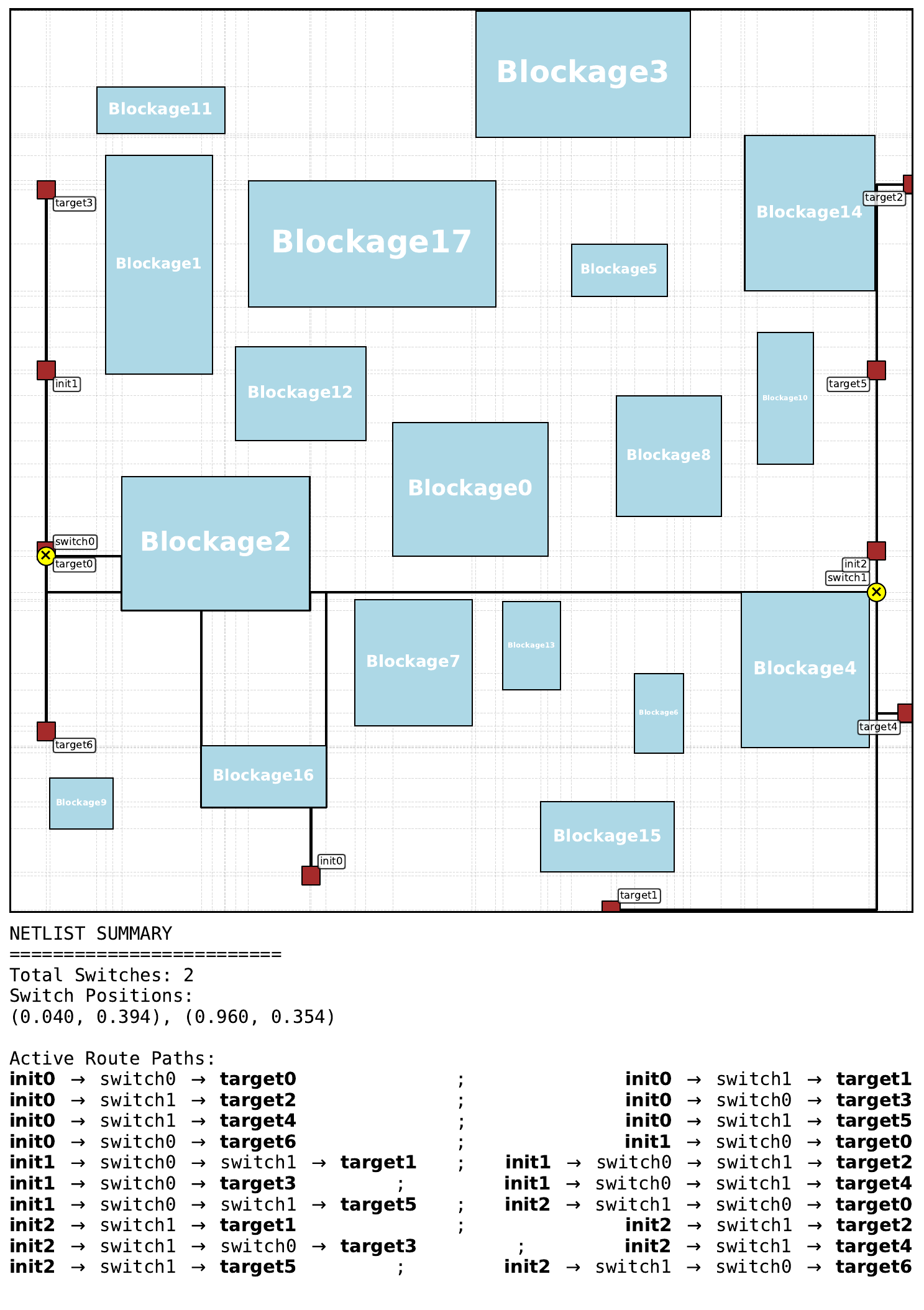}
        \caption*{PPO}
    \end{subfigure}
    \hspace{0.04\linewidth}
    \begin{subfigure}[t]{0.31\linewidth}
        \vspace{0pt}
        \centering
        \includegraphics[width=\linewidth]{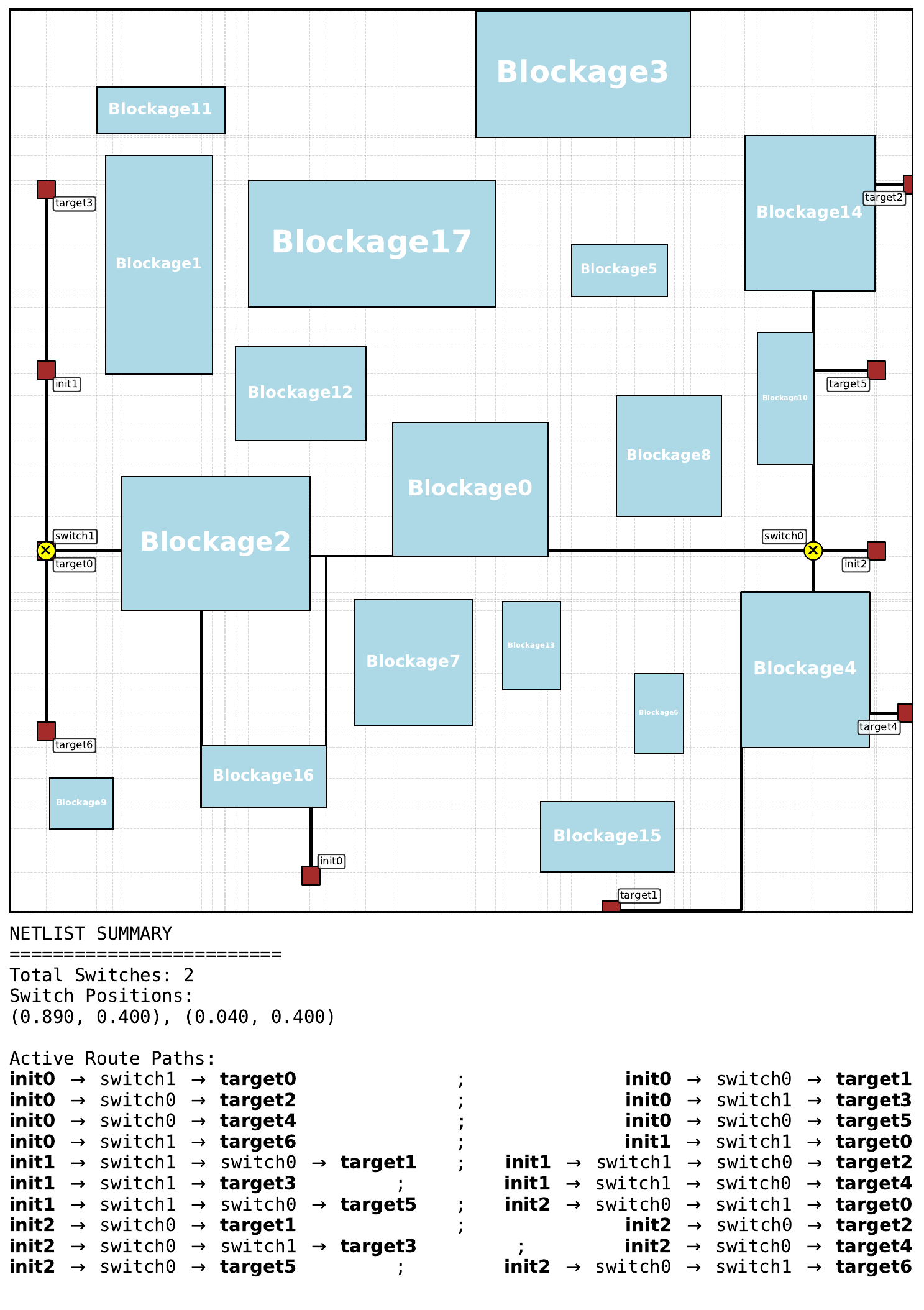}
        \caption*{MCTS}
    \end{subfigure}
\caption{Instance 16.}
\label{fig:best_pretrain_instance_16}
\end{figure*}

\begin{figure*}[h]
\centering
    \begin{subfigure}[t]{0.31\linewidth}
        \vspace{0pt}
        \centering
        \includegraphics[width=\linewidth]{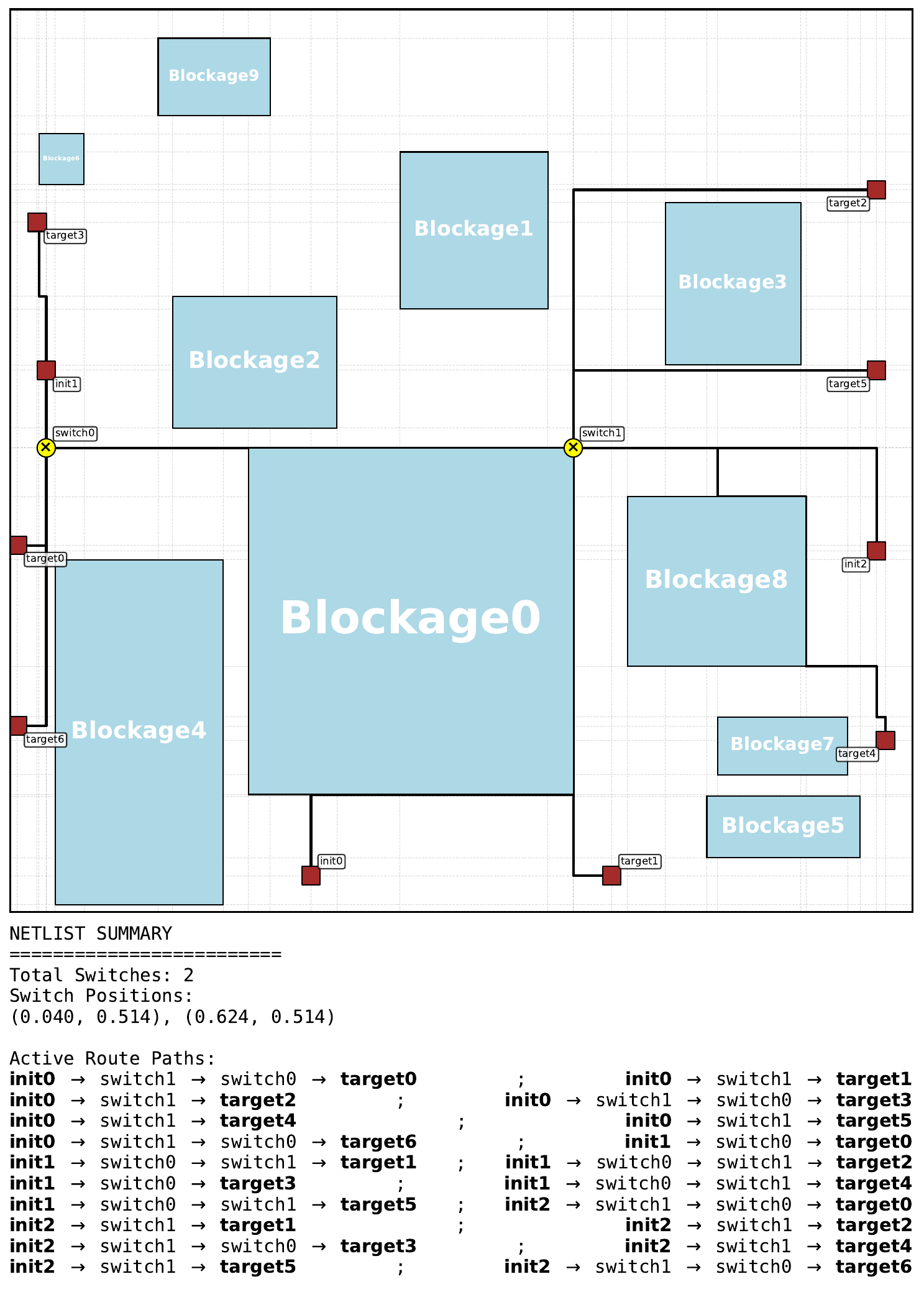}
        \caption*{Heuristic}
    \end{subfigure}
    \hfill
    \begin{subfigure}[t]{0.31\linewidth}
        \vspace{0pt}
        \centering
        \includegraphics[width=\linewidth]{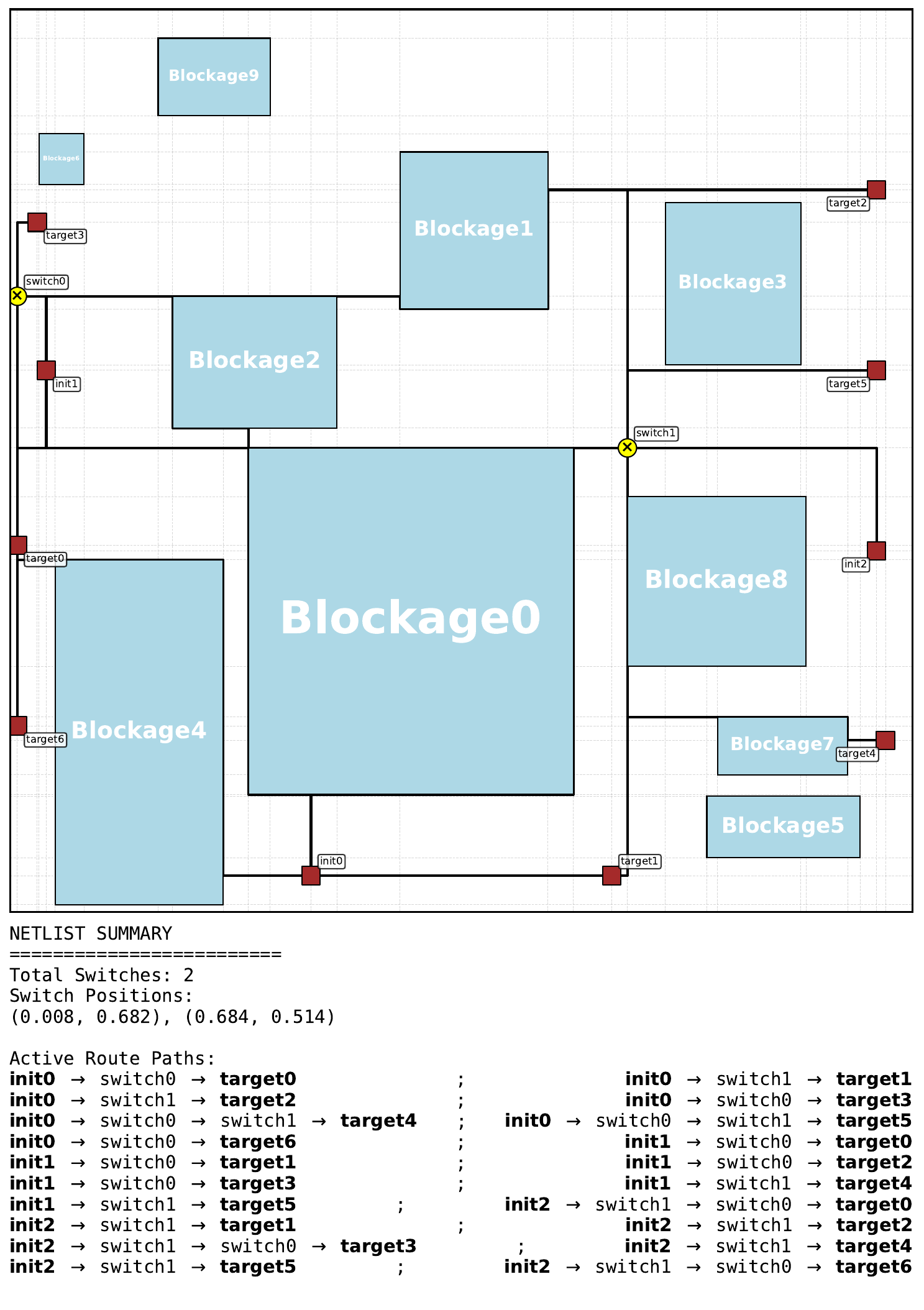}
        \caption*{Random search}
    \end{subfigure}
    \hfill
    \begin{subfigure}[t]{0.31\linewidth}
        \vspace{0pt}
        \centering
        \includegraphics[width=\linewidth]{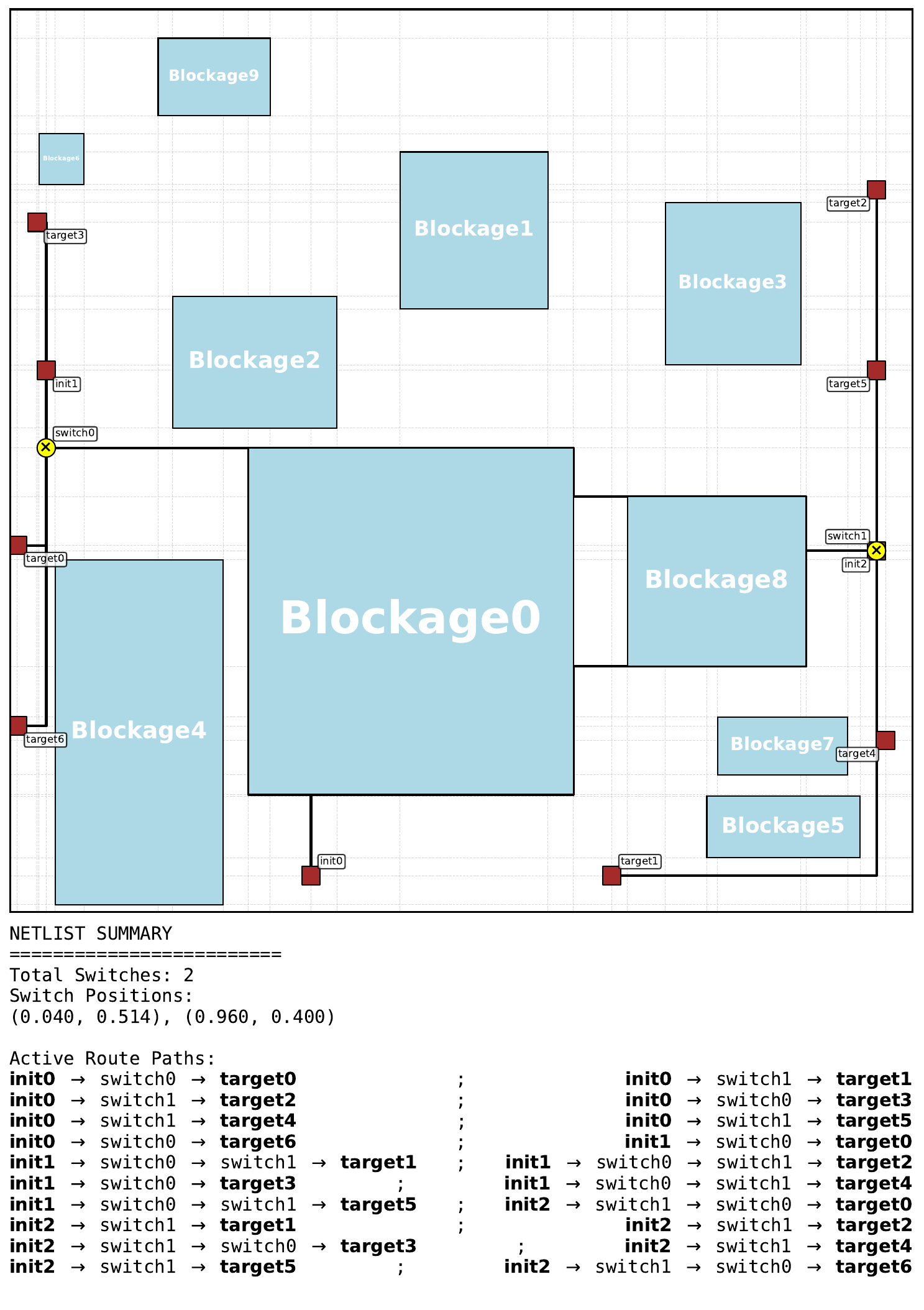}
        \caption*{Genetic algorithm}
    \end{subfigure}
    \\[0.6em]
    \begin{subfigure}[t]{0.31\linewidth}
        \vspace{0pt}
        \centering
        \includegraphics[width=\linewidth]{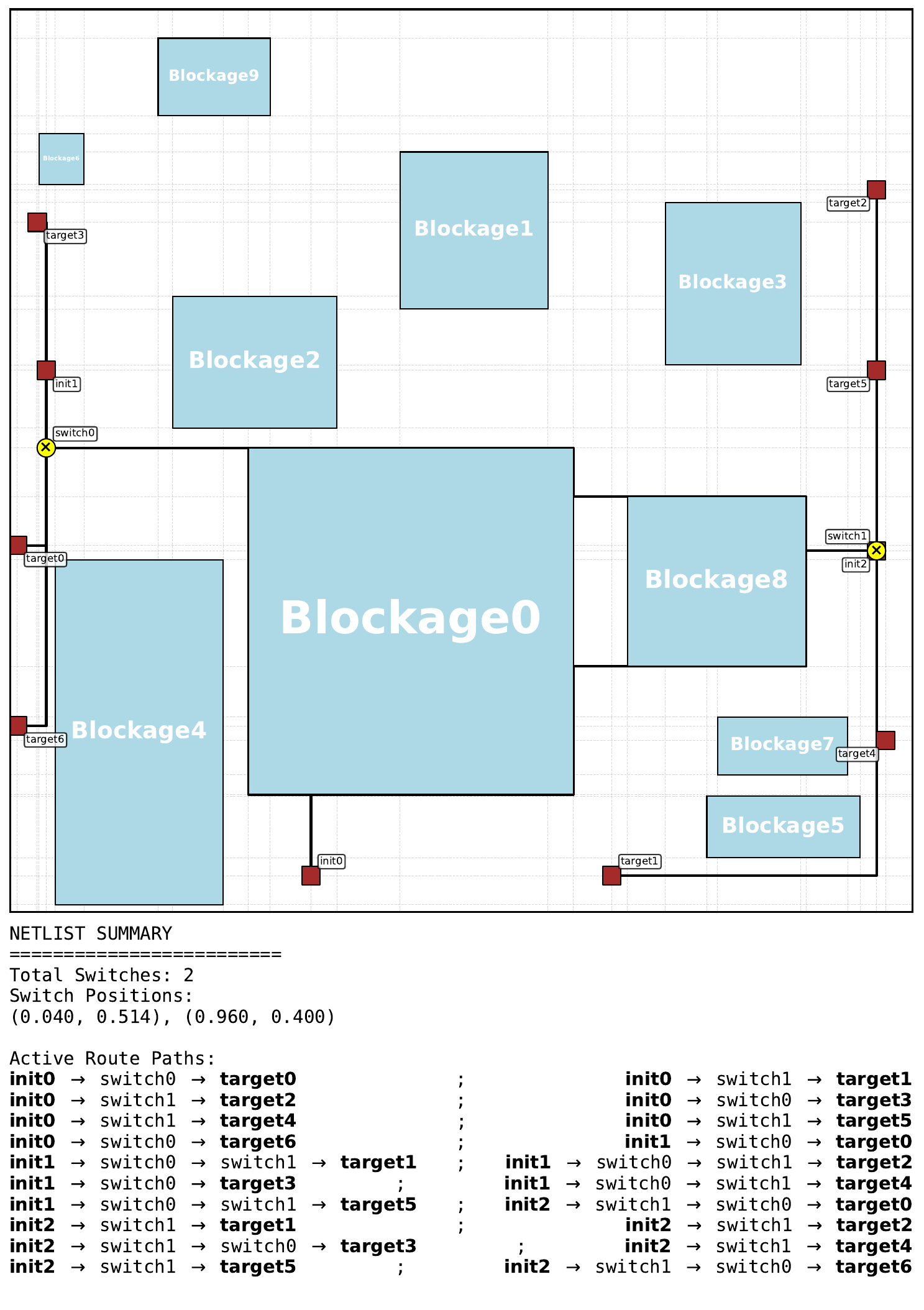}
        \caption*{PPO}
    \end{subfigure}
    \hspace{0.04\linewidth}
    \begin{subfigure}[t]{0.31\linewidth}
        \vspace{0pt}
        \centering
        \includegraphics[width=\linewidth]{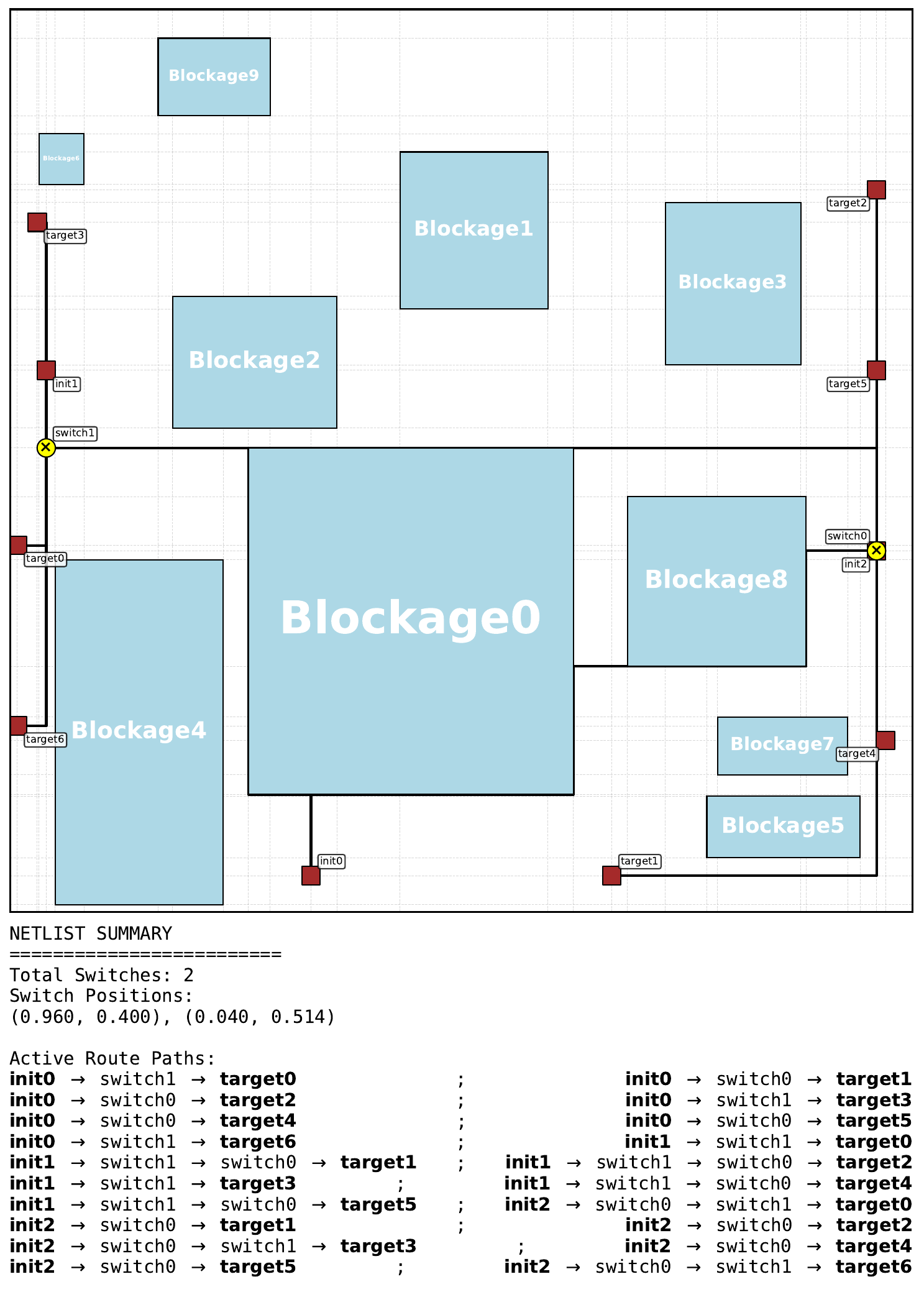}
        \caption*{MCTS}
    \end{subfigure}
\caption{Instance 17.}
\label{fig:best_pretrain_instance_17}
\end{figure*}

\begin{figure*}[h]
\centering
    \begin{subfigure}[t]{0.31\linewidth}
        \vspace{0pt}
        \centering
        \includegraphics[width=\linewidth]{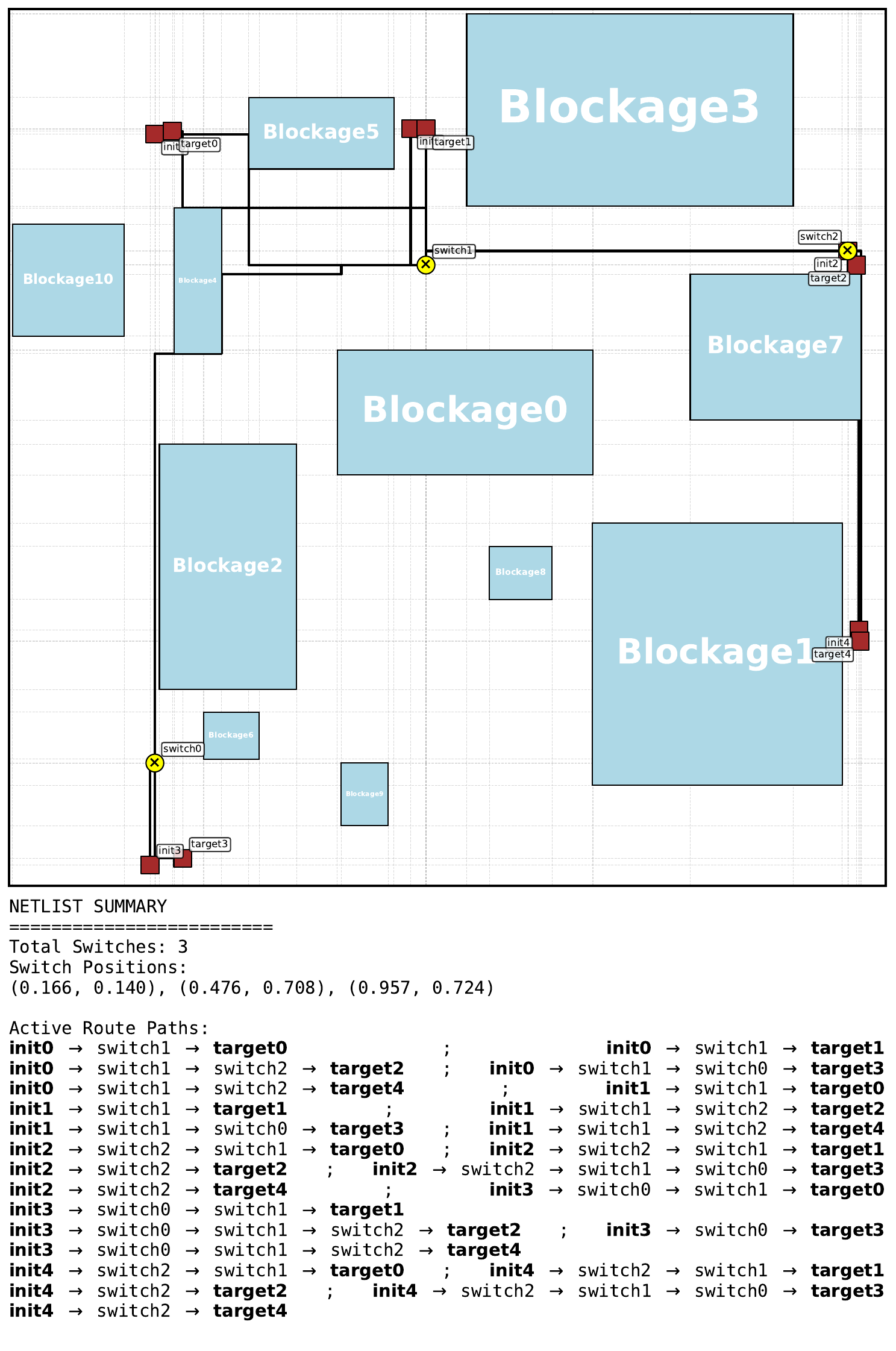}
        \caption*{Heuristic}
    \end{subfigure}
    \hfill
    \begin{subfigure}[t]{0.31\linewidth}
        \vspace{0pt}
        \centering
        \includegraphics[width=\linewidth]{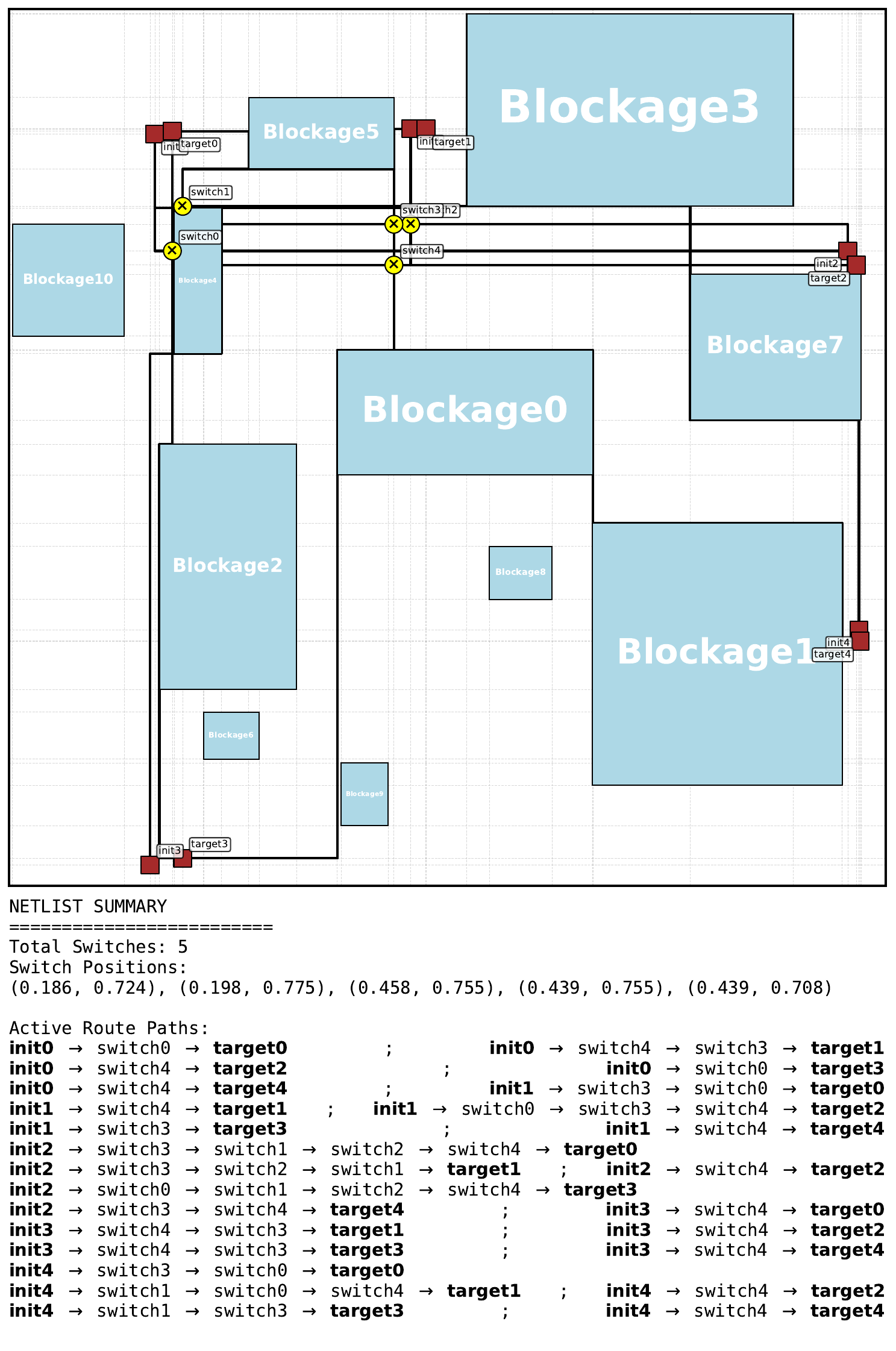}
        \caption*{Random search}
    \end{subfigure}
    \hfill
    \begin{subfigure}[t]{0.31\linewidth}
        \vspace{0pt}
        \centering
        \includegraphics[width=\linewidth]{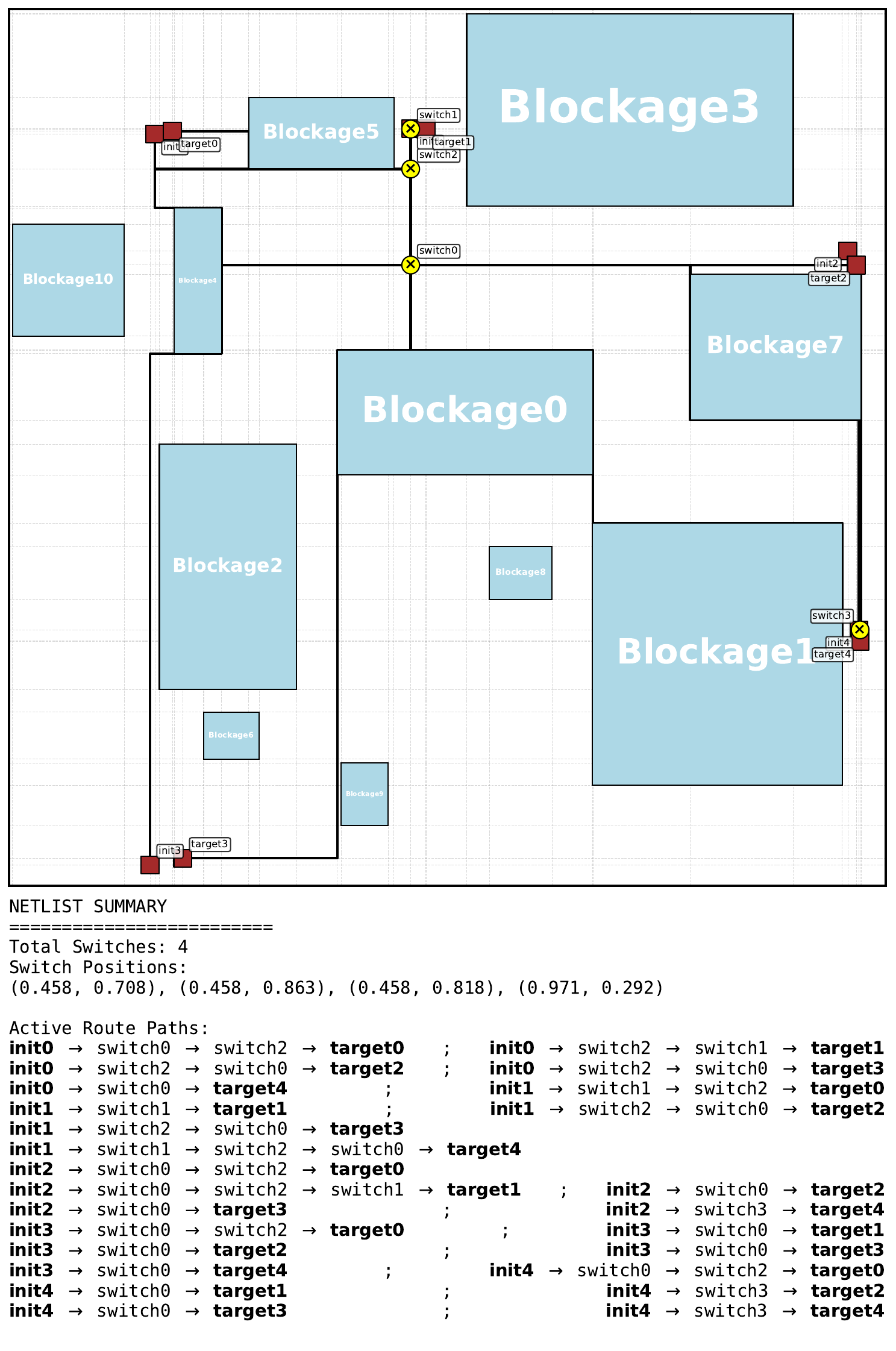}
        \caption*{Genetic algorithm}
    \end{subfigure}
    \\[0.6em]
    \begin{subfigure}[t]{0.31\linewidth}
        \vspace{0pt}
        \centering
        \includegraphics[width=\linewidth]{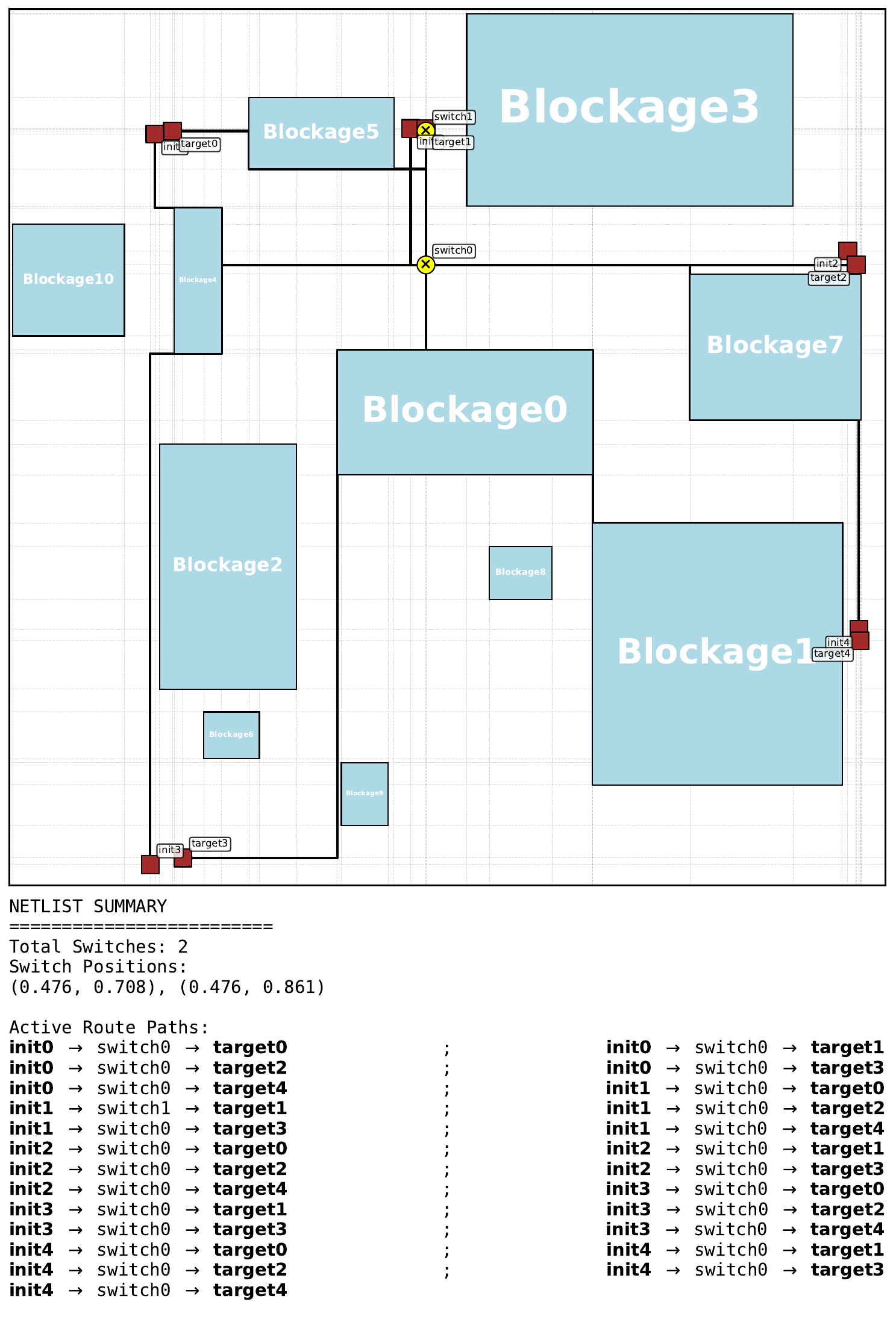}
        \caption*{PPO}
    \end{subfigure}
    \hspace{0.04\linewidth}
    \begin{subfigure}[t]{0.31\linewidth}
        \vspace{0pt}
        \centering
        \includegraphics[width=\linewidth]{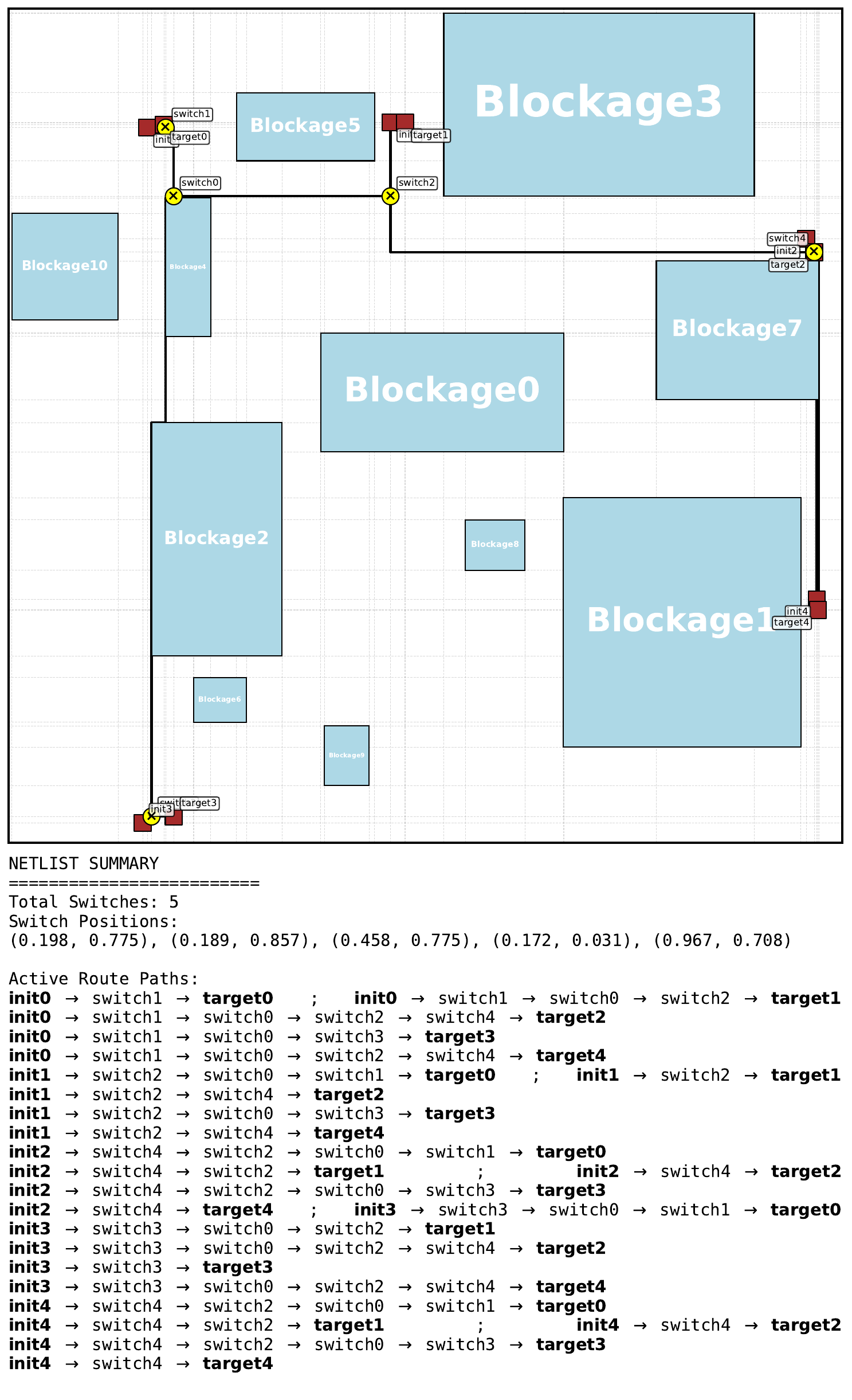}
        \caption*{MCTS}
    \end{subfigure}
\caption{Instance 18.}
\label{fig:best_pretrain_instance_18}
\end{figure*}

\begin{figure*}[h]
\centering
    \begin{subfigure}[t]{0.31\linewidth}
        \vspace{0pt}
        \centering
        \includegraphics[width=\linewidth]{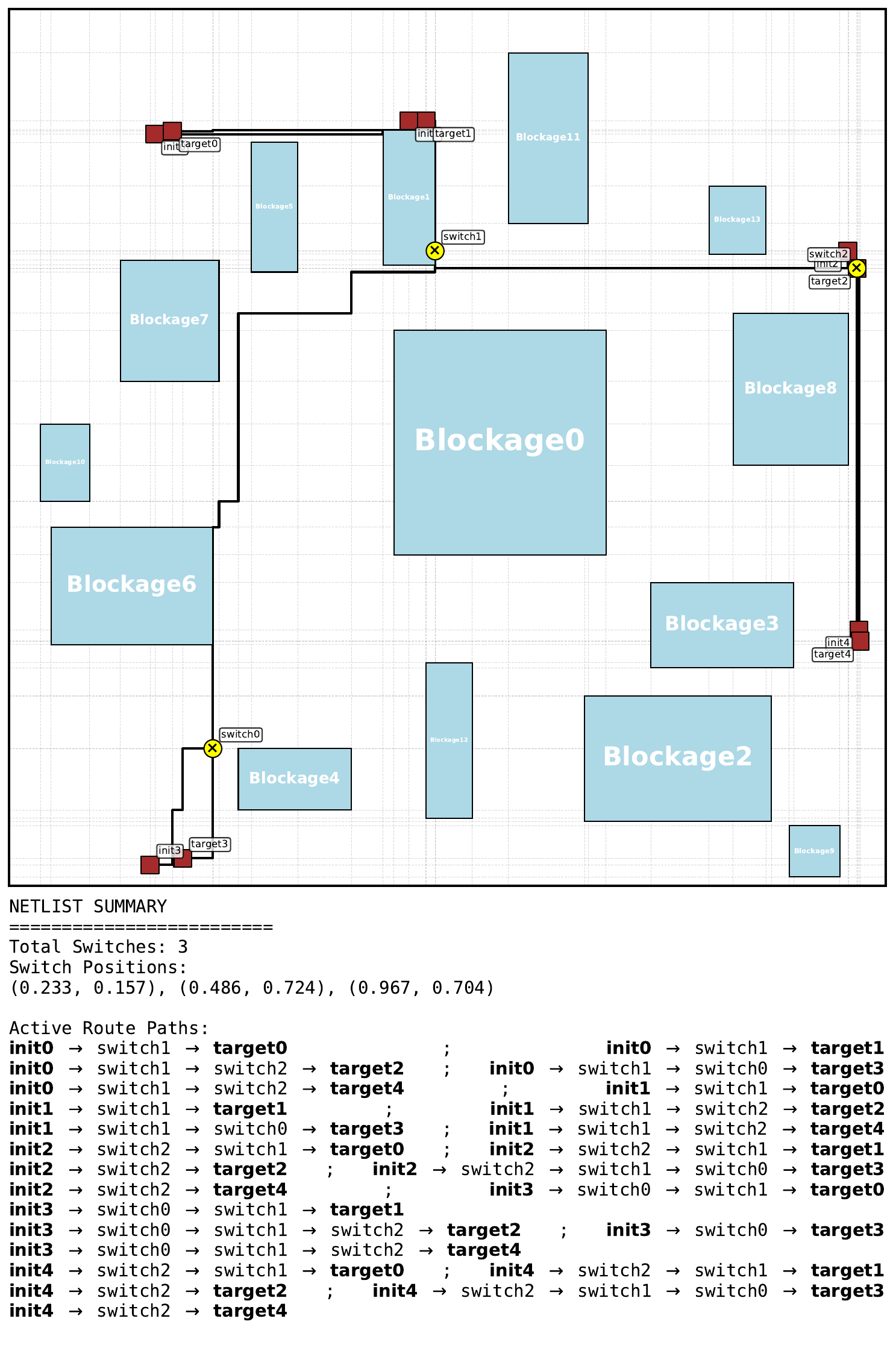}
        \caption*{Heuristic}
    \end{subfigure}
    \hfill
    \begin{subfigure}[t]{0.31\linewidth}
        \vspace{0pt}
        \centering
        \includegraphics[width=\linewidth]{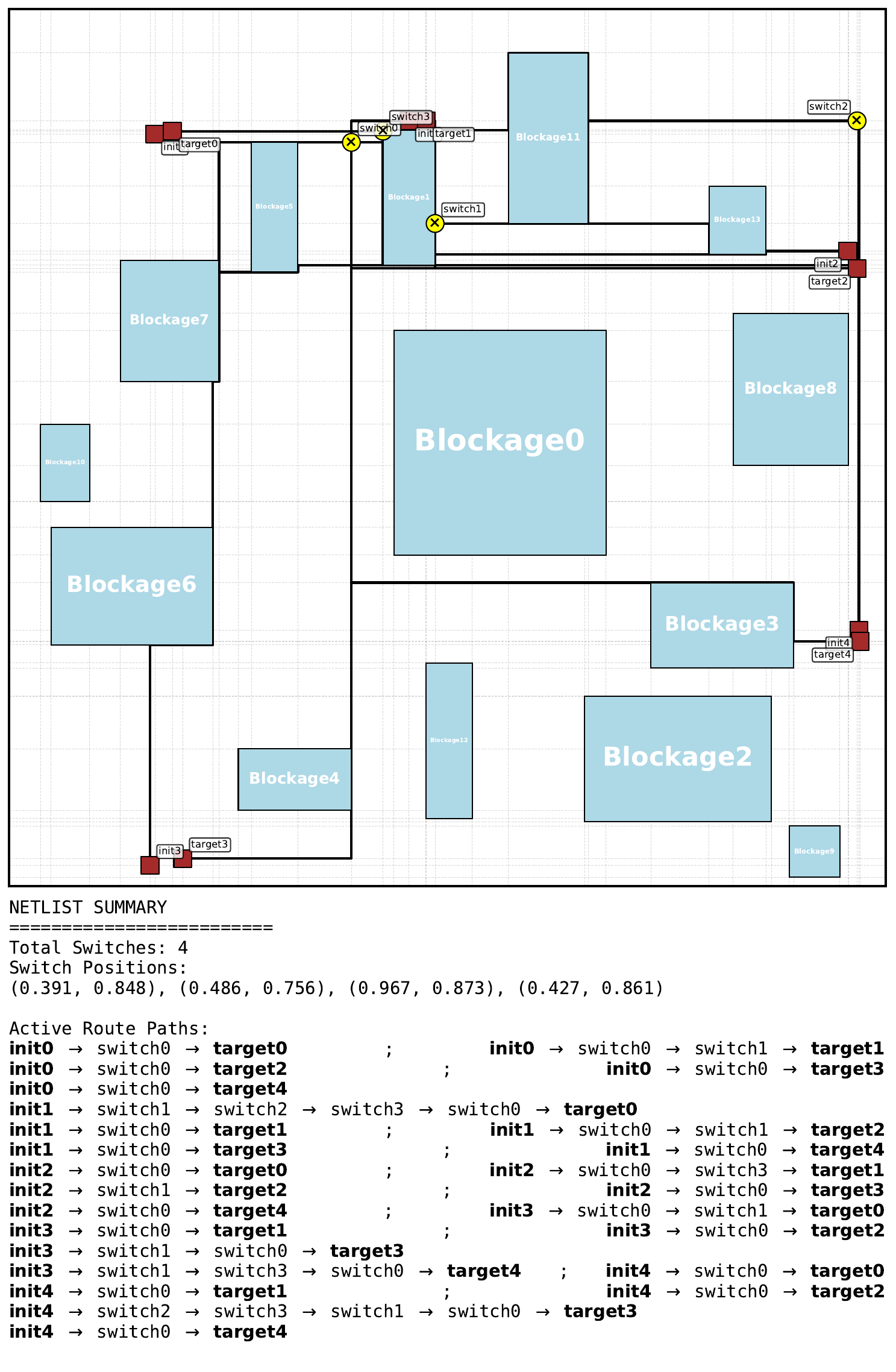}
        \caption*{Random search}
    \end{subfigure}
    \hfill
    \begin{subfigure}[t]{0.31\linewidth}
        \vspace{0pt}
        \centering
        \includegraphics[width=\linewidth]{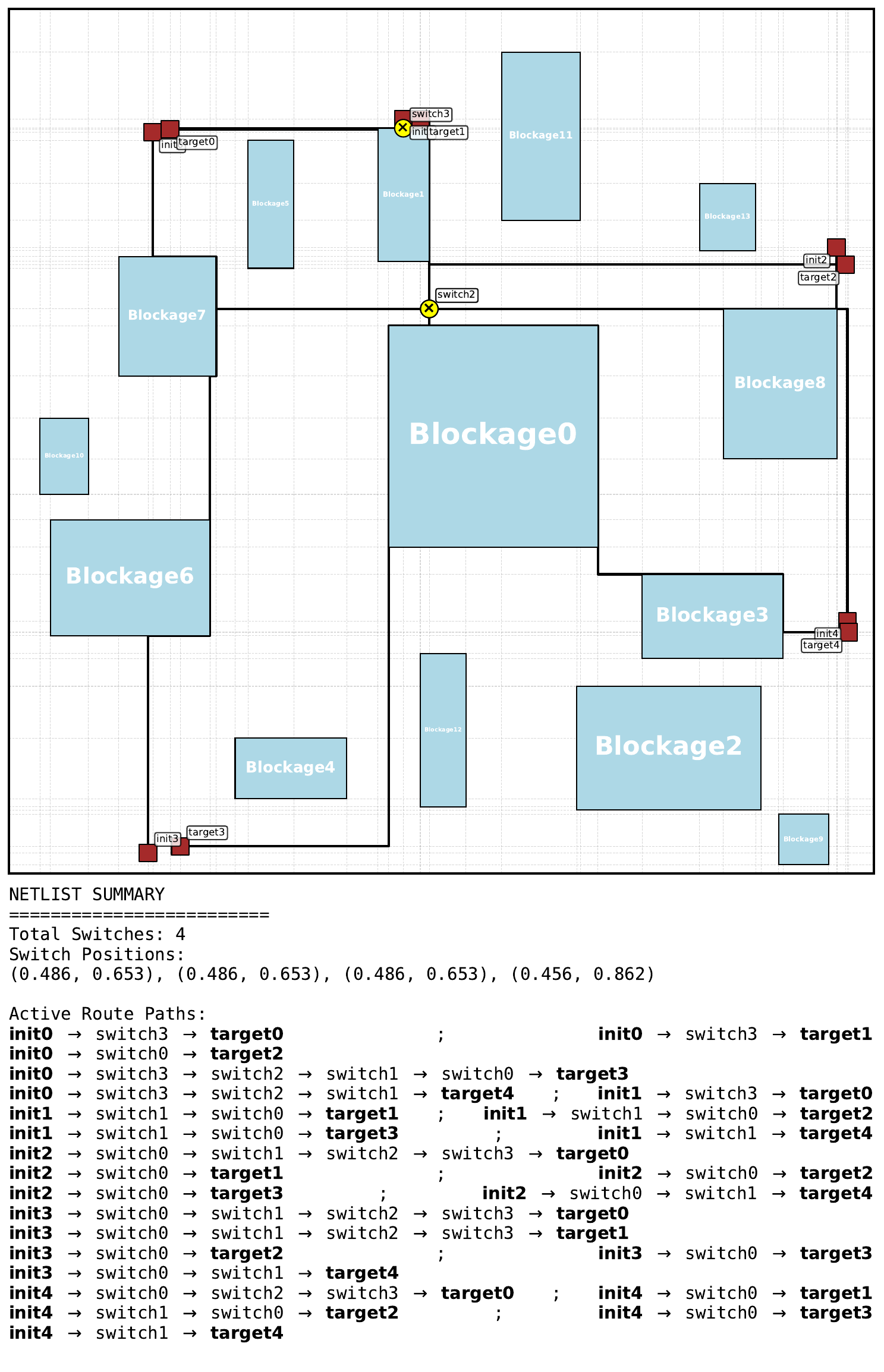}
        \caption*{Genetic algorithm}
    \end{subfigure}
    \\[0.6em]
    \begin{subfigure}[t]{0.31\linewidth}
        \vspace{0pt}
        \centering
        \includegraphics[width=\linewidth]{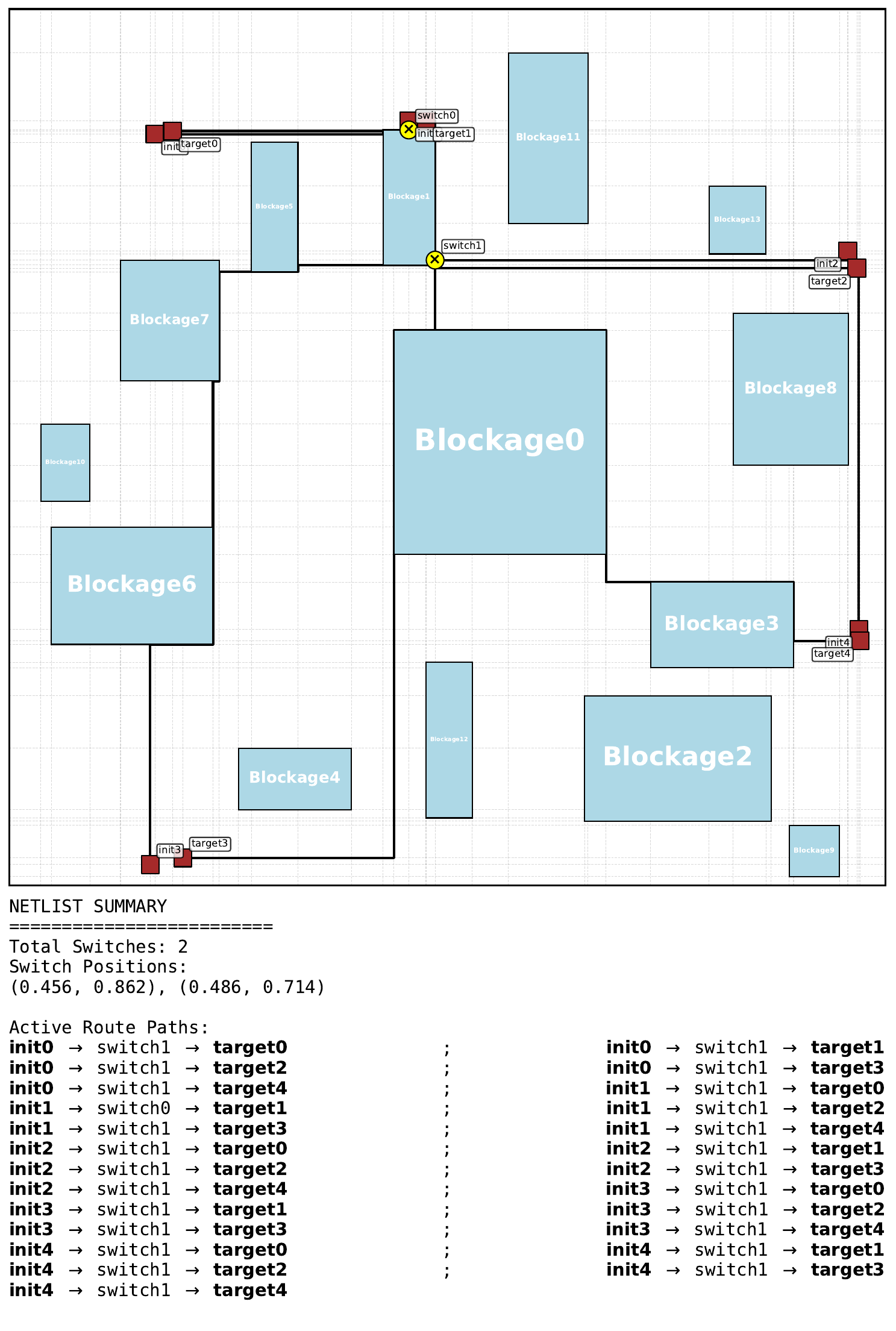}
        \caption*{PPO}
    \end{subfigure}
    \hspace{0.04\linewidth}
    \begin{subfigure}[t]{0.31\linewidth}
        \vspace{0pt}
        \centering
        \includegraphics[width=\linewidth]{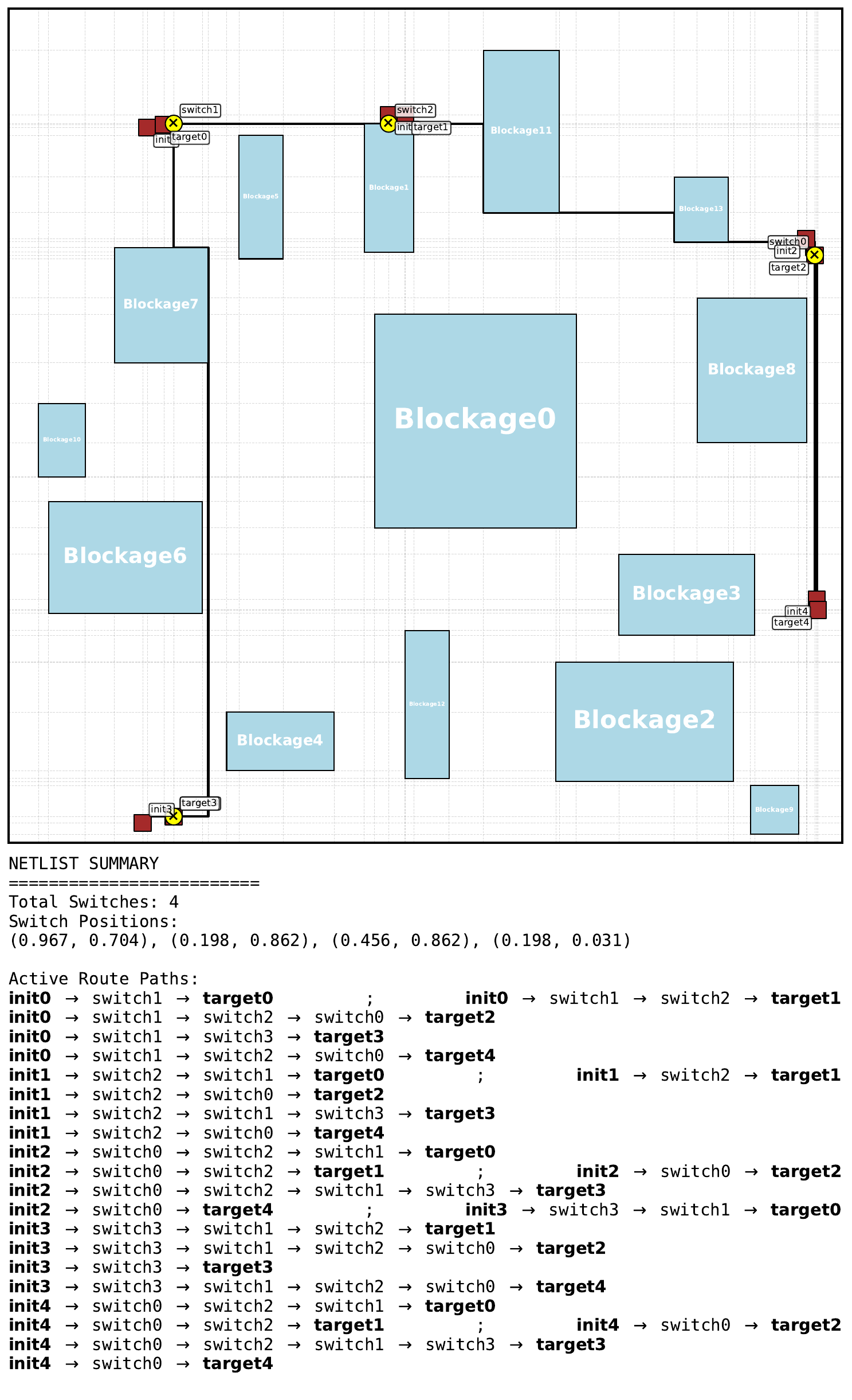}
        \caption*{MCTS}
    \end{subfigure}
\caption{Instance 19.}
\label{fig:best_pretrain_instance_19}
\end{figure*}

\begin{figure*}[h]
\centering
    \begin{subfigure}[t]{0.31\linewidth}
        \vspace{0pt}
        \centering
        \includegraphics[width=\linewidth]{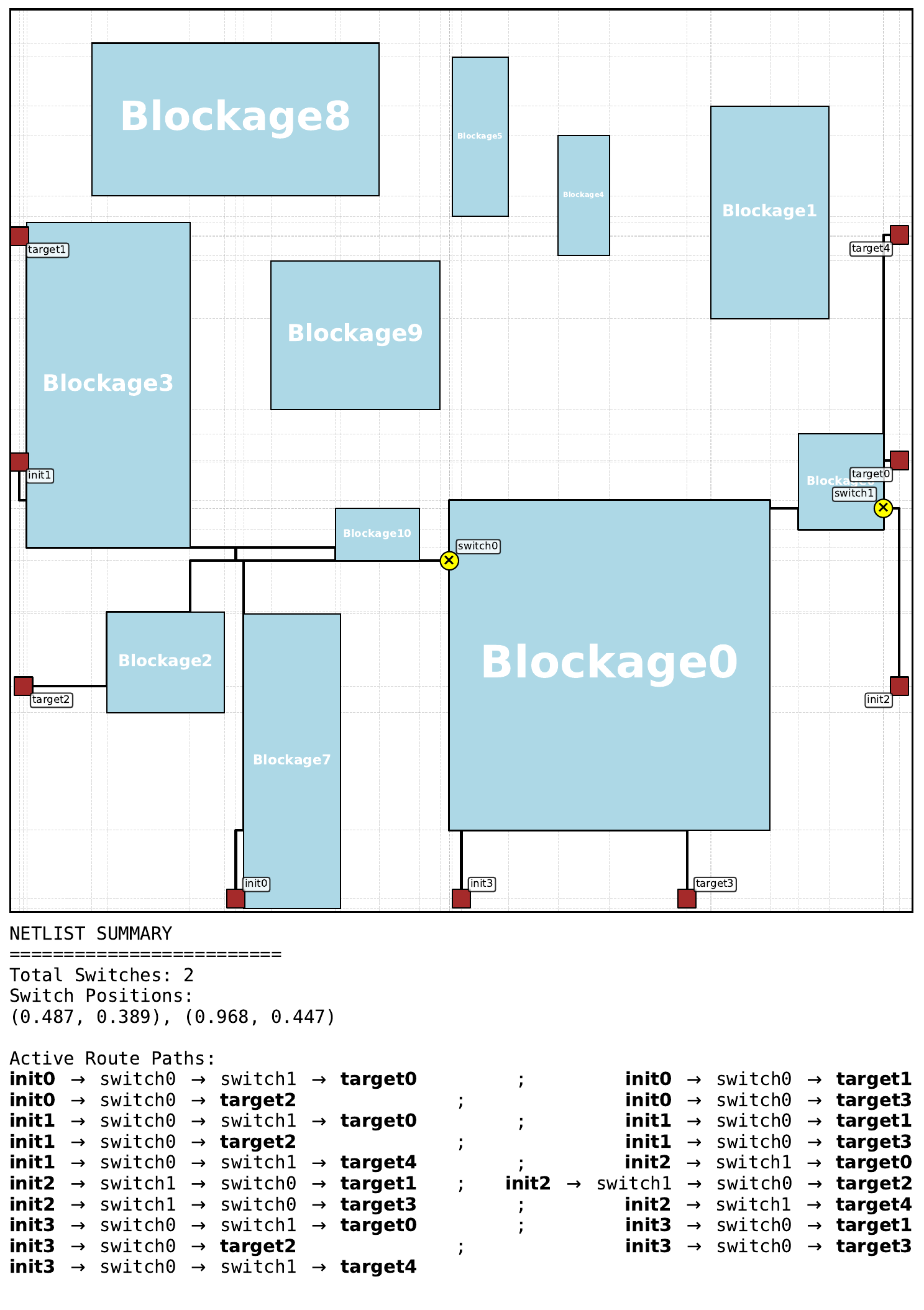}
        \caption*{Heuristic}
    \end{subfigure}
    \hfill
    \begin{subfigure}[t]{0.31\linewidth}
        \vspace{0pt}
        \centering
        \includegraphics[width=\linewidth]{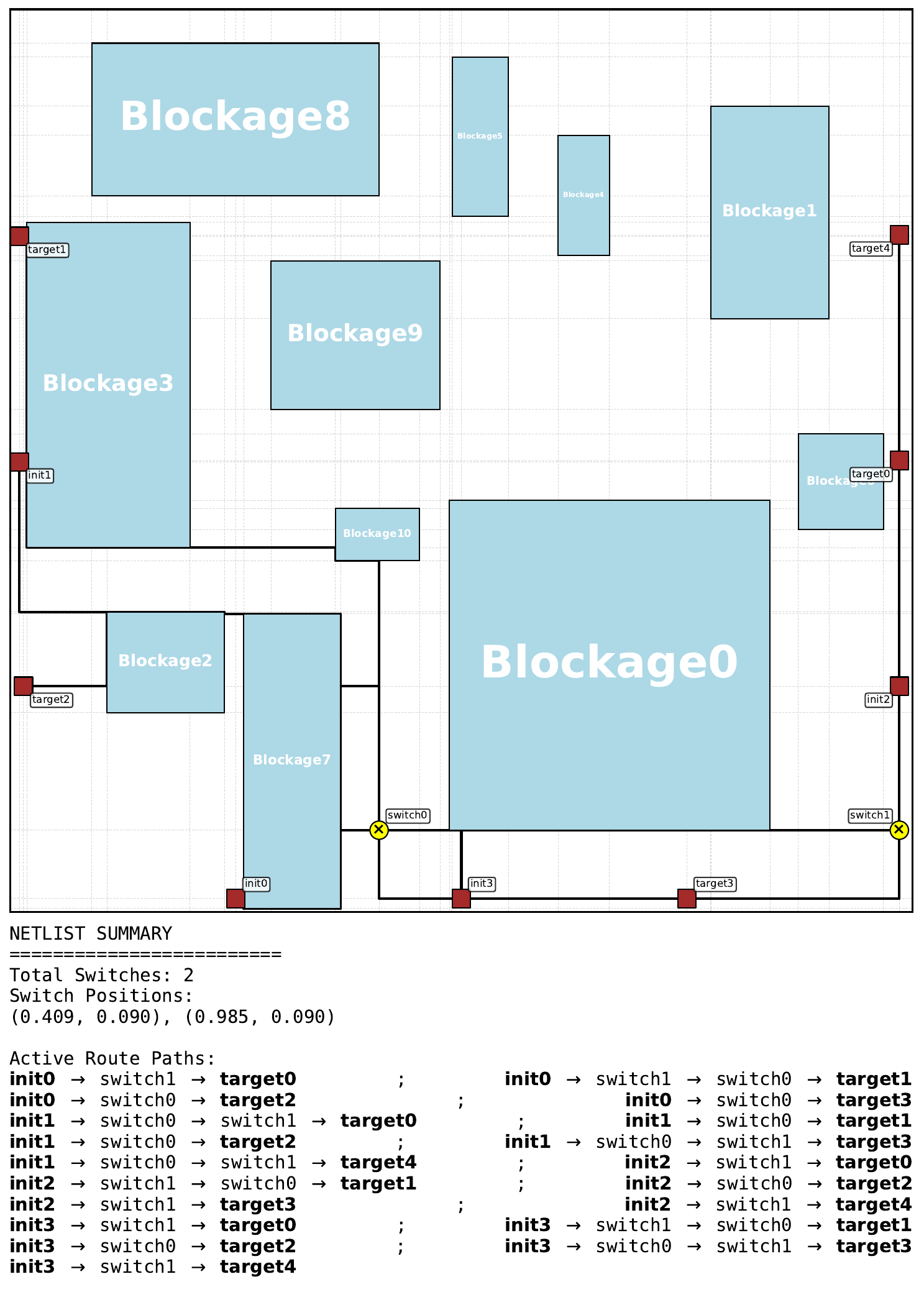}
        \caption*{Random search}
    \end{subfigure}
    \hfill
    \begin{subfigure}[t]{0.31\linewidth}
        \vspace{0pt}
        \centering
        \includegraphics[width=\linewidth]{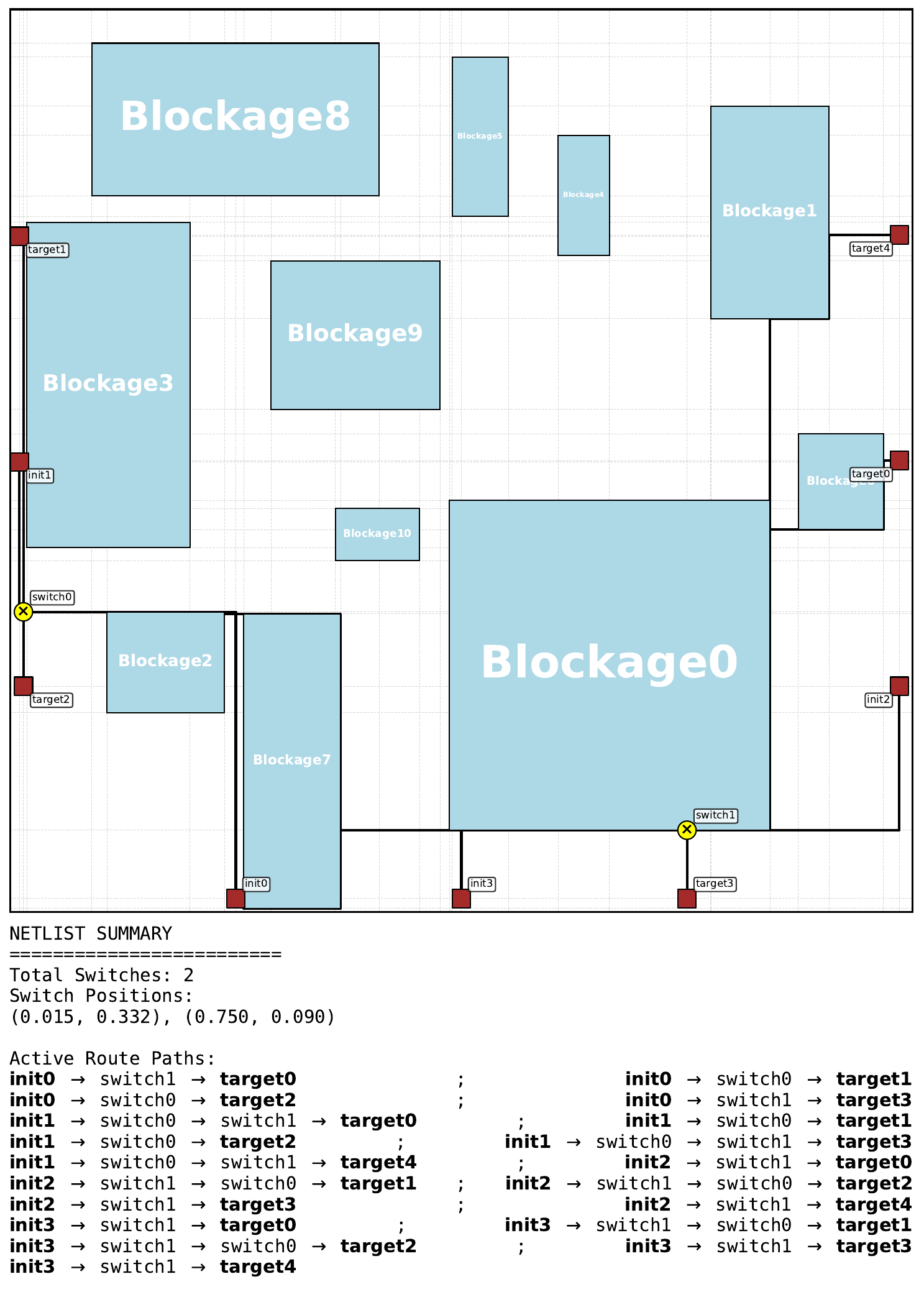}
        \caption*{Genetic algorithm}
    \end{subfigure}
    \\[0.6em]
    \begin{subfigure}[t]{0.31\linewidth}
        \vspace{0pt}
        \centering
        \includegraphics[width=\linewidth]{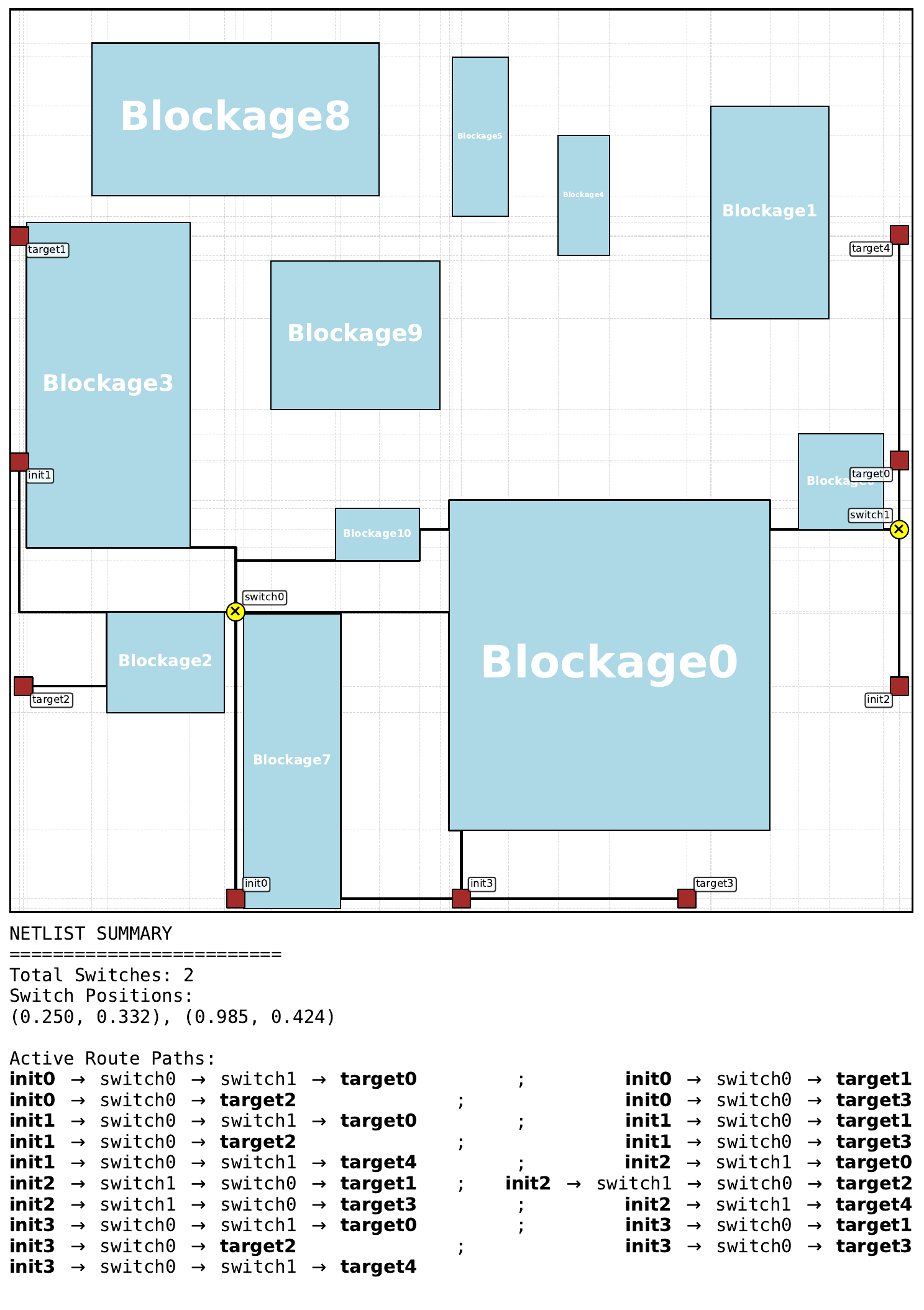}
        \caption*{PPO}
    \end{subfigure}
    \hspace{0.04\linewidth}
    \begin{subfigure}[t]{0.31\linewidth}
        \vspace{0pt}
        \centering
        \includegraphics[width=\linewidth]{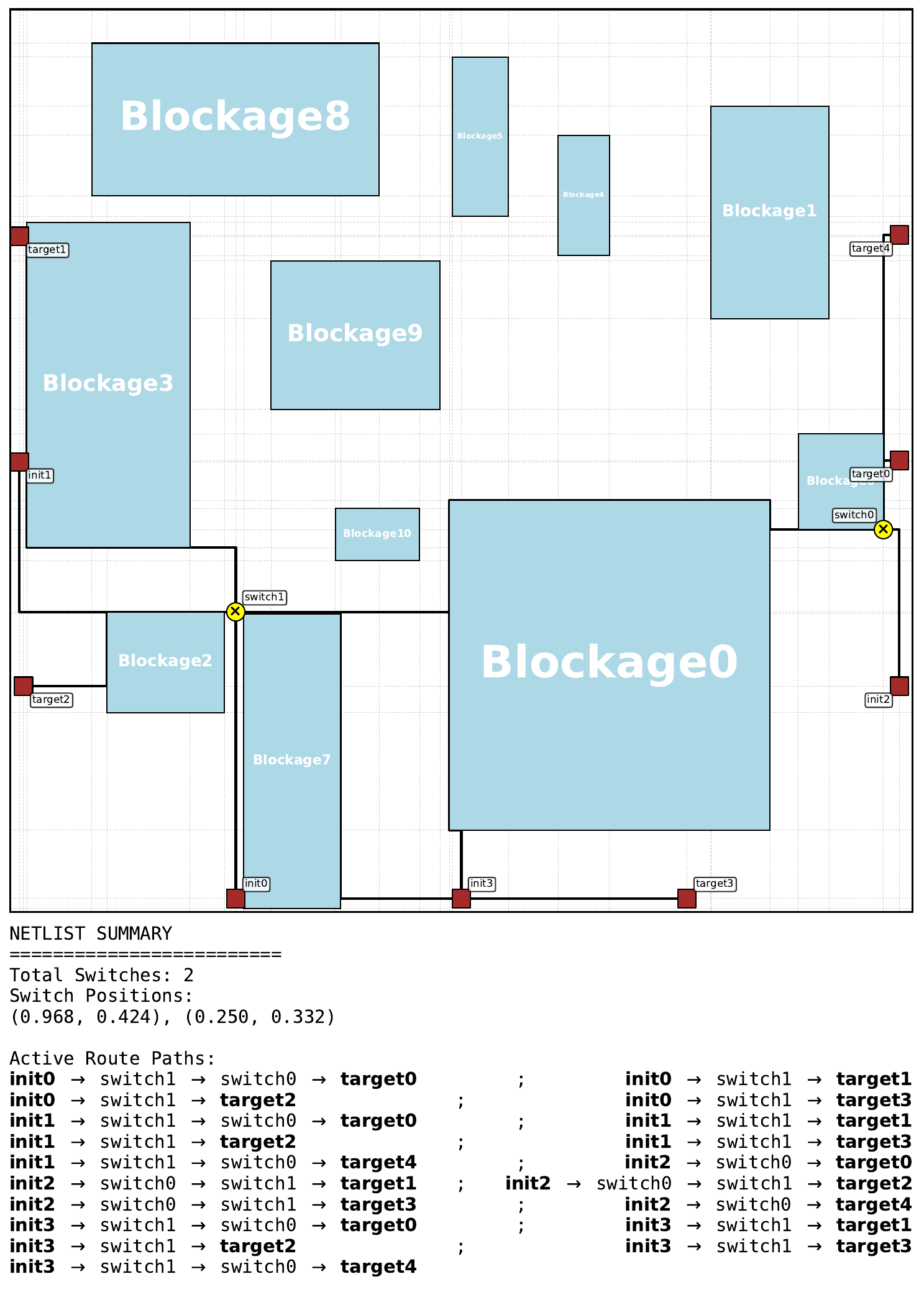}
        \caption*{MCTS}
    \end{subfigure}
\caption{Instance 20.}
\label{fig:best_pretrain_instance_20}
\end{figure*}

\begin{figure*}[h]
\centering
    \begin{subfigure}[t]{0.31\linewidth}
        \vspace{0pt}
        \centering
        \includegraphics[width=\linewidth]{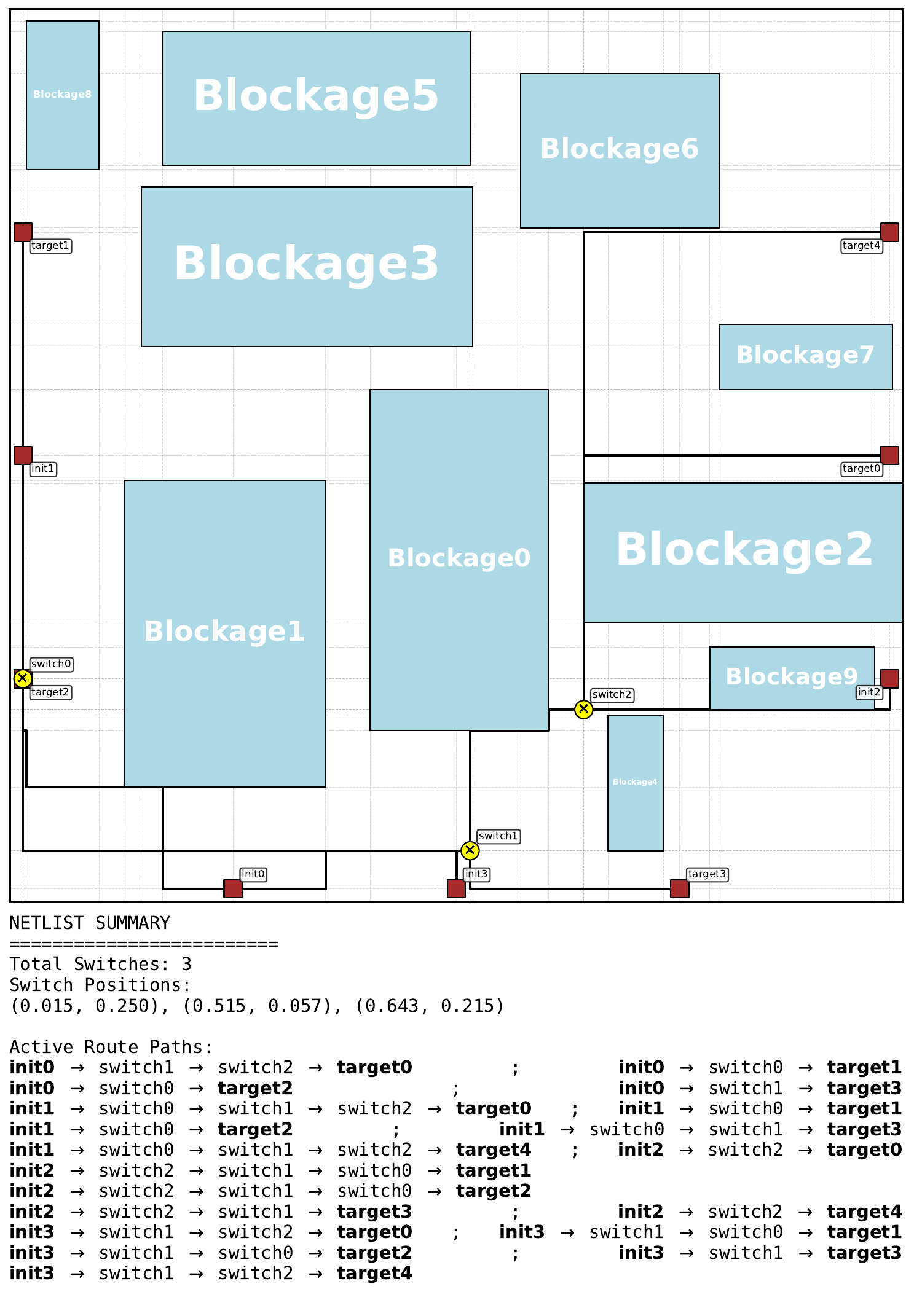}
        \caption*{Heuristic}
    \end{subfigure}
    \hfill
    \begin{subfigure}[t]{0.31\linewidth}
        \vspace{0pt}
        \centering
        \includegraphics[width=\linewidth]{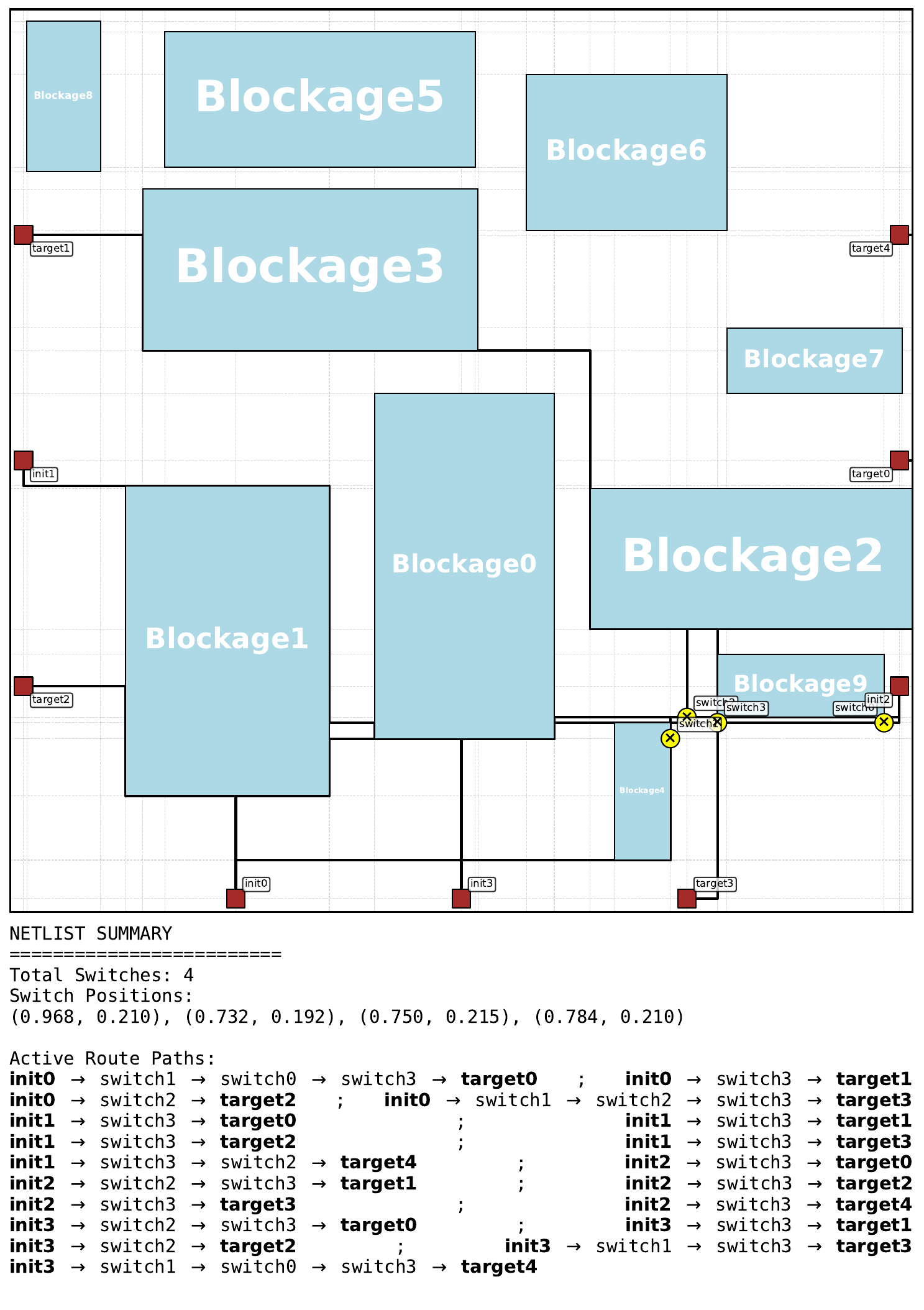}
        \caption*{Random search}
    \end{subfigure}
    \hfill
    \begin{subfigure}[t]{0.31\linewidth}
        \vspace{0pt}
        \centering
        \includegraphics[width=\linewidth]{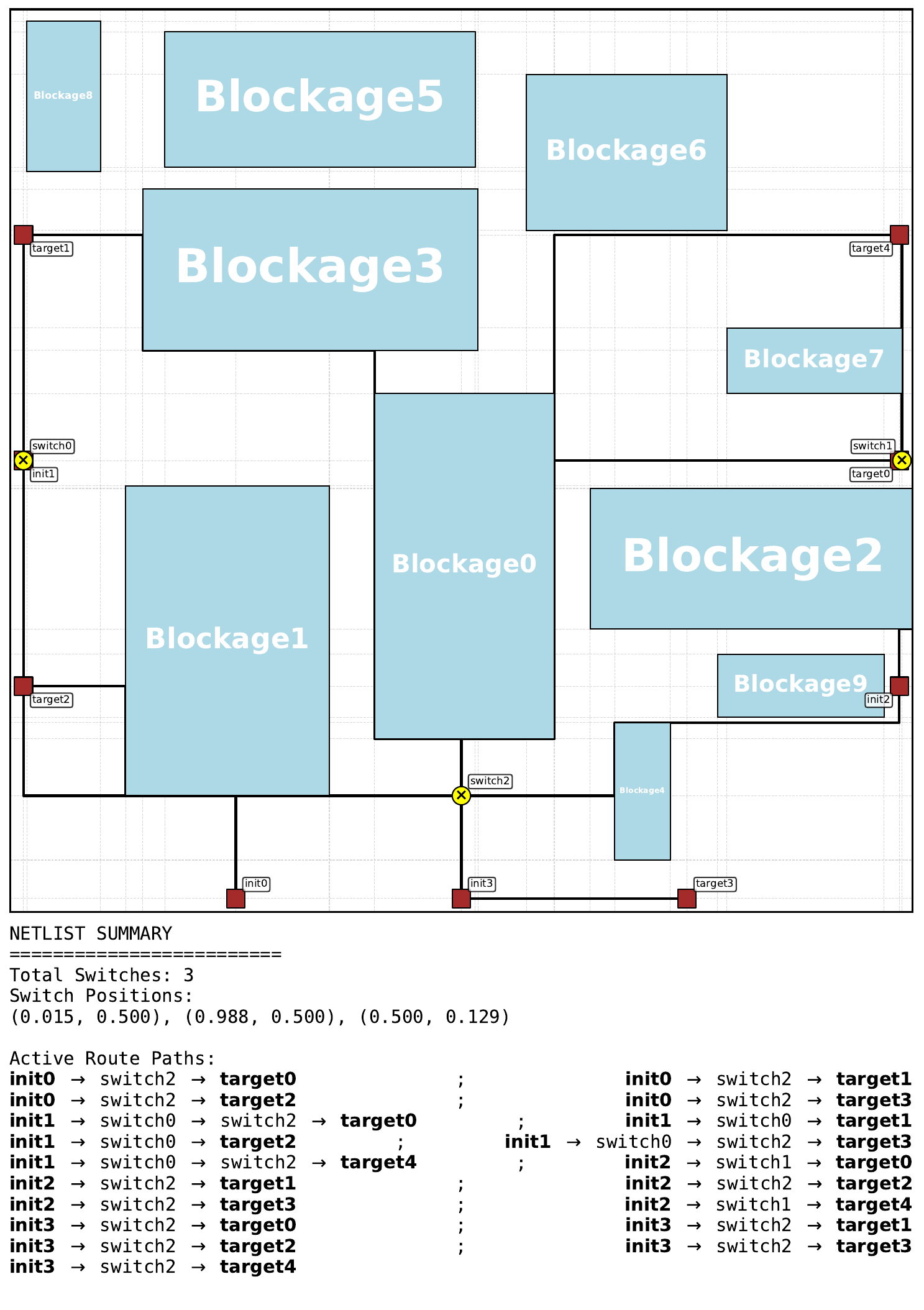}
        \caption*{Genetic algorithm}
    \end{subfigure}
    \\[0.6em]
    \begin{subfigure}[t]{0.31\linewidth}
        \vspace{0pt}
        \centering
        \includegraphics[width=\linewidth]{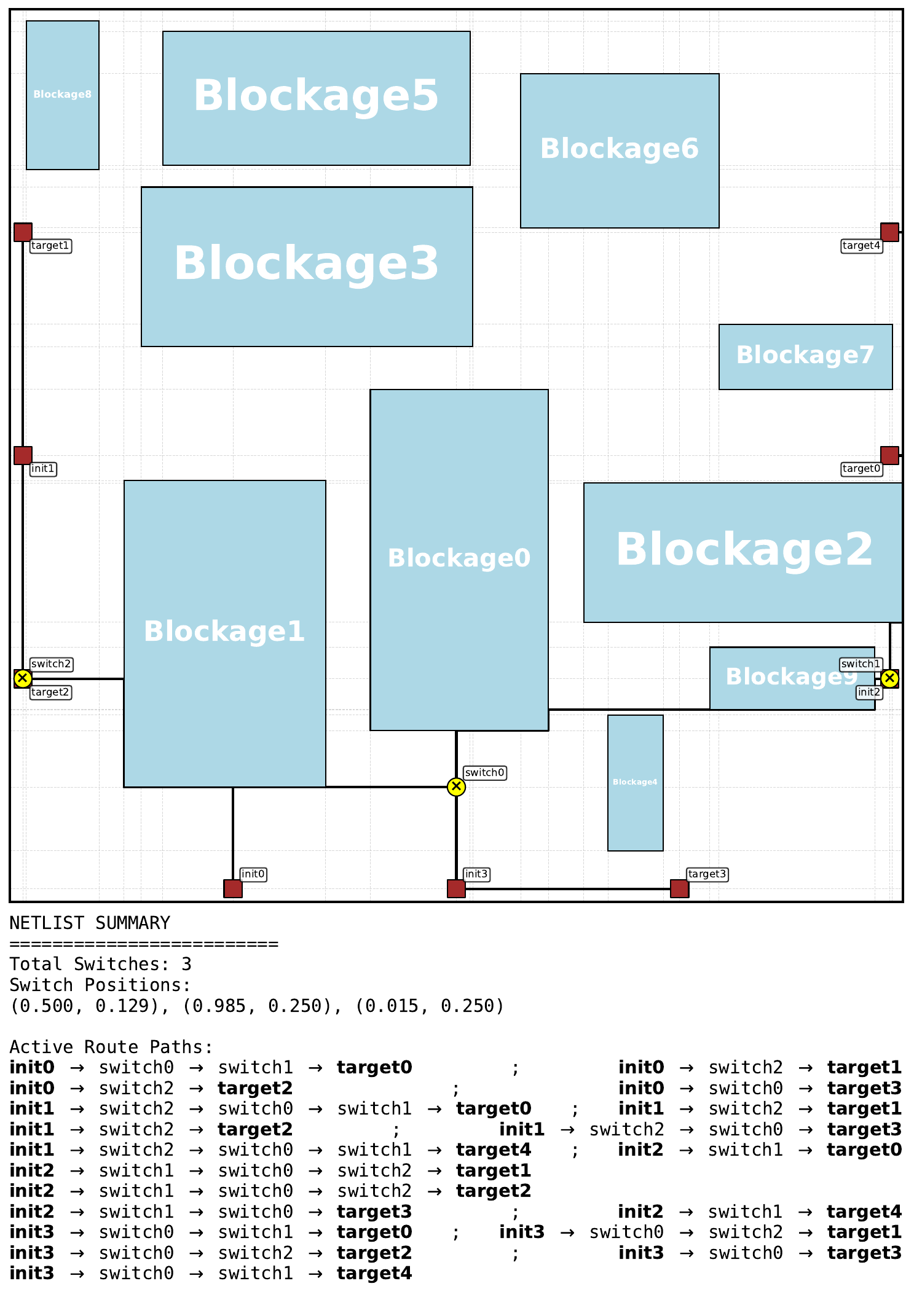}
        \caption*{PPO}
    \end{subfigure}
    \hspace{0.04\linewidth}
    \begin{subfigure}[t]{0.31\linewidth}
        \vspace{0pt}
        \centering
        \includegraphics[width=\linewidth]{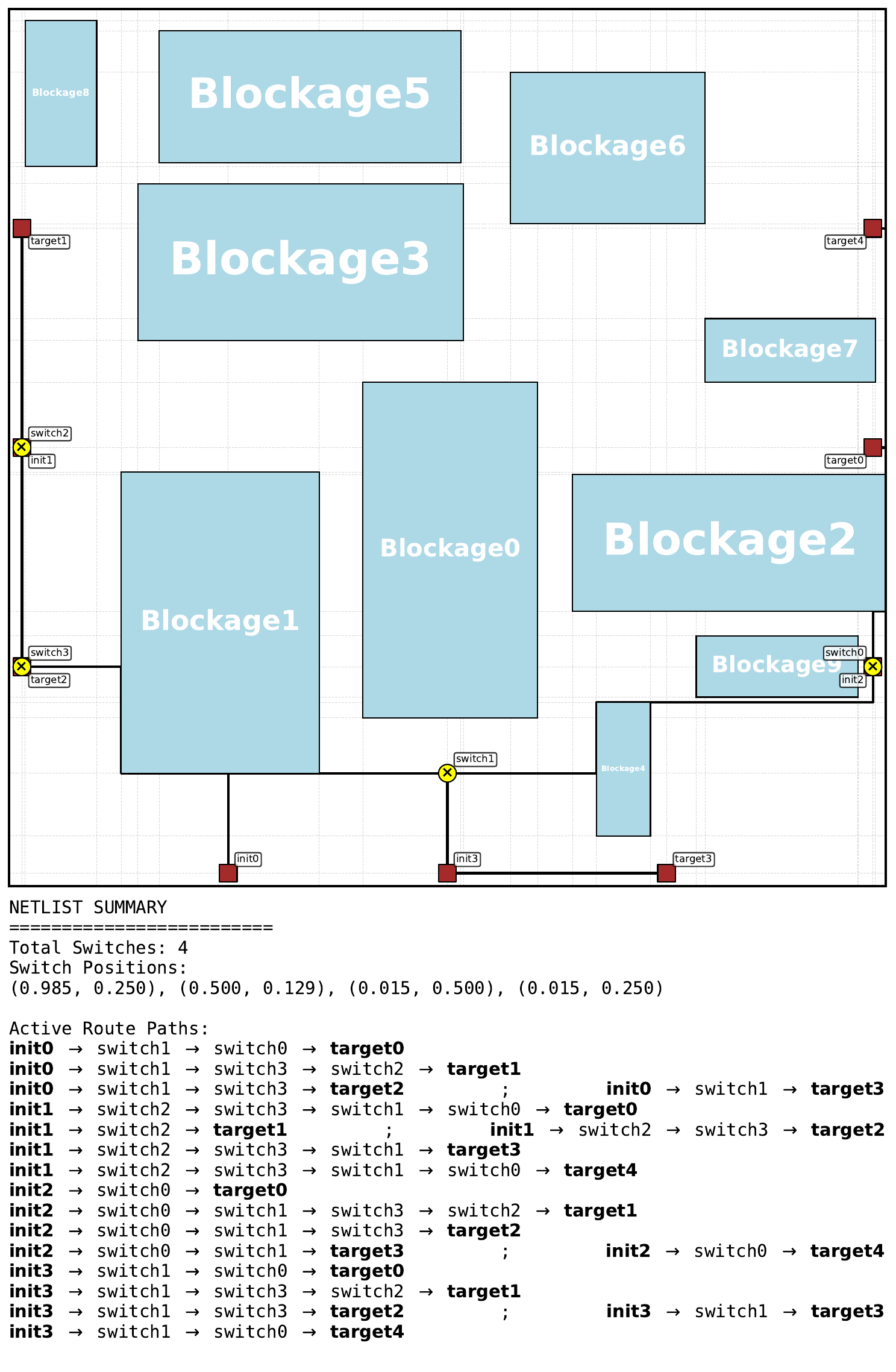}
        \caption*{MCTS}
    \end{subfigure}
\caption{Instance 21.}
\label{fig:best_pretrain_instance_21}
\end{figure*}

\begin{figure*}[h]
\centering
    \begin{subfigure}[t]{0.31\linewidth}
        \vspace{0pt}
        \centering
        \includegraphics[width=\linewidth]{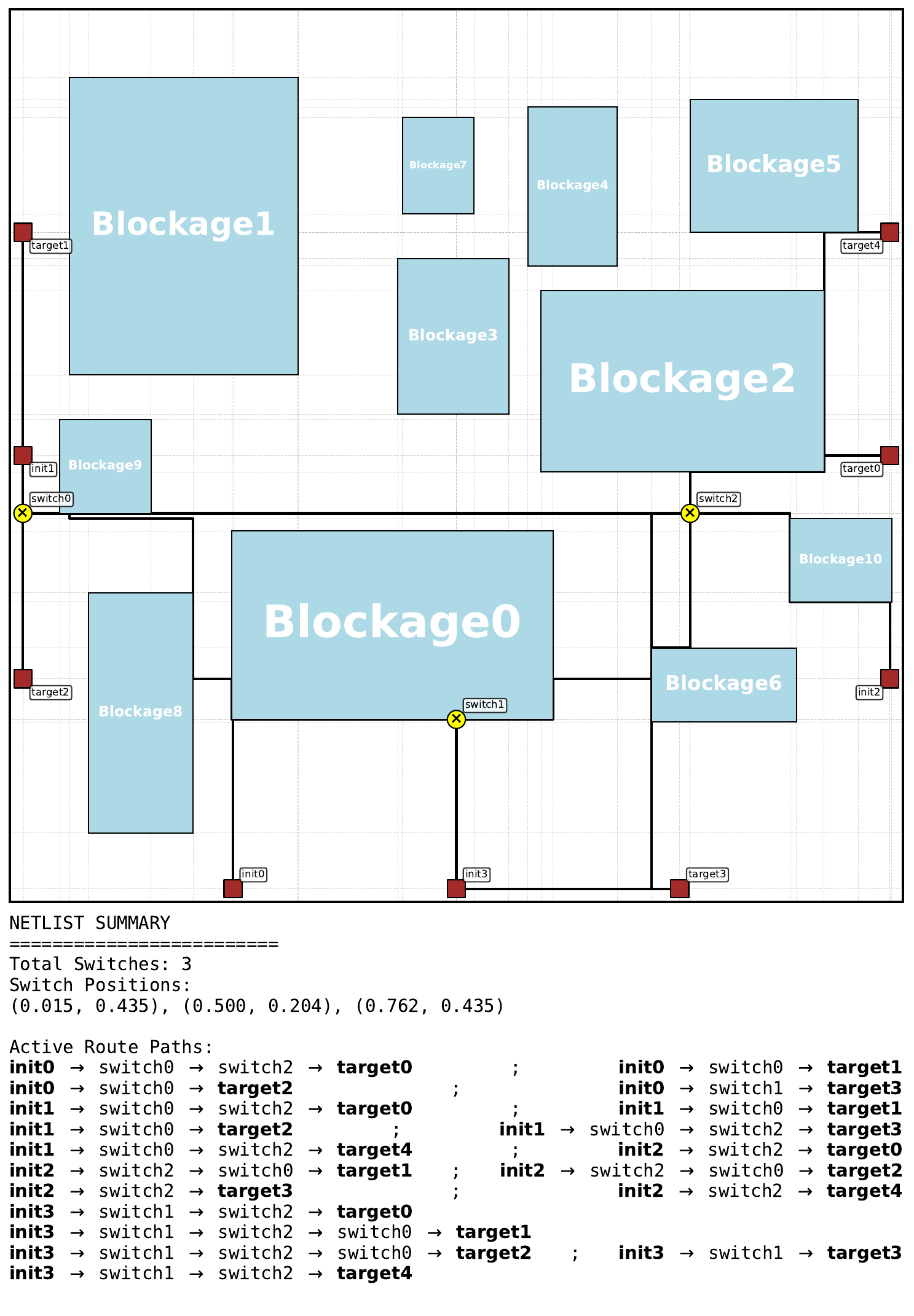}
        \caption*{Heuristic}
    \end{subfigure}
    \hfill
    \begin{subfigure}[t]{0.31\linewidth}
        \vspace{0pt}
        \centering
        \includegraphics[width=\linewidth]{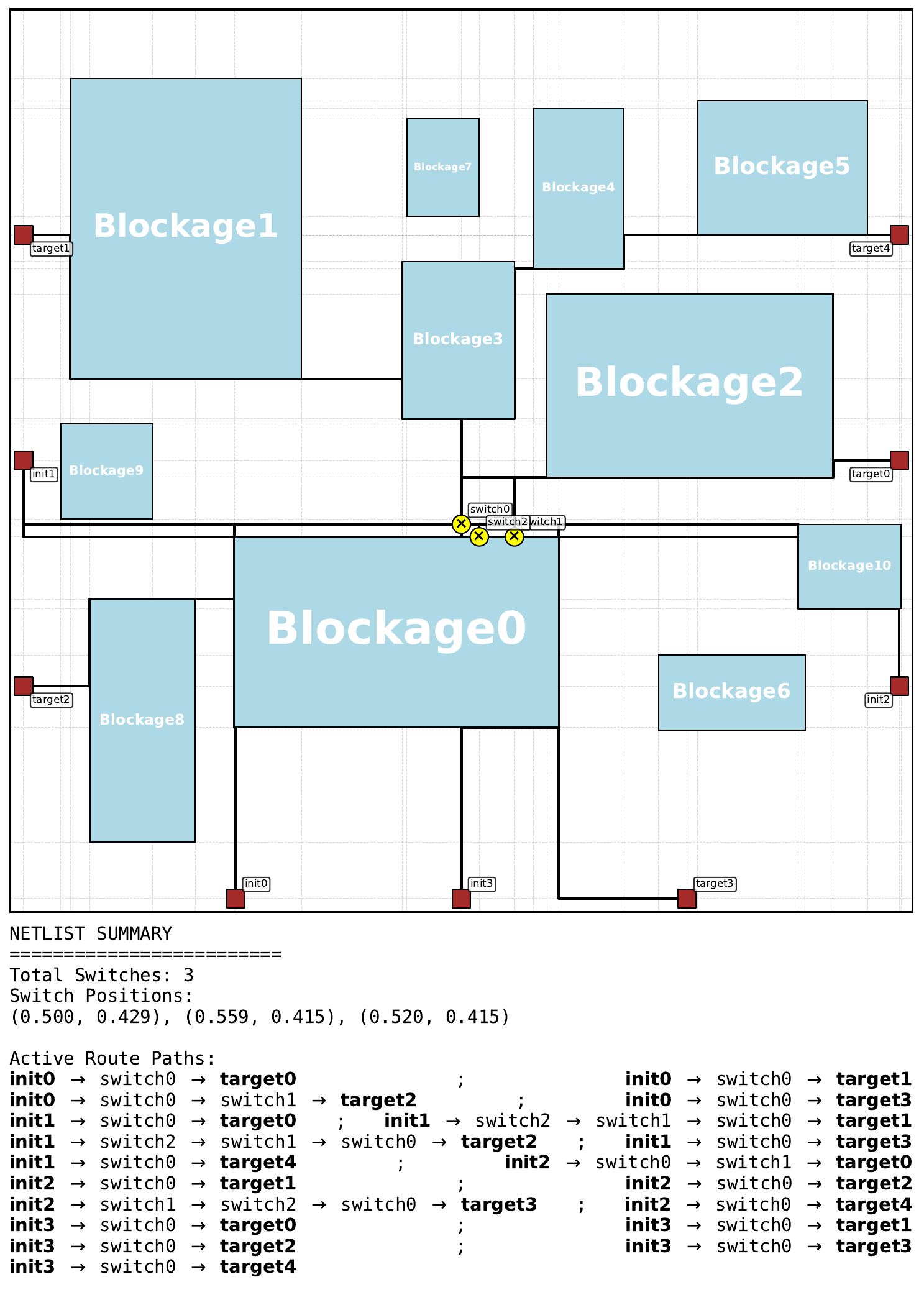}
        \caption*{Random search}
    \end{subfigure}
    \hfill
    \begin{subfigure}[t]{0.31\linewidth}
        \vspace{0pt}
        \centering
        \includegraphics[width=\linewidth]{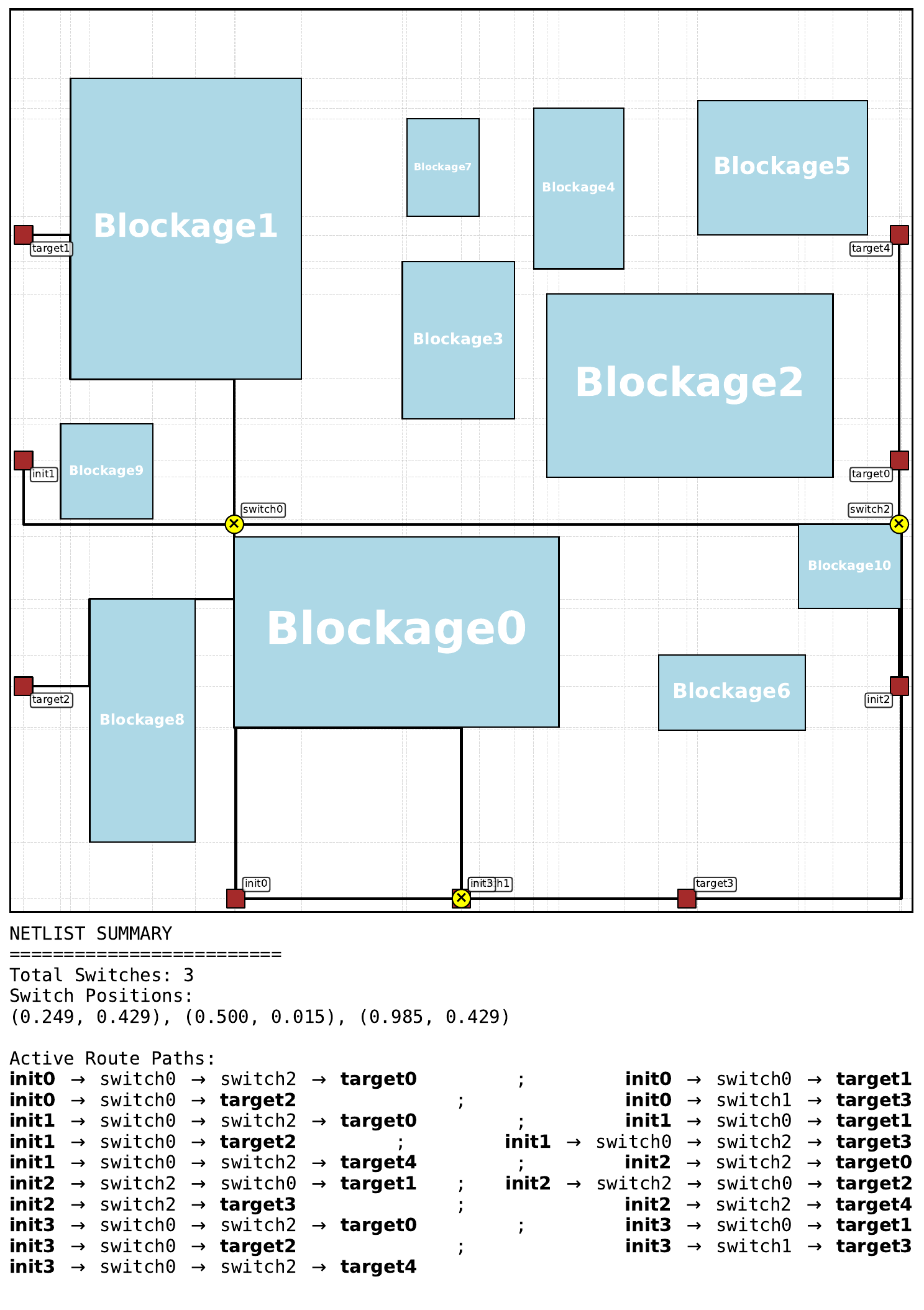}
        \caption*{Genetic algorithm}
    \end{subfigure}
    \\[0.6em]
    \begin{subfigure}[t]{0.31\linewidth}
        \vspace{0pt}
        \centering
        \includegraphics[width=\linewidth]{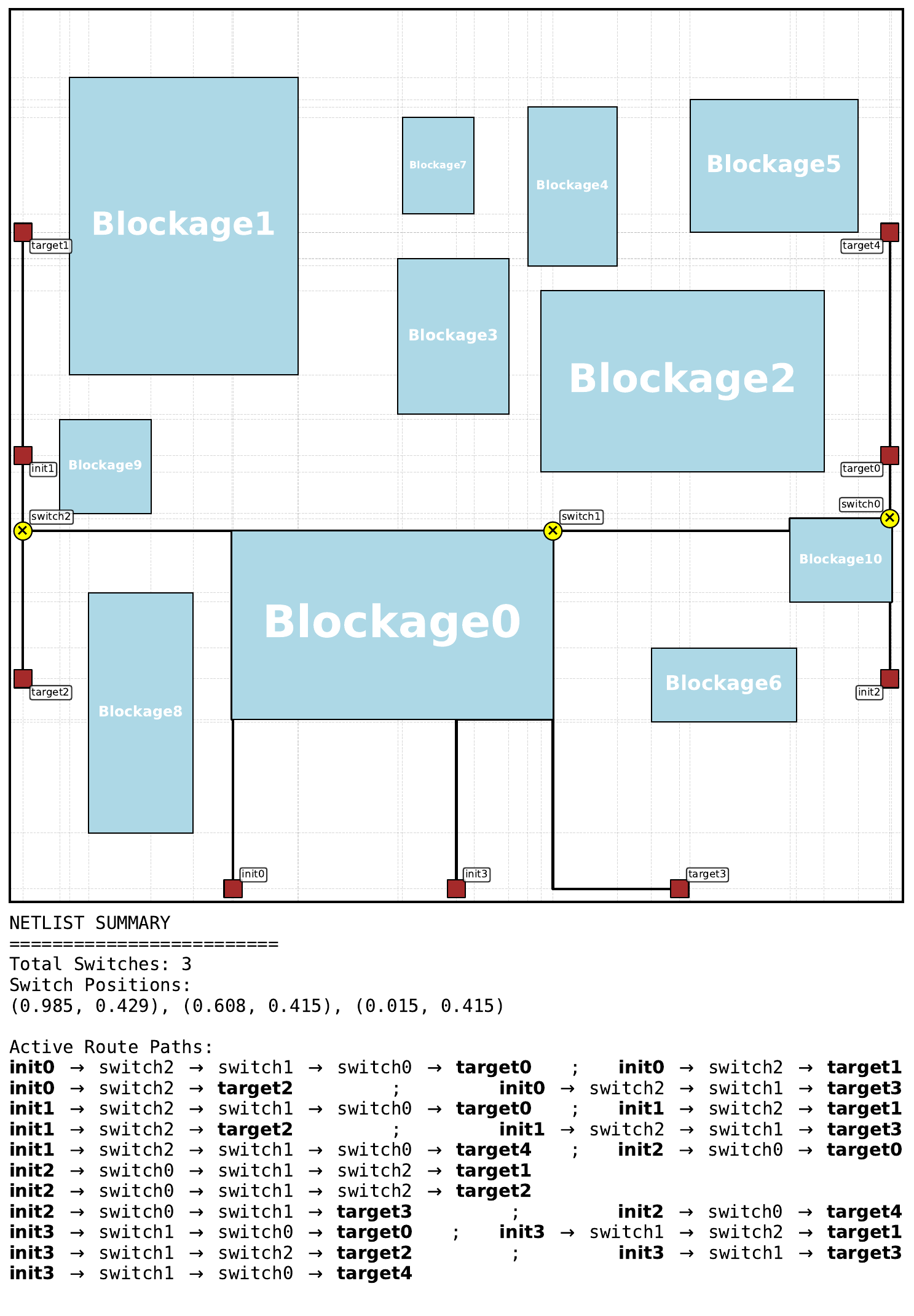}
        \caption*{PPO}
    \end{subfigure}
    \hspace{0.04\linewidth}
    \begin{subfigure}[t]{0.31\linewidth}
        \vspace{0pt}
        \centering
        \includegraphics[width=\linewidth]{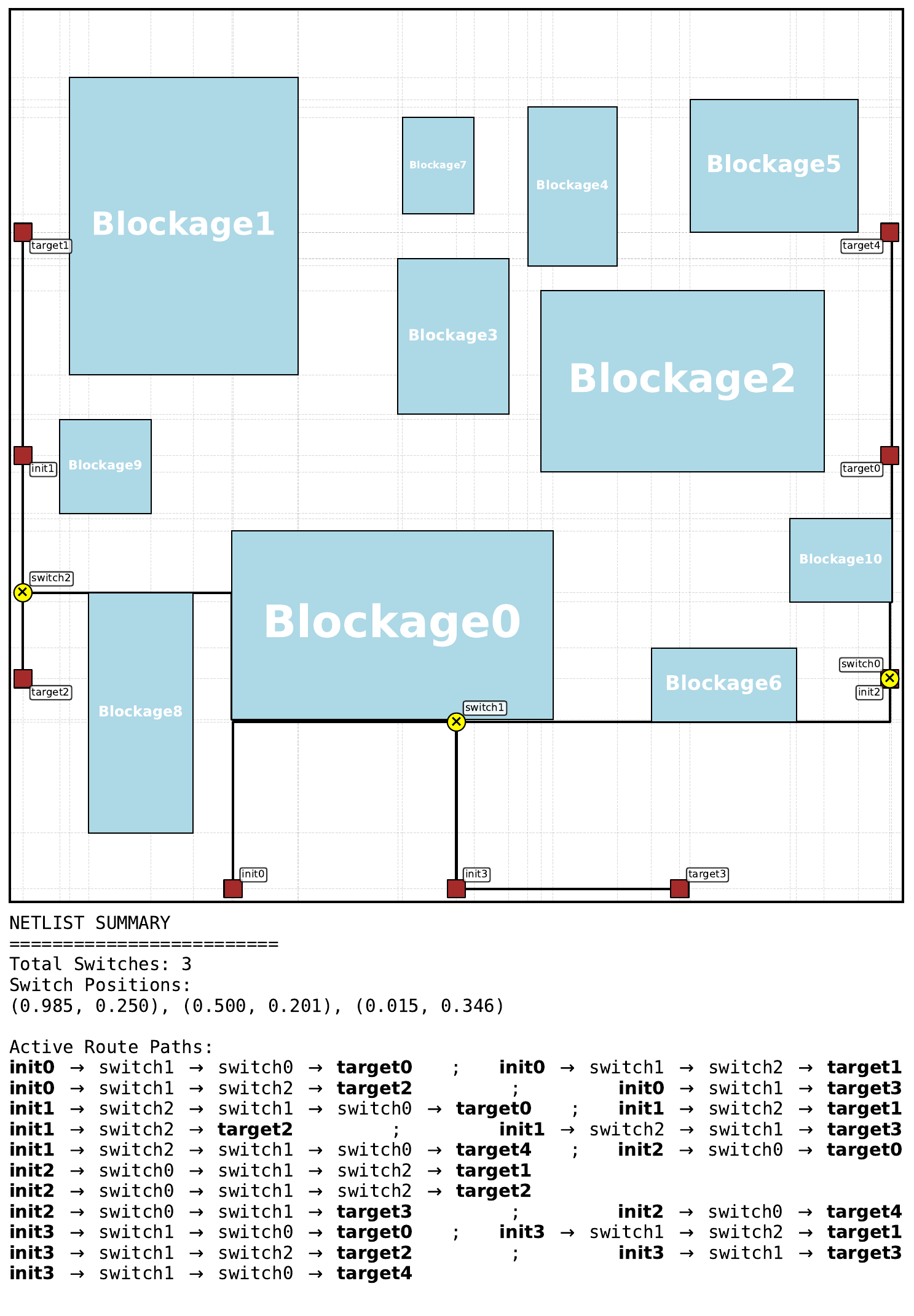}
        \caption*{MCTS}
    \end{subfigure}
\caption{Instance 22.}
\label{fig:best_pretrain_instance_22}
\end{figure*}

\begin{figure*}[h]
\centering
    \begin{subfigure}[t]{0.31\linewidth}
        \vspace{0pt}
        \centering
        \includegraphics[width=\linewidth]{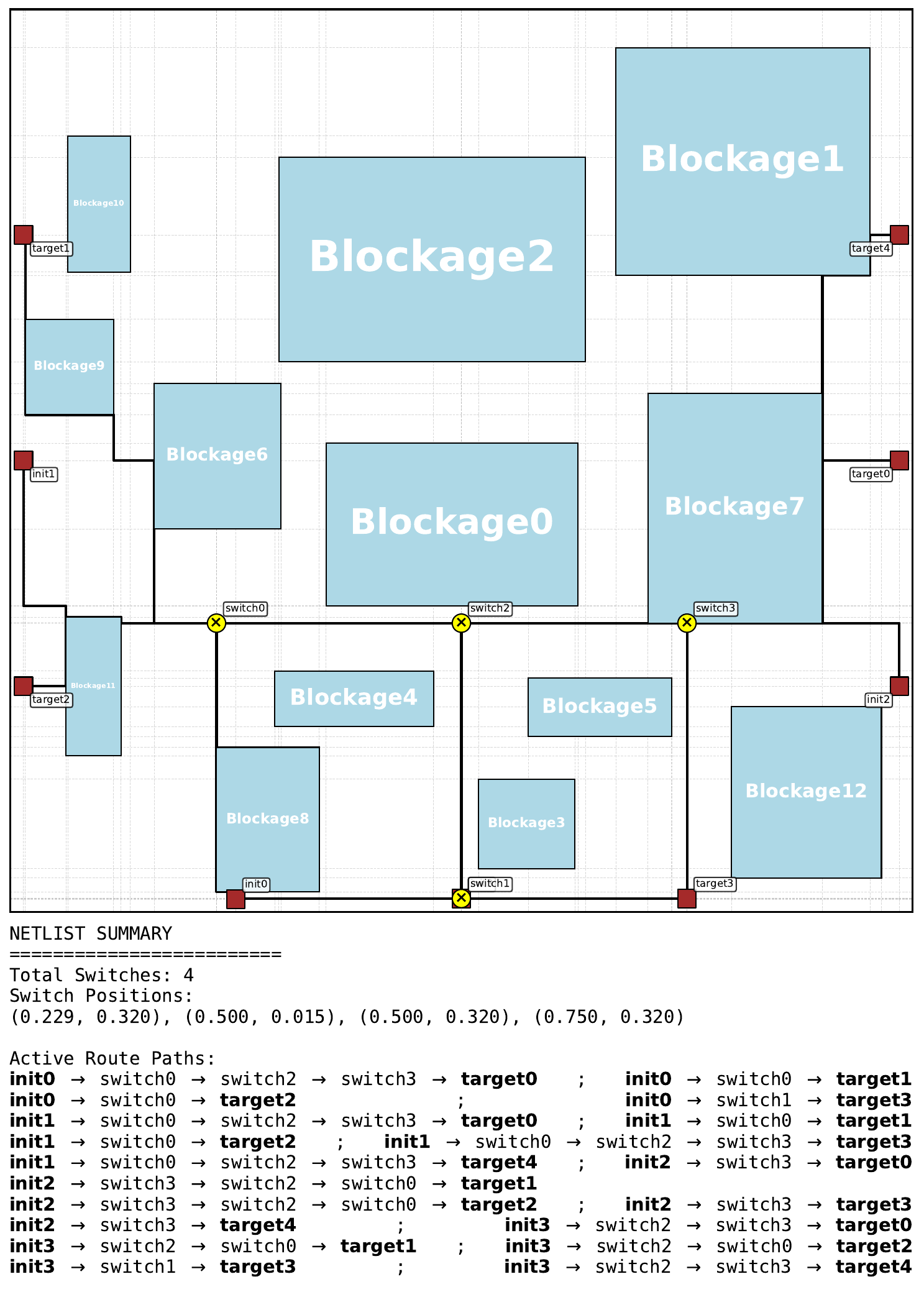}
        \caption*{Heuristic}
    \end{subfigure}
    \hfill
    \begin{subfigure}[t]{0.31\linewidth}
        \vspace{0pt}
        \centering
        \includegraphics[width=\linewidth]{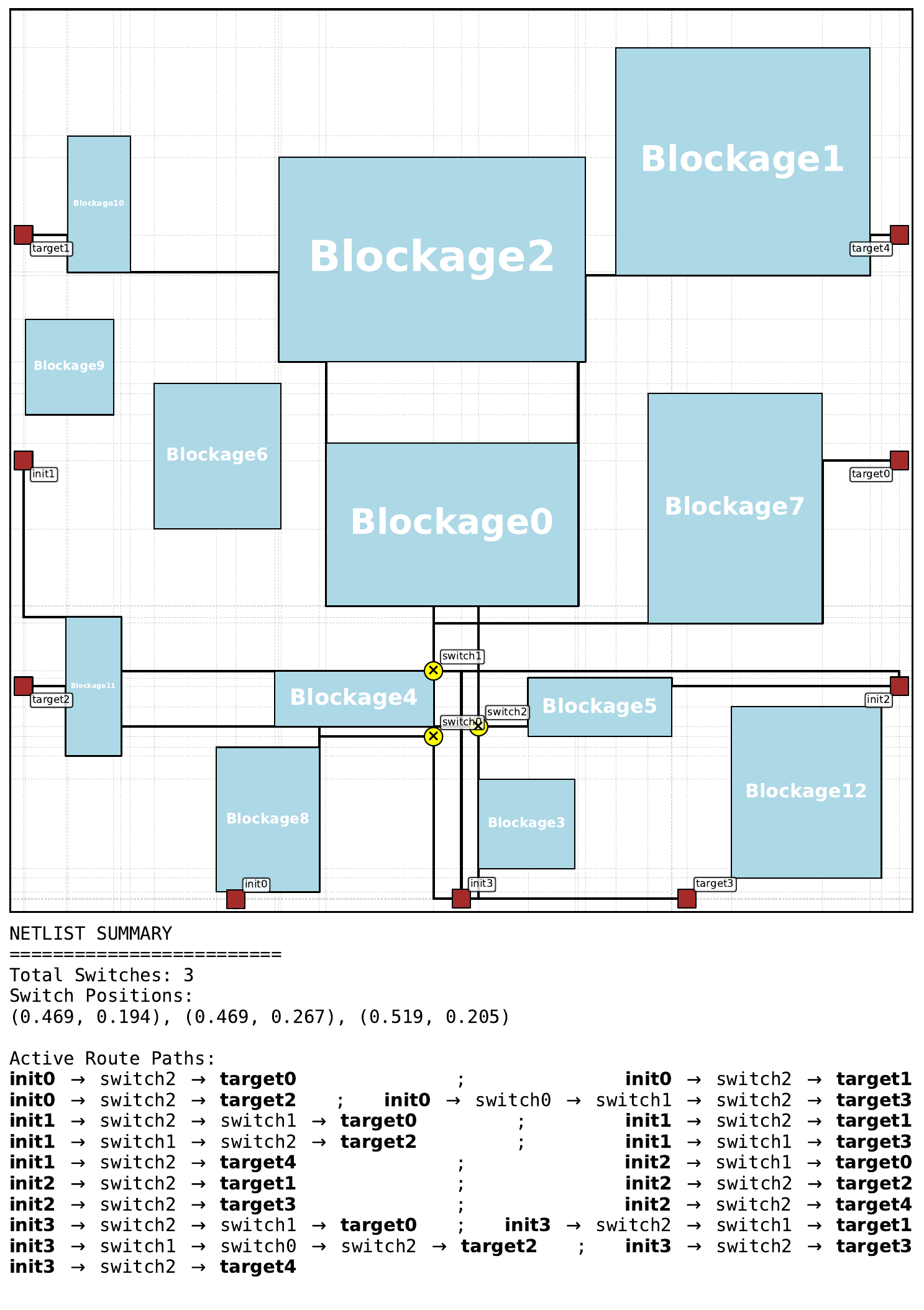}
        \caption*{Random search}
    \end{subfigure}
    \hfill
    \begin{subfigure}[t]{0.31\linewidth}
        \vspace{0pt}
        \centering
        \includegraphics[width=\linewidth]{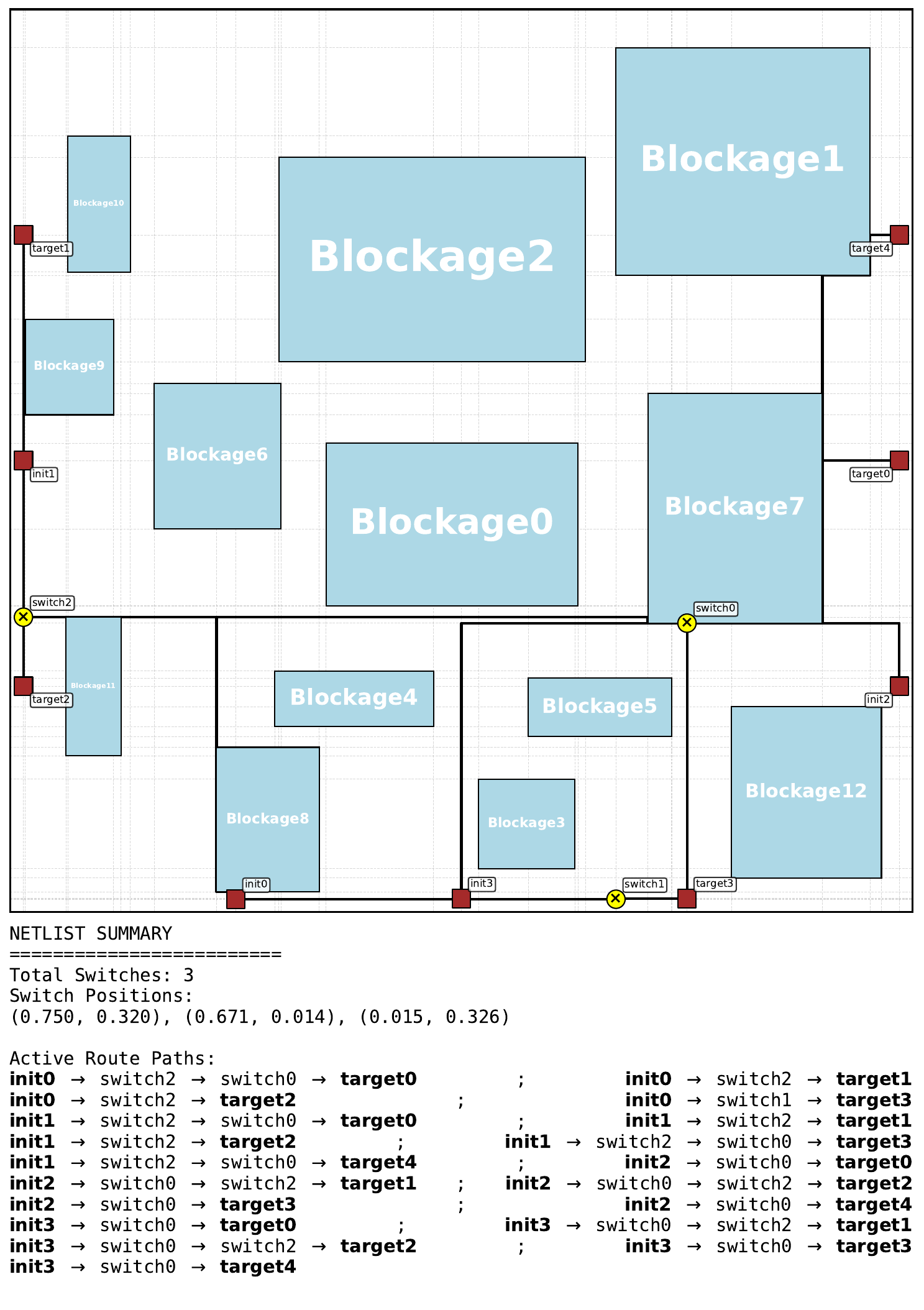}
        \caption*{Genetic algorithm}
    \end{subfigure}
    \\[0.6em]
    \begin{subfigure}[t]{0.31\linewidth}
        \vspace{0pt}
        \centering
        \includegraphics[width=\linewidth]{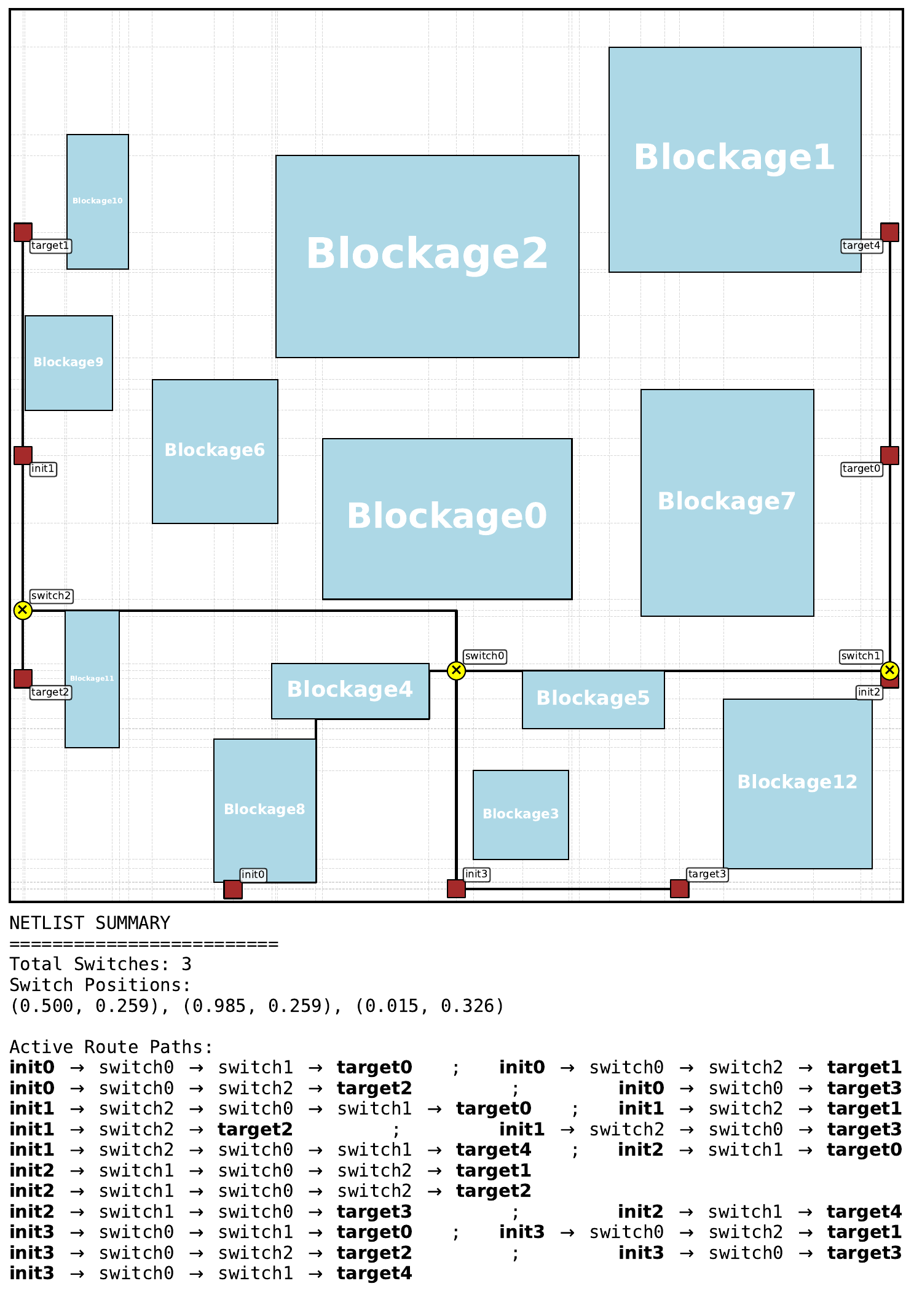}
        \caption*{PPO}
    \end{subfigure}
    \hspace{0.04\linewidth}
    \begin{subfigure}[t]{0.31\linewidth}
        \vspace{0pt}
        \centering
        \includegraphics[width=\linewidth]{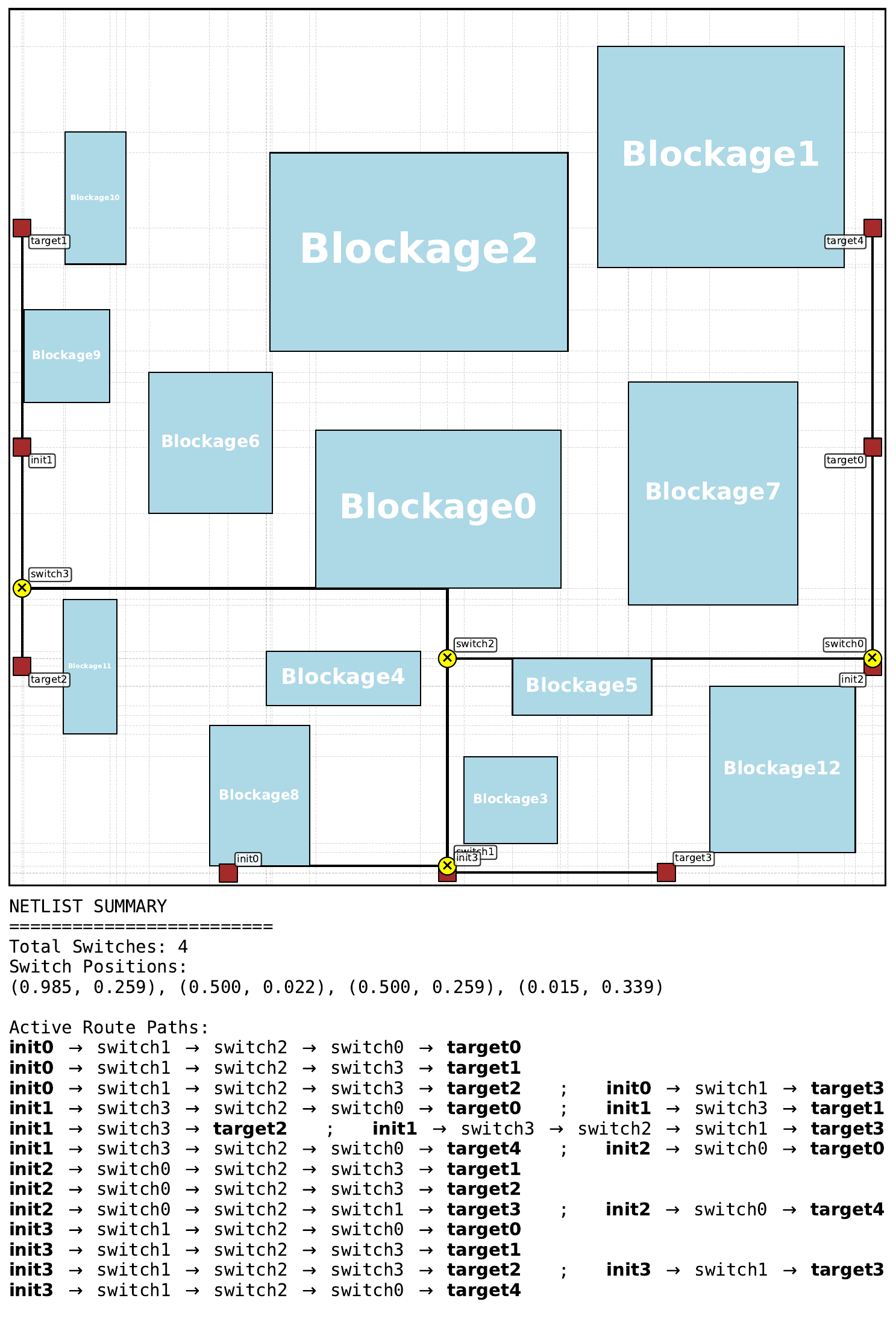}
        \caption*{MCTS}
    \end{subfigure}
\caption{Instance 23.}
\label{fig:best_pretrain_instance_23}
\end{figure*}

\begin{figure*}[h]
\centering
    \begin{subfigure}[t]{0.31\linewidth}
        \vspace{0pt}
        \centering
        \includegraphics[width=\linewidth]{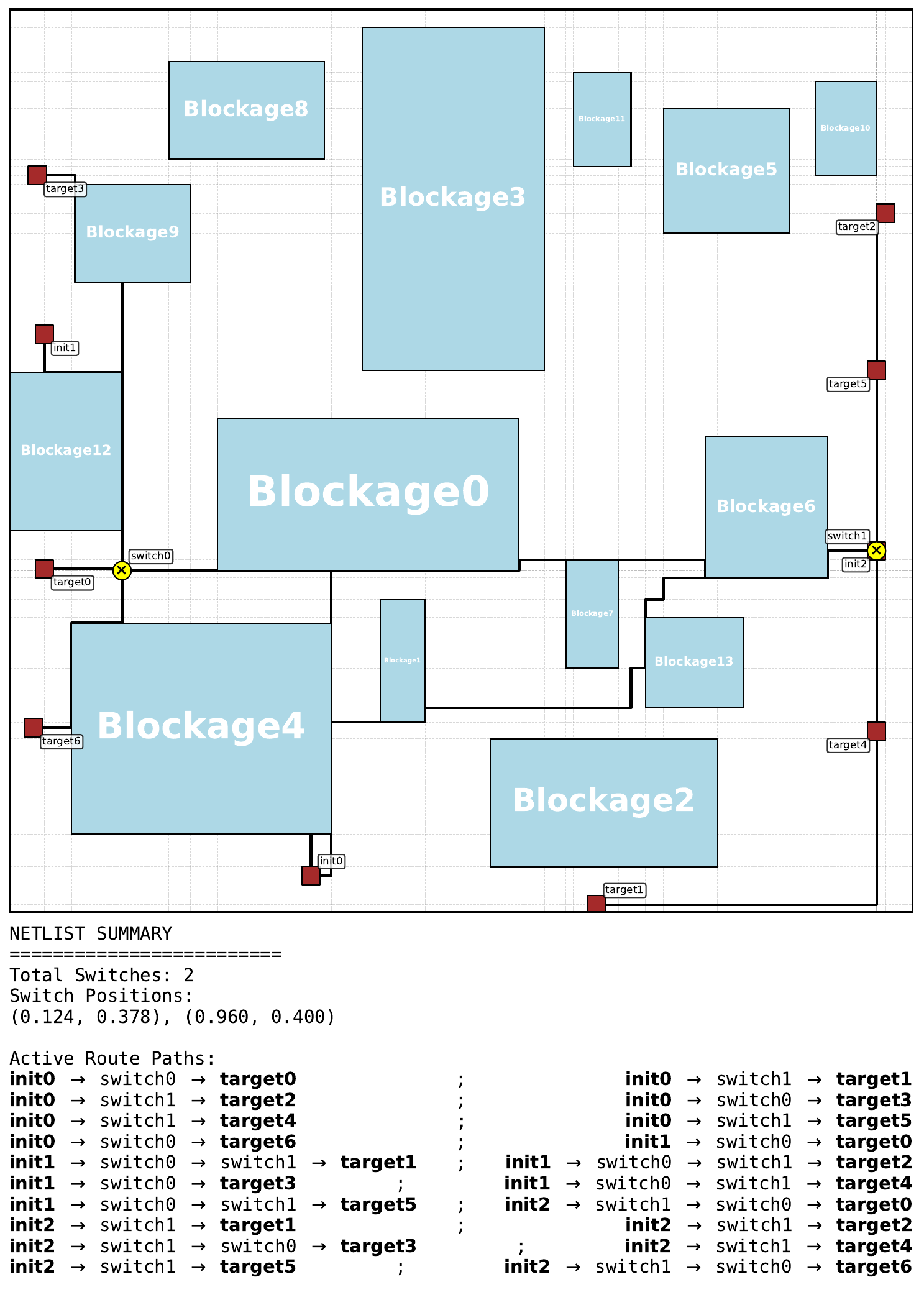}
        \caption*{Heuristic}
    \end{subfigure}
    \hfill
    \begin{subfigure}[t]{0.31\linewidth}
        \vspace{0pt}
        \centering
        \includegraphics[width=\linewidth]{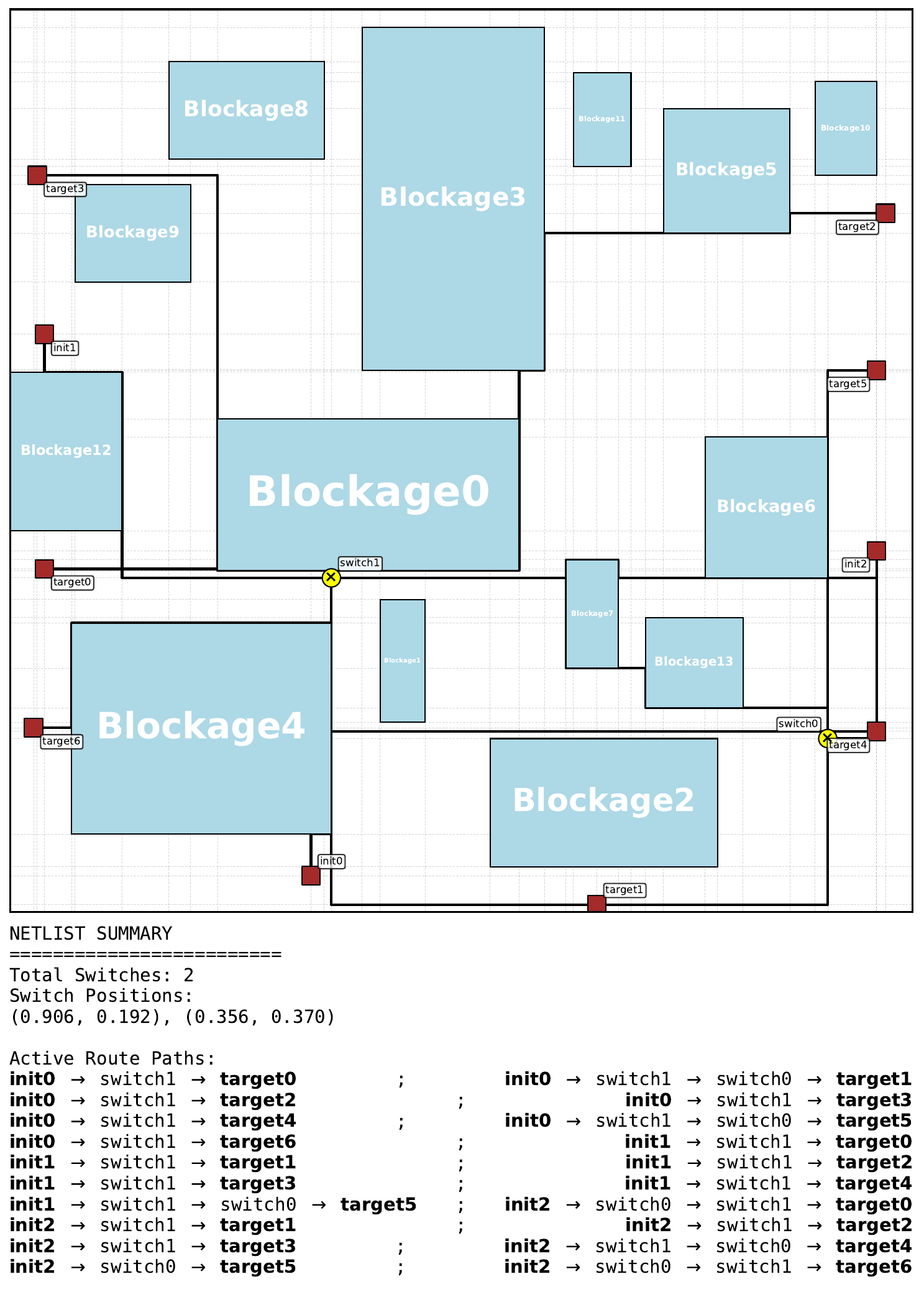}
        \caption*{Random search}
    \end{subfigure}
    \hfill
    \begin{subfigure}[t]{0.31\linewidth}
        \vspace{0pt}
        \centering
        \includegraphics[width=\linewidth]{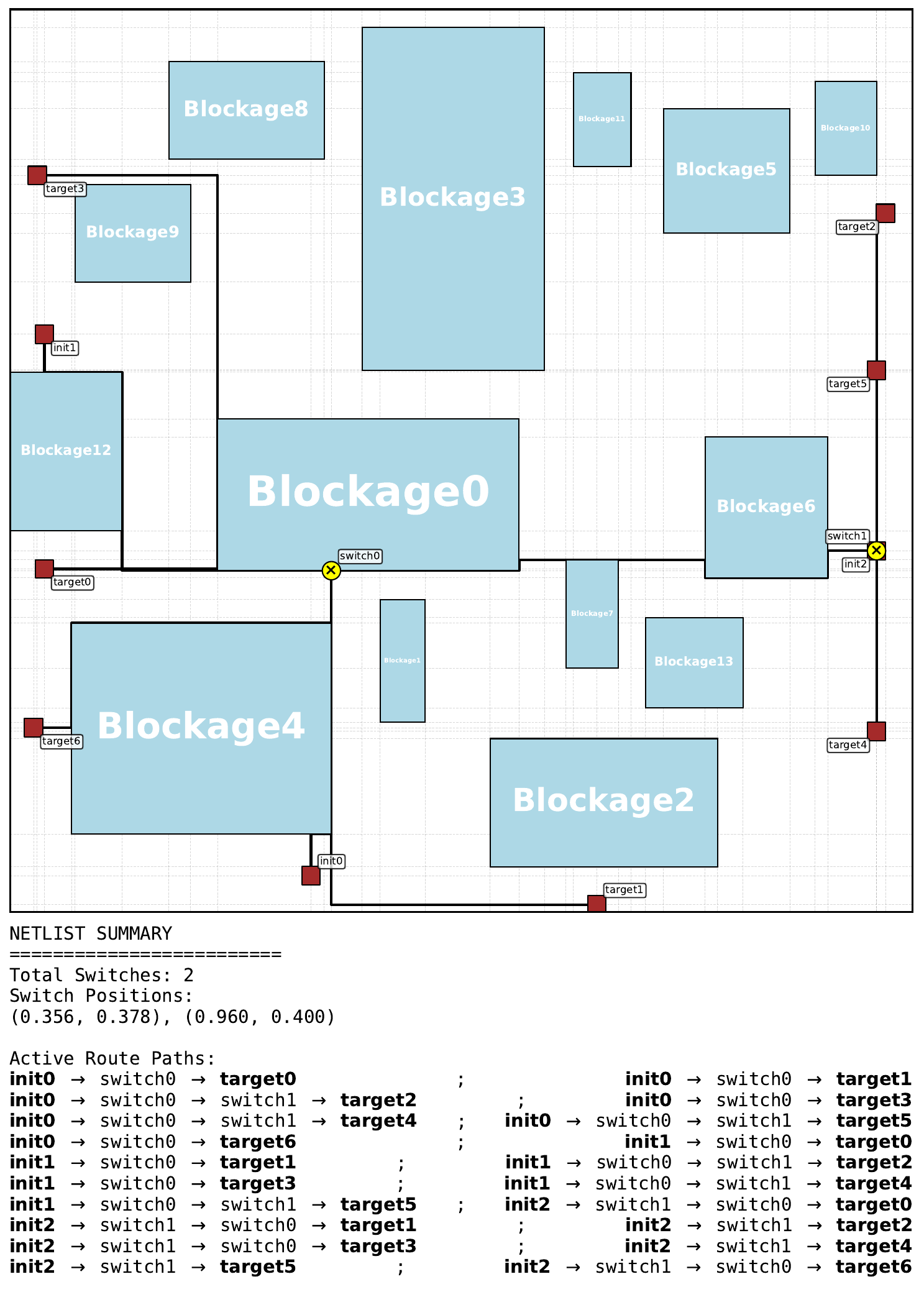}
        \caption*{Genetic algorithm}
    \end{subfigure}
    \\[0.6em]
    \begin{subfigure}[t]{0.31\linewidth}
        \vspace{0pt}
        \centering
        \includegraphics[width=\linewidth]{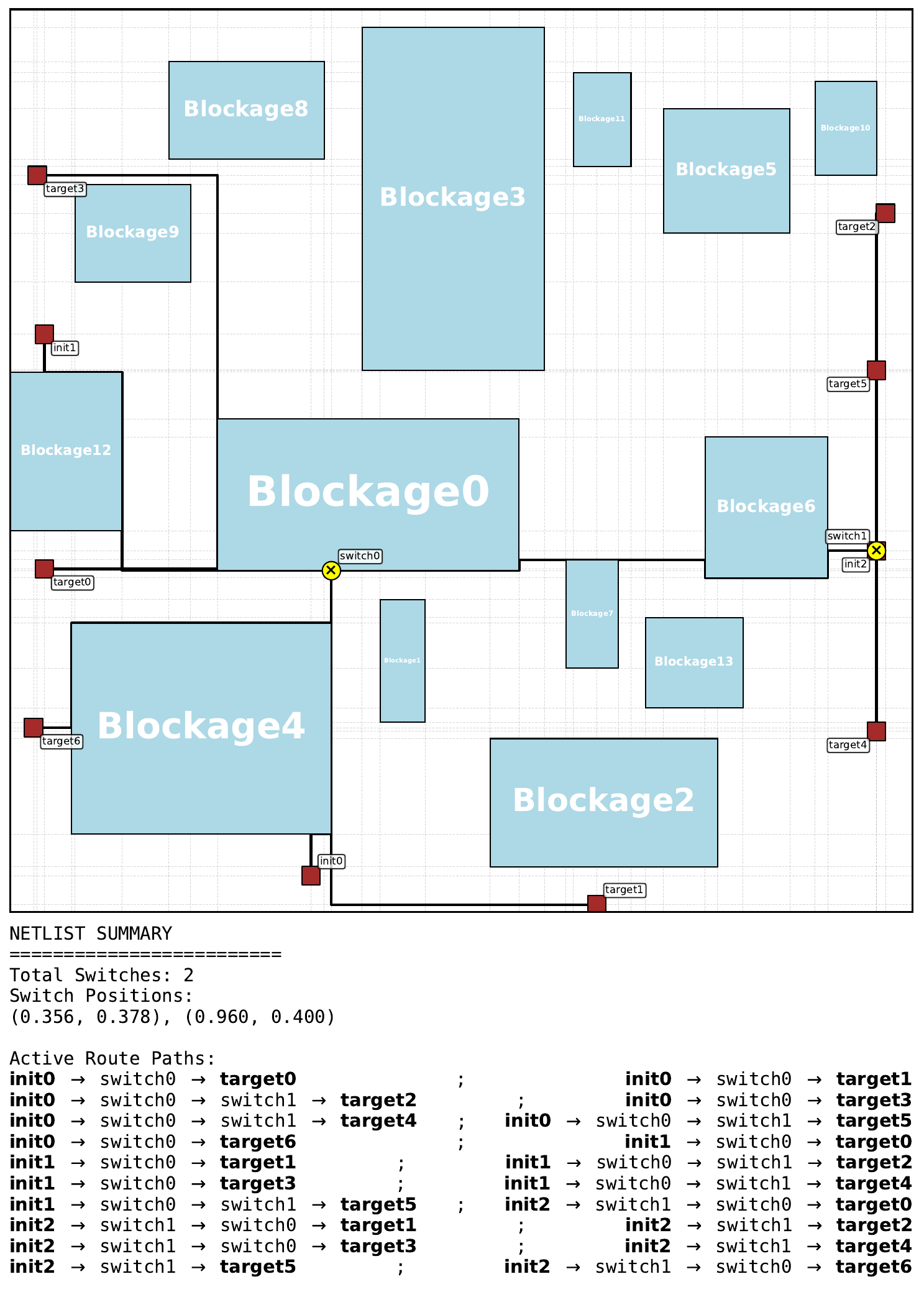}
        \caption*{PPO}
    \end{subfigure}
    \hspace{0.04\linewidth}
    \begin{subfigure}[t]{0.31\linewidth}
        \vspace{0pt}
        \centering
        \includegraphics[width=\linewidth]{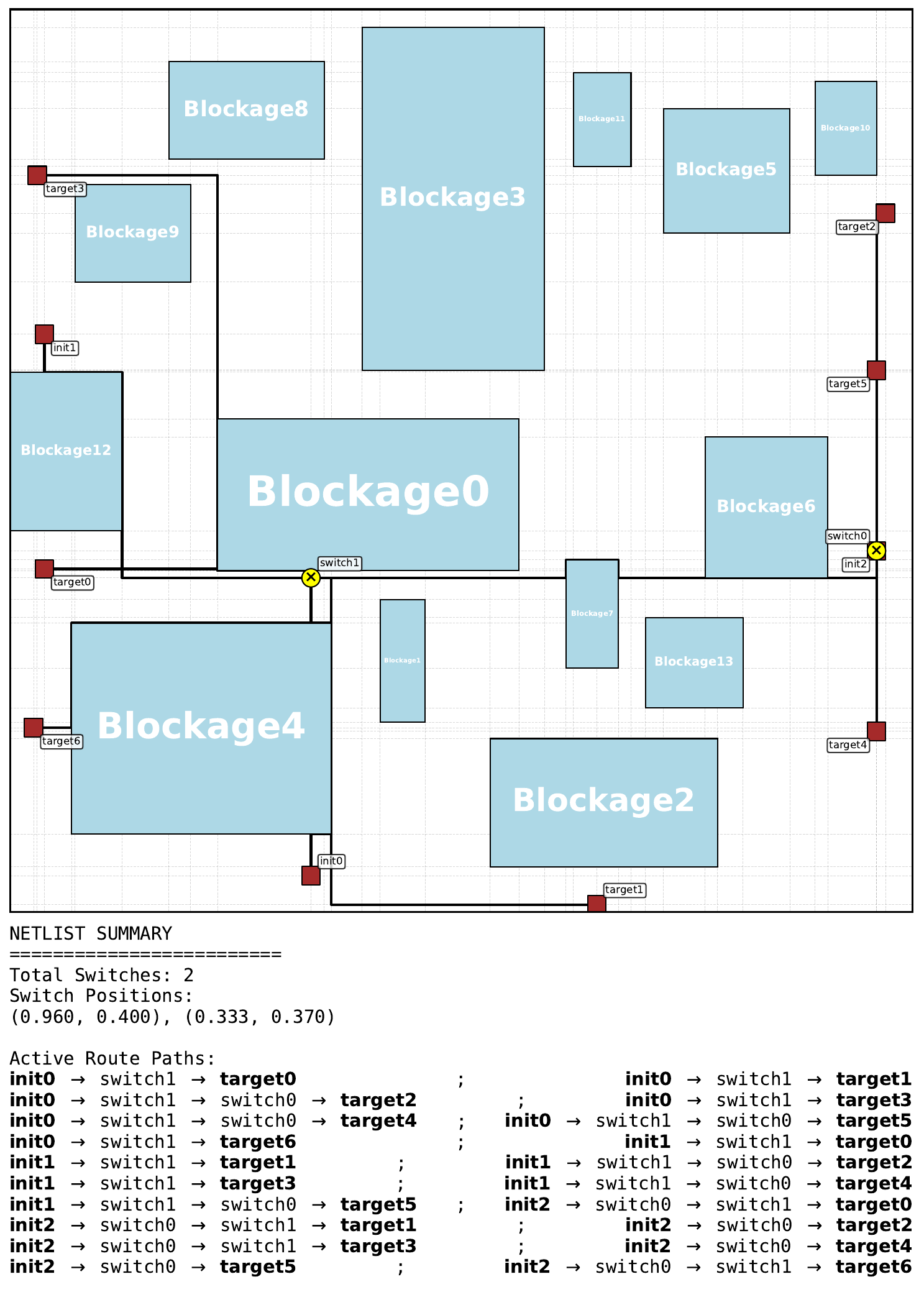}
        \caption*{MCTS}
    \end{subfigure}
\caption{Instance 24.}
\label{fig:best_pretrain_instance_24}
\end{figure*}

\clearpage

\subsection{Fine-Tuning}

\begin{figure*}[h]
\centering
    \begin{subfigure}[t]{0.48\linewidth}
        \vspace{0pt}
        \centering
        \includegraphics[width=\linewidth]{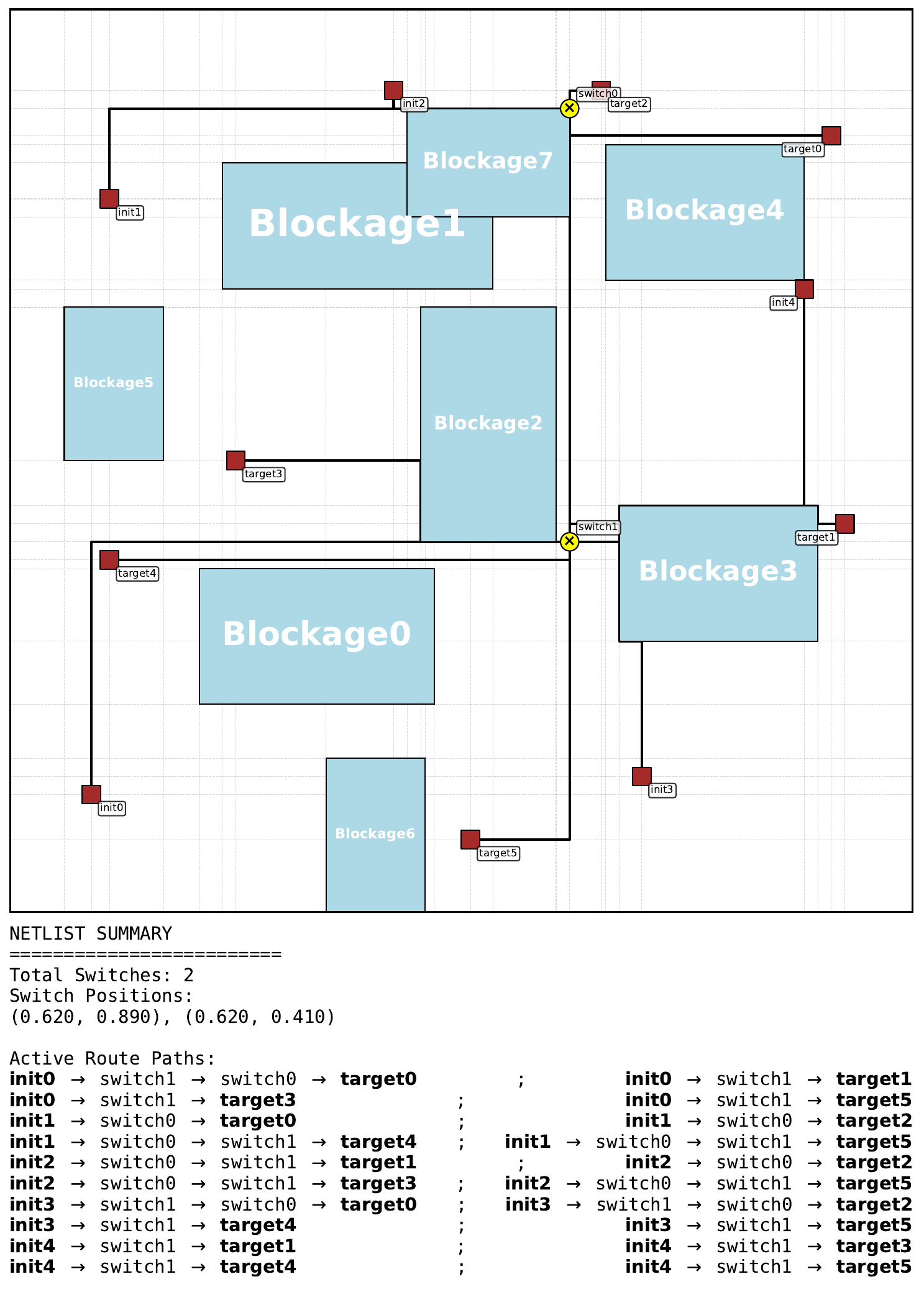}
        \caption*{PPO, pretrained}
    \end{subfigure}
    \hfill
    \begin{subfigure}[t]{0.48\linewidth}
        \vspace{0pt}
        \centering
        \includegraphics[width=\linewidth]{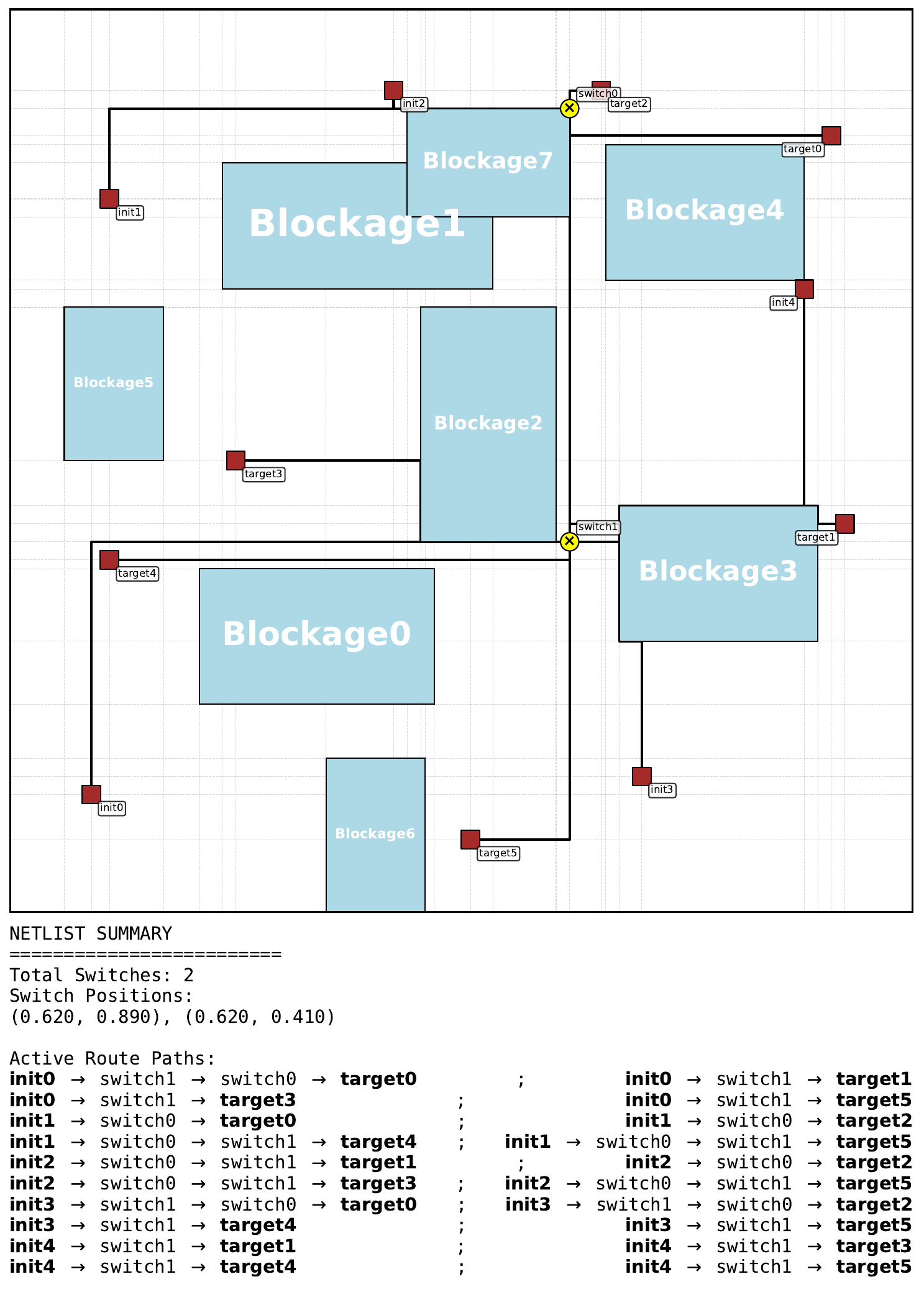}
        \caption*{PPO, from scratch}
    \end{subfigure}
    \\[0.6em]
    \begin{subfigure}[t]{0.48\linewidth}
        \vspace{0pt}
        \centering
        \includegraphics[width=\linewidth]{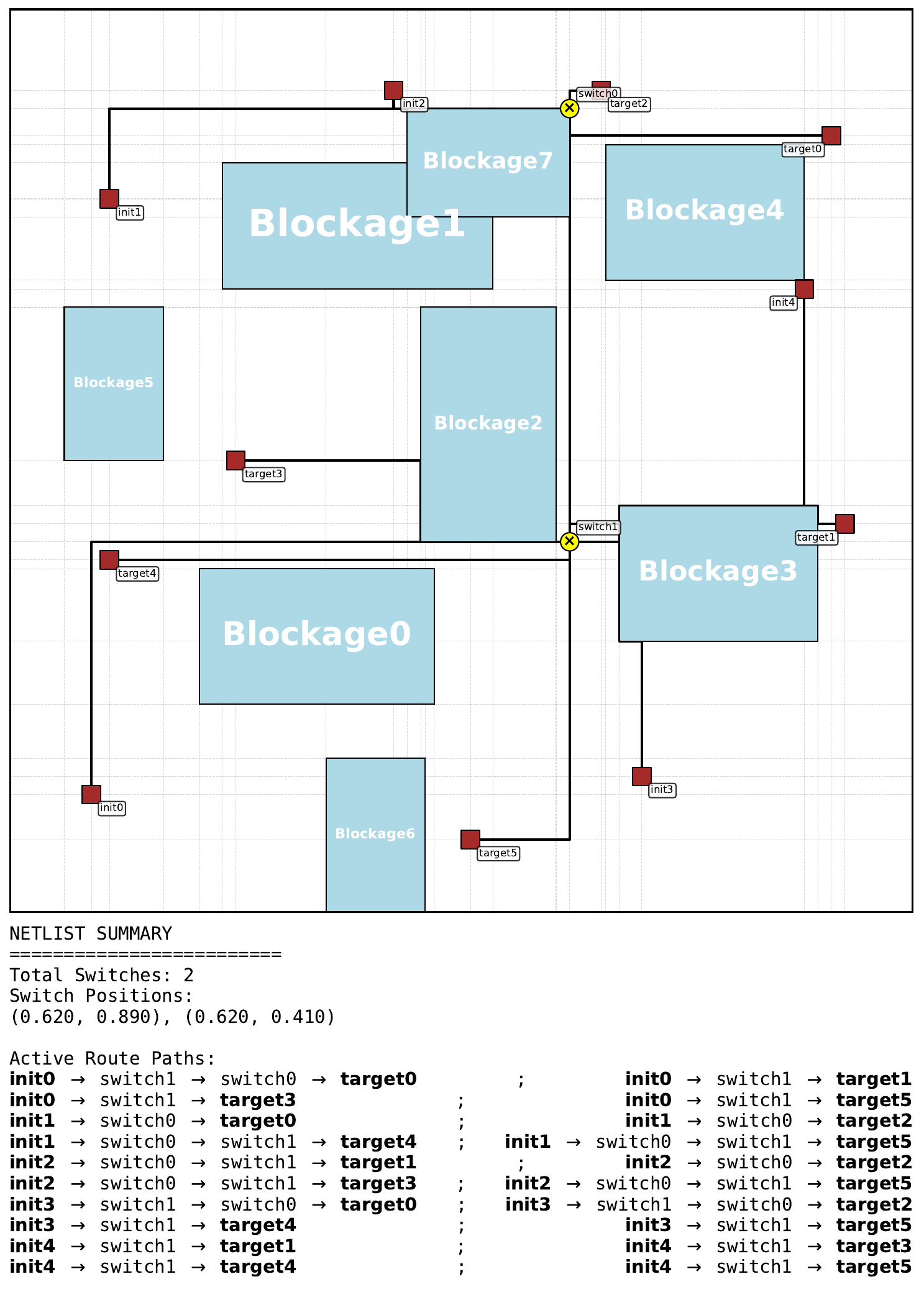}
        \caption*{MCTS, pretrained}
    \end{subfigure}
    \hfill
    \begin{subfigure}[t]{0.48\linewidth}
        \vspace{0pt}
        \centering
        \includegraphics[width=\linewidth]{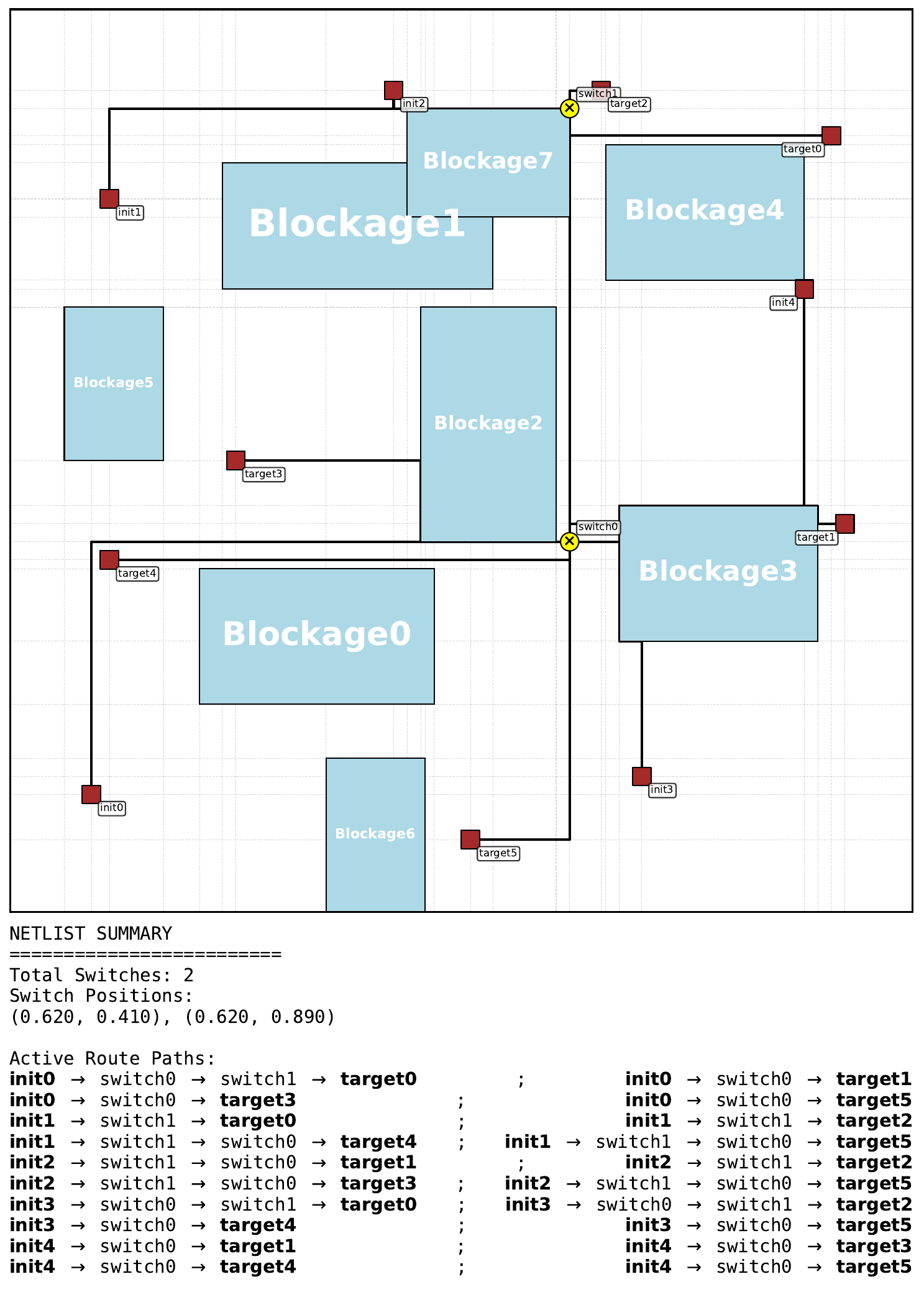}
        \caption*{MCTS, from scratch}
    \end{subfigure}
\caption{Fine-tuning instance 1.}
\label{fig:best_finetuning_instance_1}
\end{figure*}

\begin{figure*}[h]
\centering
    \begin{subfigure}[t]{0.48\linewidth}
        \vspace{0pt}
        \centering
        \includegraphics[width=\linewidth]{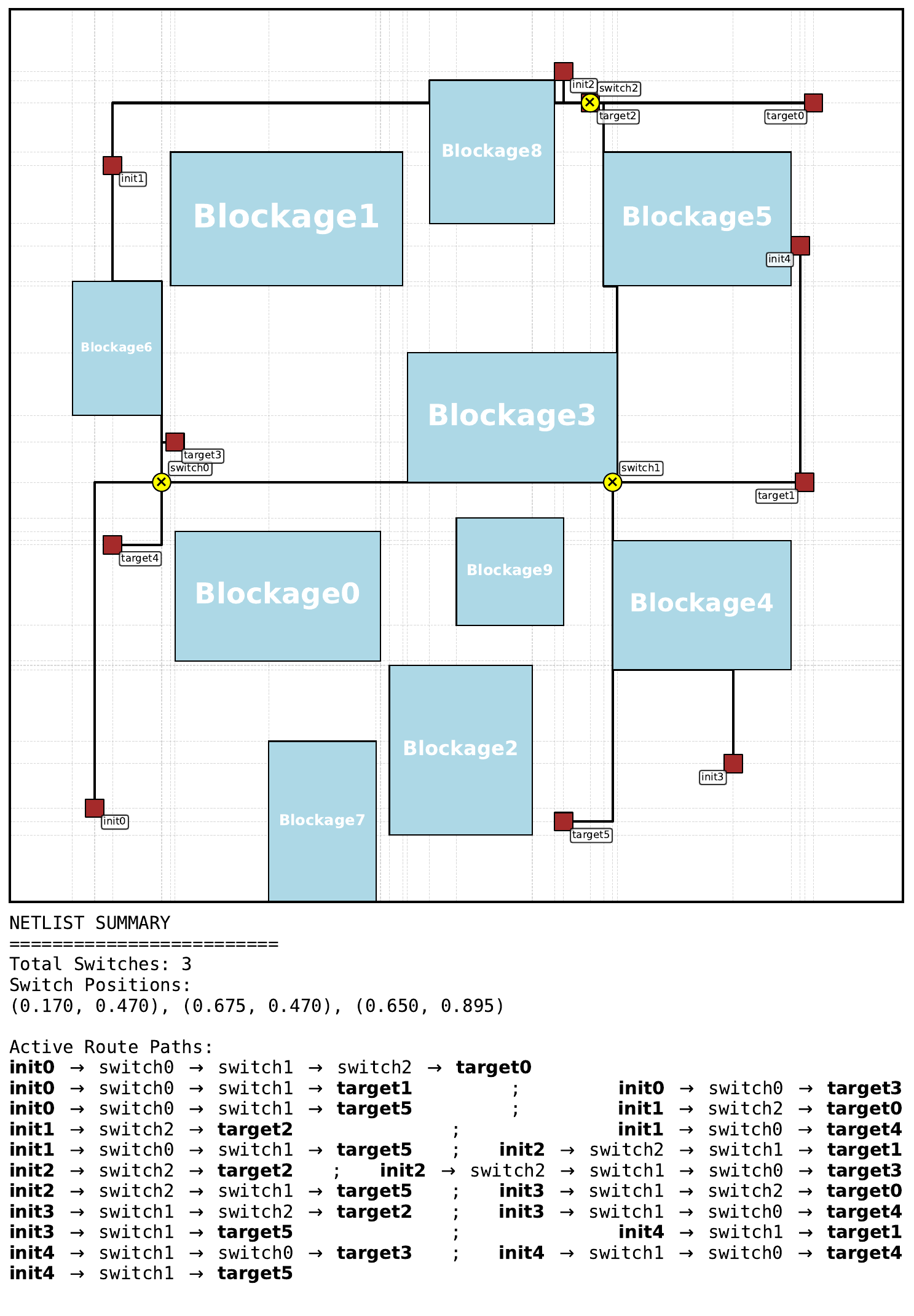}
        \caption*{PPO, pretrained}
    \end{subfigure}
    \hfill
    \begin{subfigure}[t]{0.48\linewidth}
        \vspace{0pt}
        \centering
        \includegraphics[width=\linewidth]{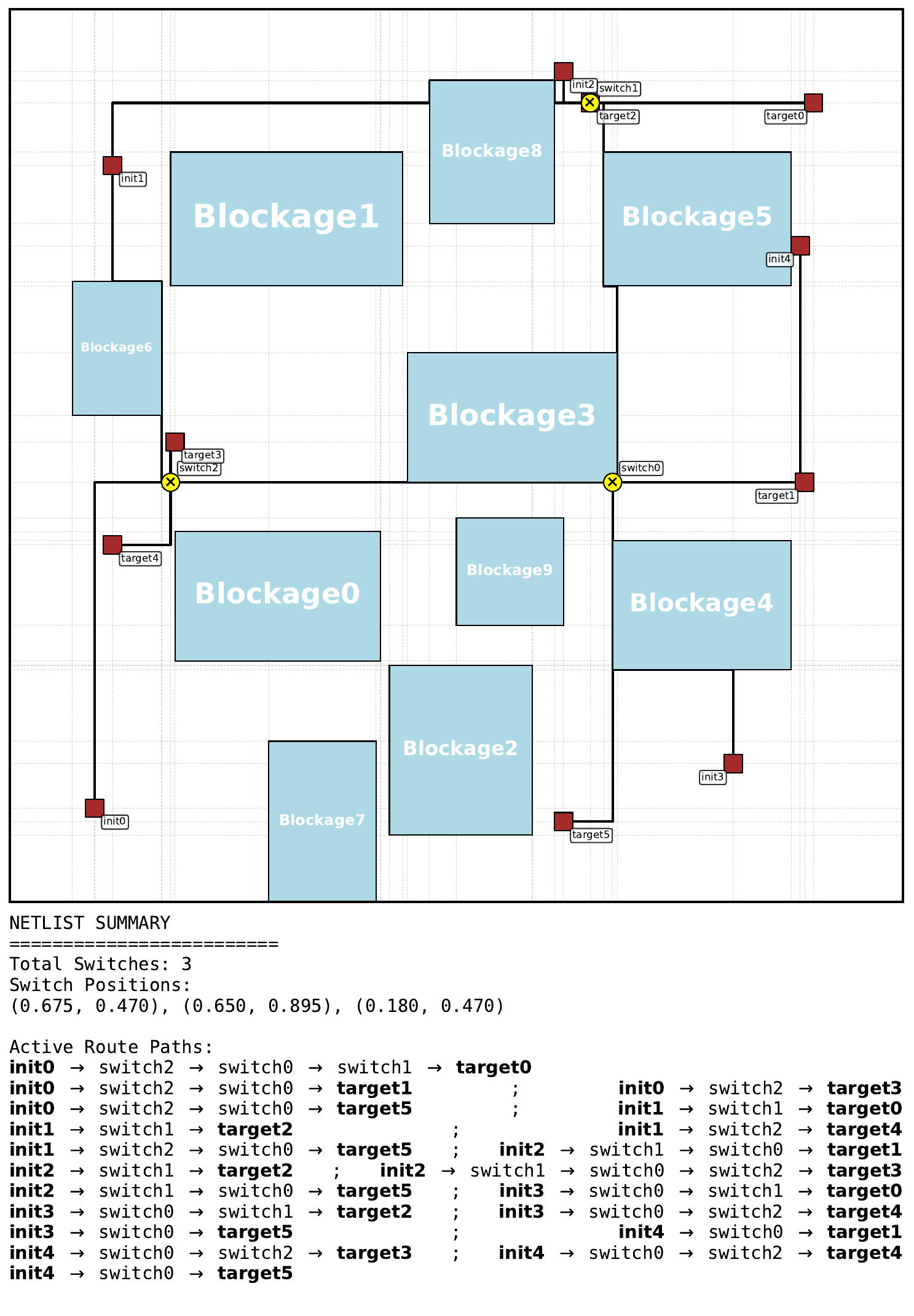}
        \caption*{PPO, from scratch}
    \end{subfigure}
    \\[0.6em]
    \begin{subfigure}[t]{0.48\linewidth}
        \vspace{0pt}
        \centering
        \includegraphics[width=\linewidth]{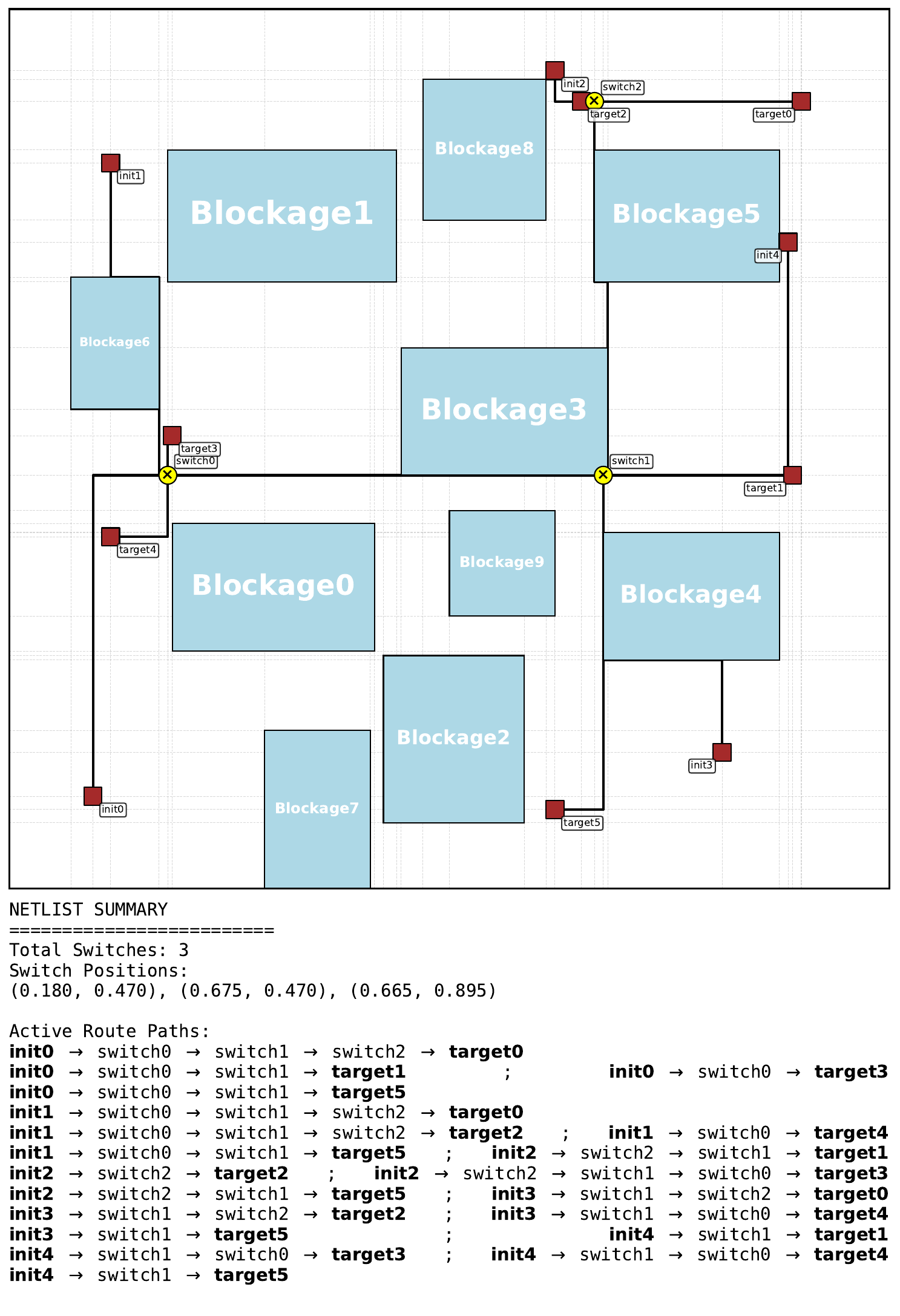}
        \caption*{MCTS, pretrained}
    \end{subfigure}
    \hfill
    \begin{subfigure}[t]{0.48\linewidth}
        \vspace{0pt}
        \centering
        \includegraphics[width=\linewidth]{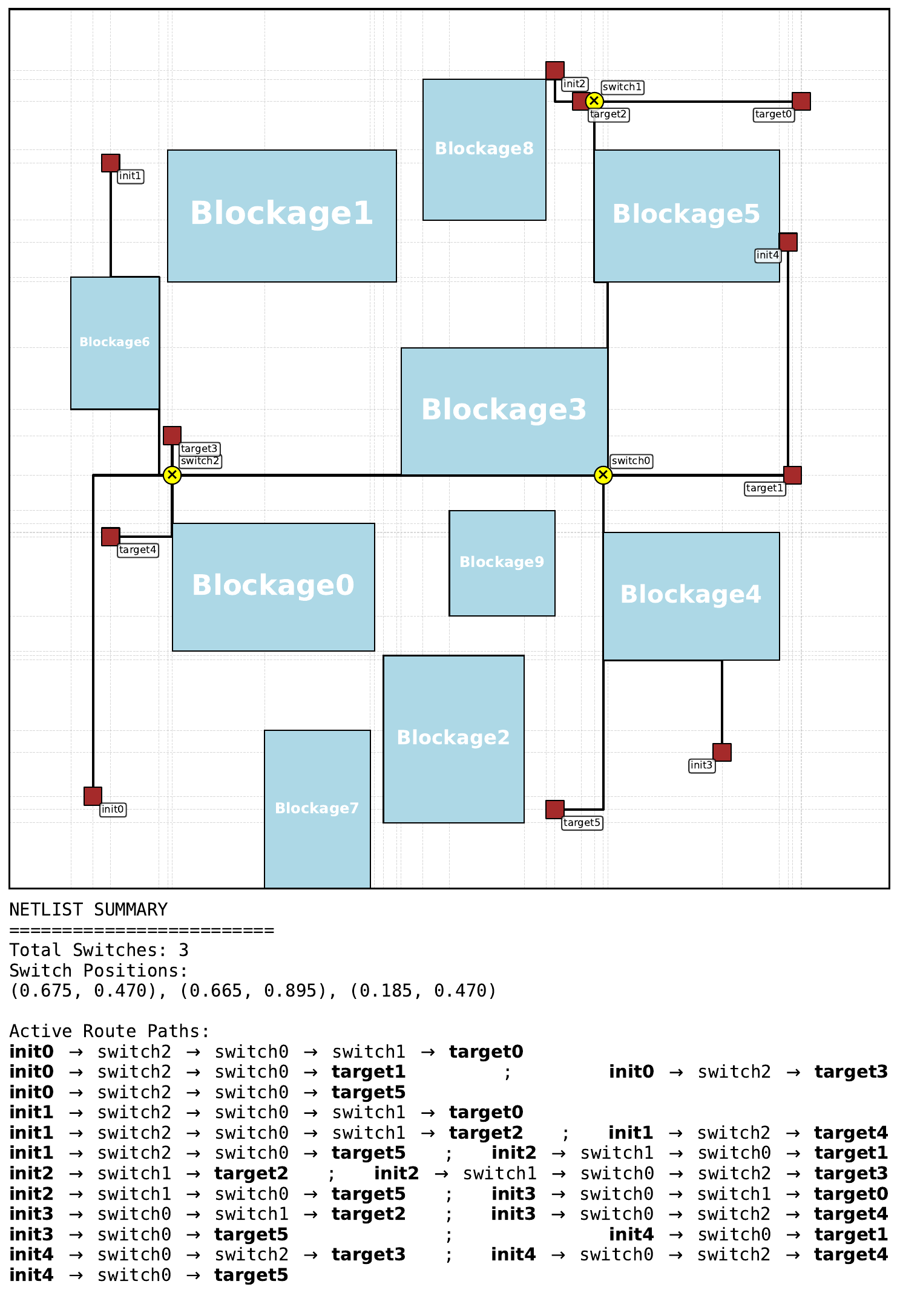}
        \caption*{MCTS, from scratch}
    \end{subfigure}
\caption{Fine-tuning instance 2.}
\label{fig:best_finetuning_instance_2}
\end{figure*}

\begin{figure*}[h]
\centering
    \begin{subfigure}[t]{0.48\linewidth}
        \vspace{0pt}
        \centering
        \includegraphics[width=\linewidth]{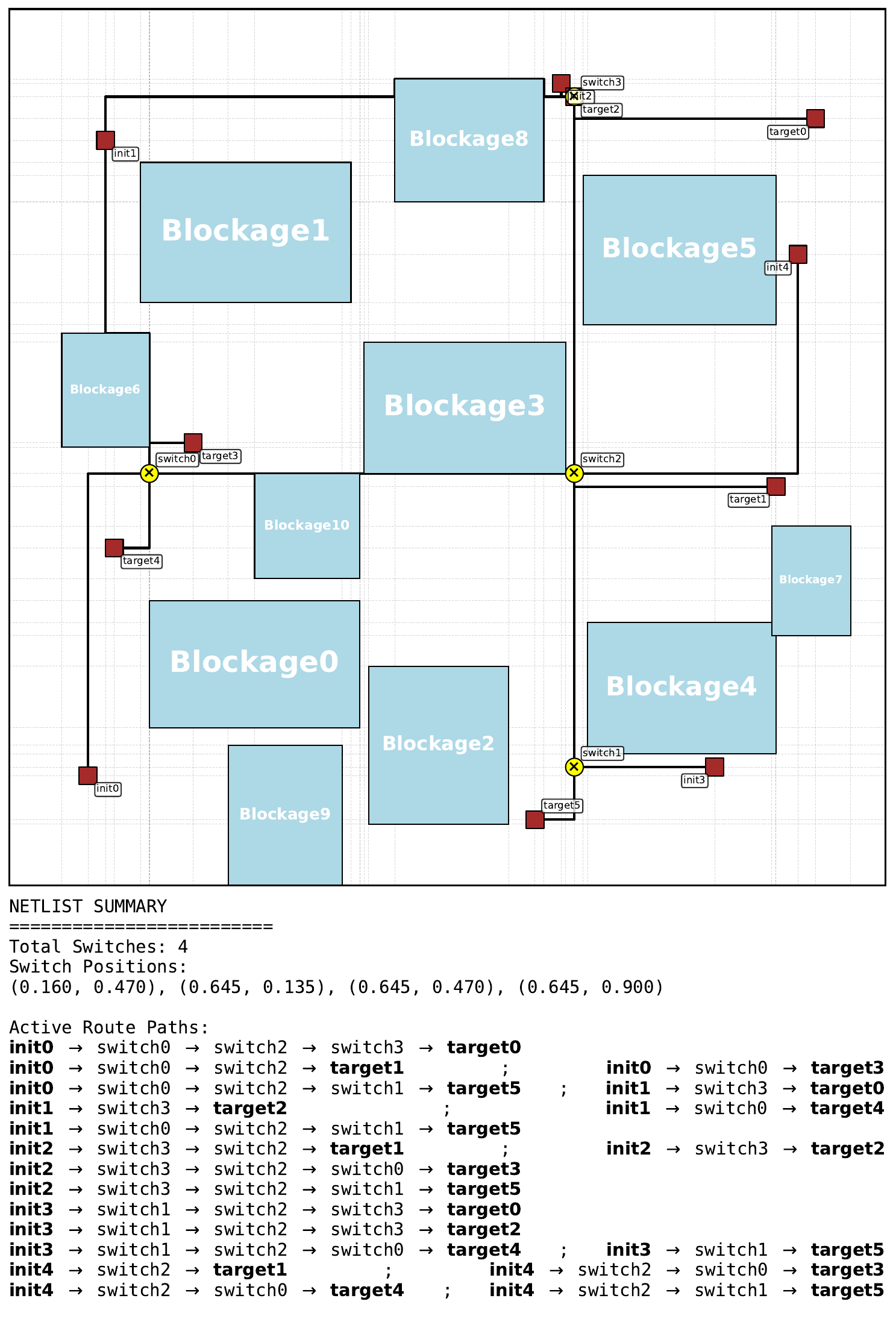}
        \caption*{PPO, pretrained}
    \end{subfigure}
    \hfill
    \begin{subfigure}[t]{0.48\linewidth}
        \vspace{0pt}
        \centering
        \includegraphics[width=\linewidth]{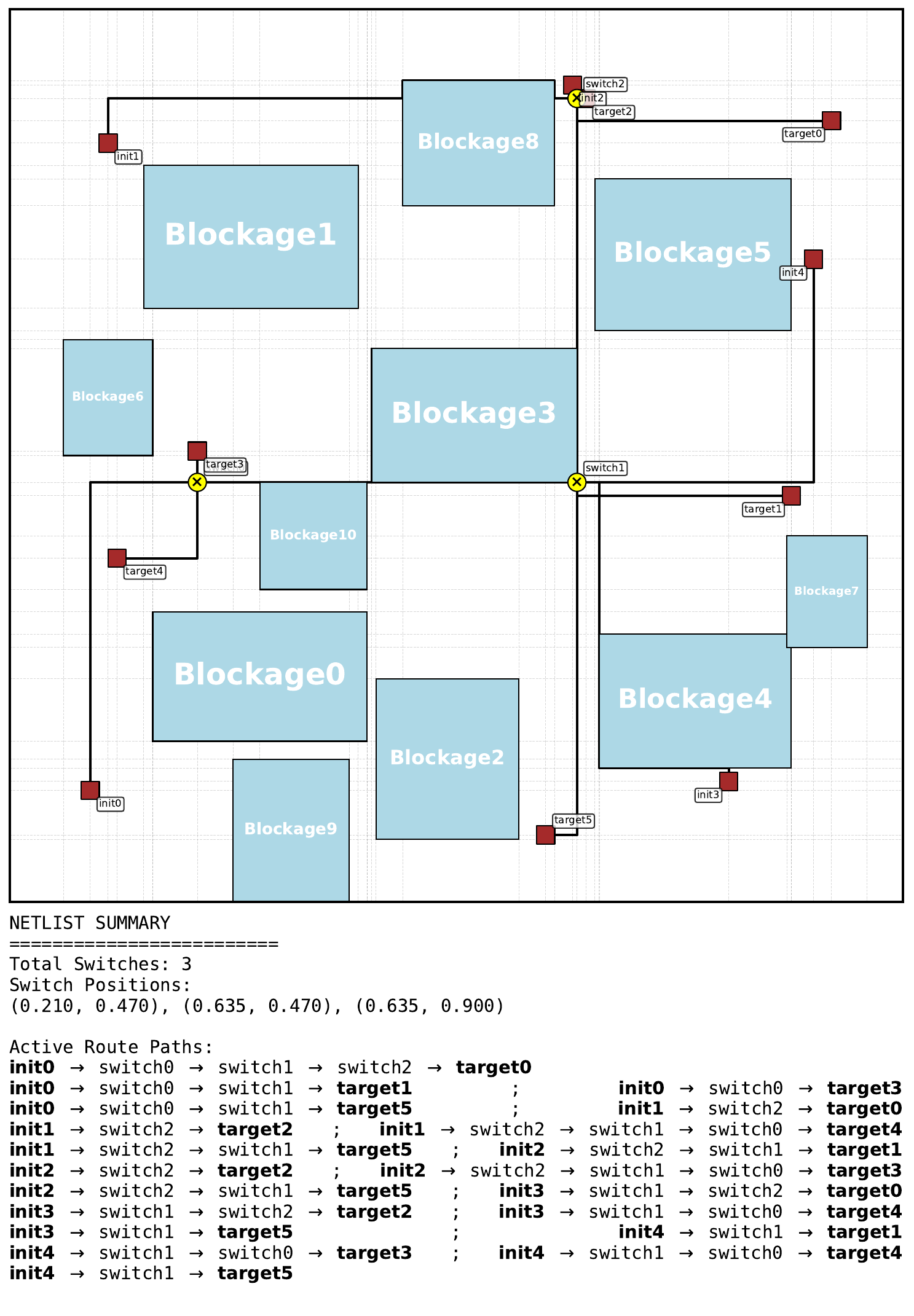}
        \caption*{PPO, from scratch}
    \end{subfigure}
    \\[0.6em]
    \begin{subfigure}[t]{0.48\linewidth}
        \vspace{0pt}
        \centering
        \includegraphics[width=\linewidth]{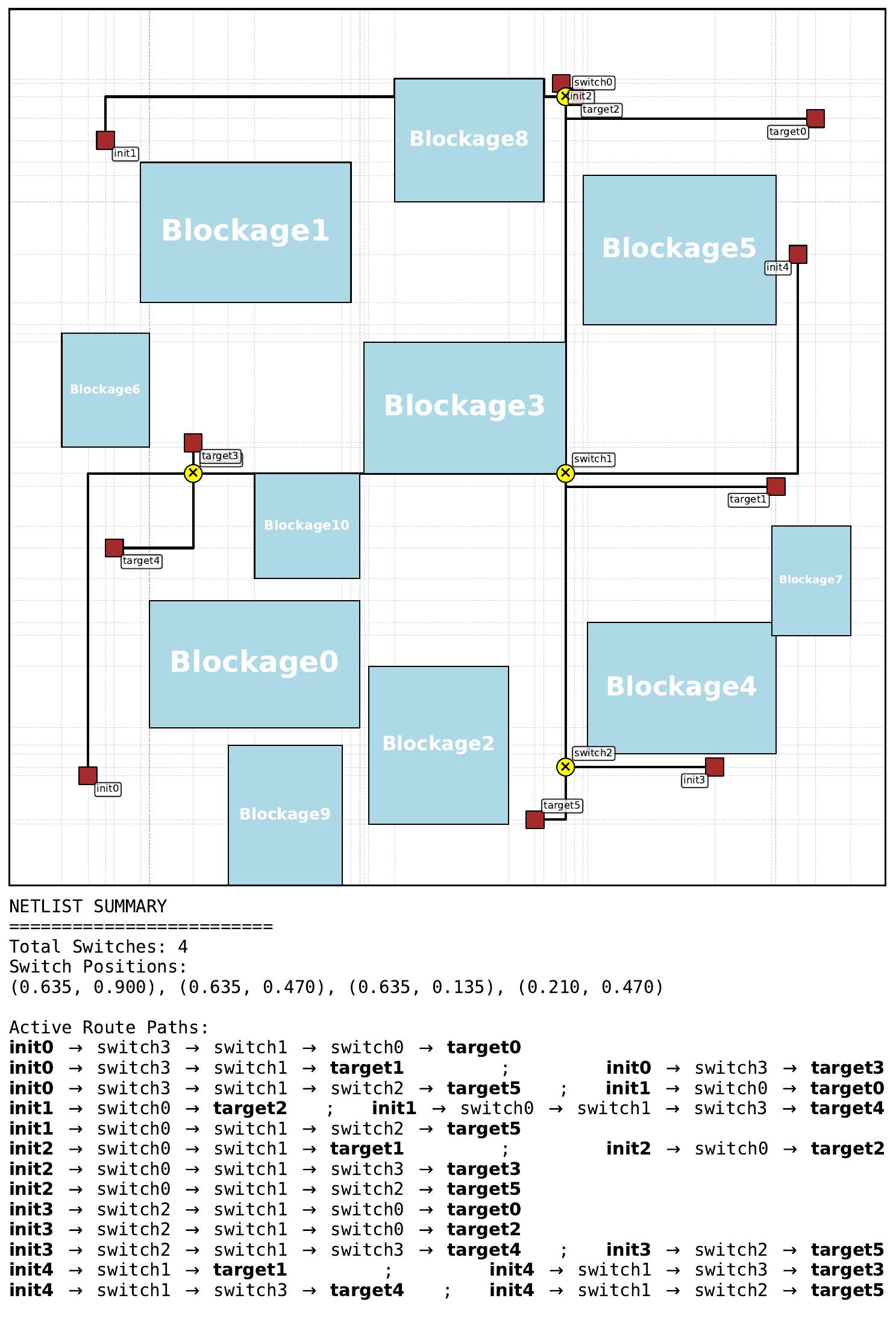}
        \caption*{MCTS, pretrained}
    \end{subfigure}
    \hfill
    \begin{subfigure}[t]{0.48\linewidth}
        \vspace{0pt}
        \centering
        \includegraphics[width=\linewidth]{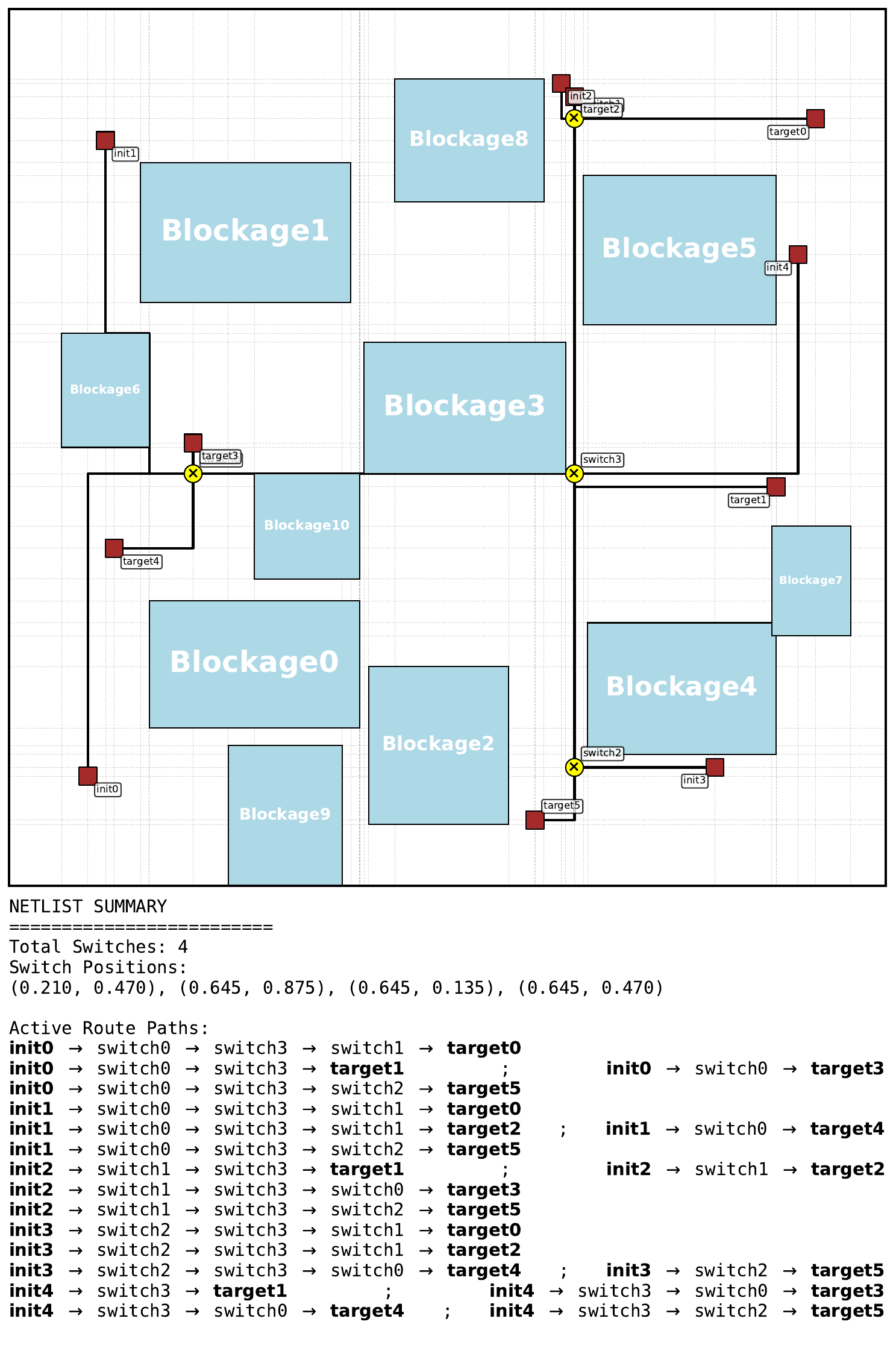}
        \caption*{MCTS, from scratch}
    \end{subfigure}
\caption{Fine-tuning instance 3.}
\label{fig:best_finetuning_instance_3}
\end{figure*}

\begin{figure*}[h]
\centering
    \begin{subfigure}[t]{0.48\linewidth}
        \vspace{0pt}
        \centering
        \includegraphics[width=\linewidth]{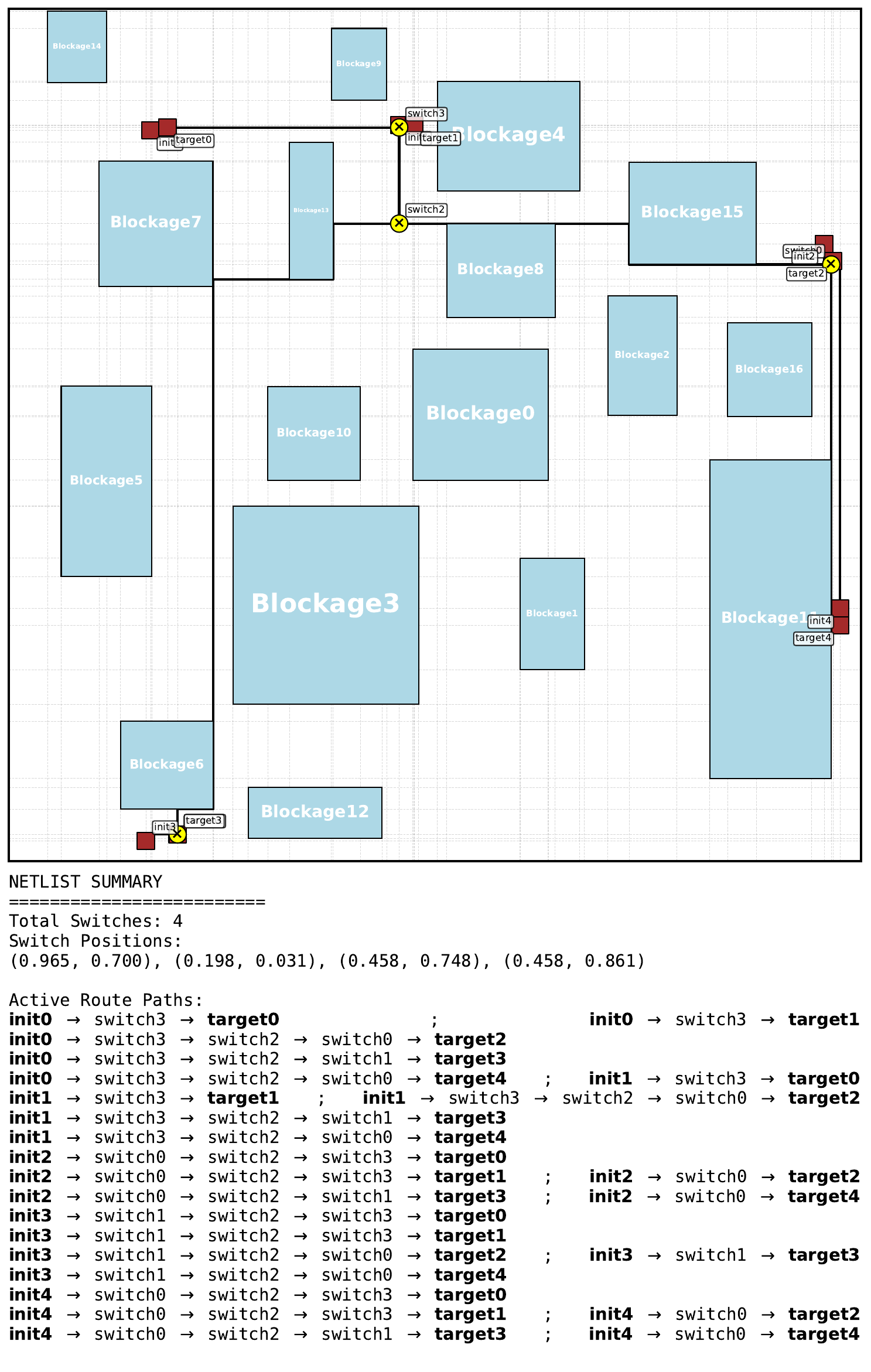}
        \caption*{PPO, pretrained}
    \end{subfigure}
    \hfill
    \begin{subfigure}[t]{0.48\linewidth}
        \vspace{0pt}
        \centering
        \includegraphics[width=\linewidth]{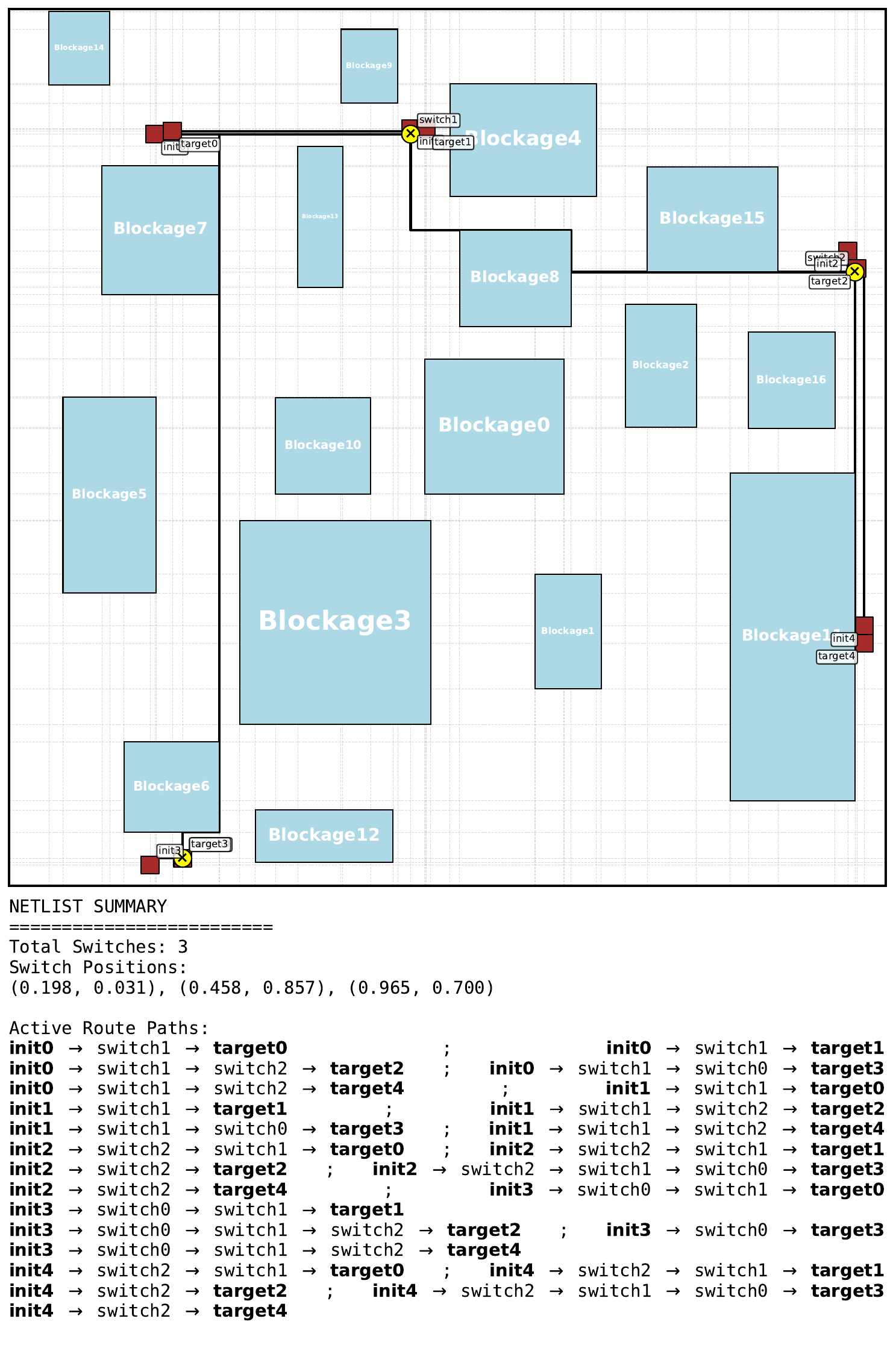}
        \caption*{PPO, from scratch}
    \end{subfigure}
    \\[0.6em]
    \begin{subfigure}[t]{0.48\linewidth}
        \vspace{0pt}
        \centering
        \includegraphics[width=\linewidth]{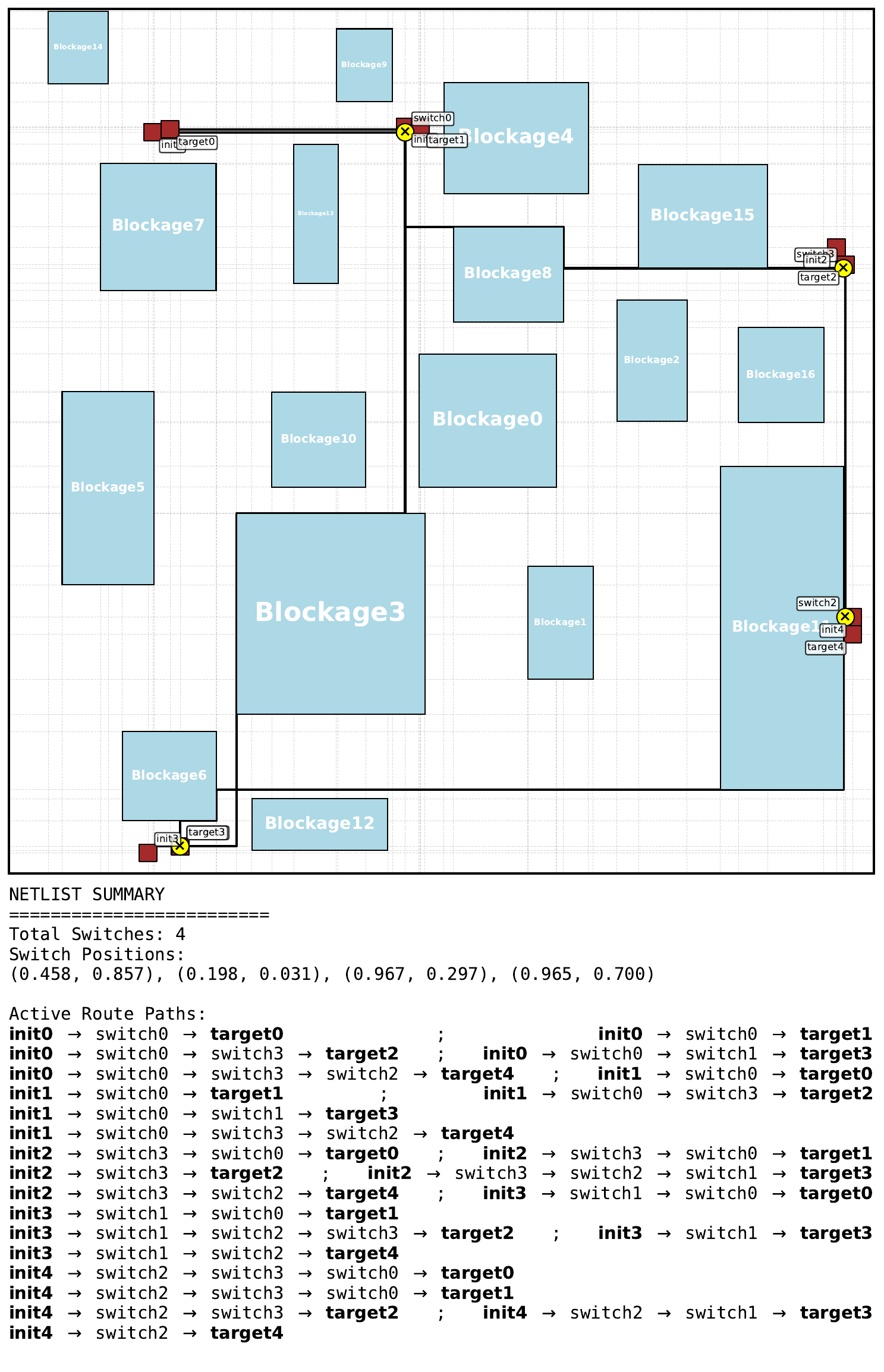}
        \caption*{MCTS, pretrained}
    \end{subfigure}
    \hfill
    \begin{subfigure}[t]{0.48\linewidth}
        \vspace{0pt}
        \centering
        \includegraphics[width=\linewidth]{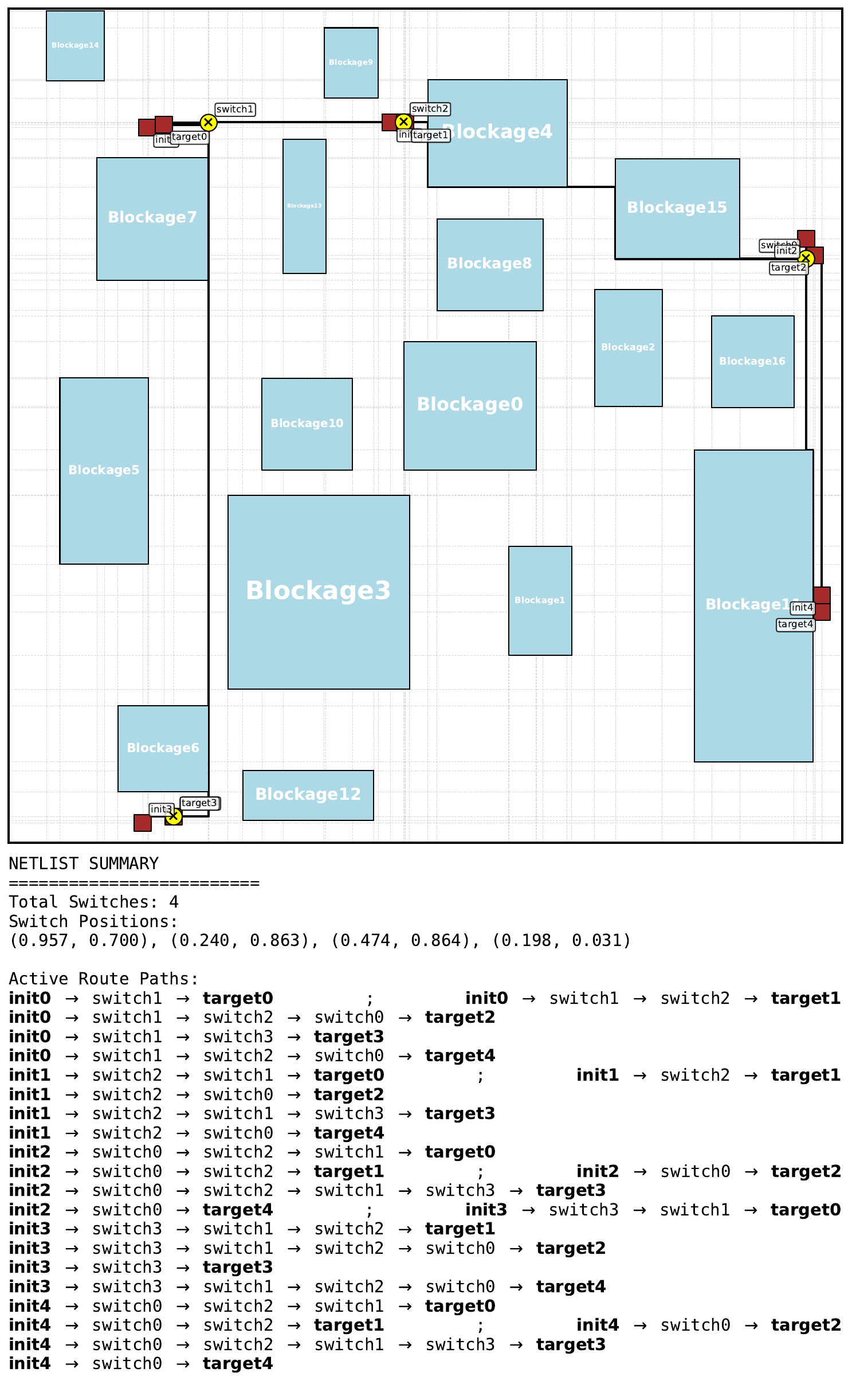}
        \caption*{MCTS, from scratch}
    \end{subfigure}
\caption{Fine-tuning instance 4.}
\label{fig:best_finetuning_instance_4}
\end{figure*}

\end{document}